\documentclass[preprint]{imsart}

\RequirePackage{amsthm,amsmath,amsfonts,amssymb}
\RequirePackage[authoryear]{natbib} 
\RequirePackage[colorlinks,citecolor=blue,urlcolor=blue]{hyperref}
\RequirePackage{graphicx,subcaption}
\RequirePackage{bm,bbm,algorithm,algpseudocode}
\RequirePackage{enumitem}
\RequirePackage{multirow}

\usepackage{array}
\newcommand{\PreserveBackslash}[1]{\let\temp=\\#1\let\\=\temp}
\newcolumntype{C}[1]{>{\PreserveBackslash\centering}p{#1}}
\newcolumntype{R}[1]{>{\PreserveBackslash\raggedleft}p{#1}}
\newcolumntype{L}[1]{>{\PreserveBackslash\raggedright}p{#1}}

\startlocaldefs
\theoremstyle{plain}

\newtheorem{theorem}{Theorem}[section]
\newtheorem{lemma}[theorem]{Lemma}
\newtheorem{corollary}[theorem]{Corollary}
\newtheorem{proposition}[theorem]{Proposition}
\theoremstyle{remark}
\newtheorem{definition}[theorem]{Definition}
\newtheorem{example}{Example}

\newtheorem{remark}{Remark}[section]

\DeclareMathOperator*{\argmin}{arg\,min}
\newcommand{\norm}[1]{\left|\left| #1 \right|\right|}
\newcommand{\Rad}{\mathtt{Rad}}

\newcommand{\Haus}{\mathtt{Haus}}
\newcommand{\Exp}{\mathtt{Exp}}
\newcommand{\Diag}{\mathtt{Diag}}
\newcommand{\rank}{\mathtt{rank}}
\newcommand{\grad}{\mathtt{grad}\,}
\renewcommand{\hat}{\widehat}
\renewcommand{\tilde}{\widetilde}

\newtheorem{innercustomthm}{Theorem}
\newenvironment{customthm}[1]
{\renewcommand\theinnercustomthm{#1}\innercustomthm}
{\endinnercustomthm}

\newtheorem{innercustomlem}{Lemma}
\newenvironment{customlem}[1]
{\renewcommand\theinnercustomlem{#1}\innercustomlem}
{\endinnercustomlem}

\newtheorem{innercustomprop}{Proposition}
\newenvironment{customprop}[1]
{\renewcommand\theinnercustomprop{#1}\innercustomprop}
{\endinnercustomprop}

\endlocaldefs

\begin{document}

\begin{frontmatter}
\title{Linear Convergence of the Subspace Constrained Mean Shift Algorithm: From Euclidean to Directional Data}
\runtitle{Linear Convergence of the Euclidean and Directional SCMS Algorithms}

\begin{aug}
\author[A]{\fnms{Yikun} \snm{Zhang}\ead[label=e1,mark]{yikun@uw.edu}}
\and
\author[A]{\fnms{Yen-Chi} \snm{Chen}\ead[label=e2,mark]{yenchic@uw.edu}}
\address[A]{Department of Statistics,
University of Washington\\
\printead{e1,e2}}
\end{aug}

\begin{abstract}
This paper studies the linear convergence of the subspace constrained mean shift (SCMS) algorithm, a well-known algorithm for identifying a density ridge defined by a kernel density estimator. By arguing that the SCMS algorithm is a special variant of a subspace constrained gradient ascent (SCGA) algorithm with an adaptive step size, we derive the linear convergence of such SCGA algorithm. While the existing research focuses mainly on density ridges in the Euclidean space, we generalize density ridges and the SCMS algorithm to directional data. In particular, we establish the stability theorem of density ridges with directional data and prove the linear convergence of our proposed directional SCMS algorithm.
\end{abstract}



\end{frontmatter}
\tableofcontents

\section{Introduction}
\label{Sec:Intro}

Identifying meaningful lower dimensional structures from a point cloud has long been a popular research topic in Statistics and Machine Learning \citep{Manifold_Learning2012,TDA_review2018}. One reliable characterization of such a low-dimensional structure is the \emph{density ridge}, which can be feasibly estimated by a kernel density estimator (KDE) from point cloud data \citep{Eberly1996ridges,Non_ridge_est2014}. Loosely speaking, an estimated density ridge signifies a high-density curve or surface in a point cloud; see the left panel of Figure~\ref{fig:ridges_Eu_Dir}. 
Let $p$ be the underlying probability density function that generates the data in the Euclidean space $\mathbb{R}^D$.
Its order-$d$ density ridge $R_d$ with $0\leq d < D$ is the set of points defined as:
\begin{equation}
\label{ridge}
R_d = \left\{\bm{x}\in \mathbb{R}^D: V_d(\bm{x})^T \nabla p(\bm{x})=\bm{0}, \lambda_{d+1}(\bm{x}) <0 \right\},
\end{equation}
where $\lambda_1(\bm{x})\geq \cdots \geq \lambda_D(\bm{x})$ are the eigenvalues of Hessian $\nabla\nabla p(\bm{x})$ and $V_d(\bm{x}) \in \mathbb{R}^{D\times (D-d)}$ has its columns as the last $D-d$ orthonormal eigenvectors.
The notion of density ridges has appeared in various scientific fields, such as medical imaging \citep{Vessel_curves2011}, seismology \citep{Ridge_Direct_Ratio2017}, and astronomy \citep{Sousbie_2007,Cosmic_Web_Catal2016}. To locate an estimated density ridge defined by (Euclidean) KDE, \cite{Principal_curve_surf2011} proposed a practical method called \emph{subspace constrained mean shift (SCMS) algorithm}.

\begin{figure}[t]
	\captionsetup[subfigure]{justification=centering}
	\centering
	\begin{subfigure}[t]{.49\textwidth}
		\centering
		\includegraphics[width=1\linewidth]{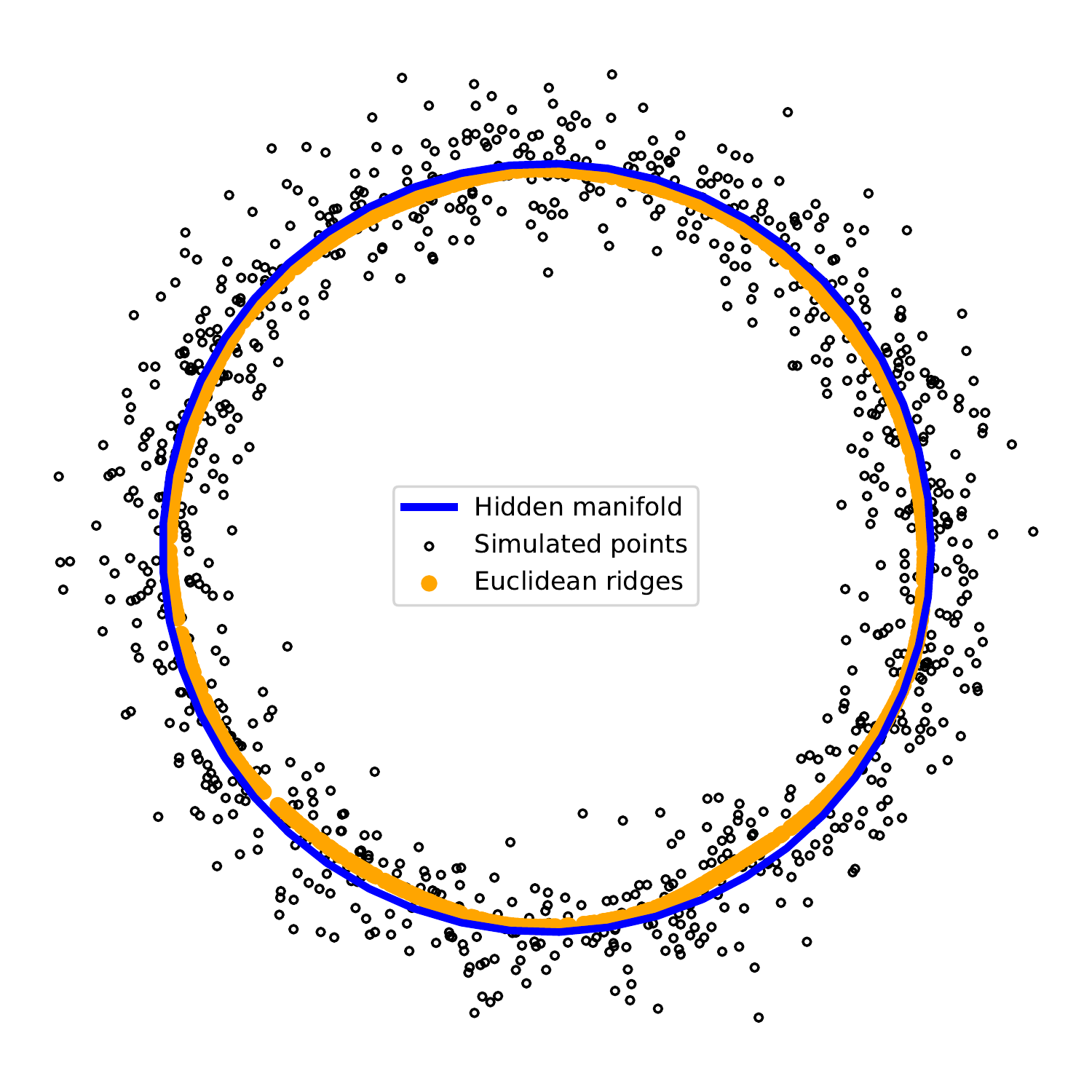}
	\end{subfigure}
	\hfil
	\begin{subfigure}[t]{.49\textwidth}
		\centering
		\includegraphics[width=1\linewidth,height=6.5cm]{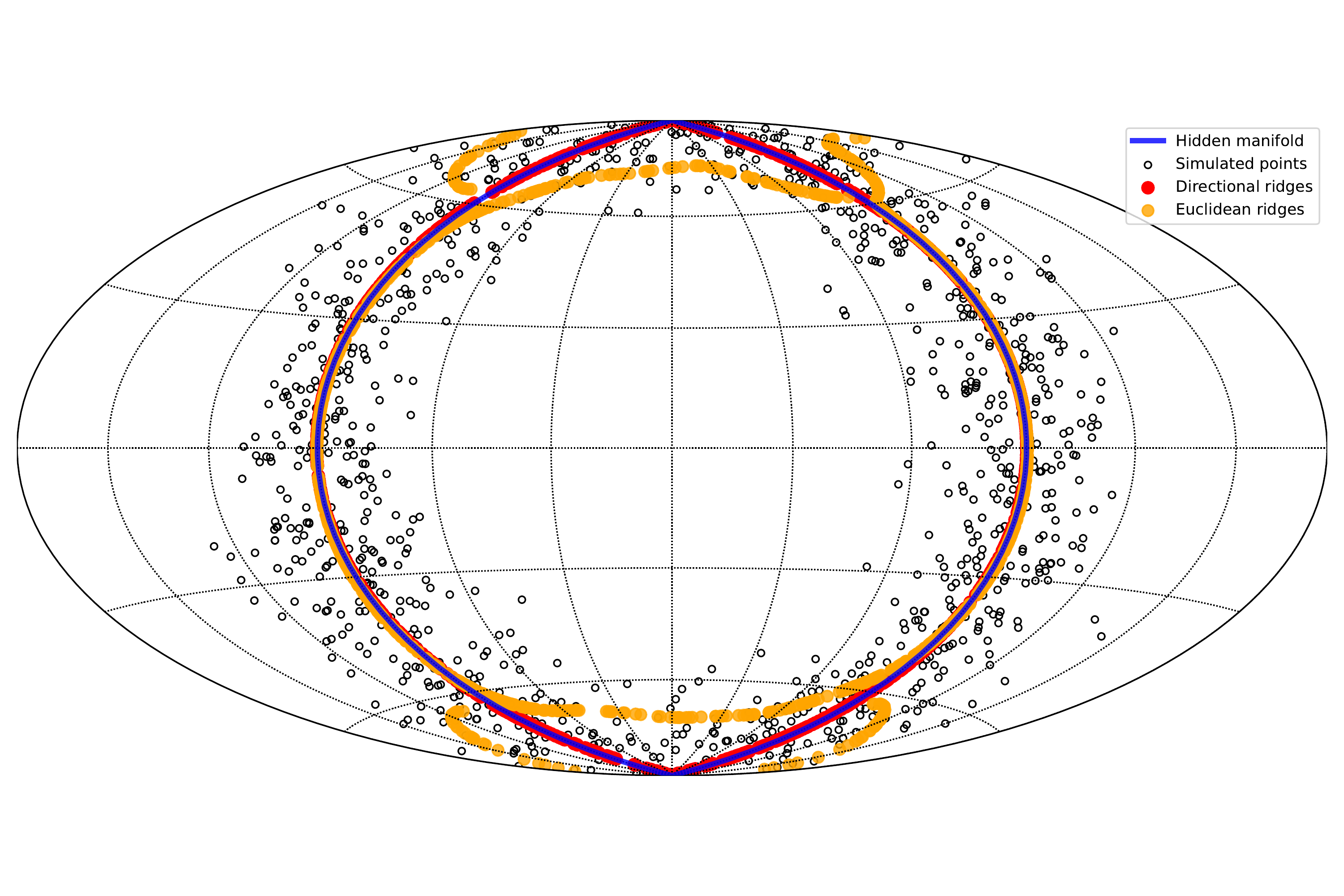}
	\end{subfigure}
	\caption{Density ridges estimated by Euclidean and directional SCMS algorithms on two synthetic datasets (drawn as black points) with hidden circular manifold structures (indicated by blue curves) on $\mathbb{R}^2$ and the unit sphere $\Omega_2 \subset \mathbb{R}^3$, respectively. {\bf Left}: The orange points indicate the estimated ridge obtained by the Euclidean SCMS algorithm from the dataset on $\mathbb{R}^2$. {\bf Right}: The red points represent the estimated directional ridge identified by our directional SCMS algorithm, while the orange points indicate the estimated ridge obtained by the Euclidean SCMS algorithm from the dataset on $\Omega_2$. This panel is presented under the Hammer projection; see Appendix~\ref{Sec:Drawback_Euc} for more details.
	}
	\label{fig:ridges_Eu_Dir}
\end{figure}

While the statistical estimation and asymptotic theories of density ridges in $\mathbb{R}^D$ have been well-studied \citep{Non_ridge_est2014,Asymp_ridge2015,Qiao2016,Op_Ridge_Detect2015,qiao2021asymptotic}, the literature falls short of addressing the algorithmic properties of the ridge-finding method, \emph{i.e.}, the SCMS algorithm. 
To the extent of our knowledge, \cite{SCMS_conv2013,SCMS_conv_mod2020} were the only available works to investigate the SCMS algorithm and its modified version from an algorithmic perspective. However, they only proved a non-decreasing property of density estimates and the validity of two stopping criteria for the SCMS algorithm. 
The algorithmic convergence of the SCMS algorithm remains an open question.
There are two challenges to answering this question. First, because every iteration of the SCMS algorithm involves a projection matrix defined by the (estimated) Hessian, it is no longer a conventional first-order method in optimization. Second, estimating a density ridge in practice is a nonconvex/nonconcave optimization problem. Thus, the first objective of this paper is to provide a theoretical study on the algorithmic convergence and its associated (linear) rate of convergence for the SCMS algorithm. 

In stark contrast to abundant research papers about density ridges in the Euclidean space, little work has been done to examine the statistical properties and any practical algorithm of estimating density ridges on the unit hypersphere $\Omega_q= \{\bm{x}\in\mathbb{R}^{q+1}:\|\bm{x}\|_2=1\} \subset \mathbb{R}^{q+1}$. Nevertheless, data on $\Omega_q$ are ubiquitous in many scientific fields of study, such as seismology (\emph{e.g.}, longitudes and latitudes of the epicenters of earthquakes) and astronomy (\emph{e.g.}, right ascensions and declinations of astronomical objects). Such data are generally known as \emph{directional data} in the statistical literature \citep{Mardia2000directional,ley2017modern}. Hence, the second objective of this paper is to generalize density ridges and the SCMS algorithm to directional data. 

More importantly, identifying an estimated density ridge from directional data on $\Omega_2$ by the Euclidean SCMS algorithm always suffers from high bias near the two poles of $\Omega_2$. Consider a synthetic dataset with independently and identically distributed (i.i.d.) observations $\left\{\bm{X}_1,...,\bm{X}_{1000}\right\}$ from a great circle connecting the North and South Poles of $\Omega_2$ with additive noises. We apply both the Euclidean and directional SCMS algorithms to this simulated dataset. While the estimated ridges by the Euclidean SCMS algorithm fail to recover the desired great circle in high latitude regions, the ridges identified by our proposed directional SCMS algorithm align well with the underlying circular structure; see the right panel in Figure~\ref{fig:ridges_Eu_Dir} for a preview and Appendix~\ref{Sec:Drawback_Euc} for a more detailed discussion.

{\em Main Results}. The main contributions of this paper are summarized as follows: 

\noindent $\bullet$ We present the convergence analysis of the SCMS and the general SCGA algorithms and prove their linear convergence properties with Euclidean data (Theorem~\ref{SCGA_LC}, Corollary~\ref{SCMS_LC}, and related discussion in Section~\ref{Sec:LC_SCGA}):
$$
\norm{\hat{\bm{x}}^{(t)} - \hat{\bm{x}}^*}_2 \leq \Upsilon^t \norm{\hat{\bm{x}}^{(0)} - \hat{\bm{x}}^*}_2,
$$
where $\big\{ \hat{\bm{x}}^{(t)} \big\}_{t=0}^{\infty}$ is a sequence of points generated by the SCGA or SCMS algorithm in $\mathbb{R}^D$, $\hat{\bm{x}}^*$ is the limit point of the sequence, and $0<\Upsilon<1$ is a constant. 

\noindent $\bullet$ We generalize density ridges and the SCMS algorithm to directional data on $\Omega_q$ (Section~\ref{Sec:Dir_ridge_SCMS}). 

\noindent $\bullet$ We prove the statistical convergence rate of a ridge estimator on the sphere $\Omega_q$ defined by the directional KDE (Theorem~\ref{Dir_ridge_stability}):
$$\Haus(\underline{R}_d, \underline{\hat{R}}_d) = O\left(h^2\right) +O_P\left(\sqrt{\frac{|\log h|}{nh^{q+4}}}\right),$$
where $\underline{R}_d$ and $\underline{\hat{R}}_d$ are the population and estimated directional density ridges, respectively, $\Haus$ is the Hausdorff distance, and $q$ is the dimension of $\Omega_q$.

\noindent $\bullet$ We establish the convergence of the SCMS and the general SCGA algorithms with directional data and derive their linear convergence results (Theorem~\ref{LC_SCGA_Dir}, Corollary~\ref{LC_Dir_SCMS}, and related expositions in Section~\ref{Sec:Dir_SCGA_LC}):
$$d_g(\hat{\underline{\bm{x}}}^{(t)}, \hat{\underline{\bm{x}}}^*) \leq \underline{\Upsilon}^t \cdot d_g(\hat{\underline{\bm{x}}}^{(0)}, \hat{\underline{\bm{x}}}^*),$$ 
where $\big\{\hat{\underline{\bm{x}}}^{(t)} \big\}_{t=0}^{\infty}$ is the sequence of points generated by the directional SCGA or SCMS algorithm, $\hat{\underline{\bm{x}}}^*$ is the convergence point, $0<\underline{\Upsilon}<1$ is a constant, and $d_g$ is the geodesic distance on $\Omega_q$.


{\em Other Related Literature}. The problem of density ridge estimation has its unique standing in both the computer science and statistics literature; see \cite{hall1992,Eberly1996ridges,Damon1999,hall2001} and references therein. Among various definition of density ridges \citep{norgard2012second,peikert2013comment}, our definition follows from \cite{Eberly1996ridges,Non_ridge_est2014,Asymp_ridge2015}, because its statistical estimation theory has been well-established and it is feasible to be directly generalized to directional densities.
Practically, the SCMS algorithm for identifying an estimated density ridge first appeared in the field of computer vision \citep{saragih2009face} before its introduction to the statistical community by \cite{Principal_curve_surf2011}.  
More recently, \cite{qiao2021algorithms} proposed alternative methods to the SCMS algorithm for finding density ridges, which are based on a gradient descent of the ridgeness and have connections to solution manifolds \citep{YC2020}. They presented the convergence analysis on continuous versions of their proposed methods and discretized them via Euler's method.
Our directional SCMS algorithm is extended from the directional mean shift algorithm \citep{Multi_Clu_Gene2005,DMS_topology2010,kobayashi2010mises,MSC_Dir2014,DirMS2020,DMS_EM2021}. As we cast the (directional) SCMS algorithms into subspace constrained gradient ascent (SCGA) algorithms (on a hypersphere), it is worth mentioning that one should not confuse the SCGA algorithm here with the projected gradient ascent/descent method for a constrained problem in the standard optimization theory; see Section 3.2 in \cite{SB2015} for some references of the latter one. The SCGA algorithm discussed in this paper is a gradient ascent algorithm but with a subspace constrained gradient. When the subspace coincides with alternating one-dimensional coordinate spaces, the SCGA algorithm reduces to the well-known coordinate ascent/descent method \citep{wright2015coordinate}. Some linear convergence results of the coordinate descent algorithms were previously established by \cite{luo1992convergence,beck2013convergence}. Other related work includes \cite{kozak2019stochastic,kozak2020stochastic}, though, in their problem setups, the projection matrix onto the subspace is random and has its expectation equal to the identity matrix. Our interested SCGA algorithm always has a deterministic constrained subspace defined by the eigenspace associated with the last several eigenvalues of the Hessian of the density $p$.


{\em Outlines and Notations}. Section~\ref{Sec:Prelim} introduces the definitions of Euclidean and directional KDEs and reviews some preliminary concepts of differential geometry on $\Omega_q$. We discuss the assumptions on the Euclidean density ridges and establish the (linear) convergence results of the SCGA and SCMS algorithms in Section~\ref{Sec:Euc_Ridges_SCMS}. In Section~\ref{Sec:Dir_ridge_SCMS}, we generalize the definition of density ridges to the directional data scenario and prove the (linear) convergence properties of the SCGA and SCMS algorithms on $\Omega_q$. Some simulation studies and real-world applications of Euclidean and directional SCMS algorithms are presented in Section~\ref{App:Experiments}, whose code is available at \url{https://github.com/zhangyk8/EuDirSCMS}. We conclude the paper and discuss some potential impacts in Section~\ref{sec::discussion}. 

Throughout the paper we use $d$ as the intrinsic dimension of density ridges, whose ambient spaces are $\mathbb{R}^D$ in the Euclidean data case and $\Omega_q =\left\{\bm{x}\in \mathbb{R}^{q+1}: \norm{\bm{x}}_2=1 \right\}$ in the directional data case. Notice that \emph{a quantity under the directional data setting that has its counterpart in the Euclidean data case will be denoted by the same notation with an extra underline}. For instance, $R_d$ is a ridge of the density $p$ in the Euclidean space $\mathbb{R}^D$ while $\underline{R}_d$ refers to a ridge of the directional density $f$ on the sphere $\Omega_q$.

Let $f:\mathbb{R}^D \to \mathbb{R}$ be a smooth function and $[\alpha]=(\alpha_1,...,\alpha_D)$ be a multi-index (that is, $\alpha_1,...,\alpha_D$ are nonnegative integers and $|[\alpha]|=\sum_{i=1}^D \alpha_i$). Define $D^{[\alpha]} = \frac{\partial^{\alpha_1}}{\partial x_1^{\alpha_1}} \cdots \frac{\partial^{\alpha_D}}{\partial x_D^{\alpha_D}}$ as the $|[\alpha]|$-th order partial derivative operator, where $D^{[\alpha]} f$ is often written as $f^{(\alpha)}$. For $j=0,1,...$, we define the functional norms 
$$\norm{f}_{\infty}^{(j)} = \max_{\alpha:|[\alpha]|=j} \sup_{\bm{x}\in \mathbb{R}^D} |f^{(\alpha)}(\bm{x})|.$$
When $j=0$, this becomes the infinity norm of $f$; for $j>0$, the above norms are indeed some semi-norms. We also define $\norm{f}_{\infty,k}^*=\max_{j=0,...,k} \norm{f}_{\infty}^{(j)}$.

The (total) gradient and Hessian of $f$ are defined as $\nabla f(\bm{x}) = \Big(\frac{\partial f(\bm{x})}{\partial x_1},..., \frac{\partial f(\bm{x})}{\partial x_D} \Big)^T$ and $\nabla\nabla f(\bm{x}) = \Big(\frac{\partial^2 f(\bm{x})}{\partial x_i \partial x_j} \Big)_{1\leq i,j \leq D}$. Inductively, the third derivative of $f(\bm{x})$ is a $D\times D\times D$ array given by $\nabla^3 f(\bm{x})= \Big(\frac{\partial^3}{\partial x_i\partial x_j \partial x_k}f(\bm{x}) \Big)_{1\leq i,j,k \leq D}$. When $f$ is a directional density supported on $\Omega_q$, the preceding functional norms are defined via the Riemannian gradient, Hessian, and high-order derivatives of $f(\bm{x})$ within the tangent space $T_{\bm{x}}$ at $\bm{x}\in \Omega_q$, and the supremum will be taken over $\Omega_q$ instead of $\mathbb{R}^D$. They are equivalent to the derivatives of $f$ with respect to the local coordinate chart on $\Omega_q$; see Section~\ref{Sec:DG_sphere} for a review.

Let $A_{jk}$ denote the $(j,k)$ entry of a matrix $A\in \mathbb{R}^{m\times n}$. Then, the Frobenius norm is $\norm{A}_F = \sqrt{\sum_{j,k} A_{jk}^2} = \mathtt{tr}(A^TA)$, where $\mathtt{tr}(A^TA)$ is the trace of the square matrix $A^TA$, and the operator norm is $\norm{A} = \sup_{||\bm{x}||=1} \norm{A\bm{x}}$. In most cases, we consider the $L_2$ (operator) norm $\norm{\cdot}_2$. We define $\norm{A}_{\max} = \max_{j,k} |A_{jk}|$. The inequality relationships between the above matrix norms are $\norm{A}_2\leq \norm{A}_F \leq \sqrt{n} \norm{A}_2$, $\norm{A}_{\max}\leq \norm{A}_2 \leq \sqrt{mn}\norm{A}_{\max}$, and $\norm{A}_F \leq \sqrt{mn}\norm{A}_{\max}$.

We use the big-O notation $h(\bm{x})=O(g(\bm{x}))$ if the absolute value of $h(\bm{x})$ is upper bounded by a positive constant multiple of $g(\bm{x})$ for all sufficiently large $\bm{x}$. In contrast, $h(\bm{x})=o(g(\bm{x}))$ when $\lim_{\norm{\bm{x}}_2\to\infty} \frac{|h(\bm{x})|}{g(\bm{x})}=0$. For random vectors, the notation $o_P(1)$ is short for a sequence of random vectors that converges to zero in probability. The expression $O_P(1)$ denotes the sequence that is bounded in probability; see Section 2.2 of \cite{VDV1998} for details.

\section{Preliminaries}
\label{Sec:Prelim}

In this section, we review the KDE with Euclidean and directional data as well as some differential geometry concepts on $\Omega_q$.

\subsection{Kernel Density Estimation with Euclidean Data}
\label{Sec:KDE}

Let $\{\bm{X}_1,...,\bm{X}_n \}$ be a random sample from a distribution $P$ with density $p$ supported on the Euclidean space $\mathbb{R}^D$. 
We call such random sample $\{\bm{X}_1,...,\bm{X}_n \}$ Euclidean data in the sequel. The (Euclidean) KDE at point $\bm{x}\in \mathbb{R}^D$ with a kernel function $K$ and bandwidth parameter $h\equiv h(n)$ is written as \citep{All_nonpara2006,Scott2015,KDE_t}:
\begin{equation}
\label{KDE_simp}
\hat{p}_n(\bm{x}) = \frac{1}{nh^D} \sum_{i=1}^n K\left(\frac{\bm{x}-\bm{X}_i}{h} \right).
\end{equation}
The kernel $K:\mathbb{R}^D \to \mathbb{R}$ is generally a unimodal function satisfying the following properties:
\begin{itemize}
	\item (K1) $\int_{\mathbb{R}^D} K(\bm{x}) d\bm{x}=1$.
	\item (K2) $K(\bm{x})$ is (radially) symmetric, \emph{i.e.}, $\int_{\mathbb{R}^D} \bm{x} \, K(\bm{x})\, d\bm{x} = 0$. 
	\item (K3) $\lim\limits_{\norm{\bm{x}}_2\to\infty} \norm{\bm{x}}_2^D K(\bm{x}) =0$ and $\int_{\mathbb{R}^D} \norm{\bm{x}}_2^2 K(\bm{x}) d\bm{x} <\infty$, where $\norm{\cdot}_2$ is the usual $L_2$ norm in $\mathbb{R}^D$.
\end{itemize}

One possible approach to construct a multivariate kernel $K(\bm{x})$ with the above properties is to derive it from a kernel profile as follows:
\begin{equation}
\label{Radial_kernel}
K(\bm{x}) = c_{k,D} \cdot k\left(\norm{\bm{x}}_2^2 \right),
\end{equation}
where $c_{k,D}$ is the normalizing constant such that $K$ satisfies (K1) and the function $k:[0,\infty) \to [0,\infty)$ is called the \emph{profile} of the kernel. This kernel form is generally used in deriving (subspace constrained) mean shift algorithms; see Section~\ref{Sec:Eu_MS_SCMS}. An important example of the profile function is $k_N(x)=\exp\left(-\frac{x}{2} \right)$ for $x\geq 0$, leading to the multivariate Gaussian kernel $K_N(\bm{x}) = \frac{1}{(2\pi)^{\frac{D}{2}}} \exp\left(-\frac{\norm{\bm{x}}_2^2}{2} \right)$.

Another approach of designing a multivariate kernel function is to leverage the product kernel technique as $K(\bm{x}) = K_1(x_1) \cdots K_D(x_D)$,
where $K_1,...,K_D$ are kernels function defined on $\mathbb{R}$ satisfying the properties (K1-3). This leads to a multivariate KDE as:
\begin{equation}
\label{prod_kernel_KDE}
\hat{p}_n(\bm{x}) = \frac{1}{nh^D} \sum_{i=1}^n K_1\left(\frac{x_1-X_{i,1}}{h} \right) \cdots K_D\left(\frac{x_D-X_{i,D}}{h} \right).
\end{equation}
In fact, the multivariate Gaussian kernel $K_N$ can be obtained by defining its kernel profile as $k_N(x)=\exp\left(-\frac{x}{2} \right)$ for $x\geq 0$ or taking $K_1(x)=\cdots =K_D(x)=\frac{1}{\sqrt{2\pi}} \exp\left(-\frac{x^2}{2} \right)$. In practice, the multivariate KDE \eqref{KDE_simp} with Gaussian kernel is the most popular nonparametric density estimator with Euclidean data. 

The most crucial part in applying the KDE is to select the bandwidth parameter $h$. Common methods in the literature aim at minimizing the mean integrated square error (MISE):
	$$\text{MISE} = \mathbb{E} \int_{\mathbb{R}^D} \left[\hat{p}_n(\bm{x}) -p(\bm{x}) \right]^2 d\bm{x}$$
	or its asymptotic part through the rule of thumb \citep{Silverman1986}, cross validation \citep{rudemo1982empirical,bowman1984alternative,hall1983large,stone1984asymptotically}, and plug-in methods \citep{sheather1991reliable}.
	As choosing the bandwidth is not the main focus of this paper, we refer the interested reader to \cite{survey_bw1996,Sheather2004} and Chapter 6.5 of \cite{Scott2015} for comprehensive reviews.

\subsection{Kernel Density Estimation with Directional Data}
\label{Sec:Dir_KDE}

The Euclidean KDE \eqref{KDE_simp} exhibits some salient drawbacks in dealing with directional data samples; see Appendix~\ref{Sec:Drawback_Euc} for a detailed exposition. Fortunately, the theory of kernel density estimation with directional data has been well-studied since late 1970s \citep{Beran1979,KDE_Sphe1987,KDE_direct1988,Zhao2001,Exact_Risk_bw2013,pewsey2021recent}. Let $\bm{X}_1,...,\bm{X}_n \in \Omega_q\subset\mathbb{R}^{q+1}$ be a random sample generated from an underlying directional density function $f$ on $\Omega_q$ with $\int_{\Omega_q} f(\bm{x}) \, \omega_q(d\bm{x})=1,$
where $\omega_q$ is the Lebesgue measure on $\Omega_q$.  
The directional KDE is given by:
\begin{equation}
\label{Dir_KDE}
\hat{f}_h(\bm{x}) = \frac{c_{L,q}(h)}{n} \sum_{i=1}^n L\left(\frac{1-\bm{x}^T \bm{X}_i}{h^2} \right) = \frac{c_{L,q}(h)}{n} \sum_{i=1}^n L\left(\frac{1}{2}\norm{\frac{\bm{x}-\bm{X}_i}{h}}_2^2 \right),
\end{equation}
where $L$ is a directional kernel (\emph{i.e.}, a rapidly decaying function with nonnegative values and defined on $(-\delta_L,\infty) \subset \mathbb{R}$ for some constant $\delta_L >0$), 
$h>0$ is the bandwidth parameter, and $c_{L,q}(h)$ is a normalizing constant satisfying $c_{L,q}(h)^{-1} = \int_{\Omega_q} L\left(\frac{1-\bm{x}^T \bm{y}}{h^2} \right) \omega_q(d\bm{y}) =O(h^q)$.

\begin{remark}
	The distance metric used by the directional KDE \eqref{Dir_KDE} on $\Omega_q$ is identical to the standard Euclidean metric in the ambient space $\mathbb{R}^{q+1}$. This is because the standard Euclidean metric $\norm{\cdot}_2$ of $\mathbb{R}^{q+1}$ is topologically equivalent (but not strongly equivalent) to the geodesic distance $d_g(\cdot,\cdot)$ on $\Omega_q$ due to the following equality:
	\begin{equation}
		\label{Geo_Eu_dist_eq}
		\norm{\bm{x}-\bm{y}}_2 =2\sin\left(\frac{d_g(\bm{x},\bm{y})}{2} \right).
	\end{equation}
	See Section C.1.5 in \cite{ok2007real} for the definition of equivalence of metrics. Hence, the distance metric in \eqref{Dir_KDE} is indeed intrinsic on $\Omega_q$ and adaptive to its geometry.
\end{remark}

As in the applications of Euclidean KDEs, the bandwidth selection is a critical part in determining the performances of directional KDEs \citep{KDE_Sphe1987,KDE_direct1988,Auto_bw_cir2008,KDE_torus2011,Oliveira2012,Exact_Risk_bw2013,Nonp_Dir_HDR2020}. On the contrary, the choice of the kernel is less crucial; see, \emph{e.g.}, Page 72 of \cite{All_nonpara2006} and Section 6.3.2 in \cite{Scott2015} for the reasoning. A popular candidate is the so-called von Mises kernel $L(r) = e^{-r}$, which serves as a counterpart of the Gaussian kernel for directional KDEs. Its name originates from the famous $q$-von Mises-Fisher distribution on $\Omega_q$, which is denoted by $\text{vMF}(\bm{\mu}, \nu)$ and has the density as:
\begin{equation}
\label{vMF_density}
f_{\text{vMF}}(\bm{x};\bm{\mu},\nu) = C_q(\nu) \cdot \exp(\nu \bm{\mu}^T \bm{x}) \quad \text{ with } \quad C_q(\nu) = \frac{\nu^{\frac{q-1}{2}}}{(2\pi)^{\frac{q+1}{2}} \mathcal{I}_{\frac{q-1}{2}}(\nu)},
\end{equation}
where $\bm{\mu} \in \Omega_q$ is the directional mean, $\nu \geq 0$ is the concentration parameter, and $\mathcal{I}_{\alpha}(\nu)$ is the modified Bessel function of the first kind at order $\nu$. For more details on statistical properties of the von Mises-Fisher distribution and directional KDE, we refer the interested reader to \cite{Mardia2000directional,spherical_EM,Dir_Linear2013}.

\subsection{Riemannian Gradient, Hessian, and Exponential Map on $\Omega_q$}
\label{Sec:DG_sphere}

Given that the unit hypersphere $\Omega_q$ is a nonlinear manifold, the Riemannian gradient and Hessian of a smooth function $f$ on $\Omega_q$ are defined within its tangent spaces. They are different from but also interconnected with the total gradient and Hessian of $f$ in the ambient Euclidean space $\mathbb{R}^{q+1}$.

\noindent $\bullet$ {\bf Riemannian Gradient on $\Omega_q$}. Let $T_{\bm{x}}$ be the tangent space of $\Omega_q$ at point $\bm{x}\in \Omega_q$, which consists of all the vectors starting from $\bm{x}$ and tangent to $\Omega_q$. Given a smooth function $f:\Omega_q\to \mathbb{R}$, its \emph{Riemannian gradient} $\grad f(\bm{x}) \in T_{\bm{x}}$ is defined as:
\begin{equation}
\label{Riem_grad}
\left\langle \bm{v}, \grad f(\bm{x}) \right\rangle_{\bm{x}} = df_{\bm{x}}(\bm{v})
\end{equation}
for any (unit) vector $\bm{v} \in T_{\bm{x}}$, where $\langle \cdot,\cdot \rangle_{\bm{x}}$ is the inner product (or Riemannian metric) in $T_{\bm{x}}$ and $df_{\bm{x}}: T_{\bm{x}} \to \mathbb{R}$ is the differential operator of $f$ at $\bm{x}\in \Omega_q$; see, \emph{e.g.}, Section 3.1 in \cite{Morse_Homology2004} for more details. Note that the Riemannian metric on $\Omega_q$ coincides with the standard inner product $\langle\cdot,\cdot \rangle$ in the ambient space $\mathbb{R}^{q+1}$; see Section 3.6.1 in \cite{Op_algo_mat_manifold2008}. If $f$ is smooth in an open neighborhood containing $\Omega_q$ and we consider $\grad f(\bm{x}), \bm{v} \in T_{\bm x}$ as vectors in $\mathbb{R}^{q+1}$, then the inner product in $T_{\bm{x}}$ reduces to the usual one in $\mathbb{R}^{q+1}$ and the Riemannian gradient $\grad f(\bm{x})$ can be expressed in terms of the total gradient $\nabla f(\bm{x})$ as:
\begin{equation}
\label{Riem_grad2}
\grad f(\bm{x}) = \left(\bm{I}_{q+1}-\bm{x}\bm{x}^T \right) \nabla f(\bm{x}),
\end{equation}
where $\bm{I}_{q+1}\in \mathbb{R}^{(q+1)\times (q+1)}$ is the identity matrix. The left-hand side of \eqref{Riem_grad2} is the projection of the total gradient $\nabla f(\bm{x})$ onto the tangent space $T_{\bm{x}}$ at $\bm{x}\in \Omega_q$.

\noindent $\bullet$ {\bf Riemannian Hessian on $\Omega_q$}. The \emph{Riemannian Hessian} $\mathcal{H}f(\bm{x})$ at point $\bm{x}\in \Omega_q$ is a symmetric bilinear map from the tangent space $T_{\bm{x}}$ into itself defined as:
	\begin{equation}
	\label{Riem_Hess}
	\mathcal{H}f(\bm{x})\left[\bm{v}\right]= \bar{\nabla}_{\bm{v}} \grad f(\bm{x})
	\end{equation}
	for any $\bm{v} \in T_{\bm{x}}$, where $\bar{\nabla}_{\bm{v}}$ is the Riemannian connection on $\Omega_q$. Similar to $\grad f(\bm{x})$, the Riemannian Hessian $\mathcal{H} f(\bm{x})$ has the following explicit formula when viewed in the ambient Euclidean space $\mathbb{R}^{q+1}$:
\begin{equation}
\label{Riem_Hess2}
\mathcal{H} f(\bm x) = (\bm{I}_{q+1} - \bm{x}\bm{x}^T) \left[\nabla\nabla f(\bm{x}) - \nabla f(\bm{x})^T \bm{x} \cdot \bm{I}_{q+1} \right](\bm{I}_{q+1} - \bm{x}\bm{x}^T),
\end{equation}
where $\nabla f(\bm{x})$ and $\nabla\nabla f(\bm{x})$ are the total gradient and Hessian of $f$ in $\mathbb{R}^{q+1}$. This formula can be derived via the Riemannian connection and Weingarten map on $\Omega_q$ (\citealt{Extrinsic_Look_Riem_Manifold} and Section 5.5 in \citealt{Op_algo_mat_manifold2008}) or geodesics on $\Omega_q$ \citep{DirMS2020}.

\noindent $\bullet$ {\bf Exponential Map}. An \emph{exponential map} $\Exp_{\bm x}: T_{\bm x} \to \Omega_q$ at $\bm x\in \Omega_q$ is a mapping that takes a vector $\bm{v}\in T_{\bm{x}}$ to a point $\bm y:=\Exp_{\bm x}(\bm v) \in \Omega_q$ along the curve $\varphi$ with $\varphi(0) =\bm x, \varphi(1)=\bm y$ and $\varphi'(0)=\bm v$. Here, $\varphi: [0,1] \to \Omega_q$ is a curve of minimum length between $\bm{x}$ and $\bm{y}$ (\emph{i.e.}, the so-called geodesic on $\Omega_q$). An intuitive way of thinking of the exponential map $\Exp_{\bm x}$ evaluated at $\bm{v}$ on $\Omega_q$ is that starting at point $\bm x$, we identify another point $\bm y$ on $\Omega_q$ along the geodesic (or great circle) in the direction of $\bm{v}$ so that the geodesic distance between $\bm{x}$ and $\bm{y}$ is $\norm{\bm{v}}_2$. 
As $\Omega_q$ is a compact Riemannian manifold, the exponential map $\Exp_{\bm x}$ is a diffeomorphism (smooth bijection) from a neighborhood of $\bm{0}\in T_{\bm{x}}$ to its image on $\Omega_q$; see Lemma 6.16 in \cite{Lee2018Riem_man}.
The inverse of an exponential map (or logarithmic map) is defined within a neighborhood $U \subset \Omega_q$ around $\bm{x}$ as a mapping $\Exp_{\bm x}^{-1}: U \to T_{\bm x}$ such that $\Exp_{\bm x}^{-1}(\bm y)$ represents the vector in $T_{\bm x}$ starting at $\bm{x}$, pointing to $\bm{y}$, and with its length equal to the geodesic distance between $\bm{x}$ and $\bm{y}$.

\section{Linear Convergence of the SCMS Algorithm With Euclidean Data}
\label{Sec:Euc_Ridges_SCMS}

Given the definition of a order-$d$ ridge $R_d$ in \eqref{ridge} of the (smooth) density $p$ on the Euclidean space $\mathbb{R}^D$, we introduce, in this section, some commonly assumed conditions to regularize $R_d$ and its stability theorem. After revisiting the frameworks of the Euclidean mean shift and SCMS algorithms as well as deriving the SCMS algorithm as the SCGA algorithm with an adaptive step size, we present our (linear) convergence analysis on the SCGA and SCMS algorithms.

\subsection{Assumptions and Stability of Euclidean Density Ridges}
\label{Sec:Assump}

Under the spectral decomposition on the Hessian $\nabla\nabla p(\bm{x})$ as $\nabla\nabla p(\bm{x}) = V(\bm{x}) \Lambda(\bm{x}) V(\bm{x})^T$, we know that $V(\bm{x}) = [\bm{v}_1(\bm{x}),...,\bm{v}_D(\bm{x})]\in \mathbb{R}^{D\times D}$ is a real orthogonal matrix with the eigenvectors of $\nabla\nabla p(\bm{x})$ as its columns and $\Lambda(\bm{x}) = \Diag\left[\lambda_1(\bm{x}),...,\lambda_D(\bm{x}) \right]\in \mathbb{R}^{D\times D}$ is a diagonal matrix with $\lambda_1(\bm{x}) \geq \cdots \geq \lambda_D(\bm{x})$. Given that $V_d(\bm{x}) = \left[\bm{v}_{d+1}(\bm{x}),...,\bm{v}_D(\bm{x}) \right] \in \mathbb{R}^{D\times (D-d)}$, we let $U_d(\bm{x}) \equiv V_d(\bm{x}) V_d(\bm{x})^T$ be the projection matrix onto the column space of $V_d(\bm{x})$ and $U_d^{\perp}(\bm{x}) = \bm{I}_D -V_d(\bm{x}) V_d(\bm{x})^T = V_{\diamond}(\bm{x}) V_{\diamond}(\bm{x})^T$ be the projection matrix onto the complement space, where $V_{\diamond}(\bm{x})=[\bm{v}_1(\bm{x}),...,\bm{v}_d(\bm{x})] \in \mathbb{R}^{D\times d}$ and $\bm{I}_D$ is the identity matrix in $\mathbb{R}^{D\times D}$. Then, the order-$d$ principal gradient $G_d(\bm{x})$ (or projected gradient in \citealt{Non_ridge_est2014,Asymp_ridge2015}) is defined as:
\begin{equation}
\label{proj_grad}
G_d(\bm{x}) = U_d(\bm{x}) \nabla p(\bm{x}) = V_d(\bm{x}) V_d(\bm{x})^T \nabla p(\bm{x}),
\end{equation}
and $U_d^{\perp}(\bm{x}) \nabla p(\bm{x})$ will be called the residual gradient. The order-$d$ density ridge can be equivalently defined as:
\begin{equation}
\label{ridge_equivalent}
R_d = \left\{\bm{x}\in \mathbb{R}^D: G_d(\bm{x})=\bm{0}, \lambda_{d+1}(\bm{x}) <0 \right\}.
\end{equation} 

It follows that the 0-ridge $R_0$ is the set of local modes of $p$, whose statistical properties and practical estimation algorithm have been well-studied in \cite{MS_consist2016,Mode_clu2016}. Thus, we only consider the case when $1\leq d<D$ in the sequel. We define the projection from point $\bm{x}\in \mathbb{R}^D$ onto a ridge $R_d$ by $\pi_{R_d}(\bm{x}) = \argmin_{\bm{y}\in R_d} \norm{\bm{x}-\bm{y}}_2$ and the distance from point $\bm{x}$ to $R_d$ by $d_E(\bm{x},R_d) = \norm{\bm{x}-\pi_{R_d}(\bm{x})}_2 = \min_{\bm{y}\in R_d} \norm{\bm{x}-\bm{y}}_2$. Note that the projection from point $\bm{x}\in \mathbb{R}^D$ to $R_d$ may not be unique. To guarantee the uniqueness of the projection, we introduce a concept called the \emph{reach} \citep{Federer1959,cuevas2009set} as:
\begin{equation}
\label{reach}
\mathtt{reach}(R_d) = \inf\left\{\delta>0: \forall \bm{x}\in R_d\oplus \delta, \bm{x} \text{ has a unique projection onto }R_d \right\},
\end{equation}
where $R_d \oplus \delta = \cup_{\bm{x}\in R_d} \text{Ball}_D(\bm{x},\delta)$ and $\text{Ball}_D(\bm{x},\delta) = \left\{\bm{z}\in \mathbb{R}^D: \norm{\bm{z}-\bm{x}}_2 \leq \delta \right\}$ is a $D$-dimensional ball of radius $\delta$ centered at $\bm{x}$. To obtain a well-behaved ridge $R_d$, some assumptions need imposing on the underlying density $p$ around a small neighborhood of $R_d$.

\begin{itemize}
	\item {\bf (A1)} (\emph{Differentiability}) We assume that $p$ is bounded and at least four times differentiable with bounded partial derivatives up to the fourth order for every $\bm{x}\in \mathbb{R}^D$.
	\item {\bf (A2)} (\emph{Eigengap}) We assume that there exist constants $\rho>0$ and $\beta_0>0$ such that $\lambda_{d+1}(\bm{y}) \leq -\beta_0$ and $\lambda_d(\bm{y}) -\lambda_{d+1}(\bm{y}) \geq \beta_0$ for any $\bm{y} \in R_d\oplus \rho$.
	\item {\bf (A3)} (\emph{Path Smoothness}) Under the same $\rho, \beta_0>0$ in (A2), we assume that there exists another constant $\beta_1 \in (0,\beta_0)$ such that
	\begin{align*}
	D^{\frac{3}{2}}\norm{U_d^{\perp}(\bm{y}) \nabla p(\bm{y})}_2 \norm{\nabla^3 p(\bm{y})}_{\max} &\leq \frac{\beta_0^2}{2},\\
	d \cdot D^{\frac{3}{2}} \norm{\nabla p(\bm{x})}_2 \norm{\nabla^3 p(\bm{x})}_{\max} & \leq \beta_0(\beta_0-\beta_1)
	\end{align*}
	for all $\bm{y} \in R_d\oplus \rho$ and $\bm{x}\in R_d$. 
\end{itemize}

Condition (A1) is a natural differentiability assumption under the context of ridge estimation. Condition (A2) is a curvature assumption on the true density $p$, ensuring that $p$ is ``strongly concave'' around $R_d$ inside the $(D-d)$-dimensional linear space spanned by the columns of $V_d(\bm{y})$. We call this property ``subspace constrained strong concavity''. It is one of the most important components in establishing the linear convergence of the SCGA and SCMS algorithms; see Remark~\ref{SC_remark} for the reasoning. Condition (A3) regularizes the gradient and third order derivatives of $p$ from being too steep around the ridge $R_d$. They are also imposed by \cite{Non_ridge_est2014} for characterizing a quadratic behavior of $p$ around $R_d$ and ensuring the stability of $R_d$, as well as by \cite{Asymp_ridge2015} to avoid the degenerate normal spaces of $R_d$. Consequently, $R_d$ is a $d$-dimensional manifold that contains neither intersections nor endpoints; see also Lemma~\ref{normal_reach_prop} in the Appendix.
Notice that the inequality assumptions in (A3) depend on both the ambient dimension $D$ and the intrinsic dimension $d$ of the ridge $R_d$. The larger the dimensions $D$ and $d$ are, the harder the assumptions will hold. This phenomenon, in some sense, reflects the \emph{curse of dimensionality} in nonparametric ridge estimation.

Given conditions (A1-3), the ridge $R_d$ will be stable under small perturbations of the underlying density $p$ and its derivatives, which is summarized in the following lemma. The stability of $R_d$ is generally measured by the Hausdorff distance defined as:
\begin{equation}
\label{Haus_dist}
\Haus(A,B) = \inf\left\{\epsilon>0: A\subset B\oplus \epsilon \text{ and } B\subset A\oplus \epsilon \right\},
\end{equation}
where $A,B$ are two sets in $\mathbb{R}^D$.

\begin{lemma}[Theorem 4 in \citealt{Non_ridge_est2014}]
	\label{stability_ridge}
	Assume conditions (A1-3) for two densities $p_1,p_2$. When $\norm{p_1-p_2}_{\infty,3}^*$ is sufficiently small, we have 
	$$\Haus(R_{d,1},R_{d,2}) = O\left(\norm{p_1-p_2}_{\infty,2}^*\right),$$
	where $R_{d,1}$ and $R_{d,2}$ are the $d$-ridges of $p_1$ and $p_2$, respectively.
\end{lemma}

When the true density $p$ that generates the Euclidean data $\left\{\bm{X}_1,...,\bm{X}_n \right\} \subset \mathbb{R}^D$ is replaced by the Euclidean KDE $\hat{p}_n$ in the definition \eqref{ridge} of density ridges, we obtain a natural (plug-in) estimator of the true ridge $R_d$ as:
$$\hat{R}_d=\left\{\bm{x}\in \mathbb{R}^D: \hat{V}_d(\bm{x})^T \nabla \hat{p}_n(\bm{x})=\bm{0}, \hat{\lambda}_{d+1}(\bm{x})<0 \right\}.$$
To regularize statistical behaviors of the estimated ridge $\hat{R}_d$, we make the following assumptions on the kernel of its form \eqref{Radial_kernel} as:
\begin{itemize}
	\item {\bf (E1)} We assume that the kernel profile $k:[0,\infty) \to [0,\infty)$ is non-increasing and at least three times continuously differentiable with bounded fourth order partial derivatives as well as
	$$\int_{\mathbb{R}^D} \norm{\bm{x}}_2^2\cdot k\left(\norm{\bm{x}}_2^2 \right)d\bm{x}, \quad \int_{\mathbb{R}^d} -k'\left(\norm{\bm{x}}_2^2 \right) d\bm{x} <\infty,\, \text{ and } \, \int_{\mathbb{R}^d} \left|k^{(\alpha)}\left(\norm{\bm{x}}_2^2 \right) \right|^2 d\bm{x} <\infty$$
		with $\alpha=0,1,2,3$.
	\item {\bf (E2)} Let 
	$$\mathcal{K}_E = \left\{\bm{y} \mapsto K^{(\alpha)}\left(\frac{\bm{x} - \bm{y}}{h} \right) = D^{[\alpha]}\left[c_{k,D} \cdot k\left(\norm{\frac{\bm{x} - \bm{y}}{h}}_2^2 \right) \right]: \bm{x}\in \mathbb{R}^D, |[\alpha]|=0,1,2,3 \right\}.$$
	We assume that $\mathcal{K}_E$ is a bounded VC (subgraph) class of measurable functions on $\mathbb{R}^D$; that is, there exist constants $A, \upsilon >0$ such that for any $0< \epsilon <1$,
	$$\sup_Q N\left(\mathcal{K}_E, L_2(Q), \epsilon \norm{F}_{L_2(Q)} \right) \leq \left(\frac{A}{\epsilon} \right)^{\upsilon},$$
	where $N(\mathcal{K}_E,L_2(Q), \epsilon)$ is the $\epsilon$-covering number of the normed space $\left(\mathcal{K}_E, \norm{\cdot}_{L_2(Q)} \right)$, $Q$ is any probability measure on $\mathbb{R}^D$, and $F$ is an envelope function of $\mathcal{K}_E$. Here, the norm $\norm{F}_{L_2(Q)}$ is defined as $\left[\int_{\mathbb{R}^D} |F(\bm{x})|^2 dQ(\bm{x}) \right]^{\frac{1}{2}}$.
\end{itemize}

\begin{remark}
		Recall that the \emph{$\epsilon$-covering number} $N(\mathcal{F},L_2(Q), \epsilon)$ is defined as the minimal number of $L_2(Q)$-balls $\left\{g: \norm{g-f}_{L_2(Q)} < \epsilon\right\}$ of radius $\epsilon$ needed to cover the (function) class $\mathcal{F}$. 
		One popular concept for controlling uniform covering number $\sup_Q N\left(\mathcal{F}, L_2(Q), \epsilon \norm{F}_{L_2(Q)} \right)$ is the notion of Vapnik-\v{C}ervonenkis (subgraph) classes, or simply VC classes.  Starting from collections of sets, we say that a collection $\mathcal{C}$ of subsets of the sample space $\mathcal{X}$ \emph{picks out} a certain subset of the finite set $\{x_1,...,x_n\} \subset \mathcal{X}$ if it can be written as $C \cap \{x_1,...,x_n\}$ for some $C\in \mathcal{C}$. The collection is said to \emph{shatter} $\{x_1,...,x_n\}$ if $\mathcal{C}$ picks out each of its $2^n$ subsets. The \emph{VC-index} $V(\mathcal{C})$ of $\mathcal{C}$ is the smallest $n$ for which no set of size $n$ is shattered by $\mathcal{C}$. A collection $\mathcal{C}$ of measurable sets is called a \emph{VC class} if its index $V(\mathcal{C})$ is finite. To generalize this concept to a class $\mathcal{F}$ of real-valued and measurable functions defined on $\mathcal{X}$, we say that $\mathcal{F}$ is a \emph{VC subgraph class} if the collection of all subgraphs of the functions in $\mathcal{F}$ forms a VC class of sets in $\mathcal{X}\times \mathbb{R}$. An important property of VC (subgraph) classes is that their $\epsilon$-covering numbers grow polynomially in $\frac{1}{\epsilon}$ as what condition (E2) is stated; see Theorem 2.6.4 in \cite{van1996weak}. More in-depth discussion on VC classes can be found in Chapter 2.6 of the same book.
\end{remark}

Condition (E1) can be relaxed such that the kernel profile $k$ is three times continuously differentiable except for finite number of points on $[0,\infty)$. Such relaxation allows us to include the Epanechnikov and other compactly supported kernel. The integrability assumption on $k$ in condition (E1) is similar to the conditions (K1) and (K3) in Section~\ref{Sec:KDE} for the purpose of bounding the expectations and variances of the KDE $\nabla \hat{p}_n(\bm{x})$ and its (partial) derivatives. Condition (E2) regularizes the complexity of the kernel and its (partial) derivatives, which is essential in establishing the uniform consistency of $\hat{p}_n$ and its derivatives to the corresponding quantities of $p$ as in \eqref{unif_bound} below. 

Given conditions (E1) and (E2), the techniques in \cite{Gine2002,Einmahl2005,Asymp_deri_KDE2011} can be utilized to show the uniform consistency of the Euclidean KDE $\hat{p}_n$ and its derivatives as:
\begin{align}
\label{unif_bound}
\begin{split}
\norm{\hat{p}_n -p}_{\infty}^{(k)} = O(h^2) + O_P\left(\sqrt{\frac{|\log h|}{nh^{D+2k}}} \right)\quad \text{ for } k=0,...,3.
\end{split}
\end{align}

\subsection{Mean Shift and SCMS Algorithms with Euclidean Data}
\label{Sec:Eu_MS_SCMS}

We begin with a quick review on the Euclidean mean shift algorithm, as the SCMS algorithm is built on top of such formulation. Given condition (E1) and the Euclidean KDE $\hat{p}_n(\bm{x}) = \frac{c_{k,D}}{nh^D} \sum\limits_{i=1}^n k\left(\norm{\frac{\bm{x}-\bm{X}_i}{h}}_2^2 \right)$ with kernel \eqref{Radial_kernel}, its gradient estimator takes the form as:
\begin{align}
\label{KDE_grad_Eu}
\begin{split}
\nabla \hat{p}_n(\bm{x}) &= \frac{2c_{k,D}}{nh^{D+2}} \sum_{i=1}^n (\bm{x}-\bm{X}_i) \cdot k'\left(\norm{\frac{\bm{x}-\bm{X}_i}{h}}_2^2 \right)\\
&= \frac{2c_{k,D}}{nh^{D+2}} \left[\sum_{i=1}^n -k'\left(\norm{\frac{\bm{x}-\bm{X}_i}{h}}_2^2 \right) \right] \left[\frac{\sum_{i=1}^n \bm{X}_i k'\left(\norm{\frac{\bm{x}-\bm{X}_i}{h}}_2^2 \right)}{\sum_{i=1}^n k'\left(\norm{\frac{\bm{x}-\bm{X}_i}{h}}_2^2 \right)} -\bm{x}\right],
\end{split}
\end{align}
where the first term is a variant of KDEs and the second term is the \emph{mean shift} vector
\begin{equation}
\label{MS}
\Xi_h(\bm{x}) = \frac{\sum_{i=1}^n \bm{X}_i k'\left(\norm{\frac{\bm{x}-\bm{X}_i}{h}}_2^2 \right)}{\sum_{i=1}^n k'\left(\norm{\frac{\bm{x}-\bm{X}_i}{h}}_2^2 \right)} -\bm{x}.
\end{equation}
This factorization suggests that the mean shift vector aligns with the direction of maximum increase in $\hat{p}_n$. Thus, moving a point along its mean shift vector successively yields an ascending path to a local mode \citep{MS1995,MS2002,MS2007_pf}. Let $\left\{\hat{\bm{x}}^{(t)} \right\}_{t=0}^{\infty}$ be the mean shift sequence with the Euclidean KDE $\hat{p}_n$. Then, one step iteration of the mean shift algorithm is written as:
\begin{align}
\label{MS_iteration}
\begin{split}
\hat{\bm{x}}^{(t+1)} \gets \hat{\bm{x}}^{(t)} + \Xi_h\left(\hat{\bm{x}}^{(t)} \right) &= \frac{\sum_{i=1}^n \bm{X}_i k'\left(\norm{\frac{\hat{\bm{x}}^{(t)}-\bm{X}_i}{h}}_2^2 \right)}{\sum_{i=1}^n k'\left(\norm{\frac{\hat{\bm{x}}^{(t)} -\bm{X}_i}{h}}_2^2 \right)}\\
&= \hat{\bm{x}}^{(t)} + \frac{1}{\frac{2 c_{k,D}}{nh^{D+2}} \sum\limits_{i=1}^n -k'\left(\norm{\frac{\hat{\bm{x}}^{(t)}-\bm{X}_i}{h}}_2^2 \right)} \cdot \nabla \hat{p}_n(\hat{\bm{x}}^{(t)}),
\end{split}
\end{align}
showing that the mean shift algorithm is a gradient ascent method with an adaptive step size 
\begin{equation}
\label{MS_stepsize}
\eta_{n,h}^{(t)} = \frac{1}{\frac{2 c_{k,D}}{nh^{D+2}} \sum\limits_{i=1}^n -k'\left(\norm{\frac{\hat{\bm{x}}^{(t)}-\bm{X}_i}{h}}_2^2 \right)} = \frac{1}{\hat{g}_n\left(\hat{\bm{x}}^{(t)} \right)}.
\end{equation}
Here, we denote by $\hat{g}_n\left(\hat{\bm{x}}^{(t)} \right) = \frac{2 c_{k,D}}{nh^{D+2}} \sum\limits_{i=1}^n -k'\left(\norm{\frac{\hat{\bm{x}}^{(t)}-\bm{X}_i}{h}}_2^2 \right)$ the denominator of the adaptive step size $\eta_{n,h}^{(t)}$.
Lemma~\ref{limit_step_MS} below shows that under condition (E1) and the differentiability assumption on $p$, $h^2\hat{g}_n\left(\bm{x} \right)$ tends to a fixed constant with probability tending to 1 for any $\bm{x}\in \mathbb{R}^D$ as $nh^D \to \infty$ and $h\to 0$. Therefore, the step size $\eta_{n,h}^{(t)}$ has its asymptotic rate as $O(h^2)$ and tends to zero as $nh^D \to \infty$ and $h\to 0$ as well. The proof of Lemma~\ref{limit_step_MS} can be found in Appendix~\ref{App:Proofs_Sec3}.

\begin{lemma}
	\label{limit_step_MS}
	Assume conditions (A1) and (E1). The convergence rate of $\hat{g}_n(\bm{x})=\frac{2c_{k,D}}{nh^{D+2}} \sum\limits_{i=1}^n -k'\left(\norm{\frac{\bm{x}-\bm{X}_i}{h}}_2^2 \right)$ is
	\begin{align*}
	h^2 \hat{g}_n(\bm{x}) &= -2c_{k,D} \cdot p(\bm{x}) \int_{\mathbb{R}^D} k'\left(\norm{\bm{u}}_2^2 \right)d\bm{u} +O\left(h^2 \right) + O_P\left(\sqrt{\frac{1}{nh^D}} \right) \\
	&= O(1) + O\left(h^2 \right) + O_P\left(\sqrt{\frac{1}{nh^D}} \right)
	\end{align*}
	for any $\bm{x}\in \mathbb{R}^D$ as $nh^D \to \infty$ and $h\to 0$.
\end{lemma}

As the mean shift algorithm is not a main focus of this paper, we will make an abuse of notation and denote by $\big\{\hat{\bm{x}}^{(t)} \big\}_{t=0}^{\infty}$ the sequence produced by the SCMS or SCGA algorithm in the sequel. Compared to the mean shift iteration \eqref{MS_iteration}, the SCMS algorithm updates the sequence $\big\{\hat{\bm{x}}^{(t)} \big\}_{t=0}^{\infty}$ through the subspace constrained mean shift vector $\hat{V}_d(\hat{\bm{x}}^{(t)}) \hat{V}_d(\hat{\bm{x}}^{(t)})^T \Xi_h(\hat{\bm{x}}^{(t)})$ as:
\begin{align}
\label{SCMS_update}
\begin{split}
\hat{\bm{x}}^{(t+1)} &\gets \hat{\bm{x}}^{(t)} + \hat{V}_d(\hat{\bm{x}}^{(t)}) \hat{V}_d(\hat{\bm{x}}^{(t)})^T \Xi_h(\hat{\bm{x}}^{(t)})\\
& = \hat{\bm{x}}^{(t)} + \frac{1}{\frac{2 c_{k,D}}{nh^{D+2}} \sum\limits_{i=1}^n -k'\left(\norm{\frac{\hat{\bm{x}}^{(t)}-\bm{X}_i}{h}}^2 \right)} \cdot \hat{V}_d(\hat{\bm{x}}^{(t)}) \hat{V}_d(\hat{\bm{x}}^{(t)})^T \nabla \hat{p}_n(\hat{\bm{x}}^{(t)}).
\end{split}
\end{align}
See Algorithm~\ref{Algo:SCMS} in Appendix~\ref{App:Algo} for the entire procedure. This also implies that the SCMS algorithm can be viewed as a sample-based SCGA method as:
\begin{equation}
\label{SCGA_update_sample}
\hat{\bm{x}}^{(t+1)} \gets \hat{\bm{x}}^{(t)} + \eta_{n,h}^{(t)} \cdot \hat{V}_d(\hat{\bm{x}}^{(t)}) \hat{V}_d(\hat{\bm{x}}^{(t)})^T \nabla \hat{p}_n(\hat{\bm{x}}^{(t)})
\end{equation}
with the same adaptive step size $\eta_{n,h}^{(t)}$ as the Euclidean mean shift algorithm in \eqref{MS_iteration}. The formulation \eqref{SCGA_update_sample} sheds light on some (linear) convergence properties of the SCMS algorithm as we will demonstrate in the next subsection.

\subsection{Linear Convergence of Population and Sample-Based SCGA Algorithms}
\label{Sec:LC_SCGA}

We have shown in \eqref{SCGA_update_sample} that the (usual/Euclidean) SCMS algorithm is a variant of the sample-based SCGA algorithm in $\mathbb{R}^D$ with an adaptive step size $\eta_{n,h}^{(t)}$. To establish the (linear) convergence results of the SCMS algorithm with Euclidean KDE $\hat{p}_n$, it suffices to study the (linear) convergence of the sample-based SCGA algorithm with objective function $\hat{p}_n$. To this end, we begin by studying the convergence of the population SCGA algorithm whose objective function is the underlying density $p$.

Let $\big\{\bm{x}^{(t)}\big\}_{t=0}^{\infty}$ be the sequence defined by the population SCGA algorithm and $\big\{\hat{\bm{x}}^{(t)}\big\}_{t=0}^{\infty}$ be the sequence defined by the sample-based SCGA algorithm. The population SCGA algorithm is defined by its iterative formula as:
	\begin{equation}
	\label{SCGA_update_pop}
	\bm{x}^{(t+1)} = \bm{x}^{(t)} + \eta \cdot V_d(\bm{x}^{(t)}) V_d(\bm{x}^{(t)})^T \nabla p(\bm{x}^{(t)}),
	\end{equation}
	where $\eta >0$ is a (fixed) step size. The sample-based SCGA algorithm has its iterative formula as \eqref{SCGA_update_sample}, except that the standard sample-based SCGA algorithm normally embraces a constant step size $\eta>0$. 
	
	\begin{remark}
		\label{SCMS_stepsize_remark}
		In \eqref{SCMS_update} and \eqref{SCGA_update_sample}, we consider the SCMS algorithm as a sample-based SCGA iteration with an adaptive step size $\eta_{n,h}^{(t)}$. Our Lemma~\ref{limit_step_MS} suggests that $\eta_{n,h}^{(t)}$ tends to zero in a rate $O(h^2)$ as $nh^D \to \infty$ and $h\to 0$. However, once the sample size $n$ is fixed and the bandwidth $h$ is chosen, the step size $\eta_{n,h}^{(t)}$ is not only upper bounded but also uniformly lower bounded away from zero with respect to the iteration number $t$ by the differentiability condition (E1) when the current iterative point $\hat{\bm{x}}^{(t)}$ lies within the compact neighborhood $R_d\oplus \rho$. Note that $R_d\oplus \rho$ is compact because $R_d$ is a finite union of connected and compact manifolds; see (d) of Lemma~\ref{normal_reach_prop}. More importantly, these upper and lower bounds of $\eta_{n,h}^{(t)}$ when $\hat{\bm{x}}^{(t)} \in R_d\oplus \rho$ are independent of the iteration number $t$. Therefore, conditioning on the case when the sample size $n$ is sufficiently large, one can always select a small bandwidth $h$ such that the adaptive step size $\eta_{n,h}^{(t)}$ of the SCMS algorithm is sufficiently small but not equal to zero.
	\end{remark}
	
	As revealed by the following proposition, our imposed conditions (A1-3) in Section~\ref{Sec:Assump} ensure that as long as the step size $\eta>0$ is small, the objective function $p$ along any population SCGA sequence $\big\{\bm{x}^{(t)}\big\}_{t=0}^{\infty}$ is non-decreasing and the sequence by itself converges to $R_d$ when it is initialized within a small neighborhood of $R_d$. 
	
	\begin{proposition}[Convergence of the SCGA Algorithm]
		\label{SCGA_conv}
		For any SCGA sequence $\big\{\bm{x}^{(t)}\big\}_{t=0}^{\infty}$ defined by \eqref{SCGA_update_pop} with $0<\eta < \frac{2}{D\norm{p}_{\infty}^{(2)}}$, the following properties hold.
		\begin{enumerate}[label=(\alph*)]
			\item Under condition (A1), the objective function sequence $\big\{p(\bm{x}^{(t)}) \big\}_{t=0}^{\infty}$ is non-decreasing and converges.
			
			\item Under condition (A1), $\lim_{t\to\infty} \norm{V_d(\bm{x}^{(t)})^T \nabla p(\bm{x}^{(t)})}_2 =\lim_{t\to\infty} \norm{\bm{x}^{(t+1)} -\bm{x}^{(t)}}_2 = 0$.
			
			\item Under conditions (A1-3), $\lim_{t\to\infty} d_E(\bm{x}^{(t)}, R_d) =0$ whenever $\bm{x}^{(0)} \in R_d\oplus r_1$ with the convergence radius $r_1$ satisfying
			$$0<r_1 < \min\left\{\frac{\rho}{2},\; \frac{\beta_1^2}{A_2 \left(\norm{p}_{\infty}^{(3)} + \norm{p}_{\infty}^{(4)} \right)},\; \frac{\beta_1}{A_4(p)} \right\},$$
			where $A_2>0$ is a constant defined in (h) of Lemma~\ref{normal_reach_prop} while $A_4(p) >0$ is a quantity depending on both the dimension $D$ and functional norm $\norm{p}_{\infty,4}^*$ up to the fourth-order (partial) derivatives of $p$. 
		\end{enumerate}
	\end{proposition}
	
	The proof of Proposition~\ref{SCGA_conv} can be found in Appendix~\ref{App:Proofs_Sec3}. We make two comments on the choice of the convergence radius $r_1$ in (c) of Proposition~\ref{SCGA_conv}. The first two quantities in the upper bound of $r_1$ ensure that $r_1\leq \mathtt{reach}(R_d)$ and therefore, the projection of $\bm{x}^{(t)} \in R_d\oplus r_1$ onto $R_d$ is well-defined. The last quantity in the upper bound of $r_1$ is critical to guarantee that the distances $\left\{d_E(\bm{x}^{(t)}, R_d) \right\}_{t=0}^{\infty}$ from the SCGA sequence $\big\{\bm{x}^{(t)}\big\}_{t=0}^{\infty}$ to the ridge $R_d$ can be controlled by the norms $\norm{V_d(\bm{x}^{(t)})^T \nabla p(\bm{x}^{(t)})}_2$ of order-$d$ principal gradients for $t=0,1,...$.
	
	\begin{corollary}[Convergence of the SCMS Algorithm]
		\label{SCMS_conv}
		When the fixed sample size $n$ is sufficiently large and the fixed bandwidth $h$ is chosen to be sufficiently small, the following properties hold for the SCMS sequence $\big\{\hat{\bm{x}}^{(t)} \big\}_{t=0}^{\infty}$ with high probability under conditions (A1-3) and (E1-2).
		\begin{enumerate}[label=(\alph*)]
			\item The Euclidean KDE sequence $\left\{\hat{p}_n(\hat{\bm{x}}^{(t)}) \right\}$ is non-decreasing and thus converges.
			
			\item $\lim_{t\to\infty} \norm{\hat{V}_d(\hat{\bm{x}}^{(t)})^T \nabla \hat{p}_n(\hat{\bm{x}}^{(t)})}_2 = \lim_{t\to\infty} \norm{\hat{\bm{x}}^{(t+1)} -\hat{\bm{x}}^{(t)}}_2=0$.
			
			\item $\lim_{t\to\infty} d_E(\hat{\bm{x}}^{(t)},\hat{R}_d)=0$ whenever $\hat{\bm{x}}^{(0)} \in \hat{R}_d \oplus r_1$ with the convergence radius $r_1>0$ defined in (c) of Proposition~\ref{SCGA_conv}.
		\end{enumerate}
	\end{corollary}
	
	Corollary~\ref{SCMS_conv} is the sample-based version of Proposition~\ref{SCGA_conv}. On the one hand, when $\frac{nh^{D+6}}{|\log h|}$ is sufficiently large and $h$ is small enough, the estimated ridge $\hat{R}_d$ also satisfies conditions (A1-3) with high probability; see Lemma~\ref{stability_ridge} and the uniform bounds \eqref{unif_bound} of $\hat{p}_n$. On the other hand, the adaptive step size $\eta_{n,h}^{(t)}$ of the SCMS algorithm can be always smaller than the threshold $\frac{2}{D\norm{p}_{\infty}^{(2)}}$ when the sample size $n$ is sufficiently large and $h$ is small; see Remark~\ref{SCMS_stepsize_remark}. Consequently, our arguments in Proposition~\ref{SCGA_conv} can be applied to establish the (local) convergence of the SCMS sequence here. In addition, we point out that Proposition 2 in \cite{SCMS_conv2013} also proved the results (a-b) of Corollary~\ref{SCMS_conv} under condition (E1) and the convexity assumption on the kernel profile $k$. The difference is that our arguments hold when $n$ is large and $h$ is small while the extra convexity assumption in \cite{SCMS_conv2013} enables the authors to prove the results (a-b) universally for any choice of the bandwidth $h$.
	
	By Proposition~\ref{SCGA_conv} and Corollary~\ref{SCMS_conv}, it is now reasonable to denote the limiting points of the population and sample-based SCGA sequences $\big\{\bm{x}^{(t)}\big\}_{t=0}^{\infty}$ and $\big\{\hat{\bm{x}}^{(t)}\big\}_{t=0}^{\infty}$ by $\bm{x}^*\in R_d$ and $\hat{\bm{x}}^* \in \hat{R}_d$, respectively.
Before stating our main linear convergence results, we introduce the concepts of Q-linear and R-linear convergence from optimization literature; see, \emph{e.g.}, Appendix A2 in \cite{Nocedal2006Numerical}.

\begin{definition}[Linear Rate of Convergence]
	\label{Q_Linear_conv}
	We say that the convergence of the sequence $\big\{\bm{x}^{(t)}\big\}_{t=0}^{\infty}$ to $\bm{x}^*$ is \emph{Q-linear} if there exists a constant $\Upsilon \in (0,1)$ such that 
	$$\frac{\norm{\bm{x}^{(t+1)}-\bm{x}^*}_2}{\norm{\bm{x}^{(t)}-\bm{x}^*}_2} \leq \Upsilon \quad \text{ for all }t \text{ sufficiently large}.$$
	We say that the convergence is \emph{R-linear} if there is a sequence of nonnegative scalars $\left\{\epsilon_t\right\}_{t=0}^{\infty}$ such that
	$$\norm{\bm{x}^{(t)}-\bm{x}^*}_2 \leq \epsilon_t \text{ for all }t, \text{ and } \left\{\epsilon_t\right\}_{t=0}^{\infty} \text{ converges Q-linearly to zero}.$$
\end{definition}

	The linear convergence of the SCGA sequence $\big\{\bm{x}^{(t)} \big\}_{t=0}^{\infty}$ will be established under the following local condition.
	\begin{itemize}
		\item {\bf (A4)} (\emph{Quadratic Behaviors of Residual Vectors}) We assume that the SCGA sequence $\big\{\bm{x}^{(t)}\big\}_{t=0}^{\infty}$ with step size $0<\eta \leq \min\left\{\frac{4}{\beta_0},\frac{1}{D\norm{p}_{\infty}^{(2)}} \right\}$ and $\bm{x}^*\in R_d$ as its limiting point satisfies
		\begin{align*}
		\nabla p(\bm{x}^{(t)})^T U_d^{\perp}(\bm{x}^{(t)}) (\bm{x}^*-\bm{x}^{(t)}) &\leq \frac{\beta_0}{4} \norm{\bm{x}^*-\bm{x}^{(t)}}_2^2,\\
		\norm{U_d^{\perp}(\bm{x}^{(t)}) (\bm{x}^*-\bm{x}^{(t)})}_2 &\leq \beta_2\norm{\bm{x}^*-\bm{x}^{(t)}}_2^2
		\end{align*}
		for some constant $\beta_2>0$, where $\beta_0>0$ is the constant defined in condition (A2).
	\end{itemize}
	
	\begin{figure}
		\centering
		\includegraphics[width=0.7\linewidth]{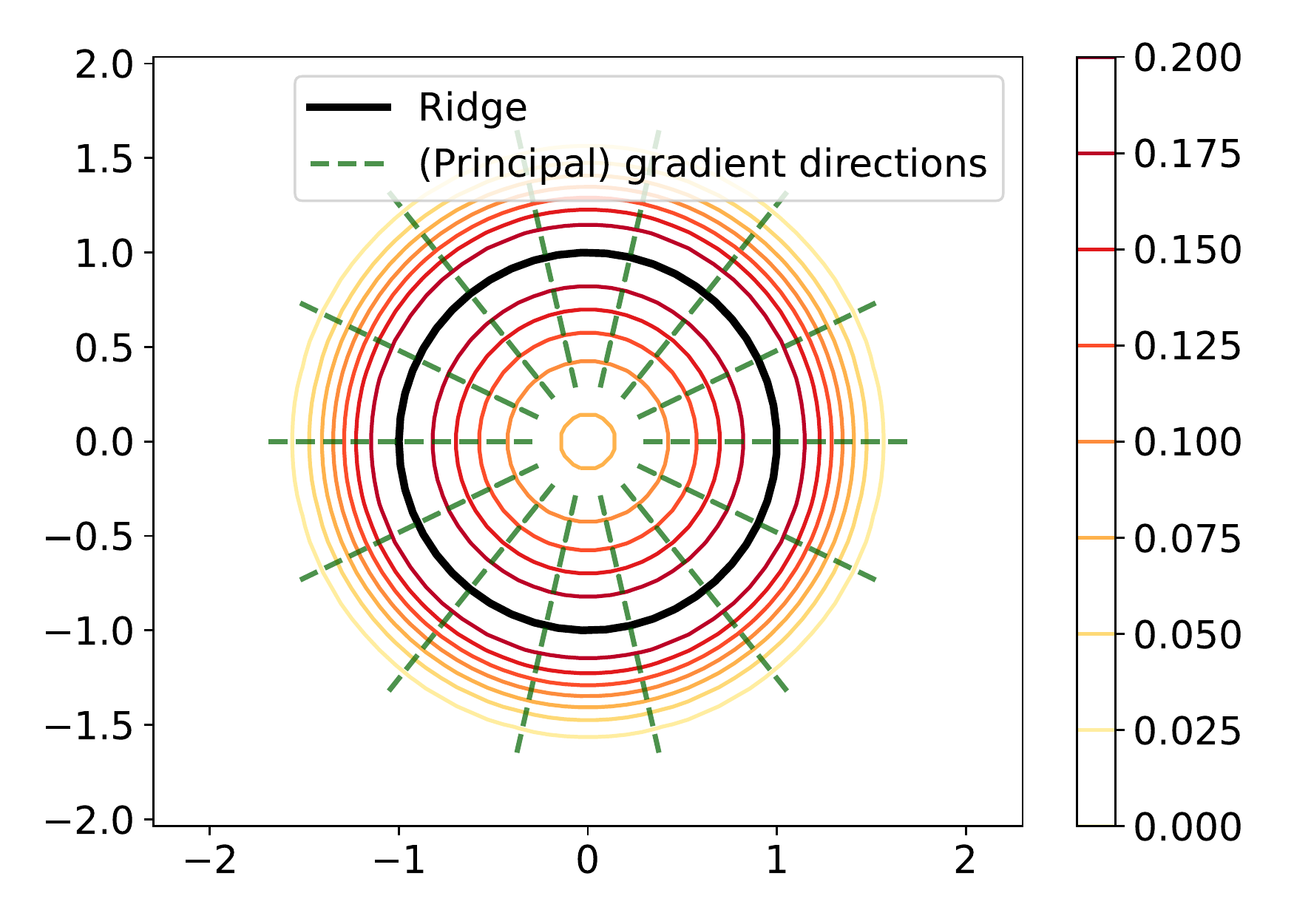}
		\caption{Contour lines of the density function \eqref{den_example} and its principal gradient flows.}
		\label{fig:den_example}
	\end{figure}
	
	Condition (A4) imposes a direct assumption on the SCGA sequence $\big\{\bm{x}^{(t)}\big\}_{t=0}^{\infty}$, under which the residual vector $U_d^{\perp}(\bm{x}^{(t)}) (\bm{x}^*-\bm{x}^{(t)})$ and its inner product with the residual gradient $U^{\perp}(\bm{x}^{(t)}) \nabla p(\bm{x}^{(t)})$ are upper bounded by a quadratic term $O\left(\norm{\bm{x}^*-\bm{x}^{(t)}}_2^2\right)$. This condition is imposed to guarantee that $p$ is ``subspace constrained strongly concave'' around $R_d$; see also Remark~\ref{SC_remark}. Our proof of Theorem~\ref{SCGA_LC} suggests that the residual vector $U_d^{\perp}(\bm{x}^{(t)}) (\bm{x}^*-\bm{x}^{(t)})$ is only required to be smaller than the first-order term $o\left(\norm{\bm{x}^{(t)}-\bm{x}^*}_2 \right)$. For the simplicity, we require it to be quadratic. When condition (A4) fails to hold, the associated SCGA sequence can only converge sublinearly to $R_d$. Therefore, it is an essential element in the linear convergence of the SCGA algorithm, and we discuss some potentially weaker assumptions that implicate condition (A4) in Appendix~\ref{App:dis_A4}. Intuitively, the SCGA path converges to $R_d$ following the direction of principal gradient $V_d(\bm{x}^{(t)}) V_d(\bm{x}^{(t)})^T \nabla p(\bm{x}^{(t)})$. To further gain more insights into the correctness of condition (A4), we consider a special density function 
	\begin{equation}
	\label{den_example}
	p(u,v) = C_{p,2}\cdot \exp\left[-\left(u^2+v^2-1 \right)^2 \right] \quad \text{ with } \quad C_{p,2}=\pi\left(\frac{\sqrt{\pi}}{2} + \int_0^1 e^{-t^2} dt \right)
	\end{equation} 
	on $\mathbb{R}^2$, whose one-dimensional ridge is $R_1=\left\{(u,v)\in \mathbb{R}^2: u^2+v^2=1 \right\}$ by the definition \eqref{ridge}. Some careful calculations suggest that its principal gradient $G_1(u,v)$ points towards the ridge $R_1$ in the direction $(u,v)$ when $\frac{2-\sqrt{2}}{2}<u^2+v^2<1$ and in the direction $(-u,-v)$ when $1<u^2+v^2<\frac{2+\sqrt{2}}{2}$; see Figure~\ref{fig:den_example} for a graphical illustration. Furthermore, the smallest eigenvalue of $\nabla\nabla p(u,v)$ is negative whenever $u^2+v^2>\frac{1}{2}$. Hence, the residual gradient $U_1^{\perp}(u,v) \nabla p(u,v)$ is perpendicular to the SCGA direction, and condition (A4) naturally holds. 

We now present our linear convergence results for the population and sample-based SCGA algorithms.

\begin{theorem}[Linear Convergence of the SCGA Algorithm]
	\label{SCGA_LC}
	Assume conditions (A1-4) throughout the theorem.
	\begin{enumerate}[label=(\alph*)]
		\item \textbf{Q-Linear convergence of $\norm{\bm{x}^{(t)}-\bm{x}^*}_2$}: Consider a convergence radius $r_2>0$ satisfying
		$$0<r_2< \min\left\{\frac{\rho}{2}, \frac{\beta_1^2}{A_2\left(\norm{p}_{\infty}^{(3)} + \norm{p}_{\infty}^{(4)} \right)}, \frac{\beta_1}{A_4(p)}, \frac{3\beta_0}{4D\left(6\norm{p}_{\infty}^{(2)}\beta_2^2 \rho + D^{\frac{1}{2}}\norm{p}_{\infty}^{(3)} \right)} \right\},$$
		where $A_2 >0$ is the constant defined in (h) of Lemma~\ref{normal_reach_prop} and $A_4(p) >0$ is a quantity defined in (c) of Proposition~\ref{SCGA_conv} that depends on both the dimension $D$ and the functional norm $\norm{p}_{\infty,4}^*$ up to the fourth-order derivative of $p$. Whenever $0<\eta \leq \min\left\{\frac{4}{\beta_0}, \frac{1}{D\norm{p}_{\infty}^{(2)}} \right\}$ and the initial point $\bm{x}^{(0)} \in \text{Ball}_D(\bm{x}^*, r_2)$ with $\bm{x}^* \in R_d$, we have that
		$$\norm{\bm{x}^{(t)}-\bm{x}^*}_2 \leq \Upsilon^t \norm{\bm{x}^{(0)}-\bm{x}^*}_2 \quad \text{ with } \quad \Upsilon = \sqrt{1-\frac{\beta_0\eta}{4}}.$$
		\item \textbf{R-Linear convergence of $d_E(\bm{x}^{(t)}, R_d)$}: Under the same radius $r_2>0$ in (a), we have that whenever $0<\eta \leq \min\left\{\frac{4}{\beta_0},\frac{1}{D\norm{p}_{\infty}^{(2)}}\right\}$ and the initial point $\bm{x}^{(0)} \in \text{Ball}_D(\bm{x}^*, r_2)$ with $\bm{x}^* \in R_d$,
		$$d_E(\bm{x}^{(t)}, R_d) \leq \Upsilon^t \norm{\bm{x}^{(0)}-\bm{x}^*}_2 \quad \text{ with } \quad \Upsilon=\sqrt{1-\frac{\beta_0 \eta}{4}}.$$
	\end{enumerate}
	We further assume conditions (E1-2) in the rest of statements. If $h \to 0$ and $\frac{nh^{D+4}}{|\log h|} \to \infty$,
	\begin{enumerate}[label=(c)]
		\item \textbf{Q-Linear convergence of $\norm{\hat{\bm{x}}^{(t)}-\bm{x}^*}_2$}: under the same radius $r_2>0$ and $\Upsilon=\sqrt{1-\frac{\beta_0 \eta}{4}}$ in (a), we have that
		$$\norm{\hat{\bm{x}}^{(t)} -\bm{x}^*}_2 \leq \Upsilon^t \norm{\hat{\bm{x}}^{(0)} -\bm{x}^*}_2 + O(h^2) + O_P\left(\sqrt{\frac{|\log h|}{nh^{D+4}}}\right)$$
		with probability tending to 1 whenever $0<\eta \leq \min\left\{\frac{4}{\beta_0}, \frac{1}{D\norm{p}_{\infty}^{(2)}} \right\}$ and the initial point $\hat{\bm{x}}^{(0)} \in \text{Ball}_D(\bm{x}^*, r_2)$ with $\bm{x}^* \in R_d$.
	\end{enumerate}
	\begin{enumerate}[label=(d)]
		\item \textbf{R-Linear convergence of $d_E(\hat{\bm{x}}^{(t)}, R_d)$}: under the same radius $r_2>0$ and $\Upsilon=\sqrt{1-\frac{\beta_0 \eta}{4}}$ in (a), we have that
		$$d_E(\hat{\bm{x}}^{(t)},R_d) \leq \Upsilon^t \norm{\hat{\bm{x}}^{(0)} -\bm{x}^*}_2 + O(h^2) + O_P\left(\sqrt{\frac{|\log h|}{nh^{D+4}}} \right)$$
		with probability tending to 1 whenever $0< \eta \leq \min\left\{\frac{4}{\beta_0}, \frac{1}{D\norm{p}_{\infty}^{(2)}} \right\}$ and the initial point $\hat{\bm{x}}^{(0)} \in \text{Ball}_D(\bm{x}^*, r_2)$ with $\bm{x}^* \in R_d$.
	\end{enumerate}
\end{theorem}

The detailed proof of Theorem~\ref{SCGA_LC} can be found in Appendix~\ref{App:Proofs_Sec3}. Note that, as in (c) of Proposition~\ref{SCGA_conv}, we elucidate a threshold value for the convergence radius $r_2>0$ in (a), under which the population SCGA algorithm converges linearly to $R_d$. The first three quantities in the threshold value are directly adopted from the upper bound of the convergence radius $r_1$ in (c) of Proposition~\ref{SCGA_conv}, while the last term controls the ``subspace constrained strongly concavity'' \eqref{key_inequality} of $p$ within $R_d\oplus r_2$.

\begin{remark}
	\label{SC_remark}
	Notice that the standard strong concavity assumption on the objective function (or density function) $p$ is not sufficient to establish the linear convergence of the population SCGA algorithm \eqref{SCGA_update_pop}. This is because, under the (quasi-)strong concavity assumption \citep{necoara2019linear}, the objective function $p$ would satisfy
	\begin{equation}
	\label{SC}
	p(\bm{x}^*) -p(\bm{y}) \leq \nabla p(\bm{y})^T (\bm{x}^*-\bm{y}) - \frac{A_6}{2}\norm{\bm{x}^*-\bm{y}}_2^2
	\end{equation}
	for some constant $A_6>0$, and those standard proofs of the linear convergence of gradient ascent methods rely on this inequality; see Section 3.4 in \cite{SB2015}. However, as indicated in our proof of Theorem~\ref{SCGA_LC}, the linear convergence of the SCGA algorithm requires the following inequality instead:
	\begin{equation}
	\label{key_inequality}
	p(\bm{x}^*) -p(\bm{y}) \leq \nabla p(\bm{y})^T V_d(\bm{y}) V_d(\bm{y})^T (\bm{x}^*-\bm{y}) - \frac{A_7}{2}\norm{\bm{x}^*-\bm{y}}_2^2 + o\left(\norm{\bm{x}^*-\bm{y}}_2^2\right)
	\end{equation}
	for some constant $A_7>0$, where $\bm{y}$ is generally chosen to be $\bm{x}^{(t)}$. We call the function $p$ satisfying \eqref{key_inequality} to be ``subspace constrained strongly concave''. Since 
	$$\nabla p(\bm{y})^T (\bm{x}^*-\bm{y}) = \nabla p(\bm{y})^T V_d(\bm{y}) V_d(\bm{y})^T (\bm{x}^*-\bm{y}) + \nabla p(\bm{y})^T U_d^{\perp}(\bm{y}) (\bm{x}^*-\bm{y}),$$ 
	the strong concavity assumption \eqref{SC} will not imply the key inequality \eqref{key_inequality} for the linear convergence of the population SCGA algorithm unless the residual gradient term $\nabla p(\bm{y})^T U_d^{\perp}(\bm{y}) (\bm{x}^*-\bm{y})$ can be upper bounded by the second-order error term $O\left(\norm{\bm{x}^*-\bm{y}}_2^2 \right)$. The imposed eigengap condition (A2) as well as condition (A4) with its related discussion in Appendix~\ref{App:dis_A4} fill in this gap, ensuring that such a quadratic upper bound holds on the residual gradients along the SCGA sequence.
\end{remark}

	\begin{corollary}[Linear Convergence of the SCMS Algorithm]
		\label{SCMS_LC}
		Assume conditions (A1-4) and (E1-2). When the fixed sample size $n$ is sufficiently large and the bandwidth $h$ is chosen to be sufficiently small, there exists a convergence radius $r_3 \in (0, r_2)$ such that the SCMS sequence $\big\{\hat{\bm{x}}^{(t)}\big\}_{t=0}^{\infty}$ satisfies the following property with high probability:
		\begin{align*}
		&d_E(\hat{\bm{x}}^{(t)}, \hat{R}_d) \leq \norm{\hat{\bm{x}}^{(t)} -\hat{\bm{x}}^*}_2 \leq \Upsilon_{n,h}^t \norm{\hat{\bm{x}}^{(0)} -\hat{\bm{x}}^*}_2 \\
		&\text{ with }\quad \Upsilon_{n,h}=\sqrt{1-\frac{\beta_0\tilde{\eta}_{n,h}}{4}} \quad \text{ and }\quad \tilde{\eta}_{n,h} = \inf_t \eta_{n,h}^{(t)}
		\end{align*}
		whenever $0<\sup_t\eta_{n,h}^{(t)}\leq \min\left\{\frac{4}{\beta_0}, \frac{1}{D\norm{p}_{\infty}^{(2)}} \right\}$ and the initial point $\hat{\bm{x}}^{(0)} \in \hat{R}_d\oplus r_3$.
	\end{corollary}
	
	Corollary~\ref{SCMS_LC} should also be regarded as the linear convergence of the sample-based SCGA algorithm to the estimated ridge $\hat{R}_d$ defined by the Euclidean KDE $\hat{p}_n$. Based on conditions (E1-2) and the uniform bounds \eqref{unif_bound}, $\hat{p}_n$ together with its ridge $\hat{R}_d$ and sample-based SCGA sequence $\big\{\hat{\bm{x}}^{(t)} \big\}_{t=0}^{\infty}$ satisfy conditions (A1-4) with probability tending to 1 as $h\to 0$ and $\frac{nh^{D+6}}{|\log h|} \to \infty$. As a result, one can follow our argument in (a) of Theorem~\ref{SCGA_LC} to establish the linear convergence of the sample-based SCGA algorithm with a fixed step size $\eta$ satisfying $0<\eta\leq \min\left\{\frac{4}{\beta_0}, \frac{1}{D\norm{p}_{\infty}^{(2)}} \right\}$. Furthermore, when the fixed sample size $n$ is sufficiently large and the bandwidth $h$ is chosen to be small, the adaptive step size $\eta_{n,h}^{(t)}$ of the SCMS algorithm always falls below the threshold $\min\left\{\frac{4}{\beta_0}, \frac{1}{D\norm{p}_{\infty}^{(2)}} \right\}$ for linear convergence but is also uniformly bounded away from zero with respect to the iteration number $t$; see our Remark~\ref{SCMS_stepsize_remark}. By taking the infimum of the adaptive step size $\eta_{n,h}^{(t)}$ with respect to $t$, one can thus establish the linear convergence of the SCMS algorithm with its rate of convergence as $\Upsilon_{n,h}=\sqrt{1-\frac{\beta_0\tilde{\eta}_{n,h}}{4}}$ and $\tilde{\eta}_{n,h} = \inf_t \eta_{n,h}^{(t)}$.

\section{The SCMS Algorithm With Directional Data and Its Linear Convergence}
\label{Sec:Dir_ridge_SCMS}

In this section, we generalize the definition \eqref{ridge} of density ridges to directional densities on $\Omega_q$ and propose our directional SCMS algorithm to identify directional density ridges. In addition, we prove the linear convergence of our directional SCMS algorithm by adjusting the arguments in Section~\ref{Sec:LC_SCGA}. Throughout this section, $\left\{\bm{X}_1,...,\bm{X}_n \right\}$ denotes a random sample from a directional distribution with density $f$ supported on the unit hypersphere $\Omega_q$ that is embedded in the ambient Euclidean space $\mathbb{R}^{q+1}$.

\subsection{Definitions, Assumptions, and Stability of Directional Density Ridges}
\label{Sec:Dir_ridge_def}

To apply the matrix forms of the Riemannian gradient $\grad f(\bm{x})$ and Hessian $\mathcal{H} f(\bm{x})$ of a directional density $f$ in the ambient space $\mathbb{R}^{q+1}$, we first extend $f$ from its support $\Omega_q$ to $\mathbb{R}^{q+1}\setminus \left\{\bm{0}\right\}$ by defining 
	\begin{equation}
	\label{DirDensity_Ext}
	f(\bm{x}) \equiv f\left(\frac{\bm{x}}{||\bm{x}||_2} \right) \quad \text{ for all } \quad \bm{x}\in \mathbb{R}^{q+1} \setminus \{\bm{0} \}.
	\end{equation}
Now, given the expressions of $\grad f(\bm{x})$ and $\mathcal{H}f(\bm{x})$ defined in \eqref{Riem_grad2} and \eqref{Riem_Hess2}, we perform the spectral decomposition on $\mathcal{H} f(\bm{x})$ as $\mathcal{H} f(\bm{x}) = \underline{V}(\bm{x}) \underline{\Lambda}(\bm{x}) \underline{V}(\bm{x})^T$, where $\underline{V}(\bm{x})=\left[\bm{x},\underline{\bm{v}}_1(\bm{x}),...,\underline{\bm{v}}_q(\bm{x})\right] \in \mathbb{R}^{(q+1)\times (q+1)}$ is a real orthogonal matrix with columns $\underline{\bm{v}}_1(\bm{x}),...,\underline{\bm{v}}_q(\bm{x})$ as the eigenvectors of $\mathcal{H} f(\bm{x})$ that are associated with the eigenvalues $\underline{\lambda}_1(\bm{x}) \geq \cdots \geq \underline{\lambda}_q(\bm{x})$ and lie within the tangent space $T_{\bm{x}}$ at $\bm{x}\in \Omega_q$, and $\underline{\Lambda}(\bm{x})=\Diag\left[0,\underline{\lambda}_1(\bm{x}),...,\underline{\lambda}_q(\bm{x})\right]$. Note that the Riemannian Hessian $\mathcal{H} f(\bm{x})$ has a unit eigenvector $\bm{x}$ that is orthogonal to $T_{\bm{x}}$ and corresponds to eigenvalue 0. 

Let $\underline{V}_d(\bm{x}) = \left[\underline{\bm{v}}_{d+1}(\bm{x}),...,\underline{\bm{v}}_q(\bm{x}) \right] \in \mathbb{R}^{(q+1)\times (q-d)}$ be the last $q-d$ columns of $\underline{V}(\bm{x})$, \emph{i.e.}, the unit eigenvectors inside the tangent space $T_{\bm{x}}$ corresponding to the $q-d$ smallest eigenvalues of $\mathcal{H} f(\bm{x})$. Let $\underline{U}_d(\bm{x}) = \underline{V}_d(\bm{x}) \underline{V}_d(\bm{x})^T$ be the projection matrix onto the linear space spanned by the columns of $\underline{V}_d(\bm{x})$, and $\underline{U}_d^{\perp}(\bm{x}) = \bm{I}_{q+1} -\underline{V}_d(\bm{x}) \underline{V}_d(\bm{x})^T$. We define the order-$d$ principal Riemannian gradient $\underline{G}_d(\bm{x})$ by:
\begin{align}
\label{proj_Riem_grad}
\begin{split}
\underline{G}_d(\bm{x}) = \underline{V}_d(\bm{x}) \underline{V}_d(\bm{x})^T \grad f(\bm{x}) &= \underline{V}_d(\bm{x}) \underline{V}_d(\bm{x})^T \left(\bm{I}_{q+1} -\bm{x}\bm{x}^T \right) \nabla f(\bm{x})\\
&= \underline{V}_d(\bm{x}) \underline{V}_d(\bm{x})^T \nabla f(\bm{x}),
\end{split}
\end{align}
where the last equality follows from the fact that the columns of $\underline{V}_d(\bm{x})$ are orthogonal to the unit vector $\bm{x}$. The order-$d$ density ridge on $\Omega_q$ (or \emph{directional density ridge}) is the set of points defined as:
\begin{align}
\label{Dir_ridge}
\begin{split}
\underline{R}_d &= \left\{\bm{x}\in \Omega_q: \underline{G}_d(\bm{x})=\bm{0}, \underline{\lambda}_{d+1}(\bm{x}) <0 \right\} \\
&= \left\{\bm{x}\in \Omega_q: \underline{V}_d(\bm{x})^T \grad f(\bm{x})=\bm{0}, \underline{\lambda}_{d+1}(\bm{x}) <0 \right\}.
\end{split}
\end{align}

Our definition of density ridges on $\Omega_q$ can be arguably generalized to any smooth function $f$ supported on an arbitrary Riemannian manifold. It also follows that the 0-ridge $\underline{R}_0$ is the set of local modes of $f$ on $\Omega_q$, whose statistical properties and practical estimation algorithm are discussed in \cite{DirMS2020}. Therefore, we only focus on the case when $1\leq d <q$ in this paper. 
To regularize the directional density ridge $\underline{R}_d$, we modify our assumptions on the Euclidean density ridge $R_d$ in Section~\ref{Sec:Assump} as follows:
\begin{itemize}
	\item {\bf (\underline{A1})} (\emph{Differentiability}) Under the extension \eqref{DirDensity_Ext} of the directional density $f$, we assume that the total gradient $\nabla f(\bm{x})$, total Hessian matrix $\nabla\nabla f(\bm{x})$, and third-order derivative tensor $\nabla^3 f(\bm{x})$ in $\mathbb{R}^{q+1}$ exist, and are continuous on $\mathbb{R}^{q+1} \setminus \{\bm{0} \}$ and square integrable on $\Omega_q$. We also assume that $f$ has bounded fourth order derivatives on $\Omega_q$. 
	
	\item {\bf (\underline{A2})} (\emph{Eigengap}) We assume that there exist constants $\underline{\rho}>0$ and $\underline{\beta}_0 >0$ such that $\underline{\lambda}_{d+1}(\bm{y}) \leq -\underline{\beta}_0$ and $\underline{\lambda}_d(\bm{y}) - \underline{\lambda}_{d+1}(\bm{y}) \geq \underline{\beta}_0$ for any $\bm{y}\in \left(\underline{R}_d \oplus \underline{\rho}\right) \cap \Omega_q$.
	
	\item {\bf (\underline{A3})} (\emph{Path Smoothness}) Under the same $\underline{\rho}, \underline{\beta}_0 >0$ in (\underline{A2}), we assume that there exists another constant $\underline{\beta}_1 \in \left(0,\underline{\beta}_0 \right)$ such that
	\begin{align*}
	\sqrt{2} \cdot q^{\frac{3}{2}} \norm{\underline{U}_d^{\perp}(\bm{y}) \grad f(\bm{y})}_2 \norm{\nabla^3 f(\bm{y})}_{\max} &\leq \frac{\underline{\beta}_0^2}{2},\\
	d \cdot q^{\frac{3}{2}} \norm{\nabla f(\bm{x})}_2 \cdot \norm{\nabla^3 f(\bm{x})}_{\max} & \leq \underline{\beta}_0\left(\underline{\beta}_0-\underline{\beta}_1 \right)
	\end{align*}
	for all $\bm{y} \in \left(\underline{R}_d\oplus \underline{\rho} \right)\cap \Omega_q$ and $\bm{x}\in \underline{R}_d$.
\end{itemize}

Recall that $\underline{R}_d\oplus \underline{\rho} = \cup_{\bm{x}\in \underline{R}_d}\text{Ball}_{q+1}\big(\bm{x},\underline{\rho}\big)$ is a $\underline{\rho}$-neighborhood of the directional ridge $\underline{R}_d$ in the ambient space $\mathbb{R}^{q+1}$. The discussions about conditions (A1-3) in Section~\ref{Sec:Assump} apply to their directional counterparts (\underline{A1-3}), except that the eigengap condition (\underline{A2}) is imposed on eigenvalues $\underline{\lambda}_1(\bm{x}) \geq \cdots \geq \underline{\lambda}_q(\bm{x})$ within the tangent space $T_{\bm{x}}$ at $\bm{x}\in \Omega_q$. However, since the only eigenvalue of $\mathcal{H} f(\bm{x})$ associated with the eigenvector outside the tangent space $T_{\bm{x}}$ is 0, the eigengap condition (\underline{A2}) is also valid to the entire spectrum of $\mathcal{H} f(\bm{x})$ in the ambient space $\mathbb{R}^{q+1}$.
The extension of $f$ in \eqref{DirDensity_Ext} has also been used by \cite{Zhao2001,Dir_Linear2013,Exact_Risk_bw2013}. Because the directional density $f$ remains unchanged along every radial direction of $\Omega_q$ under the extension \eqref{DirDensity_Ext}, the radial component of its total gradient is $\Rad\left(\nabla f(\bm{x}) \right)=0$ for all $\bm{x}\in \Omega_q$, and the Riemannian gradient \eqref{Riem_grad2} of $f$ on $\Omega_q$ becomes
\begin{equation}
\label{Riem_grad_new}
\grad (\bm{x}) = \left(\bm{I}_{q+1} -\bm{x}\bm{x}^T \right) \nabla f(\bm{x}) = \nabla f(\bm{x}).
\end{equation}
Similarly, the Riemannian Hessian \eqref{Riem_Hess2} of $f$ on $\Omega_q$ reduces to
\begin{equation}
\label{Riem_Hess_new}
\mathcal{H} f(\bm{x}) = \left(\bm{I}_{q+1} -\bm{x}\bm{x}^T \right) \nabla\nabla f(\bm{x}) \left(\bm{I}_{q+1} -\bm{x}\bm{x}^T \right).
\end{equation}
Both the Riemannian gradient and Hessian of $f$ on $\Omega_q$ are invariant under this extension.

\begin{remark}[\emph{Connection to Solution Manifolds}]
	\label{solution_manifold_remark}
	Example 4 in \cite{YC2020} showed that any Euclidean density ridge $R_d$ defined in \eqref{ridge} is a concrete example of a solution manifold $\mathcal{M}_S = \left\{\bm{x}\in \mathbb{R}^D: \Psi(\bm{x})=0 \right\}$ with $\Psi: \mathbb{R}^D\to \mathbb{R}^{D-d}$ being a vector-valued function. It is not difficult to verify that our defined directional density ridge $\underline{R}_d$ in \eqref{Dir_ridge} also belongs to the general form of the solution manifold $\mathcal{M}_S$, where we may rewrite $\underline{R}_d=\left\{\bm{x}\in \mathbb{R}^{q+1}: \Psi(\bm{x})=0 \right\}$ with $\Psi: \mathbb{R}^{q+1} \to \mathbb{R}^{q+1-d}$ defined by:
		\[
		\Psi(\bm{x})=
		\begin{bmatrix}
		\underline{\bm{v}}_{d+1}(\bm{x})^T \nabla f(\bm{x})\\
		\vdots\\
		\underline{\bm{v}}_q(\bm{x})^T \nabla f(\bm{x})\\
		\bm{x}^T\bm{x}-1
		\end{bmatrix},
		\]
		recalling that $\underline{\bm{v}}_{d+1}(\bm{x}),...,\underline{\bm{v}}_q(\bm{x})$ are the last $q-d$ eigenvectors of the Riemannian Hessian $\mathcal{H}f(\bm{x})$ of the directional density $f$. More importantly, our imposed conditions (A1-3) in the Euclidean ridge case and (\underline{A1-3}) in the directional ridge case imply all the required assumptions in \cite{YC2020}, \emph{i.e.}, the differentiability of $\Psi$ and non-degeneracy of the normal space of $\mathcal{M}_S$; see (d) of Lemmas~\ref{normal_reach_prop} and \ref{Dir_norm_reach_prop} in the Appendix.
		Therefore, the discussion about statistical properties and (normal) gradient flows of a generic solution manifold $\mathcal{M}_S$ apply to the (directional) density ridge $R_d$ or $\underline{R}_d$ here.
	\end{remark}

Similar to Euclidean density ridges, we establish the following stability theorem of directional density ridges. To measure the distance between two directional ridges $\underline{R}_d, \underline{\tilde{R}}_d \subset \Omega_q$ defined by the directional densities $f$ and $\tilde{f}$, we adopt the definition \eqref{Haus_dist} of Hausdorff distance between two sets in the ambient Euclidean space $\mathbb{R}^{q+1}$. Note that the Euclidean norm used in the definition \eqref{Haus_dist} is upper bounded by the geodesic distance when our interested sets lie on $\Omega_q$; see also \eqref{Geo_Eu_dist_eq}. We will leverage this property in our proof of Theorem~\ref{Dir_ridge_stability}; see Appendix~\ref{App:Stability_Dir_Ridge} for details.

\begin{theorem}
	\label{Dir_ridge_stability}
	Suppose that conditions (\underline{A1-3}) hold for the directional density $f$ and that condition (\underline{A1}) holds for $\tilde{f}$. When $\norm{f-\tilde{f}}_{\infty,3}^*$ is sufficiently small,
	\begin{enumerate}[label=(\alph*)]
		\item conditions (\underline{A2-3}) holds for $\tilde{f}$.
		\item $\Haus(\underline{R}_d, \underline{\tilde{R}}_d) = O\left(\norm{f-\tilde{f}}_{\infty,2}^* \right)$.
		\item $\mathtt{reach}(\underline{\tilde{R}}_d) \geq \min\left\{\underline{\rho}/2, \frac{\min\left\{\underline{\beta}_1, 1\right\}^2}{\underline{A}_2 \left(\norm{f}_{\infty}^{(3)} + \norm{f}_{\infty}^{(4)} \right)} \right\} + O\left(\norm{f-\tilde{f}}_{\infty,3}^* \right)$ for a constant $\underline{A}_2>0$.
	\end{enumerate}
\end{theorem}

One natural estimator of the directional density ridge $\underline{R}_d$ can be obtained by plugging the directional KDE $\hat{f}_h$ into the definition \eqref{Dir_ridge} as:
$$\hat{\underline{R}}_d = \left\{\bm{x}\in \Omega_q: \hat{\underline{V}}_d(\bm{x})^T \grad \hat{f}_h(\bm{x}) = \bm{0}, \hat{\underline{\lambda}}_{d+1}(\bm{x})<0 \right\}.$$
To regularize the statistical behavior of the estimated directional ridge $\hat{\underline{R}}_d$, we consider the following assumptions that are generalized from conditions (E1-2):
\begin{itemize}
	\item {\bf (D1)} Assume that $L: (-\delta_L,\infty) \to [0,\infty)$ is a bounded and three times continuously differentiable function with a bounded fourth order derivative on $(-\delta_L,\infty) \subset \mathbb{R}$ for some constant $\delta_L>0$ such that 
	$$0< \int_0^{\infty} |L^{(\ell)}(r)|^k r^{\frac{q}{2}-1} dr < \infty \quad \text{ for all } q \geq 1, k=1,2, \text{ and } \ell=0,1,2,3.$$
	
	\item {\bf (D2)} Let 
	$$\mathcal{K}_D = \left\{ \bm{u}\mapsto K\left(\frac{\bm{z}-\bm{u}}{h} \right): \bm{u}, \bm{z}\in \Omega_q, h>0, K(\bm{x}) = D^{[\tau]} L\left(\frac{\norm{\bm{x}}_2^2}{2} \right), |[\tau]|=0,1,2,3 \right\}.$$
	We assume that $\mathcal{K}_D$ is a bounded VC (subgraph) class of measurable functions on $\Omega_q$; that is, there exist constants $A,\upsilon >0$ such that for any $0<\epsilon <1$,
	$$\sup_Q N\left(\mathcal{K}_D, L_2(Q), \epsilon ||F||_{L_2(Q)} \right) \leq \left(\frac{A}{\epsilon} \right)^{\upsilon},$$
	where $N(\mathcal{K}_D,L_2(Q),\epsilon)$ is the $\epsilon$-covering number of the normed space $\left(\mathcal{K}_D,\norm{\cdot}_{L_2(Q)}\right)$, $Q$ is any probability measure on $\Omega_q$, and $F$ is an envelope function of $\mathcal{K}_D$. Here, the norm $\norm{F}_{L_2(Q)}$ is defined as $\left[\int_{\Omega_q} |F(\bm{x})|^2 dQ(\bm{x}) \right]^{\frac{1}{2}}$.
\end{itemize}

The differentiability assumption in condition (D1) can be relaxed such that $L$ is (three times) continuously differentiable except for a set of points with Lebesgue measure $0$ on $[0,\infty)$. Conditions (D1) and (\underline{A1}) are generally required for establishing the pointwise convergence rates of the directional KDE and its derivatives \citep{KDE_Sphe1987,KLEMELA2000,Zhao2001,Dir_Linear2013,Exact_Risk_bw2013}. Under these two conditions, $\norm{\nabla \hat{f}_h(\bm{x})}_2$ appearing in the step sizes $\underline{\eta}_{n,h}^{(t)}$ or $\underline{\eta}_{n,h}^{(t)'}$ of the directional mean shift or SCMS algorithm can also be shown to diverge at the order $O(h^{-2}) + O_P\left(\frac{1}{nh^{q+2}} \right)$ as $nh^q\to \infty$ and $h\to 0$; see Section~\ref{Sec:Dir_MS_SCMS} for details.
Condition (D2) regularizes the complexity of kernel $L$ and its derivatives as in condition (E2) in order for the uniform convergence rates of the directional KDE and its derivatives; see \eqref{Dir_KDE_unif_bound} below. One can justify via integration by parts that the von-Mises kernel $L(r)=e^{-r}$ and many compactly supported kernels satisfy conditions (D1-2).

Given conditions (D1-2), the techniques in \cite{KDE_Sphe1987,KDE_direct1988,Zhao2001,Dir_Linear2013,Exact_Risk_bw2013,DirMS2020} can be utilized to demonstrate that
\begin{equation}
\label{Dir_KDE_unif_bound}
\norm{\hat{f}_h-f}_{\infty}^{(k)} = \sup_{\bm{x}\in \Omega_q} \norm{\bar{\nabla}^k \,\hat{f}_h(\bm{x}) -\bar{\nabla}^k\, f(\bm{x})}_{\max} = O(h^2) + O_P\left(\sqrt{\frac{|\log h|}{nh^{q+2k}}} \right),
\end{equation}
where $\bar{\nabla}\equiv \bar{\nabla}_{\bm{v}}$ is the Riemannian connection on $\Omega_q$ with $\bm{v}\in T_{\bm{x}}$ so that $\bar{\nabla} f(\bm{x})= \grad f(\bm{x})$, $\bar{\nabla}^2 f(\bm{x}) = \mathcal{H} f(\bm{x})$, and $\bar{\nabla}^3 f(\bm{x}) = \bar{\nabla} \mathcal{H} f(\bm{x})$; see Section 5.3 in \cite{Op_algo_mat_manifold2008} and Chapter 4 in \cite{Lee2018Riem_man}.

\subsection{Mean Shift and SCMS Algorithm with Directional Data}
\label{Sec:Dir_MS_SCMS}

Before deriving our directional SCMS algorithm, we first review the mean shift algorithm with directional data $\left\{\bm{X}_1,...,\bm{X}_n\right\} \subset \Omega_q$ \citep{Multi_Clu_Gene2005,DMS_topology2010,MSC_Dir2014}. The formal derivation can be found in Section 3 of \cite{DirMS2020}. Given the directional KDE $\hat{f}_h(\bm{x}) = \frac{c_{L,q}(h)}{n} \sum\limits_{i=1}^n L\left(\frac{1-\bm{x}^T\bm{X}_i}{h^2} \right)$ in \eqref{Dir_KDE}, the directional mean shift vector can be defined as:
\begin{equation}
\label{Dir_mean_shit_vec}
\underline{\Xi}_h(\bm{x}) =\frac{\sum_{i=1}^n \bm{X}_i  L'\left(\frac{1}{2} \norm{\frac{\bm{x} -\bm{X}_i}{h}}_2^2 \right)}{\sum_{i=1}^n L'\left(\frac{1}{2} \norm{\frac{\bm{x} -\bm{X}_i}{h}}_2^2 \right)} -\bm{x}= \frac{\sum_{i=1}^n \bm{X}_i L'\left(\frac{1-\bm{x}^T\bm{X}_i}{h^2} \right)}{\sum_{i=1}^n L'\left(\frac{1-\bm{x}^T\bm{X}_i}{h^2} \right)} -\bm{x}.
\end{equation}
Similar to the Euclidean mean shift vector \eqref{MS}, $\underline{\Xi}_h(\bm{x})$ also points toward the direction of maximum increase in $\hat{f}_h(\bm{x})$ after being projected onto the tangent space $T_{\bm{x}}$. Thus, the directional mean shift iteration translates a point $\bm{x}\in \Omega_q$ as $\bm{x}+\underline{\Xi}_h(\bm{x})$ with an extra projection $\frac{\bm{x}+\underline{\Xi}_h(\bm{x})}{\norm{\bm{x}+\underline{\Xi}_h(\bm{x})}_2}$ to draw the shifted point back to $\Omega_q$. 

Let $\{\underline{\hat{\bm{x}}}^{(t)}\}_{t=0}^{\infty}$ denote the sequence defined by the above directional mean shift procedure. Later, by abuse of notation, we will use the same notation to denote the directional SCGA/SCMS sequence with $\hat{f}_h$. As $\nabla \hat{f}_h(\bm{x}) = -\frac{c_{L,q}(h)}{nh^2} \sum\limits_{i=1}^n \bm{X}_i L'\left(\frac{1-\bm{x}^T\bm{X}_i}{h^2} \right)$, some simple algebra shows that the directional mean shift algorithm can be written into the following fixed-point iteration formula:
\begin{equation}
\label{Dir_MS2}
\underline{\hat{\bm{x}}}^{(t+1)} = -\frac{\sum_{i=1}^n \bm{X}_i L'\left(\frac{1-\bm{X}_i^T\underline{\hat{\bm{x}}}^{(t)}}{h^2} \right) }{\norm{\sum_{i=1}^n \bm{X}_i L'\left(\frac{1-\bm{X}_i^T\underline{\hat{\bm{x}}}^{(t)}}{h^2} \right)}_2} \quad \text{ or } \quad \underline{\hat{\bm{x}}}^{(t+1)} = \frac{\nabla \hat{f}_h\left(\underline{\hat{\bm{x}}}^{(t)} \right)}{\norm{\nabla \hat{f}_h\left(\underline{\hat{\bm{x}}}^{(t)}\right)}_2}.
\end{equation}
From \eqref{Dir_MS2}, it is also possible to write the directional mean shift algorithm as a gradient ascent method on $\Omega_q$ with the iteration formula \citep{Geo_Convex_Op2016}:
\begin{equation}
\label{grad_ascent_Manifold}
\hat{\underline{\bm{x}}}^{(t+1)} = \Exp_{\hat{\underline{\bm{x}}}^{(t)}}\left(\underline{\eta}_{n,h}^{(t)} \cdot \grad \hat{f}_h(\hat{\underline{\bm{x}}}^{(t)}) \right),
\end{equation}
where the adaptive step size $\underline{\eta}_{n,h}^{(t)}$ is given by
\begin{equation}
\label{Dir_MS_stepsize}
\underline{\eta}_{n,h}^{(t)} = \arccos\left( \frac{\nabla \hat{f}_h\left(\underline{\hat{\bm{x}}}^{(t)} \right)^T \underline{\hat{\bm{x}}}^{(t)}}{\norm{\nabla \hat{f}_h\left(\underline{\hat{\bm{x}}}^{(t)}\right)}_2}\right) \cdot \frac{1}{\norm{\grad \hat{f}_h(\underline{\hat{\bm{x}}}^{(t)})}_2} = \frac{\theta_t}{\norm{\nabla \hat{f}_h(\underline{\hat{\bm{x}}}^{(t)})}_2 \cdot \sin\theta_t}.
\end{equation}
Here, we denote the angle between $\underline{\hat{\bm{x}}}^{(t+1)}$ and $\underline{\hat{\bm{x}}}^{(t)}$ (or equivalently, the angle between $\nabla \hat{f}_h(\underline{\hat{\bm{x}}}^{(t)})$ and $\underline{\hat{\bm{x}}}^{(t)}$) by $\theta_t$; see Section 5.2 in \cite{DirMS2020} for detailed derivations.
Within some small neighborhoods around local modes of $\hat{f}_h$, $\frac{\theta_t}{\sin \theta_t} \approx 1$ and the adaptive step size $\underline{\eta}_{n,h}^{(t)}$ will be dominated by $\norm{\nabla \hat{f}_h(\underline{\hat{\bm{x}}}^{(t)})}_2$. The following lemma characterizes the asymptotic behaviors of $\norm{\nabla \hat{f}_h(\bm{x})}_2$ on $\Omega_q$ and consequently, $\underline{\eta}_{n,h}^{(t)}$.

\begin{lemma}[Lemma 10 in \citealt{DirMS2020}]
	\label{norm_tot_grad}
	Assume conditions (D1) and (\underline{A1}). For any fixed $\bm{x} \in \Omega_q$, we have
	$$h^2 \cdot \norm{\nabla \hat{f}_h(\bm{x})}_2 = f(\bm{x}) \cdot C_{L,q} + o\left(1 \right) + O_P\left(\sqrt{\frac{1}{nh^q}} \right)$$
	as $nh^q \to \infty$ and $h\to 0$, where $C_{L,q} =-\frac{\int_0^{\infty} L'(r) r^{\frac{q}{2}-1} dr}{\int_0^{\infty} L(r) r^{\frac{q}{2}-1} dr}>0$ is a constant depending only on kernel $L$ and dimension $q$.
\end{lemma}

Lemma~\ref{norm_tot_grad} indicates that $\norm{\nabla \hat{f}_h(\bm{x})}_2 \to \infty$ with probability tending to 1 as $h\to 0$ and $nh^q \to \infty$ for any $\bm{x}\in \Omega_q$. The conclusion may seem counterintuitive at the first glance, but one should be aware that the consistency of $\nabla \hat{f}_h(\bm{x})$ holds only on its tangent component; see \eqref{Dir_KDE_unif_bound}. The radial component of $\nabla \hat{f}_h(\bm{x})$ that is perpendicular to $\Omega_q$ diverges, despite the fact that the true directional density $f$ does not have any radial component. Using Lemma~\ref{norm_tot_grad}, one can argue that the adaptive step size $\underline{\eta}_{n,h}^{(t)}$ in \eqref{Dir_MS_stepsize} of the directional mean shift algorithm as a gradient ascent method on $\Omega_q$ tends to zero at the rate $O(h^2)$ as $h\to0$ and $nh^D\to \infty$.

In the sequel, we denote by $\{\underline{\hat{\bm{x}}}^{(t)}\}_{t=0,1,...} \subset \Omega_q$ the iterative sequence generated by our directional SCMS algorithm. There are two different methods of defining a directional SCMS iteration, while we will demonstrate that one of them is superior.

\noindent $\bullet$ {\bf Method 1}: As in the Euclidean SCMS algorithm, one can define the directional SCMS sequence by the directional mean shift vector \eqref{Dir_mean_shit_vec} as:
\begin{align}
\label{Dir_SCMS_update}
\begin{split}
\underline{\hat{\bm{x}}}^{(t+1)} &\gets \underline{\hat{\bm{x}}}^{(t)} + \underline{\hat{V}}_d\left(\underline{\hat{\bm{x}}}^{(t)} \right) \underline{\hat{V}}_d\left(\underline{\hat{\bm{x}}}^{(t)} \right)^T \underline{\Xi}_h\left(\underline{\hat{\bm{x}}}^{(t)} \right)\\
&\overset{\text{(*)}}{=} \underline{\hat{\bm{x}}}^{(t)} + \underline{\hat{V}}_d(\underline{\hat{\bm{x}}}^{(t)}) \underline{\hat{V}}_d(\underline{\hat{\bm{x}}}^{(t)})^T \left[\frac{\sum_{i=1}^n \bm{X}_i L'\left(\frac{1-\bm{X}_i^T \underline{\hat{\bm{x}}}^{(t)}}{h^2} \right)}{\sum_{i=1}^n L'\left(\frac{1-\bm{X}_i^T \underline{\hat{\bm{x}}}^{(t)}}{h^2} \right)} \right]\\
&= \underline{\hat{\bm{x}}}^{(t)} + \underline{\hat{V}}_d(\underline{\hat{\bm{x}}}^{(t)}) \underline{\hat{V}}_d(\underline{\hat{\bm{x}}}^{(t)})^T \cdot \frac{\nabla \hat{f}_h(\underline{\hat{\bm{x}}}^{(t)})}{\hat{\underline{g}}_h(\underline{\hat{\bm{x}}}^{(t)})},
\end{split}
\end{align}
where $\hat{\underline{g}}_h(\bm{x})= -\frac{c_{L,q}(h)}{nh^2} \sum\limits_{i=1}^n L'\left(\frac{1-\bm{x}^T \bm{X}_i}{h^2} \right)$, $\nabla \hat{f}_h(\bm{x}) = -\frac{c_{L,q}(h)}{nh^2} \sum\limits_{i=1}^n \bm{X}_iL'\left(\frac{1-\bm{x}^T \bm{X}_i}{h^2} \right)$, and
$\underline{\hat{V}}_d(\bm{x})$ is the estimated version of $\underline{{V}}_d(\bm{x})$ defined by the directional KDE $\hat{f}_h$. Here, we plug in \eqref{Dir_mean_shit_vec} and leverage the orthogonality between the columns of $\underline{\hat{V}}_d(\underline{\hat{\bm{x}}}^{(t)})$ and $\underline{\hat{\bm{x}}}^{(t)}\in \Omega_q$ in (*).

Unlike the Euclidean SCMS algorithm, we need an extra standardization step $\underline{\hat{\bm{x}}}^{(t+1)} \gets \frac{\underline{\hat{\bm{x}}}^{(t+1)}}{\norm{\underline{\hat{\bm{x}}}^{(t+1)}}_2}$ to project the updated point back to $\Omega_q$,
which leads to the following fixed-point iteration:
\begin{align}
\label{Dir_SCMS_fix_point}
\begin{split}
\underline{\hat{\bm{x}}}^{(t+1)} &= \frac{\underline{\hat{V}}_d(\underline{\hat{\bm{x}}}^{(t)}) \underline{\hat{V}}_d(\underline{\hat{\bm{x}}}^{(t)})^T \nabla \hat{f}_h(\underline{\hat{\bm{x}}}^{(t)}) + \underline{\hat{g}}_h(\underline{\hat{\bm{x}}}^{(t)}) \cdot \underline{\hat{\bm{x}}}^{(t)}}{\norm{\underline{\hat{V}}_d(\underline{\hat{\bm{x}}}^{(t)}) \underline{\hat{V}}_d(\underline{\hat{\bm{x}}}^{(t)})^T \nabla \hat{f}_h(\underline{\hat{\bm{x}}}^{(t)}) + \underline{\hat{g}}_h(\underline{\hat{\bm{x}}}^{(t)}) \cdot \underline{\hat{\bm{x}}}^{(t)}}_2},
\end{split}
\end{align}
where the components $\underline{\hat{V}}_d(\bm{x})\underline{\hat{V}}_d(\bm{x})^T \nabla \hat{f}_h(\bm{x})$ and $\underline{\hat{g}}_h(\bm{x}) \cdot \bm{x}$ are always orthogonal for any $\bm{x} \in \Omega_q$; see Figure~\ref{fig:SCMS_One_Step} for a graphical illustration.

\begin{figure}
	\centering
	\includegraphics[width=0.75\linewidth]{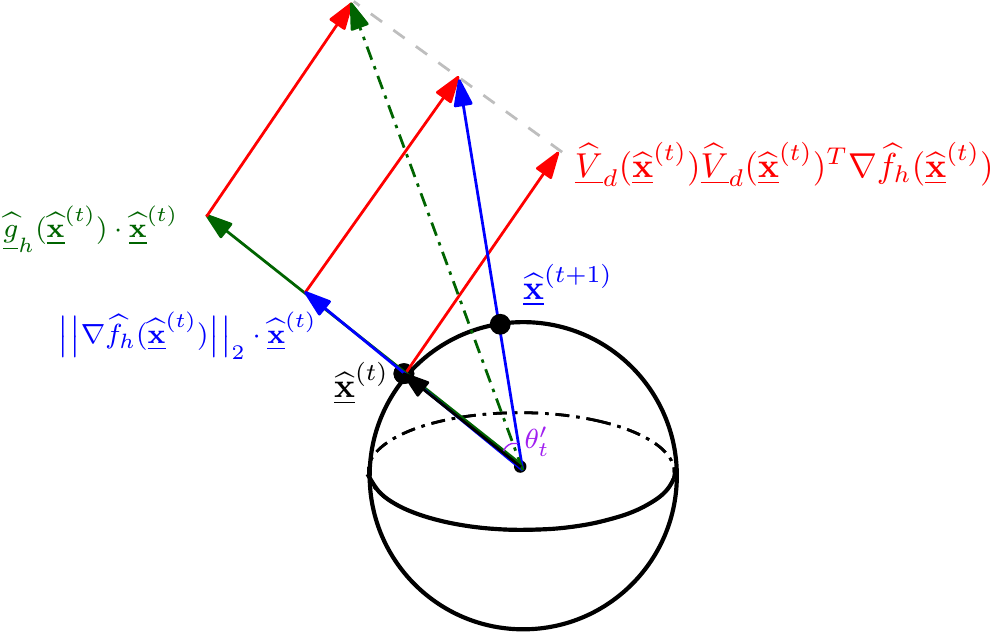}
	\caption{An illustration of one-step iterations under two candidate directional SCMS algorithms}
	\label{fig:SCMS_One_Step}
\end{figure}

\noindent $\bullet$ {\bf Method 2}: The fixed-point iteration formula \eqref{Dir_MS2} of the directional mean shift algorithm suggests a more efficient formulation of the directional SCMS algorithm as:
\begin{equation}
\label{Dir_SCMS_update_new}
\underline{\hat{\bm{x}}}^{(t+1)} \gets \underline{\hat{\bm{x}}}^{(t)} + \hat{\underline{V}}_d(\underline{\hat{\bm{x}}}^{(t)}) \hat{\underline{V}}_d(\underline{\hat{\bm{x}}}^{(t)})^T \cdot \frac{\nabla \hat{f}_h(\underline{\hat{\bm{x}}}^{(t)})}{\norm{\nabla \hat{f}_h(\underline{\hat{\bm{x}}}^{(t)})}_2} \quad \text{ and } \quad \underline{\hat{\bm{x}}}^{(t+1)} \gets \frac{\underline{\hat{\bm{x}}}^{(t+1)}}{\norm{\underline{\hat{\bm{x}}}^{(t+1)}}_2},
\end{equation}
where we replace the directional mean shift vector $\underline{\Xi}_h(\bm{x})$ with the standardized total gradient estimator $\frac{\nabla \hat{f}_h(\bm{x})}{\norm{\nabla \hat{f}_h(\bm{x})}_2}$ in \eqref{Dir_SCMS_update}. This directional SCMS is again a fixed-point iteration as:
\begin{equation}
\label{Dir_SCMS_fixed_point_new}
\underline{\hat{\bm{x}}}^{(t+1)} = \frac{\hat{\underline{V}}_d(\underline{\hat{\bm{x}}}^{(t)}) \hat{\underline{V}}_d(\underline{\hat{\bm{x}}}^{(t)})^T \nabla \hat{f}_h(\underline{\hat{\bm{x}}}^{(t)}) + \norm{\nabla \hat{f}_h(\underline{\hat{\bm{x}}}^{(t)})}_2 \cdot \underline{\hat{\bm{x}}}^{(t)}}{\norm{\hat{\underline{V}}_d(\underline{\hat{\bm{x}}}^{(t)}) \hat{\underline{V}}_d(\underline{\hat{\bm{x}}}^{(t)})^T \nabla \hat{f}_h(\underline{\hat{\bm{x}}}^{(t)}) + \norm{\nabla \hat{f}_h(\underline{\hat{\bm{x}}}^{(t)})}_2 \cdot \underline{\hat{\bm{x}}}^{(t)}}_2}.
\end{equation}

A direct computation demonstrates that, by the non-increasing property of kernel $L$ and the fact that $\norm{\bm{X}_i}_2=1$ for $i=1,...,n$,
\begin{align}
\label{norm_f_hat_g_n_ineq}
\begin{split}
\norm{\nabla \hat{f}_h(\bm{x})}_2 &= \norm{-\frac{c_{L,q}(h)}{nh^2} \sum_{i=1}^n \bm{X}_i L'\left(\frac{1-\bm{x}^T\bm{X}_i}{h^2} \right)}_2\\ &\leq \frac{c_{L,q}(h)}{nh^2} \sum_{i=1}^n \norm{\bm{X}_i L'\left(\frac{1-\bm{x}^T\bm{X}_i}{h^2} \right)}_2\\
&= -\frac{c_{L,q}(h)}{nh^2} \sum_{i=1}^n L'\left(\frac{1-\bm{x}^T\bm{X}_i}{h^2} \right) = \hat{\underline{g}}_h(\bm{x}).
\end{split}
\end{align}
Because the radial components $\underline{\hat{g}}_h(\underline{\hat{\bm{x}}}^{(t)}) \cdot \underline{\hat{\bm{x}}}^{(t)}$ and $\norm{\nabla \hat{f}_h(\underline{\hat{\bm{x}}}^{(t)})}_2 \cdot \underline{\hat{\bm{x}}}^{(t)}$ in directional SCMS iterative formulae \eqref{Dir_SCMS_fix_point} and \eqref{Dir_SCMS_fixed_point_new} respectively make no contributions to the iteration of point $\underline{\hat{\bm{x}}}^{(t)}$ on $\Omega_q$, the inequality \eqref{norm_f_hat_g_n_ineq} indicates that the directional SCMS algorithm with iterative formula \eqref{Dir_SCMS_fixed_point_new} takes a larger step size in moving the SCMS sequence $\big\{\underline{\hat{\bm{x}}}^{(t)} \big\}_{t=0}^{\infty}$ on $\Omega_q$. This helps accelerate the movements of those points that are far away from the ridge $\underline{\hat{R}}_d$ or lie in the regions with low density values of $\hat{f}_h$. In this sense, the directional SCMS algorithm with iterative formula \eqref{Dir_SCMS_fixed_point_new} will be superior to \eqref{Dir_SCMS_fix_point}; see Figure~\ref{fig:SCMS_One_Step} for a graphical demonstration. We thus choose {\bf Method 2} as our directional SCMS algorithm. 
Algorithm~\ref{Algo:Dir_SCMS} in Appendix~\ref{App:Algo} provides the detailed steps of implementing Method 2 in practice.

Inspired by Proposition 2 in \cite{SCMS_conv2013} for the Euclidean SCMS algorithm, we derive the ascending property of our directional SCMS algorithm \eqref{Dir_SCMS_update_new} and two convergent results for stopping the algorithm in the following proposition. The proof is deferred to Appendix~\ref{App:Proofs_Dir_SCMS}, in which our argument is similar to but logically different from the proof of Proposition 2 in \cite{SCMS_conv2013}.

\begin{proposition}
	\label{Dir_SCMS_prop}
	Assume that the directional kernel $L$ is non-increasing, twice continuously differentiable, and convex with $L(0) <\infty$. Given the directional KDE $\hat{f}_h(\bm{x})=\frac{c_{L,q}(h)}{n}\sum\limits_{i=1}^n L\left(\frac{1-\bm{x}^T\bm{X}_i}{h^2} \right)$ and the directional SCMS sequence $\big\{\underline{\hat{\bm{x}}}^{(t)} \big\}_{t=0}^{\infty} \subset \Omega_q$ defined by \eqref{Dir_SCMS_update_new} or \eqref{Dir_SCMS_fixed_point_new}, the following properties hold:
	\begin{enumerate}[label=(\alph*)]
		\item The estimated density sequence $\left\{\hat{f}_h(\underline{\hat{\bm{x}}}^{(t)}) \right\}_{t=0}^{\infty}$ is non-decreasing and thus converges.
		\item $\lim\limits_{t\to \infty} \norm{\hat{\underline{V}}_d(\underline{\hat{\bm{x}}}^{(t)})^T \nabla \hat{f}_h(\underline{\hat{\bm{x}}}^{(t)})}_2=0$.
		\item If the kernel $L$ is also strictly decreasing on $[0,\infty)$, then $\lim\limits_{t\to \infty} \norm{\underline{\hat{\bm{x}}}^{(t+1)} - \underline{\hat{\bm{x}}}^{(t)}}_2=0$.
	\end{enumerate}
\end{proposition}

\begin{remark}
	Our results (b) and (c) in Proposition~\ref{Dir_SCMS_prop} demonstrates that the stopping criterion of our directional SCMS algorithm can follow either the norm of the principal Riemannian gradient estimator $\norm{\hat{\underline{V}}_d(\underline{\hat{\bm{x}}}^{(t)})^T \nabla \hat{f}_h(\underline{\hat{\bm{x}}}^{(t)})}_2$ or the (Euclidean) distance $\norm{\underline{\hat{\bm{x}}}^{(t+1)} -\underline{\hat{\bm{x}}}^{(t)}}_2$ between two consecutive iterative points, where the latter one requires a strictly decreasing kernel such as the von Mises kernel $L(r)=e^{-r}$.
\end{remark}

Motivated by the iterative formula \eqref{grad_ascent_Manifold} for the gradient ascent algorithm on $\Omega_q$, we consider writing our directional SCMS algorithm as a variant of the SCGA algorithm on $\Omega_q$ with an iterative formula:
\begin{equation}
\label{SCGA_Dir_sample}
\hat{\underline{\bm{x}}}^{(t+1)} = \Exp_{\hat{\underline{\bm{x}}}^{(t)}}\left(\underline{\eta}_{n,h}^{(t)'} \cdot \hat{\underline{V}}_d(\hat{\underline{\bm{x}}}^{(t)}) \hat{\underline{V}}_d(\hat{\underline{\bm{x}}}^{(t)})^T \grad \hat{f}_h(\hat{\underline{\bm{x}}}^{(t)}) \right),
\end{equation}
where $\Exp_{\bm{x}}$ is the exponential map at $\bm{x}\in \Omega_q$ and $\underline{\eta}_{n,h}^{(t)'} >0$ is the adaptive step size. Analogous to the Euclidean SCMS algorithm and its SCGA representation \eqref{SCGA_update_sample}, the formulation \eqref{SCGA_Dir_sample} will reveal the (linear) convergence properties of our directional SCMS algorithm in the upcoming Section~\ref{Sec:Dir_SCGA_LC}.
To derive an explicit formula for $\underline{\eta}_{n,h}^{(t)'}$, we recall the fixed-point equation \eqref{Dir_SCMS_fixed_point_new} of our directional SCMS algorithm and compute the geodesic distance between $\underline{\hat{\bm{x}}}^{(t+1)}$ and $\underline{\hat{\bm{x}}}^{(t)}$ (one-step directional SCMS update) as:
\begin{align*}
\arccos\left(\left(\underline{\hat{\bm{x}}}^{(t+1)}\right)^T \underline{\hat{\bm{x}}}^{(t)} \right) &= \arccos\left(\frac{ \norm{\nabla \hat{f}_h(\underline{\hat{\bm{x}}}^{(t)})}_2}{\norm{\underline{\hat{V}}_d(\underline{\hat{\bm{x}}}^{(t)}) \underline{\hat{V}}_d(\underline{\hat{\bm{x}}}^{(t)})^T \nabla \hat{f}_h(\underline{\hat{\bm{x}}}^{(t)}) + \norm{\nabla \hat{f}_h(\underline{\hat{\bm{x}}}^{(t)})}_2 \cdot \underline{\hat{\bm{x}}}^{(t)}}_2} \right)\\
&= \norm{\underline{\eta}_{n,h}^{(t)'} \cdot \underline{\hat{V}}_d(\underline{\hat{\bm{x}}}^{(t)}) \underline{\hat{V}}_d(\underline{\hat{\bm{x}}}^{(t)})^T \grad \hat{f}_h(\underline{\hat{\bm{x}}}^{(t)})}_2,
\end{align*}
where, in the second equality, we equate the geodesic distance between $\underline{\hat{\bm{x}}}^{(t+1)}$ and $\underline{\hat{\bm{x}}}^{(t)}$ to the norm of the tangent vector inside the exponential map in \eqref{SCGA_Dir_sample}. This suggests that our directional SCMS algorithm is a (sample-based) SCGA algorithm on $\Omega_q$ with adaptive step size 
\begin{equation}
\label{Dir_SCMS_stepsize}
\underline{\eta}_{n,h}^{(t)'}
= \frac{\arccos\left(\frac{ \norm{\nabla \hat{f}_h(\underline{\hat{\bm{x}}}^{(t)})}_2}{\norm{\underline{\hat{V}}_d(\underline{\hat{\bm{x}}}^{(t)}) \underline{\hat{V}}_d(\underline{\hat{\bm{x}}}^{(t)})^T \nabla \hat{f}_h(\underline{\hat{\bm{x}}}^{(t)}) + \norm{\nabla \hat{f}_h(\underline{\hat{\bm{x}}}^{(t)})}_2 \cdot \underline{\hat{\bm{x}}}^{(t)}}_2} \right)}{\norm{\underline{\hat{V}}_d(\underline{\hat{\bm{x}}}^{(t)}) \underline{\hat{V}}_d(\underline{\hat{\bm{x}}}^{(t)})^T \nabla \hat{f}_h(\underline{\hat{\bm{x}}}^{(t)})}_2} = \frac{\theta_t'}{\norm{\nabla \hat{f}_h(\underline{\hat{\bm{x}}}^{(t)})}_2 \cdot \tan \theta_t'}
\end{equation}
for $t=0,1,...$, where $\theta_t'$ denotes the angle between $\underline{\hat{\bm{x}}}^{(t+1)}$ and $\underline{\hat{\bm{x}}}^{(t)}$. Note that the above derivation is based on the orthogonality between $\underline{\hat{\bm{x}}}^{(t)}$ and the order-$d$ principal Riemannian gradient estimator
$$\underline{\hat{G}}_d(\underline{\hat{\bm{x}}}^{(t)}) = \underline{\hat{V}}_d(\underline{\hat{\bm{x}}}^{(t)}) \underline{\hat{V}}_d(\underline{\hat{\bm{x}}}^{(t)})^T \grad \hat{f}_h(\underline{\hat{\bm{x}}}^{(t)}) = \underline{\hat{V}}_d(\underline{\hat{\bm{x}}}^{(t)}) \underline{\hat{V}}_d(\underline{\hat{\bm{x}}}^{(t)})^T \nabla \hat{f}_h(\underline{\hat{\bm{x}}}^{(t)});$$
see Figure~\ref{fig:SCMS_One_Step} for a graphical illustration. When our directional SCMS algorithm approaches the estimated ridge $\underline{\hat{R}}_d$, $\theta_t'$ tends to 0 and $\frac{\theta_t'}{\tan \theta_t'}$ is approximately equal to 1. Thus, the step size $\underline{\eta}_{n,h}^{(t)'}$ is also controlled by $\norm{\hat{f}_h(\underline{\hat{\bm{x}}}^{(t)})}_2$ as in the directional mean shift scenario; see Equation~\eqref{Dir_MS_stepsize}. Therefore, Lemma~\ref{norm_tot_grad} is still effective to argue that the step size $\underline{\eta}_{n,h}^{(t)'}$ converges to 0 with probability tending to 1 when $h\to 0$ and $nh^q \to\infty$.

\subsection{Linear Convergence of Population and Sample-Based SCGA Algorithms on $\Omega_q$}
\label{Sec:Dir_SCGA_LC}

As we have shown in \eqref{SCGA_Dir_sample} that our proposed directional SCMS algorithm is an example of the sample-based SCGA method with directional KDE $\hat{f}_h$ on $\Omega_q$ with an adaptive step size $\eta_{n,h}^{(t)'}$, our main focus in this subsection will be on the (linear) convergence of such SCGA algorithm on $\Omega_q$. We first consider the population SCGA algorithm on $\Omega_q$ defined by its iterative formula as:
\begin{equation}
\label{SCGA_manifold_update}
\underline{\bm{x}}^{(t+1)} = \Exp_{\underline{\bm{x}}^{(t)}}\left(\underline{\eta} \cdot \underline{V}_d(\underline{\bm{x}}^{(t)}) \underline{V}_d(\underline{\bm{x}}^{(t)})^T \grad f(\underline{\bm{x}}^{(t)}) \right)
\end{equation}
with a suitable choice of the step size $\underline{\eta} >0$. The sample-based version substitutes the subspace constrained Riemannian gradient $\underline{V}_d(\underline{\bm{x}}) \underline{V}_d(\underline{\bm{x}})^T \grad f(\underline{\bm{x}})$ with its estimator $\hat{\underline{V}}_d(\underline{\bm{x}}) \hat{\underline{V}}_d(\underline{\bm{x}})^T \grad \hat{f}_h(\underline{\bm{x}})$ and generally has a constant step size $\underline{\eta}>0$; see \eqref{SCGA_Dir_sample}. In the sequel, we denote the sequence defined by the population SCGA algorithm with objective function $f$ on $\Omega_q$ by $\big\{\underline{\bm{x}}^{(t)}\big\}_{t=0}^{\infty}$ and the sequence defined by the sample-based SCGA algorithm with objective function $\hat{f}_h$ on $\Omega_q$ by $\big\{\hat{\underline{\bm{x}}}^{(t)}\big\}_{t=0}^{\infty}$.  
	
	\begin{remark}
		Note that the definition \eqref{SCGA_manifold_update} of the SCGA algorithm is adaptive to any Riemannian manifold $\mathcal{M}$, not restricting to the unit hypersphere $\Omega_q$. The only requirement on $\mathcal{M}$ for \eqref{SCGA_manifold_update} to be valid is that the exponential map $\Exp_{\underline{\bm{x}}^{(t)}}: T_{\bm{\underline{\bm{x}}^{(t)}}}(\mathcal{M}) \to \mathcal{M}$ is well-defined within a small neighborhood of $\bm{0}$ on the tangent space $T_{\bm{\underline{\bm{x}}^{(t)}}}(\mathcal{M})$ for each $t \geq 0$. More importantly, our assumptions (\underline{A1-3}) and condition (\underline{{A4}}) are generalizable to any smooth function $f$ supported on $\mathcal{M}$, and our (linear) convergence results are applicable to the SCGA algorithm \eqref{SCGA_manifold_update} on $\mathcal{M}$ whose sectional curvature is lower bounded by a real number; see one of the key lemmas in our proofs (Lemma~\ref{trigo_inequality}).
	\end{remark}

	Similar to the SCGA algorithm in the Euclidean space $\mathbb{R}^D$, the following proposition demonstrates that the SCGA algorithm \eqref{SCGA_manifold_update} on $\Omega_q$ yields a non-decreasing sequence of the objective function $f$ supported on $\Omega_q$ and a convergent SCGA sequence to the directional ridge $\underline{R}_d$, as long as the step size $\underline{\eta}$ is sufficiently small.
	
	\begin{proposition}[Convergence of the SCGA Algorithm on $\Omega_q$]
		\label{SCGA_Dir_conv}
		For any SCGA sequence $\big\{\underline{\bm{x}}^{(t)} \big\}_{t=0}^{\infty} \subset \Omega_q$ defined by \eqref{SCGA_manifold_update} with $0<\underline{\eta} < \frac{2}{q\norm{\mathcal{H}f}_{\infty}^{(2)}}$, the following properties hold:
		\begin{enumerate}[label=(\alph*)]
			\item Under condition (\underline{A1}), the objective function sequence $\big\{f(\underline{\bm{x}}^{(t)}) \big\}_{t=0}^{\infty}$ is non-decreasing and thus converges.
			
			\item Under condition (\underline{A1}), $\lim_{t\to\infty} \norm{\underline{V}_d(\underline{\bm{x}}^{(t)})^T \grad f(\underline{\bm{x}}^{(t)})}_2 = \lim_{t\to\infty} d_g\left(\underline{\bm{x}}^{(t+1)}, \underline{\bm{x}}^{(t)} \right)=0$.
			
			\item Under conditions (\underline{A1-3}), $\lim_{t\to\infty} d_g\left(\underline{\bm{x}}^{(t)},\underline{R}_d \right)=0$ whenever $\underline{\bm{x}}^{(0)} \in \underline{R}_d\oplus \underline{r}_1$ with the convergence radius $\underline{r}_1$ satisfying
			$$0< \underline{r}_1 < \min\left\{\underline{\rho}/2, \frac{\min\left\{\underline{\beta}_1, 1 \right\}^2}{\underline{A}_2 \left(\norm{f}_{\infty}^{(3)} + \norm{f}_{\infty}^{(4)} \right)}, 2\sin\left(\frac{\underline{\beta}_1}{2\underline{A}_4(f)} \right) \right\},$$
			where $\underline{A}_2$ is a constant defined in (h) of Lemma~\ref{Dir_norm_reach_prop} while $\underline{A}_4(f) >0$ is a quantity depending on both the dimension $q$ and the functional norm $\norm{f}_{\infty,4}^*$ up to the fourth-order (partial) derivatives of $f$.
		\end{enumerate}
	\end{proposition}
	
	The proof of Proposition~\ref{SCGA_Dir_conv} can be found in Appendix~\ref{App:Proofs_Dir_SCMS}. The upper bound for the convergence radius $\underline{r}_1>0$ has the same meaning as in Proposition~\ref{SCGA_conv} for the Euclidean SCGA algorithm, ensuring that $\underline{r}_1 \leq \mathtt{reach}(\underline{R}_d)$ and the distances from the SCGA sequence $\big\{\underline{\bm{x}}^{(t)} \big\}_{t=0}^{\infty}$ on $\Omega_q$ to the directional ridge $\underline{R}_d$ can be upper bounded by the norms $\norm{\underline{V}_d(\underline{\bm{x}}^{(t)})^T \grad f(\underline{\bm{x}}^{(t)})}_2$ of order-$d$ principal Riemannian gradients for all $t=0,1,...$.
	
	\begin{corollary}[Convergence of the Directional SCMS Algorithm]
		\label{Dir_SCMS_conv}
		When the fixed sample size $n$ is sufficiently large and the bandwidth $h$ is chosen to be correspondingly small, the following properties hold for the directional SCMS sequence $\big\{\hat{\underline{\bm{x}}}^{(t)}\big\}_{t=0}^{\infty} \subset \Omega_q$ with high probability under conditions (\underline{A1-3}) and (D1-2):
		\begin{enumerate}[label=(\alph*)]
			\item The directional KDE sequence $\left\{\hat{f}_h(\hat{\underline{\bm{x}}}^{(t)}) \right\}_{t=0}^{\infty}$ is non-decreasing and thus converges.
			
			\item $\lim_{t\to\infty}\norm{\hat{\underline{V}}(\hat{\underline{\bm{x}}}^{(t)})^T \grad \hat{f}_h(\hat{\underline{\bm{x}}}^{(t)})}_2= \lim_{t\to\infty} d_g\left(\hat{\underline{\bm{x}}}^{(t)}, \hat{\underline{\bm{x}}}^{(t+1)} \right)=0$.
			
			\item $\lim_{t\to\infty} d_g\left(\hat{\underline{\bm{x}}}^{(t)}, \hat{\underline{R}}_d\right)=0$ whenever $\hat{\underline{\bm{x}}}^{(0)} \in \hat{\underline{R}}_d \oplus \underline{r}_1$ with the convergence radius $\underline{r}_1 >0$ defined in (c) of Proposition~\ref{SCGA_Dir_conv}.
		\end{enumerate}
	\end{corollary}
	
	Corollary~\ref{Dir_SCMS_conv} should also be considered as the convergence results of the sample-based SCGA algorithm on $\Omega_q$. To justify Corollary~\ref{Dir_SCMS_conv}, we know from Theorem~\ref{Dir_ridge_stability} that conditions (\underline{A1-3}) also hold with high probability for the directional KDE $\hat{f}_h$ and its estimated directional ridge $\hat{\underline{R}}_d$ when $\frac{nh^{q+6}}{|\log h|}$ is sufficiently large and $h$ is small enough. Further, by Lemma~\ref{norm_tot_grad}, the adaptive step size $\underline{\eta}_{n,h}^{(t)'}$ of our directional SCMS algorithm can be smaller than the threshold value $\frac{2}{q\norm{\mathcal{H}f}_{\infty}^{(2)}}$ in Proposition~\ref{SCGA_Dir_conv} but also universally bounded away from zero with respect to the iteration number $t$, given a sufficiently large but fixed sample $n$ and a sufficiently small bandwidth $h$; recall our Remark~\ref{SCMS_stepsize_remark}. As a result, Corollary~\ref{Dir_SCMS_conv} follows from Proposition~\ref{SCGA_Dir_conv}. Notice that the statements in Proposition~\ref{Dir_SCMS_prop} are essentially the same as the results (a-b) in Corollary~\ref{Dir_SCMS_conv} here. However, similar to Proposition 2 in \cite{SCMS_conv2013} for the Euclidean SCMS algorithm, Proposition~\ref{Dir_SCMS_prop} for the directional SCMS algorithm is established under the convexity assumption on the directional kernel $L$ and holds for any sample size $n$ and bandwidth $h$. On the contrary, the results (a-b) in Corollary~\ref{Dir_SCMS_conv} are asymptotic and probabilistic properties, in which we require $\frac{nh^{q+6}}{|\log h|} \to \infty$ and $h\to 0$.
	
	According to Proposition~\ref{SCGA_Dir_conv} and  Corollary~\ref{Dir_SCMS_conv}, we can denote the limiting points of the population and sample-based SCGA algorithms on $\Omega_q$ by $\underline{\bm{x}}^* \in \underline{R}_d$ and $\hat{\underline{\bm{x}}}^* \in \hat{\underline{R}}_d$, respectively.
	The definition of the linear convergence of any converging sequence on $\Omega_q$ (or an arbitrary Riemannian manifold) is similar to the one in the flat Euclidean space $\mathbb{R}^D$ (see Definition~\ref{Q_Linear_conv}), except that the Euclidean distance is replaced with the geodesic distance on $\Omega_q$ in the definition; see Section 4.5 in \cite{Op_algo_mat_manifold2008}.
	
	Using the notation in \cite{Geo_Convex_Op2016}, we let $\zeta(\kappa, c) \equiv \frac{\sqrt{|\kappa|}c}{\tanh(\sqrt{|\kappa|}c)}$. Given that the sectional curvature is $\kappa=1$ on $\Omega_q$, we have $\zeta(1, c) = \frac{c}{\tanh(c)}$. One can show by differentiating $\zeta(1, c)$ that $\zeta(1, c)$ is strictly increasing with respect to $c$ and $\zeta(1, c) > 1$ for any $c>0$.  Analogous to the Euclidean SCGA algorithms, we will establish the linear convergence of the SCGA sequence $\big\{\underline{\bm{x}}^{(t)} \big\}_{t=0}^{\infty}$ on $\Omega_q$ (or any Riemannian manifold whose sectional curvature is lower bounded by a real number) as well as its sample-based version under the following local condition. 
	
	\begin{itemize}
		\item {\bf (\underline{A4})} (\emph{Quadratic Behaviors of Residual Vectors}) We assume that the SCGA sequence $\big\{\underline{\bm{x}}^{(t)} \big\}_{t=0}^{\infty}$ on $\Omega_q$ with step size $0 < \underline{\eta} \leq \min\left\{\frac{4}{\underline{\beta}_0}, \frac{1}{q\norm{\mathcal{H}f}_{\infty}^{(2)} \cdot \zeta(1,\underline{\rho})} \right\}$ and $\underline{\bm{x}}^* \in \underline{R}_d$ as its limiting point satisfies
		\begin{align*}
		\left\langle \underline{U}_d^{\perp}(\underline{\bm{x}}^{(t)}) \grad f(\underline{\bm{x}}^{(t)}),\, \Exp_{\underline{\bm{x}}^{(t)}}^{-1}(\underline{\bm{x}}^*) \right\rangle &\leq \frac{\underline{\beta}_0}{4} \cdot d_g\left(\underline{\bm{x}}^{(t)},\underline{\bm{x}}^* \right)^2,\\
		\norm{\underline{U}_d^{\perp}(\underline{\bm{x}}^{(t)}) \Exp_{\underline{\bm{x}}^{(t)}}^{-1}(\underline{\bm{x}}^*)}_2 &\leq \underline{\beta}_2 \cdot d_g\left(\underline{\bm{x}}^{(t)},\underline{\bm{x}}^* \right)^2
		\end{align*}
		for some constant $\underline{\beta}_2 >0$, where $\underline{\beta}_0 >0$ is the constant defined in condition (\underline{A2}) and $\Exp_{\underline{\bm{x}}^{(t)}}^{-1}(\underline{\bm{x}}^*) \in T_{\underline{\bm{x}}^{(t)}}$ is the logarithmic map.
	\end{itemize}
	
	Condition (\underline{A4}) serves as a generalization of its Euclidean counterpart condition (A4) to $\Omega_q$, which again requires a quadratic behavior of the residual vector $\underline{U}_d^{\perp}(\underline{\bm{x}}^{(t)}) \Exp_{\underline{\bm{x}}^{(t)}}^{-1}(\underline{\bm{x}}^*)$ within the tangent space $T_{\underline{\bm{x}}^{(t)}}$. Under this condition, the objective (density) function $f$ is ``subspace constrained geodesically strongly concave'' around the directional ridge $\underline{R}_d$; see also Remark~\ref{SC_remark2}. Some discussions about potentially weaker assumptions that imply condition (\underline{A4}) in Appendix~\ref{App:dis_A4} are also applicable in the manifold setting under some modifications; see Remark~\ref{quad_bound_prop_Dir}. One intuitive example that condition (\underline{A4}) holds is presented at the second row of Figure~\ref{fig:LC_plot_Dir}, where the directional SCMS/SCGA iterative vector $\Exp_{\underline{\bm{x}}^{(t)}}^{-1}(\underline{\bm{x}}^*)$ is always orthogonal to the residual space $\underline{U}_d^{\perp}(\bm{x}^{(t)})$ for all $t\geq 0$ around the (estimated) ridge on $\Omega_q$.

\begin{theorem}[Linear Convergence of the SCGA Algorithm on $\Omega_q$]
	\label{LC_SCGA_Dir}
	Assume conditions (\underline{A1-4}) throughout the theorem.
	\begin{enumerate}[label=(\alph*)]
		\item \textbf{Q-Linear convergence of $d_g(\underline{\bm{x}}^{(t)}, \underline{\bm{x}}^*)$}: Consider a convergence radius $\underline{r}_2>0$ satisfying 
			\begin{align*}
			0<\underline{r}_2 &\leq \min\Bigg\{\underline{\rho}/2, \frac{\underline{\beta}_1^2}{\underline{A}_2\left(\norm{f}_{\infty}^{(3)} + \norm{f}_{\infty}^{(4)} \right)}, \frac{\underline{\beta}_1}{\underline{A}_4(f)},\\
			& 2\sin\Bigg[\frac{3\underline{\beta}_0}{8q \left(12\norm{\mathcal{H}f}_{\infty}^{(2)}\underline{\beta}_2^2 \arcsin\left(\underline{\rho}/2 \right) + \sqrt{q}\norm{f}_{\infty}^{(3)} \right)}\Bigg] \Bigg\},
			\end{align*}
			where $\underline{A}_2>0$ is the constant defined in (h) of Lemma~\ref{Dir_norm_reach_prop} and $\underline{A}_4(f) >0$ is a quantity defined in (c) of Proposition~\ref{SCGA_Dir_conv} that depends on both the dimension $q$ and the functional norm $\norm{f}_{\infty,4}^*$ up to the fourth-order (partial) derivatives of $f$. Whenever $0 < \underline{\eta} \leq \min\left\{\frac{4}{\underline{\beta}_0}, \frac{1}{q\norm{\mathcal{H}f}_{\infty}^{(2)} \cdot \zeta(1,\underline{\rho})} \right\}$ and the initial point $\underline{\bm{x}}^{(0)} \in \text{Ball}_{q+1}(\underline{\bm{x}}^*, \underline{r}_2) \cap \Omega_q$ with $\underline{\bm{x}}^* \in \underline{R}_d$, we have that
		$$d_g(\underline{\bm{x}}^{(t)}, \underline{\bm{x}}^*) \leq \underline{\Upsilon}^t \cdot d_g(\underline{\bm{x}}^{(0)}, \underline{\bm{x}}^*) \quad \text{ with } \quad \underline{\Upsilon}= \sqrt{1-\frac{\underline{\beta}_0\underline{\eta}}{4}}.$$
		\item \textbf{R-Linear convergence of $d_g(\underline{\bm{x}}^{(t)}, \underline{R}_d)$}: Under the same radius $\underline{r}_2>0$ in (a), we have that whenever $0 < \underline{\eta} \leq \min\left\{\frac{4}{\underline{\beta}_0}, \frac{1}{q\norm{\mathcal{H}f}_{\infty}^{(2)} \cdot \zeta(1,\underline{\rho})} \right\}$ and the initial point $\underline{\bm{x}}^{(0)} \in \text{Ball}_{q+1}(\underline{\bm{x}}^*, \underline{r}_2) \cap \Omega_q$ with $\underline{\bm{x}}^* \in \underline{R}_d$,
		$$d_g(\underline{\bm{x}}^{(t)}, \underline{R}_d) \leq \underline{\Upsilon}^t \cdot d_g(\underline{\bm{x}}^{(0)}, \underline{\bm{x}}^*) \quad \text{ with } \quad \underline{\Upsilon}=\sqrt{1-\frac{\underline{\beta}_0\underline{\eta}}{4}}.$$
	\end{enumerate}
	We further assume (D1-2) in the rest of statements. Suppose that $h \to 0$ and $\frac{nh^{q+4}}{|\log h|} \to \infty$.
	\begin{enumerate}[label=(c)]
		\item \textbf{Q-Linear convergence of $d_g(\hat{\underline{\bm{x}}}^{(t)}, \underline{\bm{x}}^*)$}: Under the same radius $\underline{r}_2>0$ and $\underline{\Upsilon}=\sqrt{1-\frac{\underline{\beta}_0\underline{\eta}}{4}}$ in (a), we have that
		$$d_g(\hat{\underline{\bm{x}}}^{(t)}, \underline{\bm{x}}^*) \leq \underline{\Upsilon}^t \cdot d_g(\hat{\underline{\bm{x}}}^{(0)}, \underline{\bm{x}}^*) + O(h^2) + O_P\left(\sqrt{\frac{|\log h|}{nh^{q+4}}} \right)$$
		with probability tending to 1 whenever $0 < \underline{\eta} \leq \min\left\{\frac{4}{\underline{\beta}_0}, \frac{1}{q\norm{\mathcal{H}f}_{\infty}^{(2)} \cdot \zeta(1,\underline{\rho})} \right\}$ and the initial point $\hat{\underline{\bm{x}}}^{(0)} \in \text{Ball}_{q+1}(\underline{\bm{x}}^*, \underline{r}_2) \cap \Omega_q$ with $\underline{\bm{x}}^* \in \underline{R}_d$.
	\end{enumerate}
	\begin{enumerate}[label=(d)]
		\item \textbf{R-Linear convergence of $d_g(\hat{\underline{\bm{x}}}^{(t)}, \underline{R}_d)$}: Under the same radius $\underline{r}_2>0$ and $\underline{\Upsilon}=\sqrt{1-\frac{\underline{\beta}_0 \underline{\eta}}{4}}$ in (a), we have that
		$$d_g(\hat{\underline{\bm{x}}}^{(t)}, \underline{R}_d) \leq \underline{\Upsilon}^t \cdot d_g(\hat{\underline{\bm{x}}}^{(0)}, \underline{\bm{x}}^*) + O(h^2) + O_P\left(\sqrt{\frac{|\log h|}{nh^{q+4}}} \right)$$
		with probability tending to 1 whenever $0 < \underline{\eta} \leq \min\left\{\frac{4}{\underline{\beta}_0}, \frac{1}{q\norm{\mathcal{H}f}_{\infty}^{(2)} \cdot \zeta(1,\underline{\rho})} \right\}$ and the initial point $\hat{\underline{\bm{x}}}^{(0)} \in \text{Ball}_{q+1}(\underline{\bm{x}}^*, \underline{r}_2) \cap \Omega_q$ with $\underline{\bm{x}}^* \in \underline{R}_d$.
	\end{enumerate}
\end{theorem}

The detailed proof of Theorem~\ref{LC_SCGA_Dir} is in Appendix~\ref{App:Proofs_Dir_SCMS}. The theorem illuminates both the step size requirement and the convergence radius $\underline{r}_2 >0$ for the linear convergence of SCGA algorithms on $\Omega_q$. 
Similar to Euclidean SCGA algorithms in Theorem~\ref{SCGA_LC}, the upper bound of the convergence radius $\underline{r}_2$ consists of the three quantities adopted from Proposition~\ref{SCGA_Dir_conv} and a quantity controlling the ``subspace constrained geodesically strong concavity'' around the directional ridge $\underline{R}_d$.

\begin{remark}
	\label{SC_remark2}
	Similar to Euclidean SCGA algorithms, the geodesically strong concavity assumption \citep{Geo_Convex_Op2016} on the objective function $f$ is not sufficient to prove the linear convergence of the SCGA algorithm \eqref{SCGA_manifold_update} on $\Omega_q$. We instead establish the following ``subspace constrained geodesically strong concavity'' under some mild conditions (\underline{A1-4}):
	\begin{equation}
	\label{Sub_con_SC_Dir}
	f(\underline{\bm{x}}^*) - f(\bm{y}) \leq \left\langle \underline{V}_d(\bm{y}) \underline{V}_d(\bm{y})^T \grad f(\bm{y}), \Exp_{\bm{y}}^{-1}(\underline{\bm{x}}^*) \right\rangle -A_8 \cdot d_g(\underline{\bm{x}}^*, \bm{y})^2 + o\left(d_g(\underline{\bm{x}}^*, \bm{y})^2 \right)
	\end{equation}
	for some constant $A_8>0$, where $\bm{y}$ is generally chosen to be $\underline{\bm{x}}^{(t)}$. In fact, the most critical factors for establish this property is the eigengap condition (\underline{A2}) and the quadratic behaviors of residual vectors stated in condition (\underline{A4}).
\end{remark}

	\begin{corollary}[Linear Convergence of the Directional SCMS Algorithm]
		\label{LC_Dir_SCMS}
		Assume conditions (\underline{A1-4}) and (D1-2). When the fixed sample size $n$ is sufficiently large and the fixed bandwidth is chosen to be sufficiently small, there exists a convergence radius $\underline{r}_3 \in \left(0,\underline{r}_2 \right)$ such that the directional SCMS sequence $\left\{\hat{\underline{\bm{x}}}^{(t)} \right\}_{t=0}^{\infty}$ satisfies
		\begin{align*}
		&d_g\left(\hat{\underline{\bm{x}}}^{(t)}, \hat{\underline{R}}_d \right) \leq d_g\left(\hat{\underline{\bm{x}}}^{(t)},\hat{\underline{\bm{x}}}^* \right) \leq \underline{\Upsilon}_{n,h}^t \cdot  d_g\left(\hat{\underline{\bm{x}}}^{(t)}, \hat{\underline{\bm{x}}}^* \right) \\ 
		&\text{ with }\quad  \underline{\Upsilon}_{n,h} = \sqrt{1- \frac{\underline{\beta}_0 \tilde{\underline{\eta}}_{n,h}}{4}} \text{ and } \tilde{\underline{\eta}}_{n,h} =\inf_t \underline{\eta}_{n,h}^{(t)'}
		\end{align*}
		with high probability whenever $0 < \sup_t \underline{\eta}_{n,h}^{(t)'} \leq \min\left\{\frac{4}{\underline{\beta}_0}, \frac{1}{q\norm{\mathcal{H}f}_{\infty}^{(2)} \cdot \zeta(1,\underline{\rho})} \right\}$ and the initial point $\hat{\underline{\bm{x}}}^{(0)} \in \left(\hat{\underline{R}}_d \oplus \underline{r}_3\right) \cap \Omega_q$.
	\end{corollary}
	
	We also identify Corollary~\ref{LC_Dir_SCMS} as the linear convergence of the sample-based SCGA algorithm on $\Omega_q$ to the estimated directional ridge $\hat{\underline{R}}_d$ defined by the directional KDE $\hat{f}_h$. The corollary can be justified by noticing that, under conditions (D1-2) and the uniform bounds \eqref{Dir_KDE_unif_bound}, $\hat{f}_h$ satisfies conditions (\underline{A1-3}) with probability tending to 1 as $h\to 0$ and $\frac{nh^{q+6}}{|\log h|} \to \infty$; see Theorem~\ref{Dir_ridge_stability}. With this fact, one can leverage our argument in (a) of Theorem~\ref{LC_SCGA_Dir} to prove the linear convergence of the sample-based SCGA algorithm on $\Omega_q$ with a fixed step size $\underline{\eta}$ satisfying $0<\underline{\eta} \leq \min\left\{\frac{4}{\underline{\beta}_0}, \frac{1}{q\norm{\mathcal{H}f}_{\infty}^{(2)} \cdot \zeta(1,\underline{\rho})} \right\}$. Additionally, when the fixed sample size $n$ is sufficiently large and the bandwidth is chosen to be accordingly small, the adaptive step size $\underline{\eta}_{n,h}^{(t)'}$ of our directional SCMS algorithm in \eqref{Dir_SCMS_stepsize} always falls below the threshold value $\min\left\{\frac{4}{\underline{\beta}_0}, \frac{1}{q\norm{\mathcal{H}f}_{\infty}^{(2)} \cdot \zeta(1,\underline{\rho})} \right\}$ for linear convergence by Lemma~\ref{norm_tot_grad} but is also bounded away from zero; recall Remark~\ref{SCMS_stepsize_remark}. Taking the infimum of $\underline{\eta}_{n,h}^{(t)'}$ with respect to the iteration number $t$ under a fixed $n$ and $h$ yields our results in Corollary~\ref{LC_Dir_SCMS}.

\section{Experiments}
\label{App:Experiments}

In this section, we first validate our linear convergence results of both Euclidean and directional SCMS algorithms on some simulated datasets. Then, we apply these two algorithms to a real-world earthquake dataset so as to identify its density ridges and compare the estimated ridges with boundaries of tectonic plates and fault lines, on which earthquakes are known to happen frequently. 

We leverage the Gaussian kernel profile $k_N(x)=\exp\left(-\frac{x}{2}\right)$ in the Euclidean SCMS algorithm and the von Mises kernel $L(r)=e^{-r}$ in the directional SCMS algorithm. In addition, the logarithms of the estimated densities are utilized in our actual implementations (Step 2 in Algorithms~\ref{Algo:SCMS} and \ref{Algo:Dir_SCMS} in Appendix~\ref{App:Algo}) of the Euclidean and directional SCMS algorithms because of two advantages. First, using the log-density in the Euclidean SCMS algorithm leads to a faster convergence process \citep{SCMS_conv2013}; see our empirical illustration in Figure~\ref{fig:SCMS_log_comp}. Second, estimating a hidden manifold with a density ridge defined by a log-density stabilizes the valid region for a well-defined ridge compared to the corresponding ridge defined by the original density; see Theorem 7 (Surrogate theorem) in \cite{Non_ridge_est2014}.

Unless stated otherwise, we set the default bandwidth parameter of the Euclidean SCMS algorithm to the normal reference rule in \cite{Asymp_deri_KDE2011,Mode_clu2016}, which is
\begin{equation}
\label{NR_rule}
h_{\text{NR}} = \bar{S}_n \times \left(\frac{4}{D+4} \right)^{\frac{1}{D+6}} n^{-\frac{1}{D+6}}, \quad \bar{S}_n=\frac{1}{D}\sum_{j=1}^D S_{n,j},
\end{equation}
where $S_{n,j}$ is the sample standard deviation along $j$-th coordinate and $D$ is the (Euclidean) dimension of the data in $\mathbb{R}^D$. As mentioned by \cite{Mode_clu2016}, there are two advantages of applying the normal reference rule \eqref{NR_rule} in our context. First, the KDE $\hat{p}_n$ under $h_{\text{NR}}$ tends to be oversmoothing \citep{Sheather2004}, because the bandwidth minimizes the asymptotic MISE for estimating the first-order derivatives of a multivariate Gaussian distribution with covariance matrix $\sigma^2 \bm{I}_D$; see Corollary 4 in \cite{Asymp_deri_KDE2011}. More importantly, the Euclidean SCMS algorithm with an oversmoothed KDE $\hat{p}_n$ would not produce too many spurious ridges. Second, compared to cross validation methods, $h_{\text{NR}}$ is easy to compute in practice, especially when the dimension of data is high. The default bandwidth parameter of the directional SCMS algorithm is selected via the rule of thumb in Proposition 2 of \cite{Exact_Risk_bw2013}, which optimizes the asymptotic MISE for a $\text{vMF}(\bm{\mu},\nu)$ distribution. The concentration parameter $\nu$ is estimated by Equation (4.4) in \cite{spherical_EM}. That is, 
\begin{equation}
\label{bw_ROT}
h_{\text{ROT}} = \left[\frac{4\pi^{\frac{1}{2}} \mathcal{I}_{\frac{q-1}{2}}(\hat{\nu})^2}{\hat{\nu}^{\frac{q+1}{2}}\left[2 q\cdot\mathcal{I}_{\frac{q+1}{2}}(2\hat{\nu}) + (q+2)\hat{\nu} \cdot \mathcal{I}_{\frac{q+3}{2}}(2\hat{\nu}) \right]n} \right]^{\frac{1}{q+4}}, \quad \hat{\nu} = \frac{\bar{R}(q+1-\bar{R})}{1-\bar{R}^2},
\end{equation}
where $\bar{R}=\frac{\norm{\sum_{i=1}^n \bm{X}_i}_2}{n}$ given the directional dataset $\left\{\bm{X}_1,...,\bm{X}_n \right\} \subset \Omega_q \subset \mathbb{R}^{q+1}$ and we recall that $\mathcal{I}_{\alpha}(\nu)$ is the modified Bessel function of the first kind of order $\nu$. As $q$-von Mises-Fisher distribution behaves as the Gaussian distribution on $\Omega_q$, choosing the bandwidth \eqref{bw_ROT} also helps smooth out the resulting directional KDE. The tolerance level is always set to be $\epsilon=10^{-9}$ for any SCMS algorithm.

\subsection{Simulation Study on the Euclidean SCMS Algorithm}
\label{Sec:EuSCMS_exp}

To evaluate the algorithmic rate of convergence of the Euclidean SCMS algorithm (Algorithm~\ref{Algo:SCMS}), we generate the first simulated dataset by randomly drawing 1000 data points from a Gaussian mixture model with density $0.4 \cdot N(\bm{\mu}_1, \Sigma_1) + 0.6 \cdot N(\bm{\mu}_2, \Sigma_2)$, where $\bm{\mu}_1=-\bm{\mu}_2 = (1,1)^T$, $\Sigma_1=\Diag\left(\frac{1}{4},\frac{1}{4} \right)$, and $\Sigma_2 = \begin{pmatrix}
\frac{1}{2} & \frac{1}{4}\\
\frac{1}{4} & \frac{1}{2}
\end{pmatrix}$. Another simulated dataset consists of 1000 data points randomly generated from an upper half circle with radius 2 and i.i.d. Gaussian noises $N(0,0.3^2)$. When applying Algorithm~\ref{Algo:SCMS} with the estimated log-density on each of these two simulated datasets, we choose the set of initial mesh points as the simulated dataset itself and remove those initial points whose density values are below 25\% of the maximum density from the set of mesh points in order to obtain a cleaner ridge structure.

Figure~\ref{fig:LC_plot} presents the Euclidean KDE plots, estimated density ridges from the Euclidean SCMS algorithm, and their (linear) convergence plots on the two simulated datasets. The linear trends of those plots in the second and third columns of Figure~\ref{fig:LC_plot} empirically demonstrate the correctness of our Theorem~\ref{SCGA_LC} and Corollary~\ref{SCMS_LC} about the linear convergence of the Euclidean SCMS algorithm. 

\begin{figure}[t]
	\captionsetup[subfigure]{justification=centering}
	\centering
	\begin{subfigure}[t]{.32\textwidth}
		\centering
		\includegraphics[width=1\linewidth]{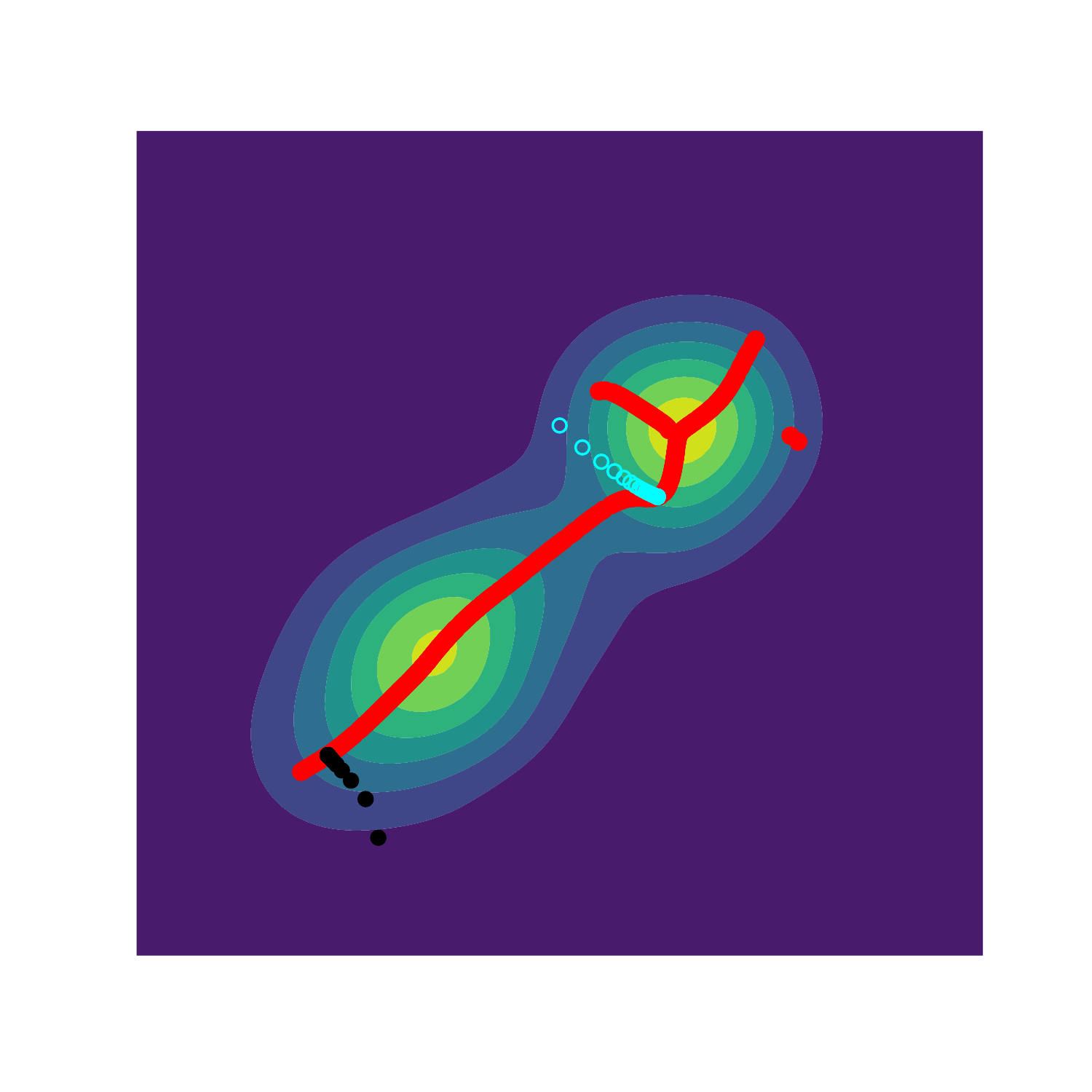}
	\end{subfigure}
	\hfil
	\begin{subfigure}[t]{.32\textwidth}
		\centering
		\includegraphics[width=1\linewidth, height=4.5cm]{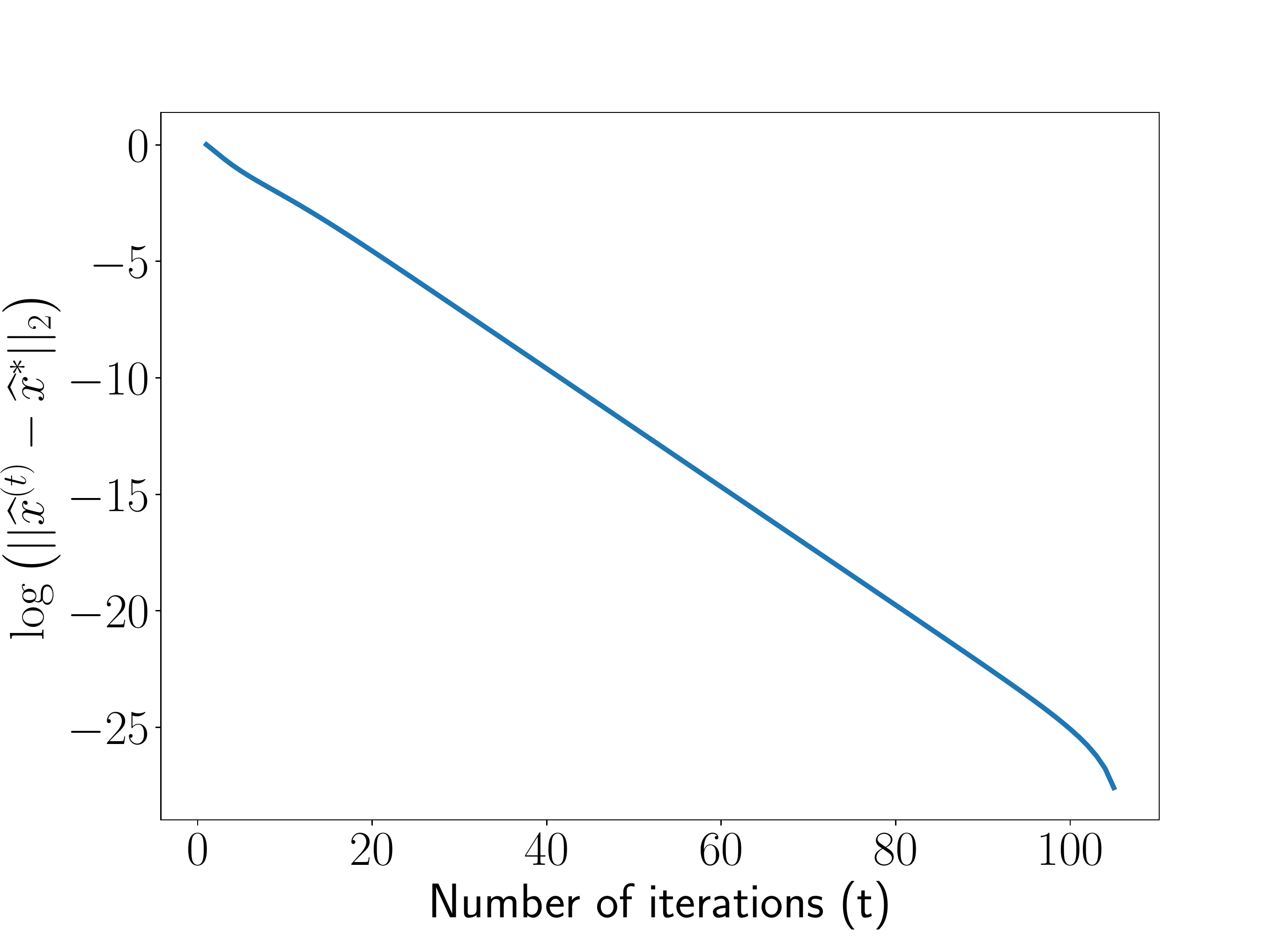}
	\end{subfigure}
	\hfil
	\begin{subfigure}[t]{.32\textwidth}
		\centering
		\includegraphics[width=1\linewidth,height=4.5cm]{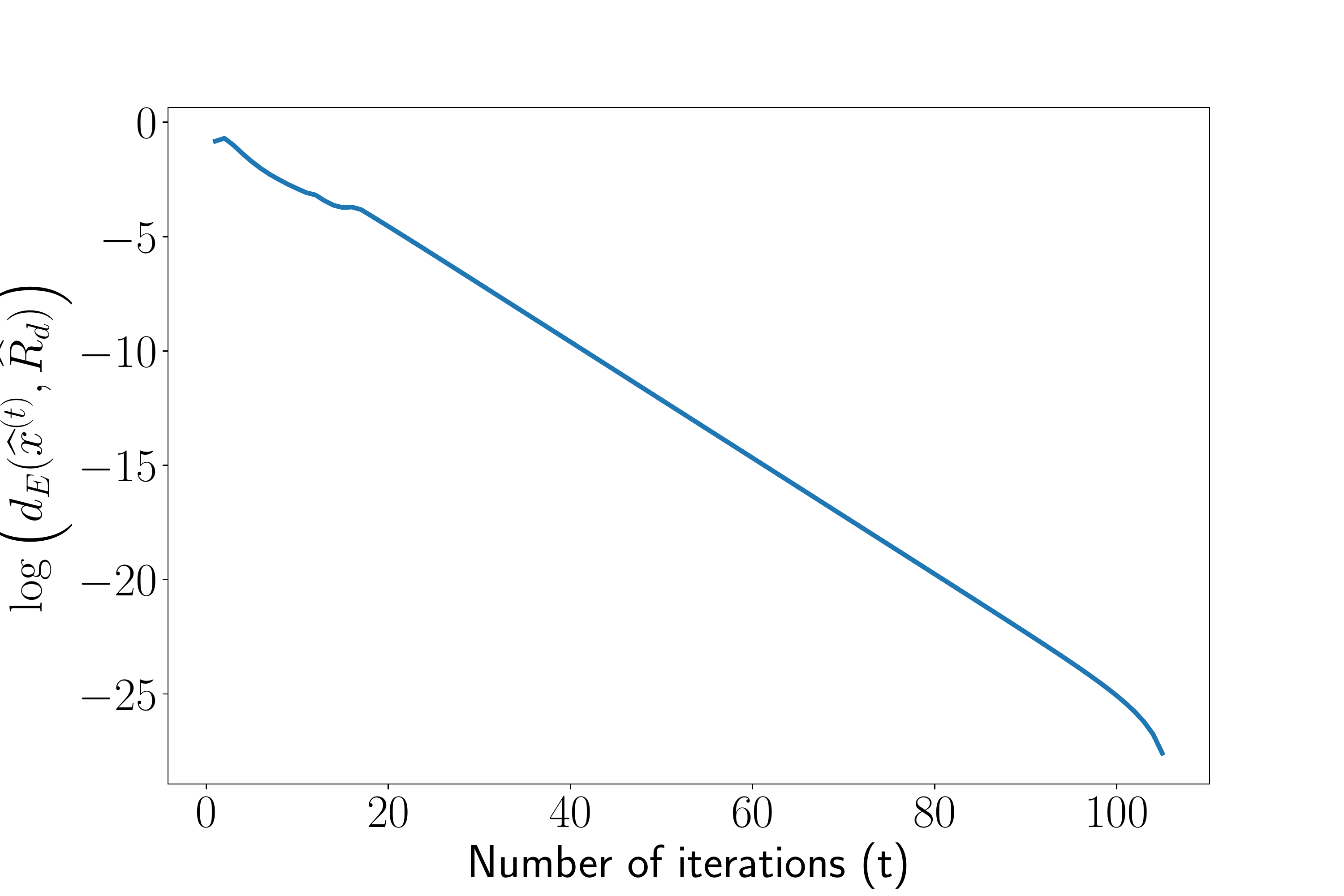}
	\end{subfigure}
	\begin{subfigure}[t]{.32\textwidth}
		\centering
		\includegraphics[width=1\linewidth]{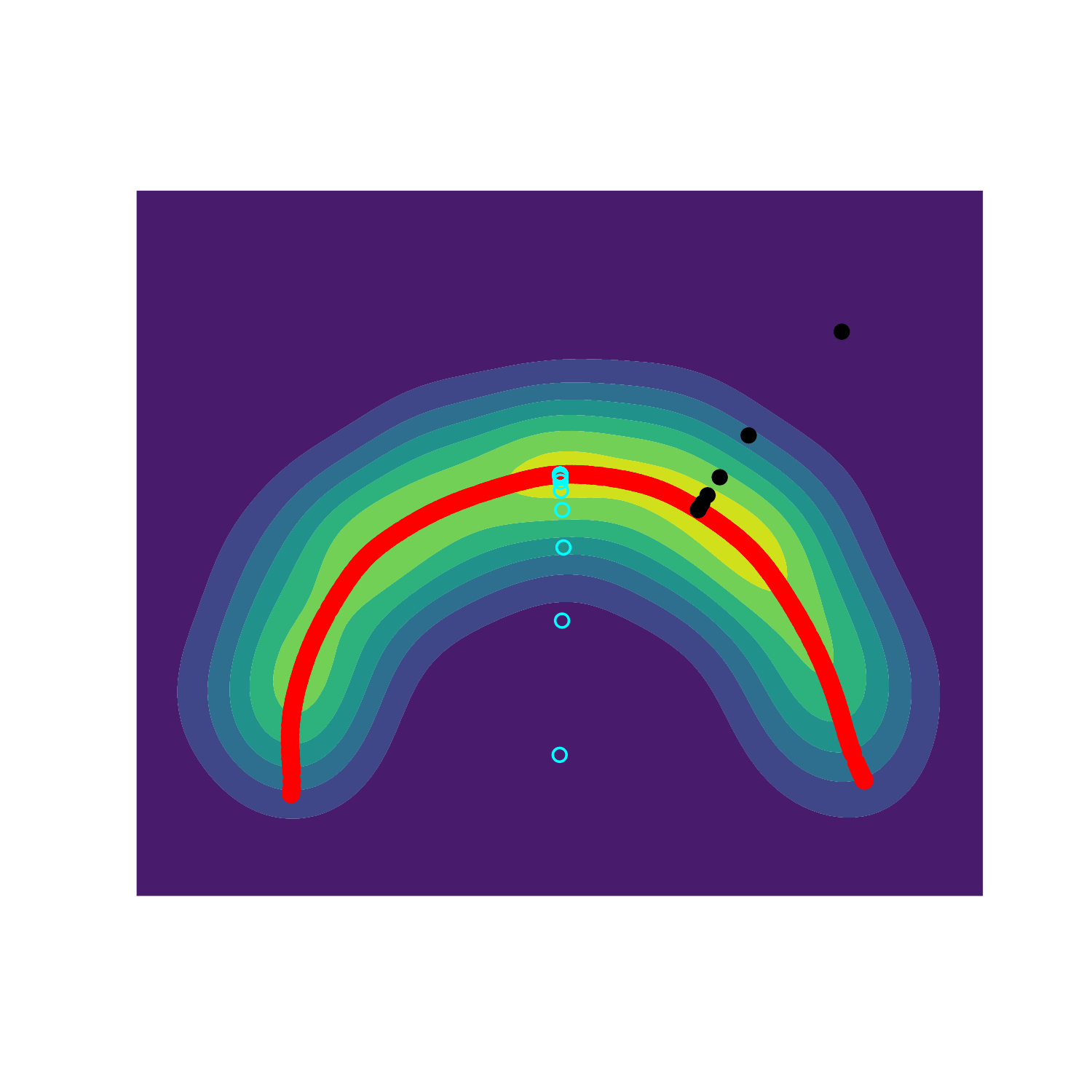}
	\end{subfigure}%
	\hfil
	\begin{subfigure}[t]{.32\textwidth}
		\centering
		\includegraphics[width=1\linewidth,height=4.5cm]{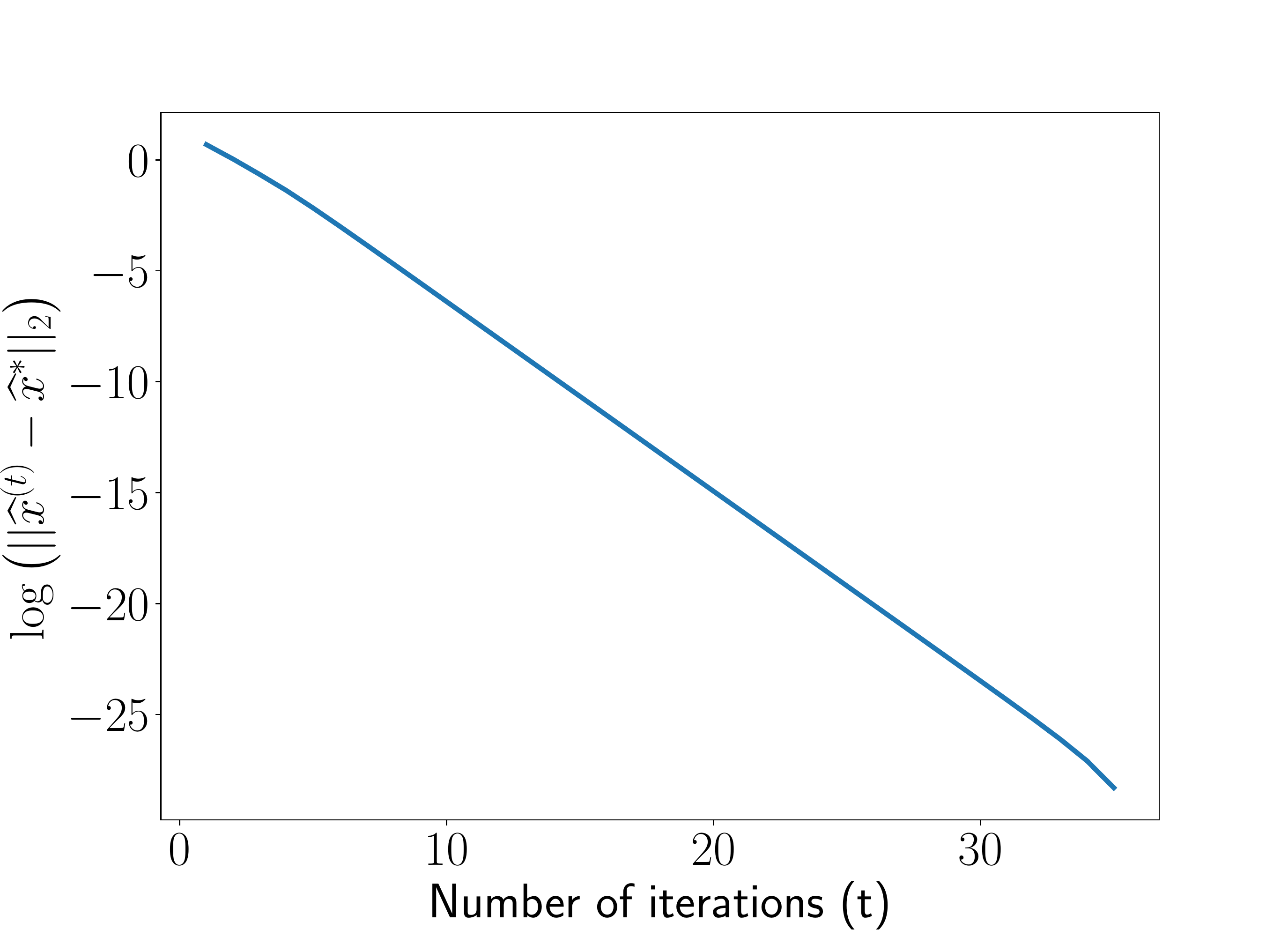}
	\end{subfigure}
	\hfil
	\begin{subfigure}[t]{.32\textwidth}
		\centering
		\includegraphics[width=1\linewidth,height=4.5cm]{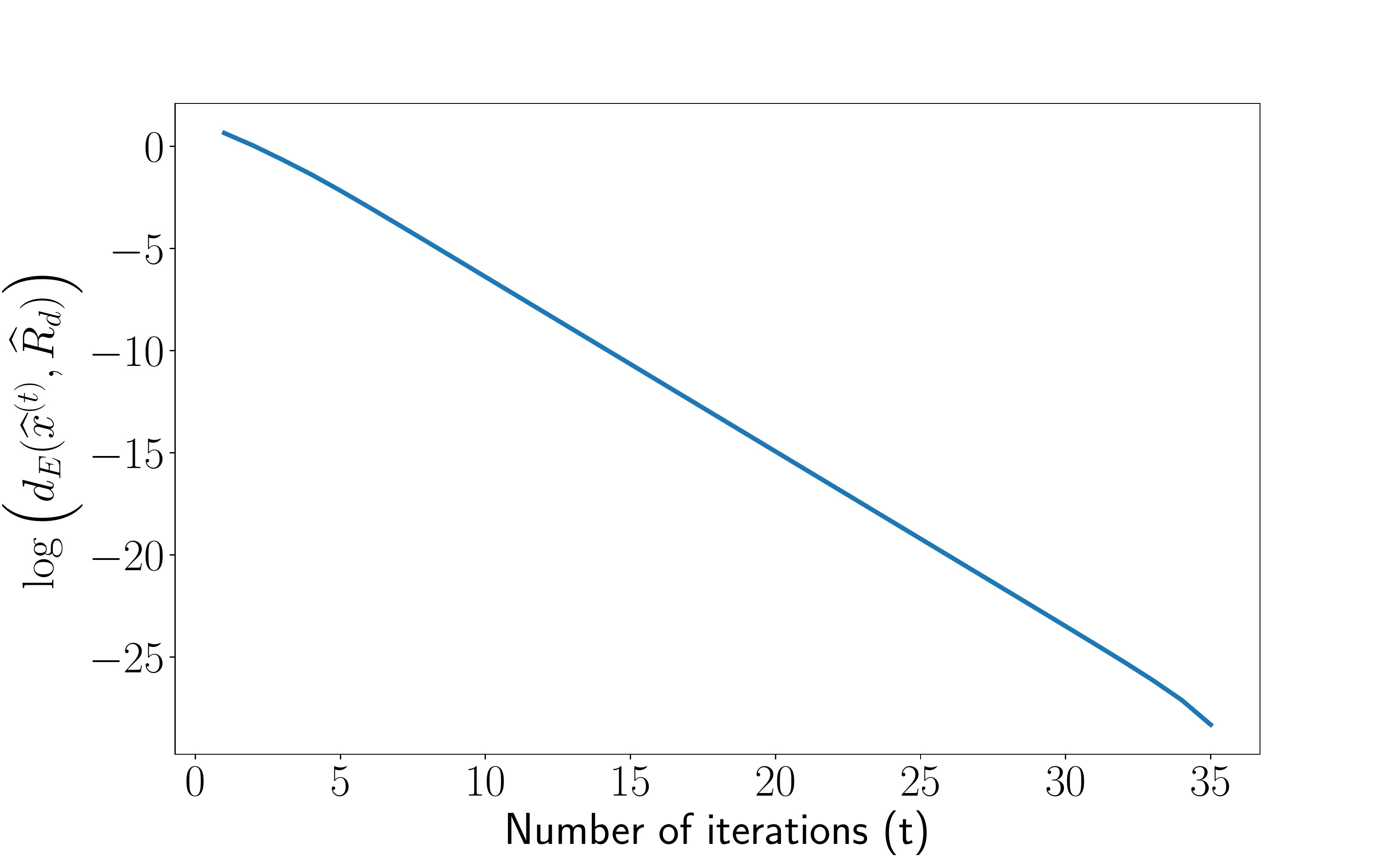}
	\end{subfigure}
	\caption{Density ridges estimated by the Euclidean SCMS algorithm on the two simulated datasets and their (linear) convergence plots. Horizontally, the first row displays the results of the simulated Gaussian mixture dataset, while the second row presents the results of the half circle simulated dataset. Vertically, the first column includes plots with Euclidean KDE, estimated ridges, and trajectories of SCMS sequences from two (randomly) chosen initial points. The second and third columns present the (linear) convergence plots for the log-distances of points in the highlighted sequences (indicated by hollow cyan points) to their limiting points or the estimated ridges.}
	\label{fig:LC_plot}
\end{figure}

\subsection{Simulation Study on the Directional SCMS Algorithm}
\label{Sec:DirSCMS_exp}

Analogous to our simulation study for the linear convergence of the Euclidean SCMS algorithm, we verify the linear convergence of our directional SCMS algorithm (Algorithm~\ref{Algo:Dir_SCMS}) on two different simulated datesets. One of them comprises 1000 data points randomly generated from a vMF mixture model $0.4 \cdot \text{vMF}\left(\underline{\bm{\mu}}_1, \nu_1 \right) + 0.6 \cdot \text{vMF}\left(\underline{\bm{\mu}}_2, \nu_2 \right)$ with $\underline{\bm{\mu}}_1=(0,0,1)^T \in \Omega_2 \subset \mathbb{R}^3$, $\underline{\bm{\mu}}_2=(1,0,0)^T \in \Omega_2 \subset \mathbb{R}^3$, and $\nu_1=\nu_2=10$. The other simulated dataset is identical to the example in the right panel of Figure~\ref{fig:ridges_Eu_Dir} and the underlying dataset in Figure~\ref{fig:add_Cart}, which consists of 1000 randomly sampled points from a circle connecting two poles on $\Omega_2$ with i.i.d. additive Gaussian noises $N(0, 0.2^2)$ to their Cartesian coordinates and additional $L_2$ normalization onto $\Omega_2$. In our implementation of Algorithm~\ref{Algo:Dir_SCMS} with the directional log-density on the two simulated datasets, we also set each initial mesh as the dataset itself and remove those points whose density values are below 10\% of the maximal density value from each set of mesh points.

Figure~\ref{fig:LC_plot_Dir} shows the directional KDE plots, estimated density ridges on $\Omega_2$ from the directional SCMS algorithm, and their (linear) convergence plots on the aforementioned simulated datasets. Those linear decreasing trends in the convergence plots, possibly after several pilot iterations, illustrate the locally linear convergence of the directional SCMS algorithm that we proved in Theorem~\ref{LC_SCGA_Dir} and Corollary~\ref{LC_Dir_SCMS}. Note that those minor perturbations at the tails of some linear convergence plots in Figure~\ref{fig:LC_plot_Dir} are due to precision errors.

\begin{figure}[t]
	\captionsetup[subfigure]{justification=centering}
	\centering
	\begin{subfigure}[t]{.32\textwidth}
		\centering
		\includegraphics[width=1\linewidth]{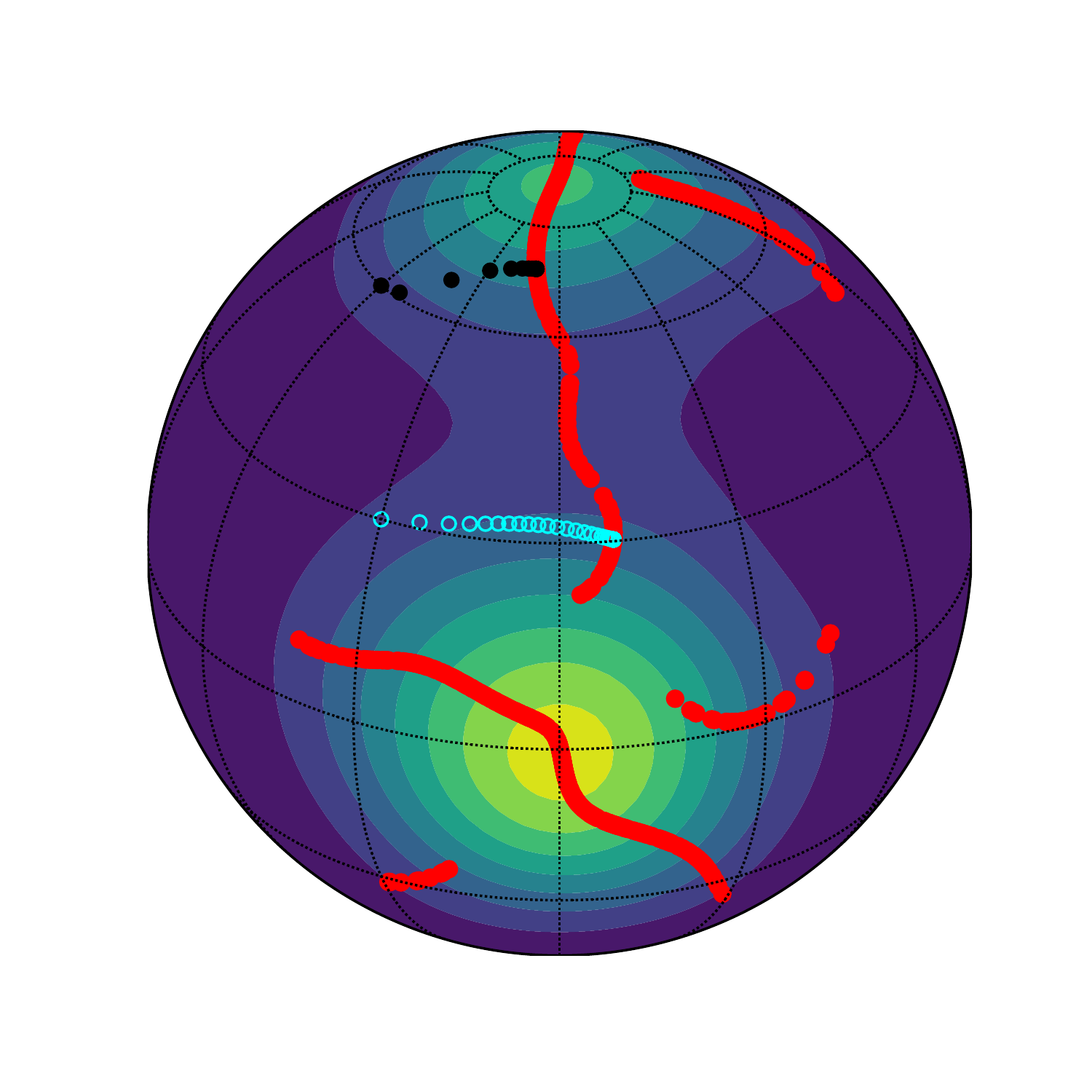}
	\end{subfigure}
	\hfil
	\begin{subfigure}[t]{.32\textwidth}
		\centering
		\includegraphics[width=1\linewidth, height=4.5cm]{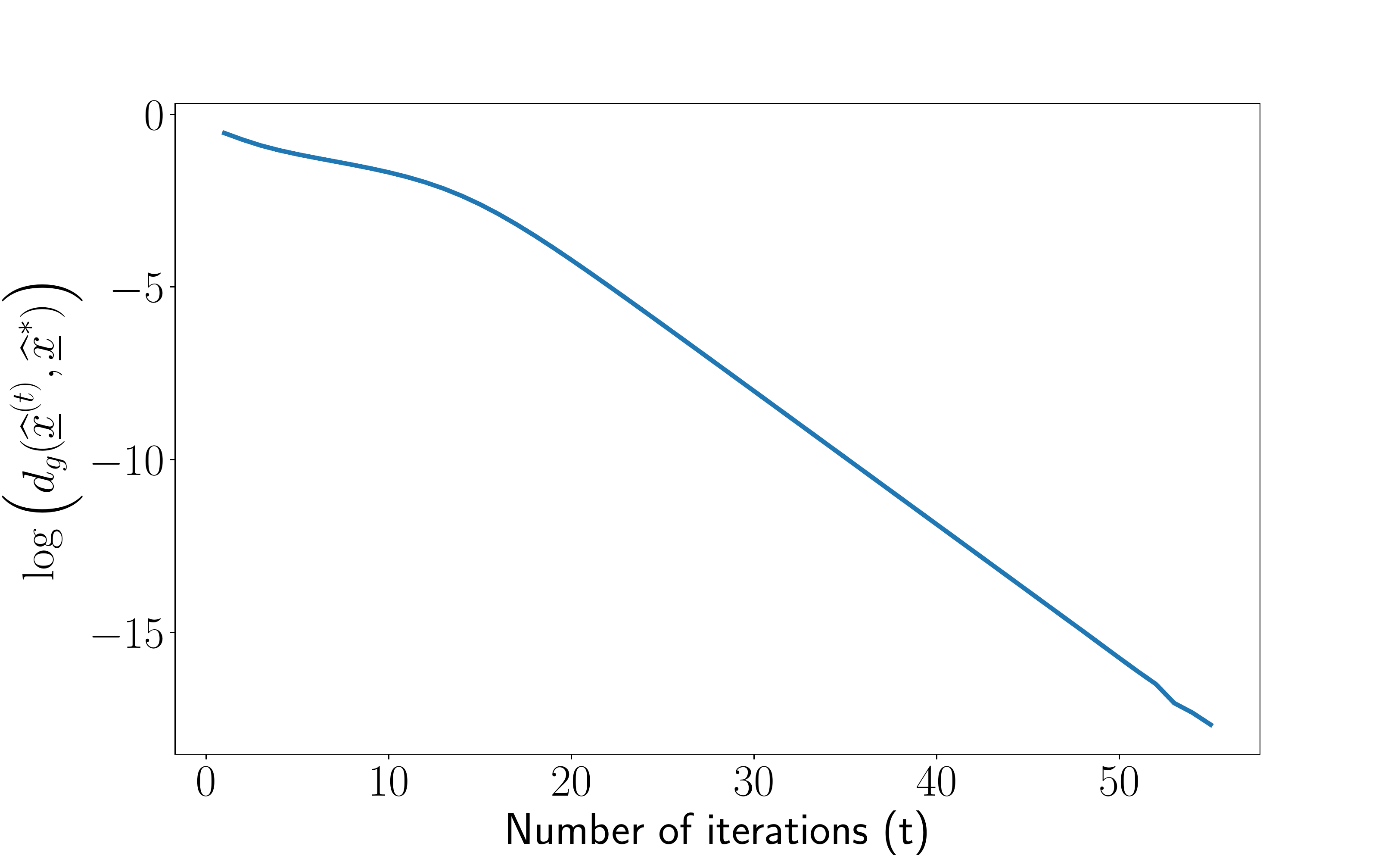}
	\end{subfigure}
	\hfil
	\begin{subfigure}[t]{.32\textwidth}
		\centering
		\includegraphics[width=1\linewidth,height=4.5cm]{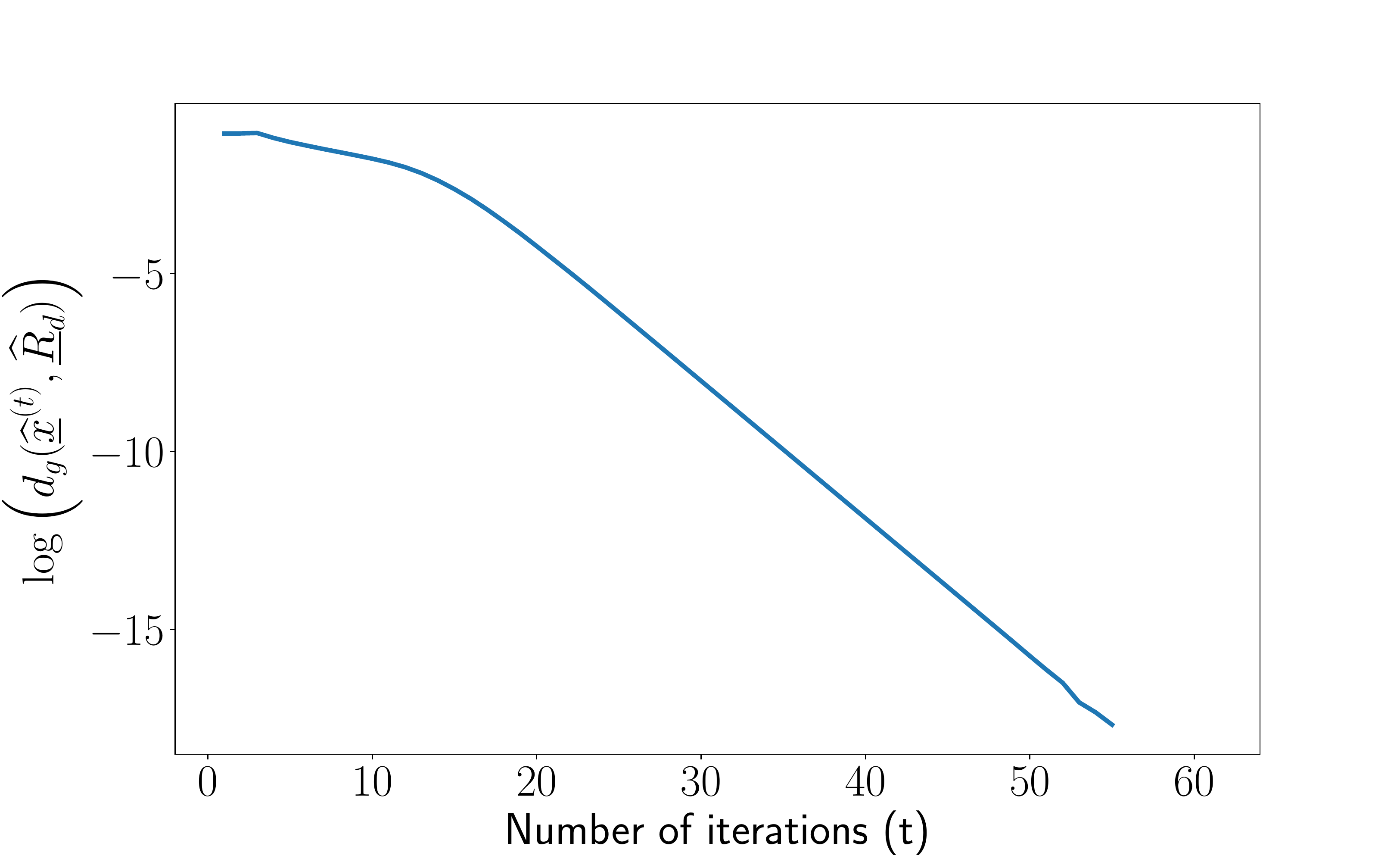}
	\end{subfigure}
	\begin{subfigure}[t]{.32\textwidth}
		\centering
		\includegraphics[width=\linewidth]{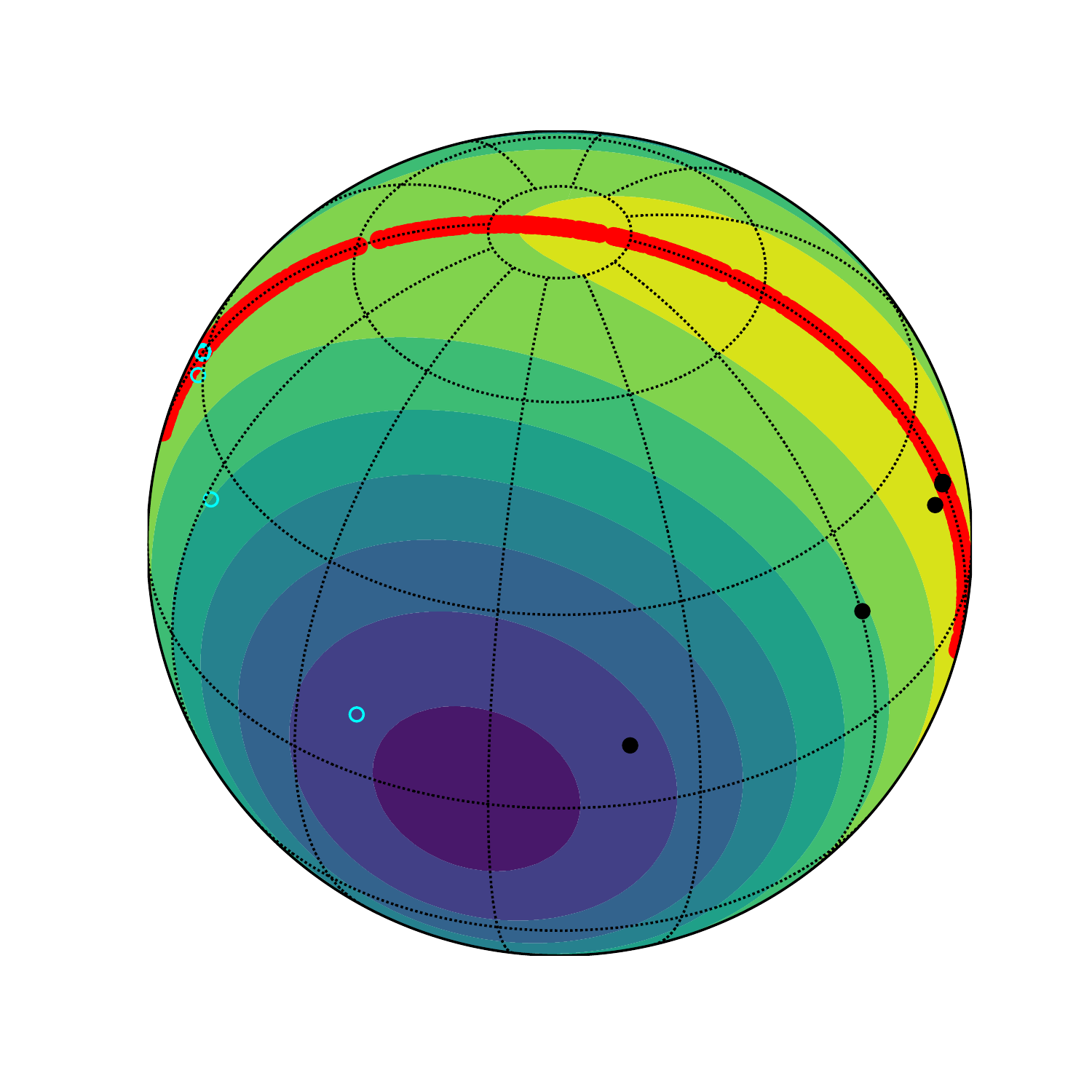}
	\end{subfigure}%
	\hfil
	\begin{subfigure}[t]{.32\textwidth}
		\centering
		\includegraphics[width=\linewidth,height=4.5cm]{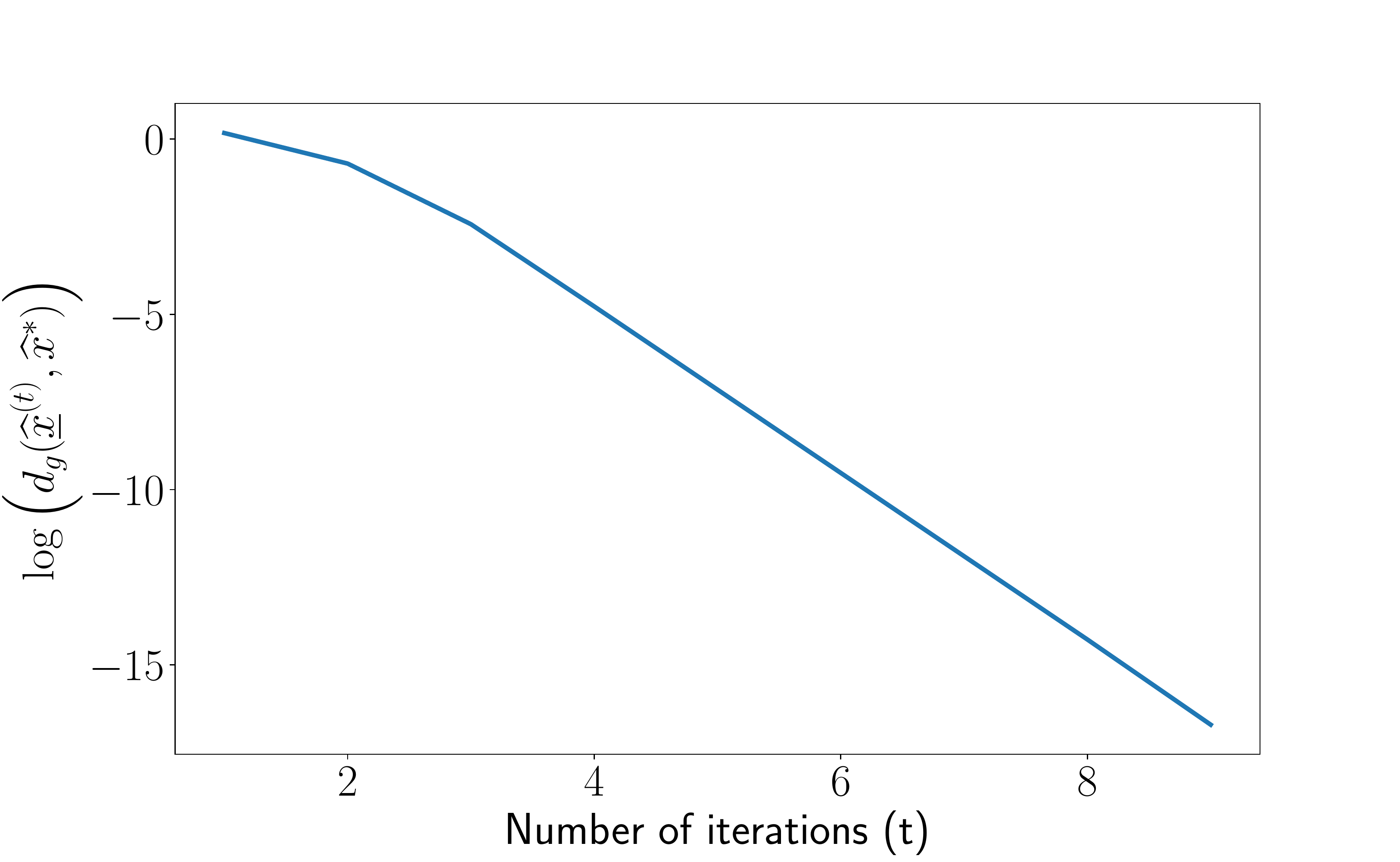}
	\end{subfigure}
	\hfil
	\begin{subfigure}[t]{.32\textwidth}
		\centering
		\includegraphics[width=\linewidth,height=4.5cm]{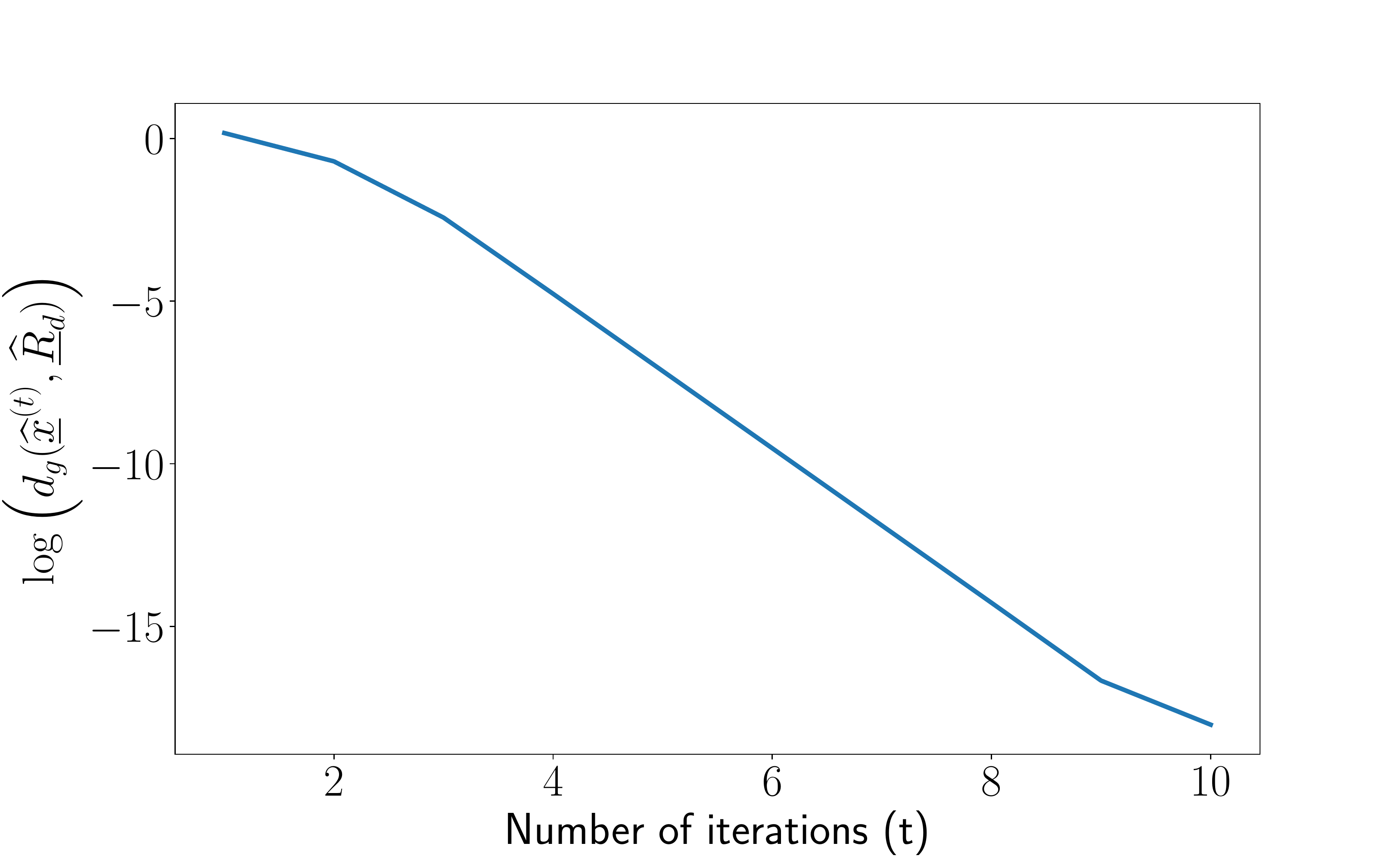}
	\end{subfigure}
	\caption{Density ridges estimated by the directional SCMS algorithm performed on the two simulated datasets and their (linear) convergence plots. Horizontally, the first row displays the results on the simulated vMF mixture dataset, while the second row presents the results on the circular simulated dataset on $\Omega_2$. Vertically, the first column includes plots with directional KDE, estimated ridges, and trajectories of directional SCMS sequences from two (randomly) chosen initial points on $\Omega_2$. The second and third columns present the convergence plots for the log-distances of points in the highlighted sequences (indicated by hollow cyan points) to their limiting points or the estimated ridges on $\Omega_2$.}
	\label{fig:LC_plot_Dir}
\end{figure}

\subsection{Density Ridges on Earthquake Data}

It is well-known that earthquakes on Earth tend to strike more frequently along the boundaries of tectonic plates and fault lines (\emph{i.e.}, sections of a plate or two plates are moving in different directions); see \cite{subarya2006plate,harris2017large} for more details. We analyze earthquakes with magnitudes of 2.5+ occurring between 2020-10-01 00:00:00 UTC and 2021-03-31 23:59:59 UTC, which can be obtained from the Earthquake Catalog (\url{https://earthquake.usgs.gov/earthquakes/search/}) of the United States Geological Survey. The dataset $\mathcal{D}$ contains 15049 earthquakes worldwide in this half-year period.

The normal reference rule \eqref{NR_rule} leads to the bandwidth parameter $h_{\text{NR}} \approx 16.0035$ and the rule of thumb \eqref{bw_ROT} yields $h_{\text{ROT}} \approx 0.1584$ under the earthquake dataset $\mathcal{D}$.
However, as these bandwidths lead to oversmoothing density estimates,
we decrease the bandwidths for the Euclidean and directional SCMS algorithms to $h_{\text{Eu}}=7.0$ and $h_{\text{Dir}}=0.1$ respectively in order to detect more ridge structures. We generate 5000 points uniformly on the sphere $\Omega_2$ as the initial mesh points.

\begin{figure}[t]
	\captionsetup[subfigure]{justification=centering}
	\begin{subfigure}[t]{.49\textwidth}
		\centering
		\includegraphics[width=\linewidth,height=4.5cm]{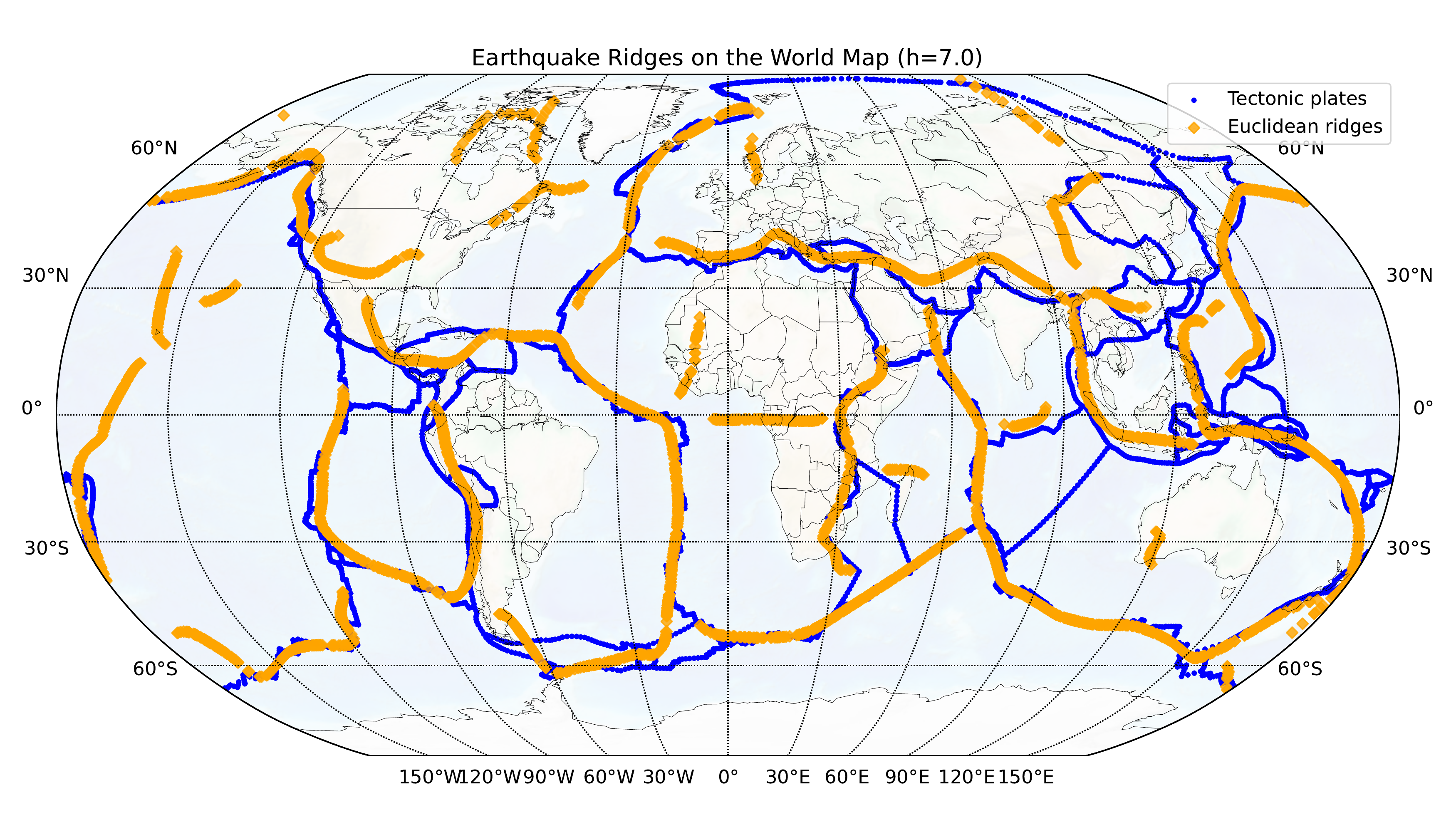}
		\caption{Estimated ridges via Euclidean SCMS}
	\end{subfigure}
	\begin{subfigure}[t]{.49\textwidth}
		\centering
		\includegraphics[width=\linewidth,height=4.5cm]{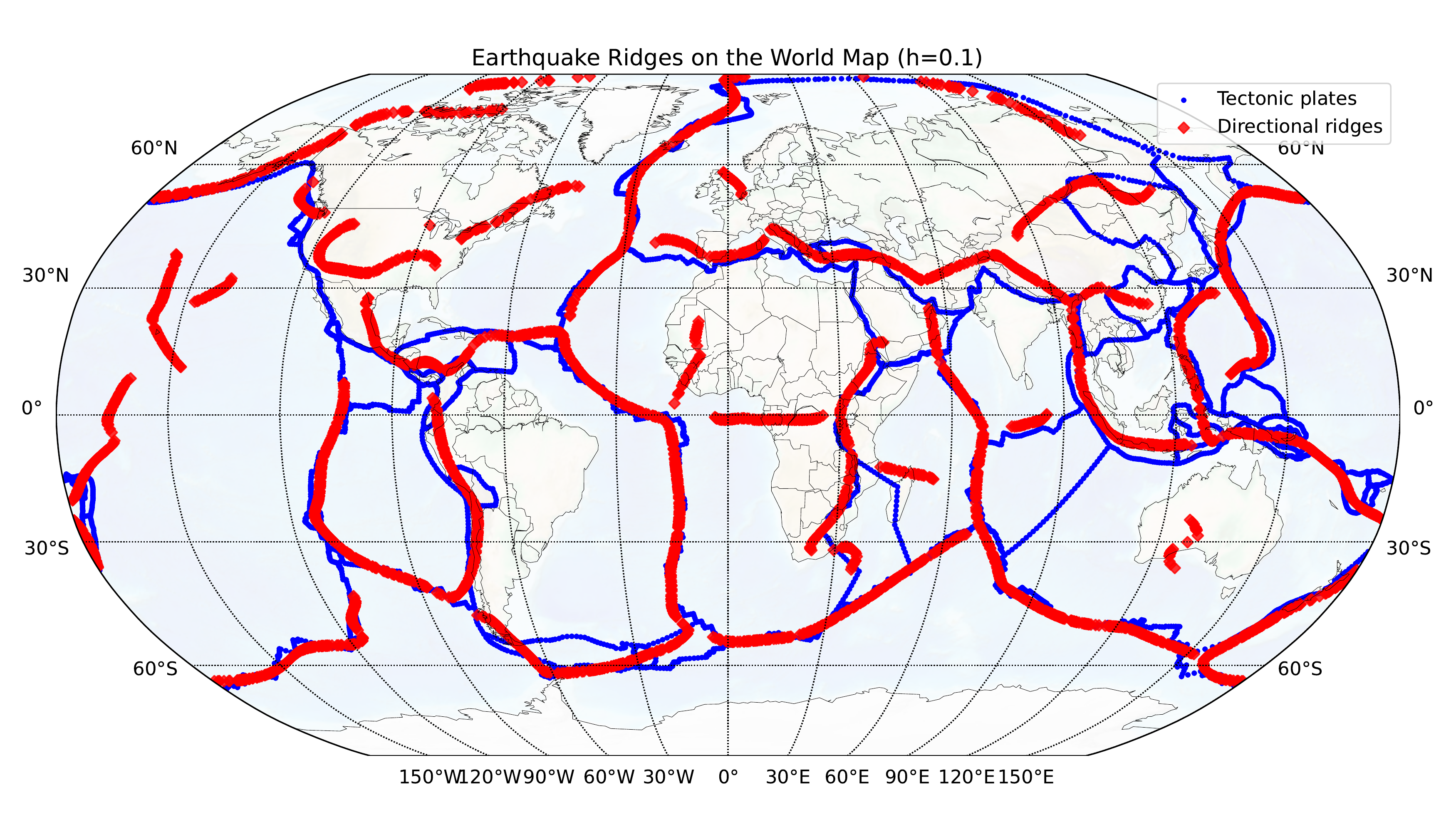}
		\caption{Estimated ridges via directional SCMS}
	\end{subfigure}%
	\caption{Comparisons between density ridges obtained by the Euclidean SCMS algorithm on angular coordinates and the directional SCMS algorithm on Cartesian coordinates from the earthquake dataset. On each panel, the ground-truth boundaries of tectonic plates are plots in blue curves.}
	\label{fig:Earthquake_ridges}
\end{figure}

To compare the earthquake ridges obtained by the Euclidean and directional SCMS algorithms with the boundaries of tectonic plates, we download the boundary geometry file of the 56 tectonic plates from \url{https://www.kaggle.com/cwthompson/tectonic-plate-boundaries} according to the models of \cite{Bird2003,argus2011geologically} and overlap them with the estimated ridges in Figure~\ref{fig:Earthquake_ridges}. The results suggest that the ridges identified by the Euclidean and directional SCMS algorithms on the earthquake dataset coincide with the boundaries of tectonic plates to a large extent.
Note that the Euclidean and directional ridges on the earthquake dataset $\mathcal{D}$ do not show too much difference, because most of the observed earthquakes are in the low latitude region ($\leq 60^{\circ}$) where most human beings live. 
Yet, the ridges estimated by our proposed directional SCMS algorithm do align better with the boundary of the Eurasian Plate near the North Pole than the ones estimated by the Euclidean SCMS algorithm, which confirms the superiority of our directional SCMS algorithm in the high latitude region; see also Appendix~\ref{Sec:Drawback_Euc} for more in-depth analysis.
	
	We further quantify the performances of earthquake ridges $\hat{R}_1$ and $\underline{\hat{R}}_1$ estimated by the Euclidean and directional SCMS algorithms from two different perspectives. First, given the fact that an estimated ridge should lie on the region where earthquakes happen more intensively, we compute the mean geodesic distances from each point in the earthquake dataset $\mathcal{D}$ to the ridges $\hat{R}_1$ and $\underline{\hat{R}}_1$ respectively as:
	$$\frac{1}{|\mathcal{D}|} \sum_{\bm{x}\in \mathcal{D}} d_g(\bm{x}, \hat{R}_1) \approx 0.02241 \quad \text{ and } \quad \frac{1}{|\mathcal{D}|} \sum_{\bm{x}\in \mathcal{D}} d_g(\bm{x}, \underline{\hat{R}}_1) \approx 0.02150,$$
	where $|\mathcal{D}|=15049$ is the number of earthquakes in the dataset. The ridge $\underline{\hat{R}}_1$ estimated by our directional SCMS algorithm is around 4\% closer to the earthquakes in $\mathcal{D}$ on average. Second, we assess the estimation errors of $\hat{R}_1$ and $\underline{\hat{R}}_1$ with respect to the boundaries of tectonic plates. To this end, we view the surface of the Earth as a unit sphere $\Omega_2$ and define a manifold-recovering error measure \citep{DLSCMSProd2021} between the set of boundary points $\mathcal{B}$ and an estimated ridge $\hat{R}$ as:
	\begin{equation}
	\label{Manifold_rec}
	d_E\left(\mathcal{B},\hat{R} \right) = \frac{1}{2} \left[\frac{1}{|\hat{R}|} \sum_{\bm{x}\in \hat{R}} d_g(\bm{x},\mathcal{B}) + \frac{1}{|\mathcal{B}|} \sum_{\bm{y}\in \mathcal{B}} d_g(\bm{y}, \hat{R}) \right],
	\end{equation}
	where $|\hat{R}|$ and $|\mathcal{B}|$ are the cardinalities of $\hat{R}$ and $\mathcal{B}$, respectively. Note that although the density ridge $\hat{R}$ and the boundaries of tectonic plates $\mathcal{B}$ are continuous structures in theory, they are generally represented by sets of discrete points in practice. That is why we can calculate their cardinalities without computing complicated integrals. Moreover, the manifold-recovering error measure is an average between the mean geodesic distances from each point in $\hat{R}$ to $\mathcal{B}$ and from each point in $\mathcal{B}$ to $\hat{R}$. We define such a balanced error measure to avoid biasing toward an estimated ridge $\hat{R}$ that only approximates a small portion of $\mathcal{B}$ in high accuracy but fails to cover other parts of $\mathcal{B}$; see Figure 4 in \cite{DLSCMSProd2021} for an illustrative example. The manifold-recovering error measures of the ridges $\hat{R}_1$ and $\underline{\hat{R}}_1$ estimated by the Euclidean and directional SCMS algorithms with respect to the boundaries of tectonic plates $\mathcal{B}$ are 
	$$d_E\left(\mathcal{B},\hat{R}_1 \right) \approx 0.05332 \quad \text{ and } \quad d_E\left(\mathcal{B},\underline{\hat{R}}_1 \right) \approx 0.05121.$$
	Our directional SCMS algorithm again reduces the estimation error by around 3.9\%. In summary, the earthquake ridges yielded by our directional SCMS algorithm are not only closer to the earthquakes on average than the ones identified by the Euclidean SCMS algorithm but also have a lower error in approximating the boundaries of tectonic plates. 


\section{Discussions}	
\label{sec::discussion}

In this paper, we have provided a rigorous proof for the linear convergence of the well-known SCMS algorithm by viewing it as an example of the SCGA algorithm. We have also generalized the definition of density ridges from the usual densities supported on compact sets in $\mathbb{R}^D$ to the directional densities supported on $\Omega_q$ with nonzero curvature. The stability theorem of directional density ridges has been established, and the linear convergence of our proposed directional SCMS algorithm has been proved. Table~\ref{table:Stepsize_summary} summarizes the frameworks of considering the (directional) mean shift/SCMS algorithms as gradient ascent/SCGA methods (on $\Omega_q$) and our results of asymptotic convergence rates of their corresponding step sizes.

\begin{table}[t]
	\centering
	\caption{Comparisons between the Euclidean and directional mean shift (MS) or SCMS algorithms and summary of the asymptotic convergence rates of their adaptive sizes when viewed as GA/SCGA algorithms in $\mathbb{R}^D$ or on $\Omega_q$.} 
	\label{table:Stepsize_summary}
	\begin{tabular}{C{0.09\textwidth}C{0.58\textwidth}c}
		\hline
		\bf Algorithms & \bf Recast forms as GA/SCGA (in $\mathbb{R}^D$ or on $\Omega_q$) & \bf Asymptotic step sizes\\
		\hline
		& \multirow{4}{*}{$\hat{\bm{x}}^{(t+1)} \gets 
			\begin{cases}
			\hat{\bm{x}}^{(t)} + \eta_{n,h}^{(t)} \cdot \nabla \hat{p}_n(\hat{\bm{x}}^{(t)})\\
			\hat{\bm{x}}^{(t)} + \eta_{n,h}^{(t)} \cdot \hat{V}_d(\hat{\bm{x}}^{(t)}) \hat{V}_d(\hat{\bm{x}}^{(t)})^T \nabla \hat{p}_n(\hat{\bm{x}}^{(t)})
			\end{cases}$}  &  \multirow{2}{*}{$\eta_{n,h}^{(t)} \asymp O(h^2) + o_P(h^2)$} \\
		MS / SCMS in $\mathbb{R}^D$  & &  \multirow{3}{*}{(See Lemma~\ref{limit_step_MS})} \\
		\hline
		& \multirow{4}{*}{$\hat{\underline{\bm{x}}}^{(t+1)} \gets 
			\begin{cases}
			\Exp_{\hat{\underline{\bm{x}}}^{(t)}}\left(\underline{\eta}_{n,h}^{(t)} \cdot \grad \hat{f}_h(\hat{\underline{\bm{x}}}^{(t)}) \right) \\
			\Exp_{\hat{\underline{\bm{x}}}^{(t)}}\left(\underline{\eta}_{n,h}^{(t)'} \cdot \hat{\underline{V}}_d(\hat{\underline{\bm{x}}}^{(t)}) \hat{\underline{V}}_d(\hat{\underline{\bm{x}}}^{(t)})^T \grad \hat{f}_h(\hat{\underline{\bm{x}}}^{(t)}) \right)
			\end{cases}$}& $\underline{\eta}_{n,h}^{(t)} \asymp \underline{\eta}_{n,h}^{(t)'}$ \\
		MS / & & \multirow{2}{*}{$\quad =O(h^2) + o_P(h^2),$} \\
		SCMS on $\Omega_q$ &  & \multirow{2}{*}{(See Lemma~\ref{norm_tot_grad})}\\
		\hline
	\end{tabular}
\end{table}

Our theoretical analyses of the SCGA algorithm in the Euclidean space $\mathbb{R}^D$ and on the unit hypersphere $\Omega_q$ has potential implications beyond proving the linear convergence of SCMS algorithms. In the optimization literature \citep{Nocedal2006Numerical,Op_algo_mat_manifold2008,Geo_Convex_Op2016,nesterov2018lectures}, it is well-known that a standard gradient ascent method (on a smooth manifold) will converge linearly given an appropriate step size when the objective function is smooth and (geodesically) strongly concave. However, as we have discussed in Remarks~\ref{SC_remark} and \ref{SC_remark2}, the smoothness and (geodesically) strong concavity assumptions are not sufficient for the linear convergence of the SCGA algorithms. Therefore, identifying density ridges with the SCGA algorithms is not only a nonconvex optimization problem, but also fundamentally more complex than standard gradient ascent methods. The assumptions and proof arguments developed in this paper may give some insights into the linear convergence of the SCGA algorithms with other forms of subspace constrained gradients.

There are still many open problems related to the SCMS algorithm. First, a central issue in determining the performance of a SCMS algorithm is the bandwidth selection. 
There is a variety of bandwidth selection mechanisms available to the Euclidean KDE and its derivatives in the literature \citep{Asymp_deri_KDE2011,Scott2015}, but it is unclear how they can be applied to the SCMS algorithm. We plan to specialize or generalize such techniques to the SCMS algorithm under both the Euclidean and directional data. Second, our definition of density ridges is generalizable to any density supported on an arbitrary Riemannian manifold. As \cite{hauberg2015principal} has formulated the principal curve on a Riemannian manifold based on its classical definition in \cite{hastie1989principal}, it will be interesting to propose a new definition of principal curves from the perspective of density ridges on Riemannian manifolds and derive a more general SCMS algorithm, possibly based on some existing nonlinear mean shift methods on manifolds \citep{subbarao2006nonlinear,subbarao2009nonlinear}.

\section*{Acknowledgements}
YC is supported by NSF DMS-1810960 and DMS-1952781, NIH U01-AG0169761.

\bibliographystyle{imsart-nameyear} 
\bibliography{Bibliography}       

\begin{appendix}

\section{Algorithmic Summaries of Euclidean and Directional SCMS Algorithms}
\label{App:Algo}

\begin{figure}[t]
	\captionsetup[subfigure]{justification=centering}
	\centering
	\begin{subfigure}[t]{.485\textwidth}
		\centering
		\includegraphics[width=\linewidth]{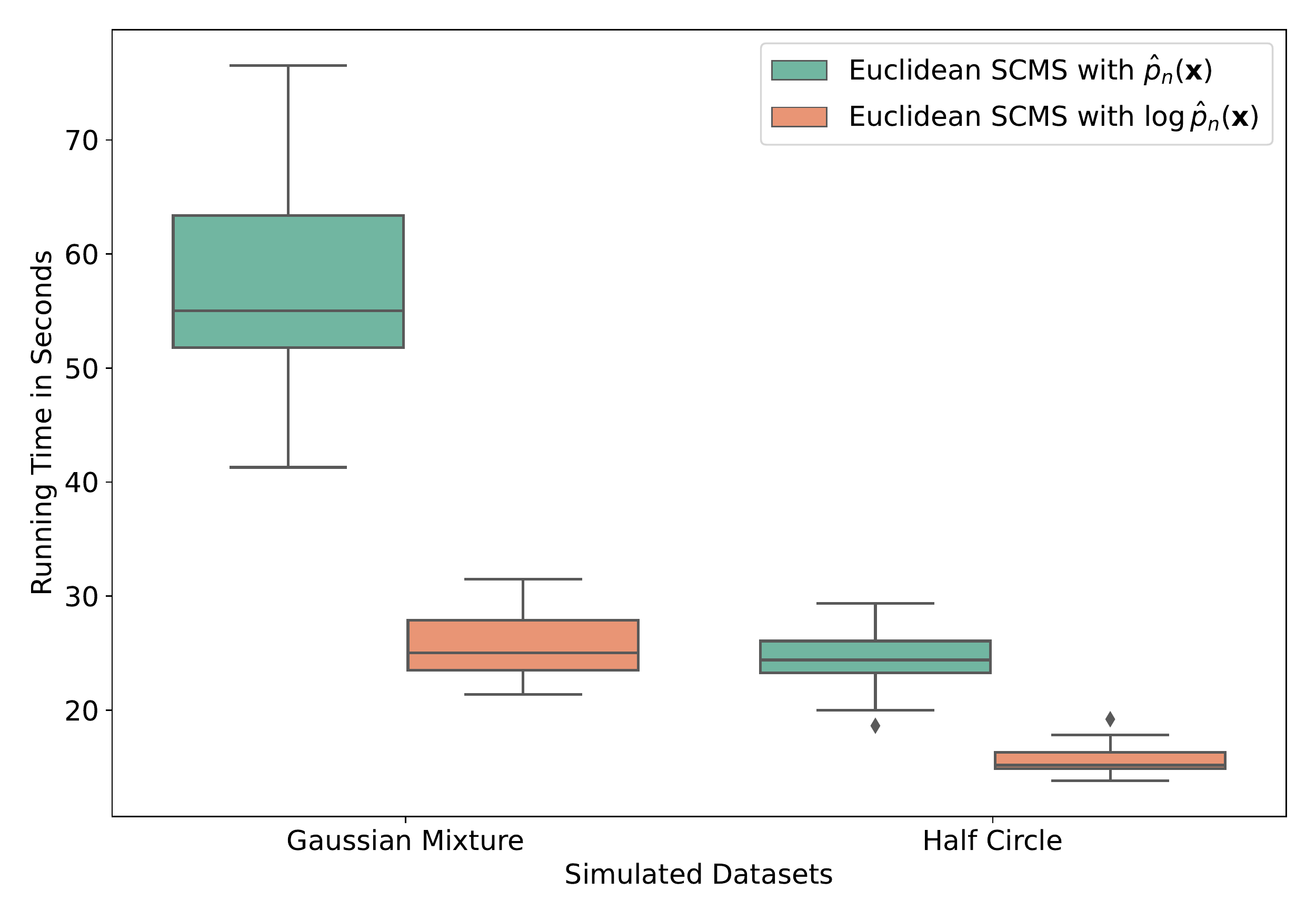}
		\caption{Repeated experiments with the simulated Euclidean datasets in Figure~\ref{fig:LC_plot}.}
	\end{subfigure}
	\hfil
	\begin{subfigure}[t]{.485\textwidth}
		\centering
		\includegraphics[width=\linewidth]{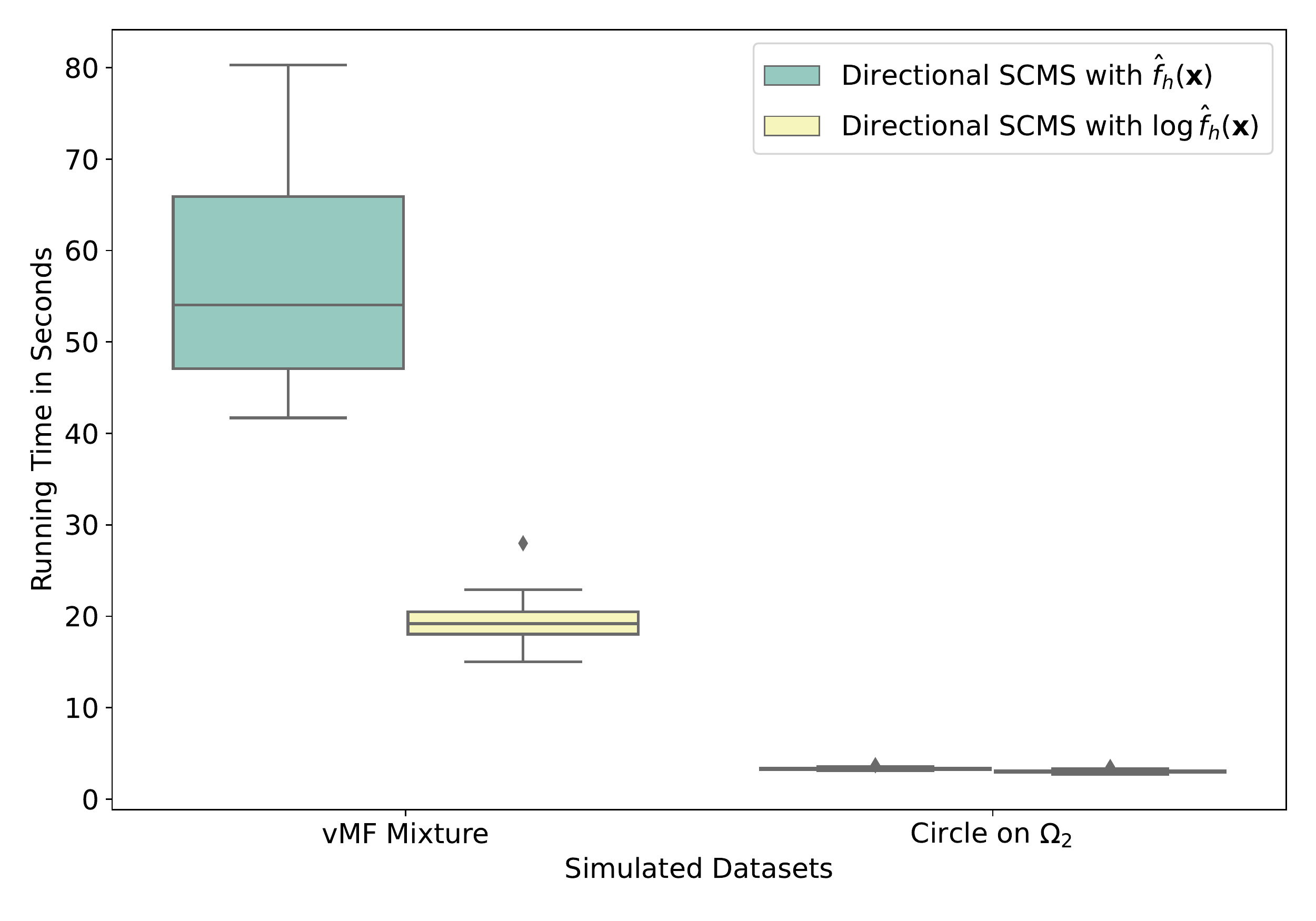}
		\caption{Repeated experiments with the simulated directional datasets in Figure~\ref{fig:LC_plot_Dir}.}
	\end{subfigure}
	\caption{Running time comparisons between the (directional) SCMS algorithms with the original density and the log-density applied to our simulated datasets in Figures~\ref{fig:LC_plot} and \ref{fig:LC_plot_Dir}.}
	\label{fig:SCMS_log_comp}
\end{figure}

In this section, we provide algorithmic summaries of the Euclidean and directional SCMS algorithms for practical reference. 
Algorithm~\ref{Algo:SCMS} describes each step of the Euclidean SCMS algorithm in detail. In our actual implementation of the algorithm, we replace the density estimator $\hat{p}_n(\bm{x})$ with $\log \hat{p}_n(\bm{x})$.
To demonstrate that the (directional) SCMS algorithms under the log-density implementation give rise to a faster convergence process, we repeat our experiments in Sections~\ref{Sec:EuSCMS_exp} and \ref{Sec:DirSCMS_exp} (\emph{i.e.}, Figures~\ref{fig:LC_plot} and \ref{fig:LC_plot_Dir}) 20 times for each simulated dataset with the (directional) SCMS algorithms under the original (estimated) density and the (estimated) log-density, respectively. The comparisons between their running times are shown in Figure~\ref{fig:SCMS_log_comp}, in which the (directional) SCMS algorithms under the log-density implementation clearly outperform their counterparts with the original density in terms of the average elapsed time until convergence.

Additionally, when the observational data in practice are noisy, it is common to incorporate an extra denoising step before Step 2 of Algorithm~\ref{Algo:SCMS} to remove observations in low-density areas and stabilize the (Euclidean) SCMS algorithm; see \cite{Non_ridge_est2014, Cosmic_Web2015} for comparative studies that demonstrate the significance of denoising.

\begin{algorithm}[t]
	\caption{(Euclidean) Subspace Constrained Mean Shift (SCMS) Algorithm}
	\label{Algo:SCMS}
	\begin{algorithmic}
		\State \textbf{Input}: 
		\begin{itemize}
			\item A data sample $\bm{X}_1,...,\bm{X}_n \sim p(\bm{x})$ in $\mathbb{R}^D$.
			\item The order $d$ of the ridge, smoothing bandwidth $h >0$, and tolerance level $\epsilon >0$.
			\item A suitable mesh $\mathbb{M}_E \subset \mathbb{R}^D$ of initial points. By default, $\mathbb{M}_E=\left\{\bm{X}_1,...,\bm{X}_n \right\}$.
		\end{itemize} 
		\State \textbf{Step 1}: Compute the density estimator $\hat{p}_n(\bm{x})=\frac{c_{k,D}}{nh^D}\sum\limits_{i=1}^n k\left(\norm{\frac{\bm{x}-\bm{X}_i}{h}}_2^2 \right)$ on the mesh $\mathbb{M}_E$.
		\State \textbf{Step 2}: For each initial point $\hat{\bm{x}}^{(0)} \in \mathbb{M}_E$, iterate the following SCMS update until convergence:
		\While {$\norm{\hat{V}_d(\hat{\bm{x}}^{(t)})^T \nabla \hat{p}_n(\hat{\bm{x}}^{(t)})}_2 > \epsilon$}
		\State \textbf{Step 2-1}: Compute the estimated Hessian matrix as:
		\begin{align*}
		\label{Est_Hessian}
		\begin{split}
		\nabla\nabla \hat{p}_n(\hat{\bm{x}}^{(t)}) = \frac{c_{k,D}}{nh^{D+2}} \sum_{i=1}^n &\Bigg[2\bm{I}_D \cdot k'\left(\norm{\frac{\hat{\bm{x}}^{(t)}-\bm{X}_i}{h}}_2^2 \right) \\
		&\quad + \frac{4}{h^2} (\hat{\bm{x}}^{(t)}-\bm{X}_i) (\hat{\bm{x}}^{(t)} - \bm{X}_i)^T \cdot k''\left(\norm{\frac{\hat{\bm{x}}^{(t)}-\bm{X}_i}{h}}_2^2 \right) \Bigg].
		\end{split}
		\end{align*}
		\State \textbf{Step 2-2}: Perform the spectral decomposition on the Hessian $\nabla\nabla \hat{p}_n(\hat{\bm{x}}^{(t)})$ and obtain that $\hat{V}_d(\hat{\bm{x}}^{(t)}) = \left[\hat{\bm{v}}_{d+1}(\hat{\bm{x}}^{(t)}),...,\hat{\bm{v}}_D(\hat{\bm{x}}^{(t)}) \right]$ whose columns are orthonormal eigenvectors associated with the smallest $(D-d)$ eigenvalues of $\nabla\nabla \hat{p}_n(\hat{\bm{x}}^{(t)})$.
		\State \textbf{Step 2-3}: Update $\hat{\bm{x}}^{(t+1)} \gets \hat{\bm{x}}^{(t)} + \hat{V}_d(\hat{\bm{x}}^{(t)}) \hat{V}_d(\hat{\bm{x}}^{(t)})^T \left[\frac{\sum_{i=1}^n \bm{X}_i k'\left(\norm{\frac{\hat{\bm{x}}^{(t)}-\bm{X}_i}{h}}_2^2 \right)}{\sum_{i=1}^n k'\left(\norm{\frac{\hat{\bm{x}}^{(t)}-\bm{X}_i}{h}}_2^2 \right)} -\hat{\bm{x}}^{(t)} \right]$.
		\EndWhile
		\State \textbf{Output}: An estimated $d$-ridge $\hat{R}_d$ represented by the collection of resulting points.
	\end{algorithmic}
\end{algorithm}

\begin{algorithm}[t]
	\caption{Directional Subspace Constrained Mean Shift (SCMS) Algorithm}
	\label{Algo:Dir_SCMS}
	\begin{algorithmic}
		\State \textbf{Input}: 
		\begin{itemize}
			\item A directional data sample $\bm{X}_1,...,\bm{X}_n \sim f(\bm{x})$ on $\Omega_q$.
			\item The order $d$ of the directional ridge, smoothing bandwidth $h>0$, and tolerance level $\epsilon >0$.
			\item A suitable mesh $\mathbb{M}_D \subset \Omega_q$ of initial points. By default, $\mathbb{M}_D = \left\{\bm{X}_1,...,\bm{X}_n \right\}$.
		\end{itemize} 
		\State \textbf{Step 1}: Compute the directional KDE $\hat{f}_h(\bm{x})=\frac{c_{L,q}(h)}{n} \sum\limits_{i=1}^n L\left(\frac{1-\bm{x}^T\bm{X}_i}{h^2}\right)$ on the mesh $\mathbb{M}_D$.
		\State \textbf{Step 2}: For each $\underline{\hat{\bm{x}}}^{(0)}\in \mathbb{M}_D$, iterate the following directional SCMS update until convergence:
		\While {$\norm{\frac{nh^2}{c_{L,q}(h)} \cdot \underline{\hat{G}}_d\left(\underline{\hat{\bm{x}}}^{(t)} \right)}_2 > \epsilon$}:
		\State \textbf{Step 2-1}: Compute the scaled version of the estimated Hessian matrix as: 
		\begin{align*}
		\frac{nh^2}{c_{L,q}(h)}\mathcal{H} \hat{f}_h(\underline{\hat{\bm{x}}}^{(t)}) 
		&= \left[\bm{I}_{q+1} -\underline{\hat{\bm{x}}}^{(t)}\left(\underline{\hat{\bm{x}}}^{(t)}\right)^T \right] \Bigg[\frac{1}{h^2}\sum_{i=1}^n \bm{X}_i\bm{X}_i^T\cdot L''\left(\frac{1-\bm{X}_i^T \underline{\hat{\bm{x}}}^{(t)}}{h^2} \right) \\
		& \hspace{10mm} + \sum_{i=1}^n \bm{X}_i^T \underline{\hat{\bm{x}}}^{(t)} \bm{I}_{q+1}\cdot L'\left(\frac{1-\bm{X}_i^T \underline{\hat{\bm{x}}}^{(t)}}{h^2} \right) \Bigg] \left[\bm{I}_{q+1} -\underline{\hat{\bm{x}}}^{(t)}\left(\underline{\hat{\bm{x}}}^{(t)}\right)^T \right].
		\end{align*}
		\State \textbf{Step 2-2}: Perform the spectral decomposition on $\frac{nh^2}{c_{L,q}(h)} \mathcal{H} \hat{f}_h\left(\underline{\hat{\bm{x}}}^{(t)} \right)$ and compute $\underline{\hat{V}}_d(\underline{\hat{\bm{x}}}^{(t)}) = \left[\underline{\hat{\bm{v}}}_{d+1}(\underline{\hat{\bm{x}}}^{(t)}),...,\underline{\hat{\bm{v}}}_q(\underline{\hat{\bm{x}}}^{(t)}) \right]$, whose columns are orthonormal eigenvectors corresponding to the smallest $q-d$ eigenvalues inside the tangent space $T_{\underline{\hat{\bm{x}}}^{(t)}}$.
		\State \textbf{Step 2-3}: Update $\underline{\hat{\bm{x}}}^{(t+1)} \gets \underline{\hat{\bm{x}}}^{(t)} - \underline{\hat{V}}_d(\underline{\hat{\bm{x}}}^{(t)}) \underline{\hat{V}}_d(\underline{\hat{\bm{x}}}^{(t)})^T  \left[\frac{\sum_{i=1}^n \bm{X}_i L'\left(\frac{1-\bm{X}_i^T \underline{\hat{\bm{x}}}^{(t)}}{h^2} \right)}{\norm{\sum_{i=1}^n \bm{X}_i L'\left(\frac{1-\bm{X}_i^T \underline{\hat{\bm{x}}}^{(t)}}{h^2} \right)}_2} \right]$.
		\State \textbf{Step 2-4}: Standardize $\underline{\hat{\bm{x}}}^{(t+1)}$ as $\underline{\hat{\bm{x}}}^{(t+1)} \gets \frac{\underline{\hat{\bm{x}}}^{(t+1)}}{\norm{\underline{\hat{\bm{x}}}^{(t+1)}}_2}$.
		\EndWhile
		\State \textbf{Output}: An estimated directional $d$-ridge $\underline{\hat{R}}_d$ represented by the collection of resulting points.
	\end{algorithmic}
\end{algorithm}

We summarize the directional SCMS algorithm in Algorithm~\ref{Algo:Dir_SCMS}. 
Note that in Step 2-1 of Algorithm~\ref{Algo:Dir_SCMS}, we compute the scaled versions $\frac{nh^2}{c_{L,q}(h)} \cdot \underline{\hat{G}}_d(\bm{x})$ and $\frac{nh^2}{c_{L,q}(h)}\mathcal{H} \hat{f}_h(\bm{x})$ for $\bm{x}\in \Omega_q$ because the estimated principal Riemannian gradient $\underline{\hat{G}}_d(\bm{x})$ and Hessian $\mathcal{H} \hat{f}_h(\bm{x})$ are often very small. 
The scaling stabilizes the numerical computation.
The spectral decomposition is thus performed on the scaled Hessian estimator $\frac{nh^2}{c_{L,q}(h)}\mathcal{H} \hat{f}_h(\bm{x})$, and the scaled principal Riemannian gradient estimator is calculated as
\begin{align*}
\frac{nh^2}{c_{L,q}(h)} \cdot \underline{\hat{G}}_d(\bm{x}) &= \underline{\hat{V}}_d(\bm{x}) \underline{\hat{V}}_d(\bm{x})^T \left[\sum_{i=1}^n \left(\bm{x}\cdot \bm{x}^T\bm{X}_i -\bm{X}_i \right) L'\left(\frac{1-\bm{x}^T\bm{X}_i}{h^2} \right) \right] \\
&= - \sum_{i=1}^n \underline{\hat{V}}_d(\bm{x}) \underline{\hat{V}}_d(\bm{x})^T \bm{X}_i \cdot L'\left(\frac{1-\bm{x}^T\bm{X}_i}{h^2} \right),
\end{align*}
where $\underline{\hat{V}}_d(\bm{x}) = [\underline{\hat{\bm{v}}}_{d+1}(\bm{x}),...,\underline{\hat{\bm{v}}}_q(\bm{x})]$ has its columns equal to the orthonormal eigenvectors associated with the $d$ smallest eigenvalues of the scaled Hessian estimator $\frac{nh^2}{c_{L,q}(h)}\mathcal{H} \hat{f}_h(\bm{x})$ (or equivalently, $\mathcal{H} \hat{f}_h(\bm{x})$) inside the tangent space $T_{\bm{x}}$.

\section{Limitations of Euclidean KDE in Handling Directional Data}
\label{Sec:Drawback_Euc}

In this section, we demonstrate with examples and simulation studies that it is inadequate to analyze angular or directional data with Euclidean KDE \eqref{KDE_simp} and SCMS algorithm (Algorithm~\ref{Algo:SCMS}). Consider a directional data sample $\left\{\bm{X}_1,...,\bm{X}_n\right\} \subset \Omega_2$ generated from a directional density $f$ on $\Omega_2$. In real-world applications, the random observations $\bm{X}_1,...,\bm{X}_n$ on $\Omega_2$ are commonly represented by their angular coordinates $\bm{Y}_1,...,\bm{Y}_n$ with $\bm{Y}_i=(Y_{i,1}, Y_{i,2}) \in [-180^{\circ},180^{\circ})\times \left[-90^{\circ},90^{\circ} \right]$ or equivalently, $\bm{Y}_i=(Y_{i,1}, Y_{i,2}) \in [-\pi,\pi)\times \left[-\frac{\pi}{2}, \frac{\pi}{2}\right]$ for $i=1,...,n$, where $\left\{Y_{i,1} \right\}_{i=1}^n$ are longitudes and $\left\{Y_{i,2} \right\}_{i=1}^n$ are latitudes.

\subsection{Case I: Density Estimation}
\label{Sec:Drawback_Euc_CaseI}

As the angular coordinates $\{\bm{Y}_1,...,\bm{Y}_n\}$ of the directional dataset $\{\bm{X}_1,...,\bm{X}_n\} \subset \Omega_2$ have their ranges in a subset $[-\pi,\pi)\times \left[-\frac{\pi}{2},\frac{\pi}{2} \right]$ of the flat Euclidean space $\mathbb{R}^2$, it is tempting to apply the Euclidean KDE on $\{\bm{Y}_1,...,\bm{Y}_n\}$ to construct a density estimator as:
\begin{equation}
\label{Euc_KDE_2d}
\hat{p}_n(\bm{y}) = \frac{1}{nh^2} \sum_{i=1}^n K\left(\frac{\bm{y}-\bm{Y}_i}{h} \right)=
\begin{cases}
\frac{c_{k,2}}{nh^2} \sum\limits_{i=1}^n k\left(\norm{\frac{\bm{y}-\bm{Y}_i}{h}}_2^2 \right) := \hat{p}_n^{(1)}(\bm{y}),\\
\frac{1}{nh^2}\sum\limits_{i=1}^n K_1\left(\frac{y_1-Y_{i,1}}{h} \right) \cdot K_2\left(\frac{y_2-Y_{i,2}}{h} \right) := \hat{p}_n^{(2)}(\bm{y}),
\end{cases}
\end{equation}
where $\hat{p}_n^{(1)}$ uses a radial symmetric kernel with profile $k$, and $\hat{p}_n^{(2)}$ leverages a product kernel. However, the Euclidean KDEs in \eqref{Euc_KDE_2d} (both $\hat{p}_n^{(1)}$ and $\hat{p}_n^{(2)}$) exhibit two potential drawbacks of dealing with directional data.

\noindent $\bullet$ First, $\hat{p}_n(\bm{y})$ in \eqref{Euc_KDE_2d} is an estimator of the directional density $f$ under its angular representation $p_f:[-\pi,\pi)\times \left[-\frac{\pi}{2},\frac{\pi}{2} \right] \to \mathbb{R}$. Here, $f$ is $2\pi$-periodic in its first coordinate and $\pi$-periodic in its second coordinate. Then, the bias of $\hat{p}_n(\bm{y})$ in estimating $p_f(\bm{y})$ is $$\mathbb{E}\left[\hat{p}_n(\bm{y}) \right]-p_f(\bm{y})=\frac{h^2}{2} C_K^2 \Delta p_f(\bm{y}) + o(h^2),$$ 
where $C_K^2 = \int_{\mathbb{R}^2} \norm{\bm{y}}_2^2 K(\bm{y}) d\bm{y}$ and $\Delta p_f= \frac{\partial^2 p_f}{\partial y_1^2} + \frac{\partial^2 p_f}{\partial y_2^2}$ is the Laplacian of $p_f$; see \cite{KDE_t} for details. However, the second-order partial derivative $\frac{\partial^2 p_f}{\partial y_1^2}$ along the lines of constant latitude (or parallels) would tend to infinity as we approach the north and south poles, given that the first-order partial derivative $\frac{\partial p_f}{\partial y_1}$ is bounded. One method to justify this claim is that the curvatures of these parallels, which are equivalent to the reciprocals of their radii, tend to infinity as these radii shrink. In addition, one should recall that the curvature of a function $y=g(x)$ is defined as $\varkappa := \frac{|g''|}{(1+g'^2)^{\frac{3}{2}}}$. Therefore, applying \eqref{Euc_KDE_2d} to estimate the angular representation $p_f$ of the directional density $f$ will produce high bias as the estimator $\hat{p}_n$ approaches the high-latitude regions (around the north and south poles); see also Panel (c) of Figure~\ref{fig:add_Cart}.

\noindent $\bullet$ Second, the Euclidean KDE $\hat{p}_n$ leverages the Euclidean distances between any query point $\bm{y} \in [-\pi,\pi)\times \left[-\frac{\pi}{2},\frac{\pi}{2} \right]$ and observations $\{\bm{Y}_1,...,\bm{Y}_n\} \subset [-\pi,\pi)\times \left[-\frac{\pi}{2},\frac{\pi}{2} \right]$ under their angular coordinates to construct the density estimates, instead of using the (intrinsic) geodesic distances. Note that the Euclidean distance in the angular coordinate system is not equivalent to the Euclidean distance in the ambient Euclidean space $\mathbb{R}^3$ containing the directional data on $\Omega_2$. As a result, some observations that have dramatically different geodesic distances to density query points can have the same density contributions in $\hat{p}_n$, as illustrated in Example~\ref{counterexample}.

\begin{figure}[t]
	\centering
	\includegraphics[width=0.4\linewidth]{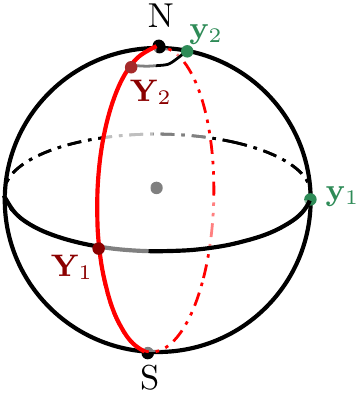}
	\caption{Graphical illustration of geodesic distances between $\bm{y}_1$ and $\bm{Y}_1$ as well as $\bm{y}_2$ and $\bm{Y}_2$.}
	\label{fig:counter_example}
\end{figure}

\begin{example}
	\label{counterexample}
	Suppose that we want to estimate the density values at $\bm{y}_1=(0,0)$ and $\bm{y}_2=(0,\frac{\pi}{2}-\epsilon)$, where $\epsilon >0$ is of a small value. Consider a random sample consisting of only two observations $\bm{Y}_1=(\frac{\pi}{2}, 0)$ and $\bm{Y}_2=(\frac{\pi}{2}, \frac{\pi}{2} -\epsilon)$. 
	If we use the Euclidean distance, the distance between $(\bm{y}_1, \bm{Y}_1)$ and the distance between $(\bm{y}_2, \bm{Y}_2)$
	are the same. Therefore, when we use the Euclidean KDE $\hat{p}_n$ to estimate the underlying density, the contribution of $\bm{Y}_1$ to $\bm{y}_1$ will be the same as the contribution of $\bm{Y}_2$ to $\bm{y}_2$.
	Nevertheless, their geodesic distances are very different, because $d_g(\bm{y}_1,\bm{Y}_1) = \frac{\pi}{2}$ while $d_g(\bm{y}_2,\bm{Y}_2) = \arccos\left[\sin^2\left(\frac{\pi}{2}-\epsilon \right) \right] = \arccos\left[\frac{1+\cos(2\epsilon)}{2} \right]$ is a quantity close to zero; see Figure~\ref{fig:counter_example} for a graphical illustration. 
	It explains, from a different angle, why the Euclidean KDE $\hat{p}_n$ will have a large bias in estimating the underlying density when the query point $\bm{y}$ is within the high latitude region.
\end{example}

\subsection{Case II: Ridge-Finding Problem}

Consider the following simulated example of identifying a density ridge via the Euclidean SCMS algorithm (Algorithm~\ref{Algo:SCMS}) and our proposed directional SCMS algorithm (Algorithm~\ref{Algo:Dir_SCMS}). We generate 1000 data points $\{\bm{X}_1,...,\bm{X}_{1000}\} \subset \Omega_2$ uniformly frbecauseom a great circle connecting the North and South Poles of $\Omega_2$ with some i.i.d. additive Gaussian noises $N(0,0.2^2)$ to their Cartesian coordinates. Then, all the simulated points will be standardized back to $\Omega_2$ via $L_2$ normalization. The angular coordinates of these simulated points are denoted by $\{\bm{Y}_1,...,\bm{Y}_{1000} \} \subset [-180^{\circ},180^{\circ})\times \left[-90^{\circ},90^{\circ} \right]$ accordingly. Figure~\ref{fig:add_Cart} presents the result of applying both the Euclidean SCMS algorithm (with the Gaussian kernel) to angular coordinates and the directional SCMS algorithm (with the von Mises kernel) to Cartesian coordinates of our simulated dataset. As shown in the panel (b) of Figure~\ref{fig:add_Cart}, the Euclidean SCMS algorithm exhibits high bias in estimating the true circular structure near two poles of $\Omega_2$, while our directional SCMS algorithm is able to seek out the true circular structure under negligible errors. The density plot in the panel (c) of Figure~\ref{fig:add_Cart} exhibits two nonsmoothing peaks on the North Pole due to the infinite Hessian matrices of the underlying density in its angular coordinate; recall our discussion in Section~\ref{Sec:Drawback_Euc_CaseI}. This also explains the chaotic behavior of the Euclidean KDE in high-latitude regions. 

\begin{figure}[!t]
	\captionsetup[subfigure]{justification=centering}
	\centering
	\begin{subfigure}[t]{.32\textwidth}
		\centering
		\includegraphics[width=1\linewidth,height=5cm]{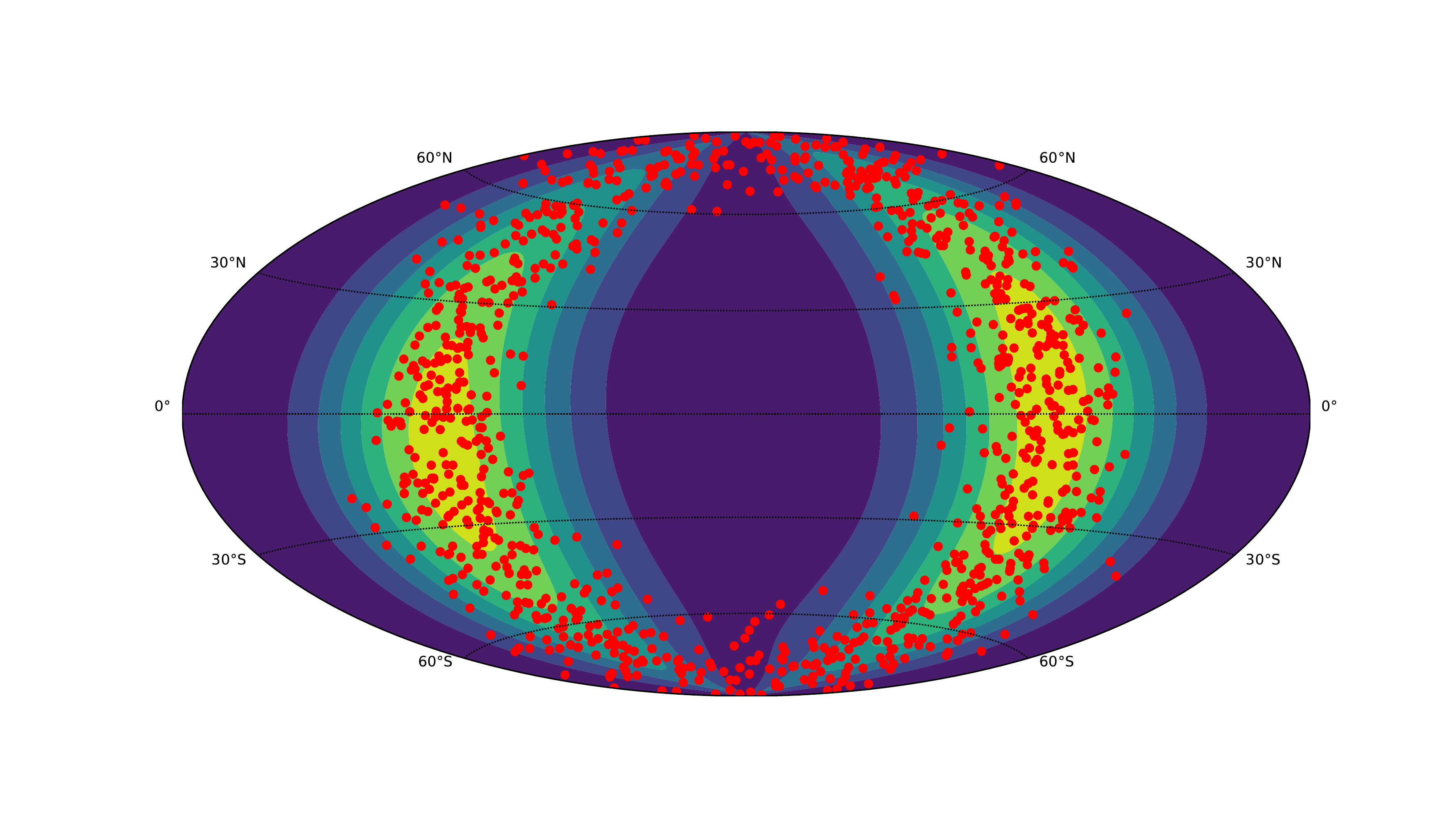}
		\caption{Initial points}
	\end{subfigure}
	\hfil
	\begin{subfigure}[t]{.32\textwidth}
		\centering
		\includegraphics[width=1\linewidth, height=5cm]{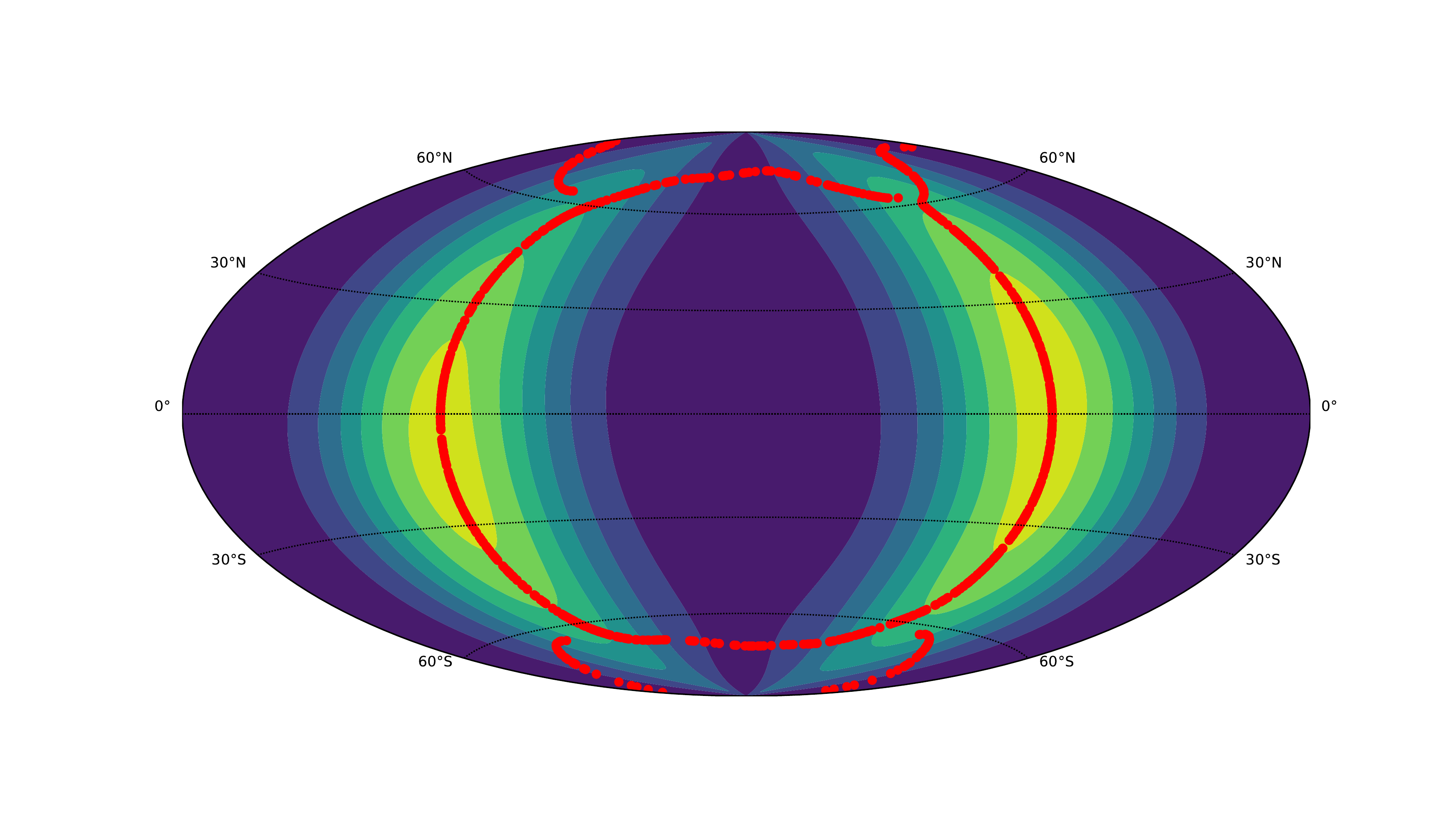}
		\caption{Converged points (hammer projection)}
	\end{subfigure}
	\hfil
	\begin{subfigure}[t]{.32\textwidth}
		\centering
		\includegraphics[width=1\linewidth]{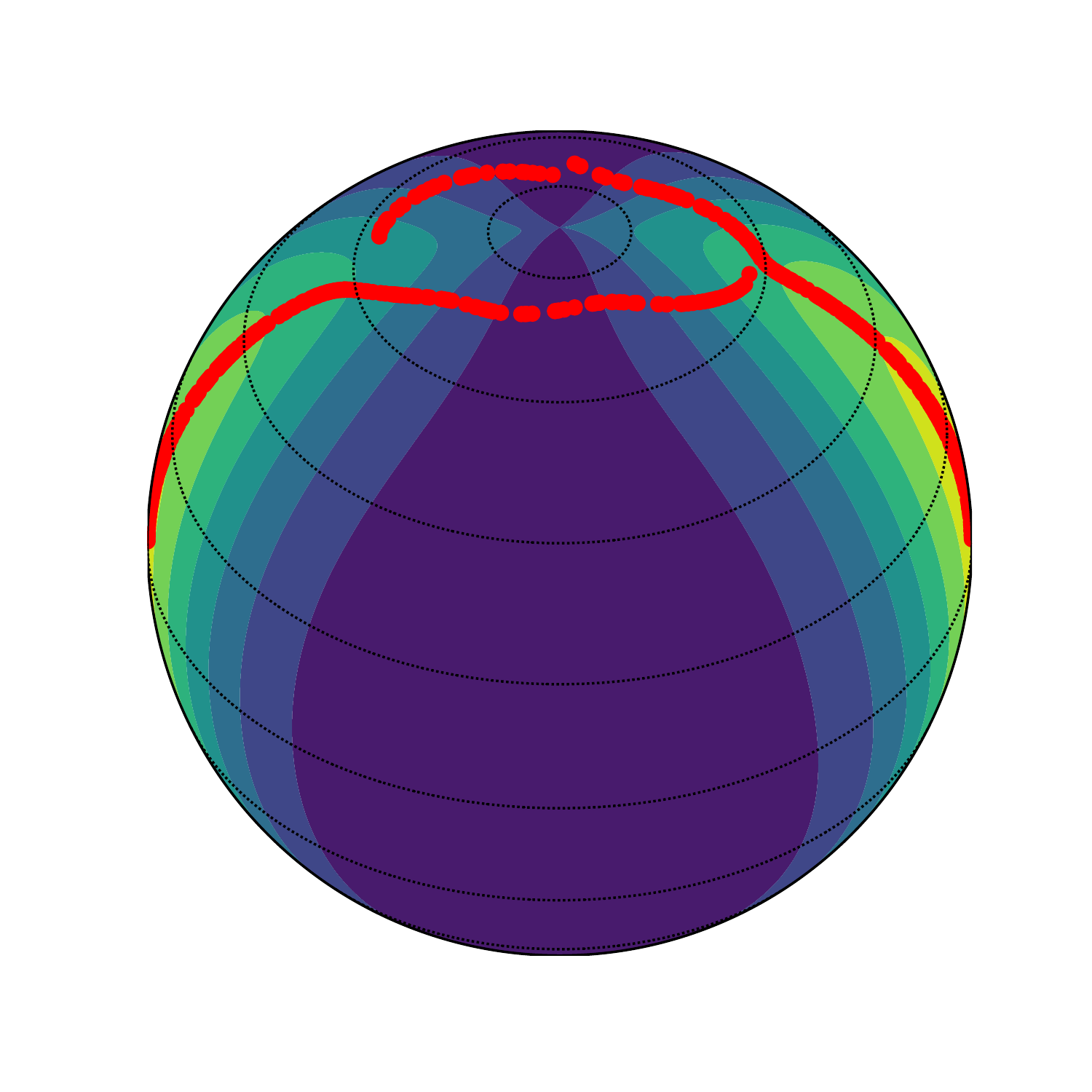}
		\caption{Converged points (orthographic projection)}
	\end{subfigure}
	\begin{subfigure}[t]{.32\textwidth}
		\centering
		\includegraphics[width=1\linewidth, height=5cm]{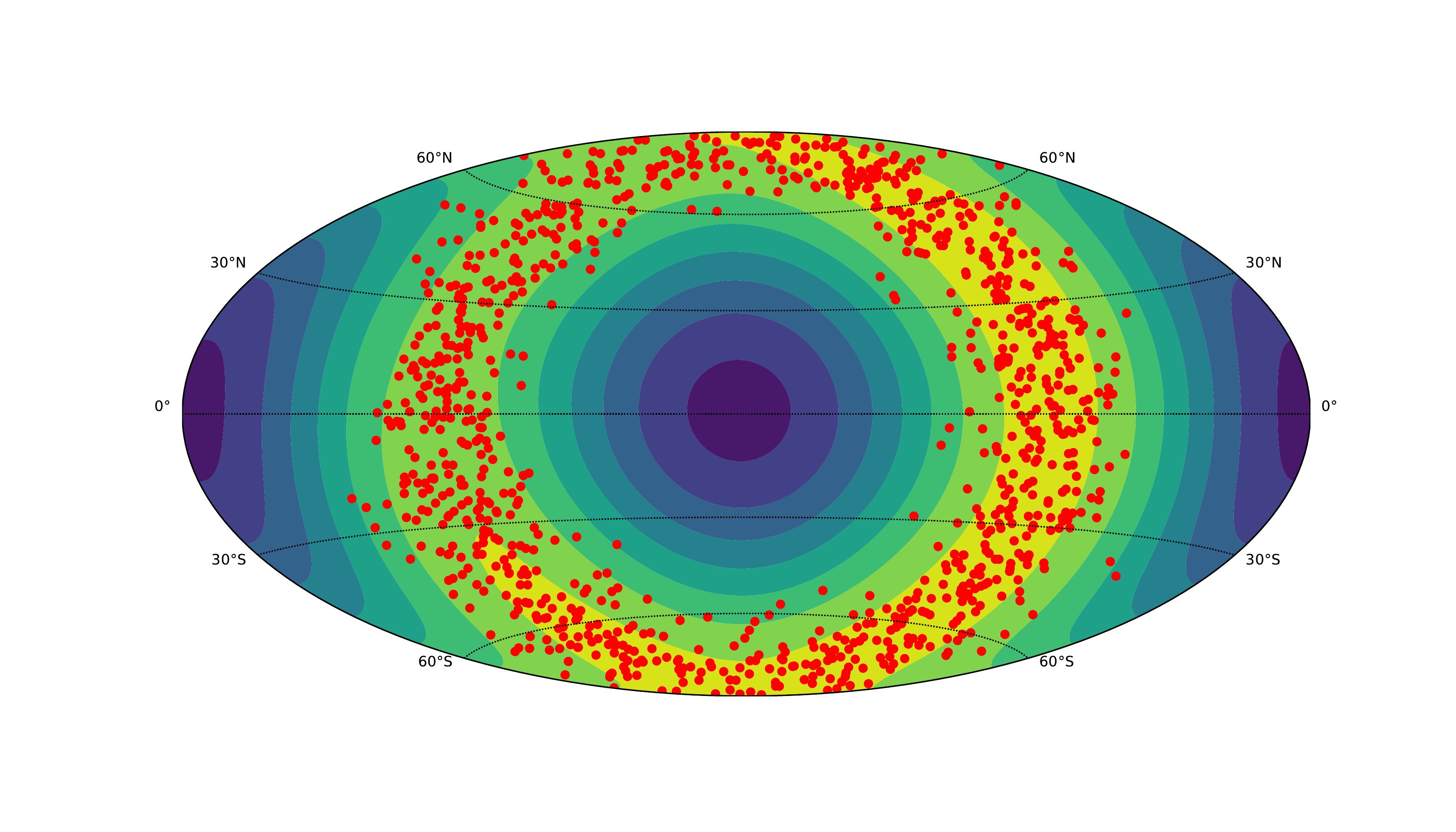}
		\caption{Initial points}
	\end{subfigure}%
	\hfil
	\begin{subfigure}[t]{.32\textwidth}
		\centering
		\includegraphics[width=1\linewidth, height=5cm]{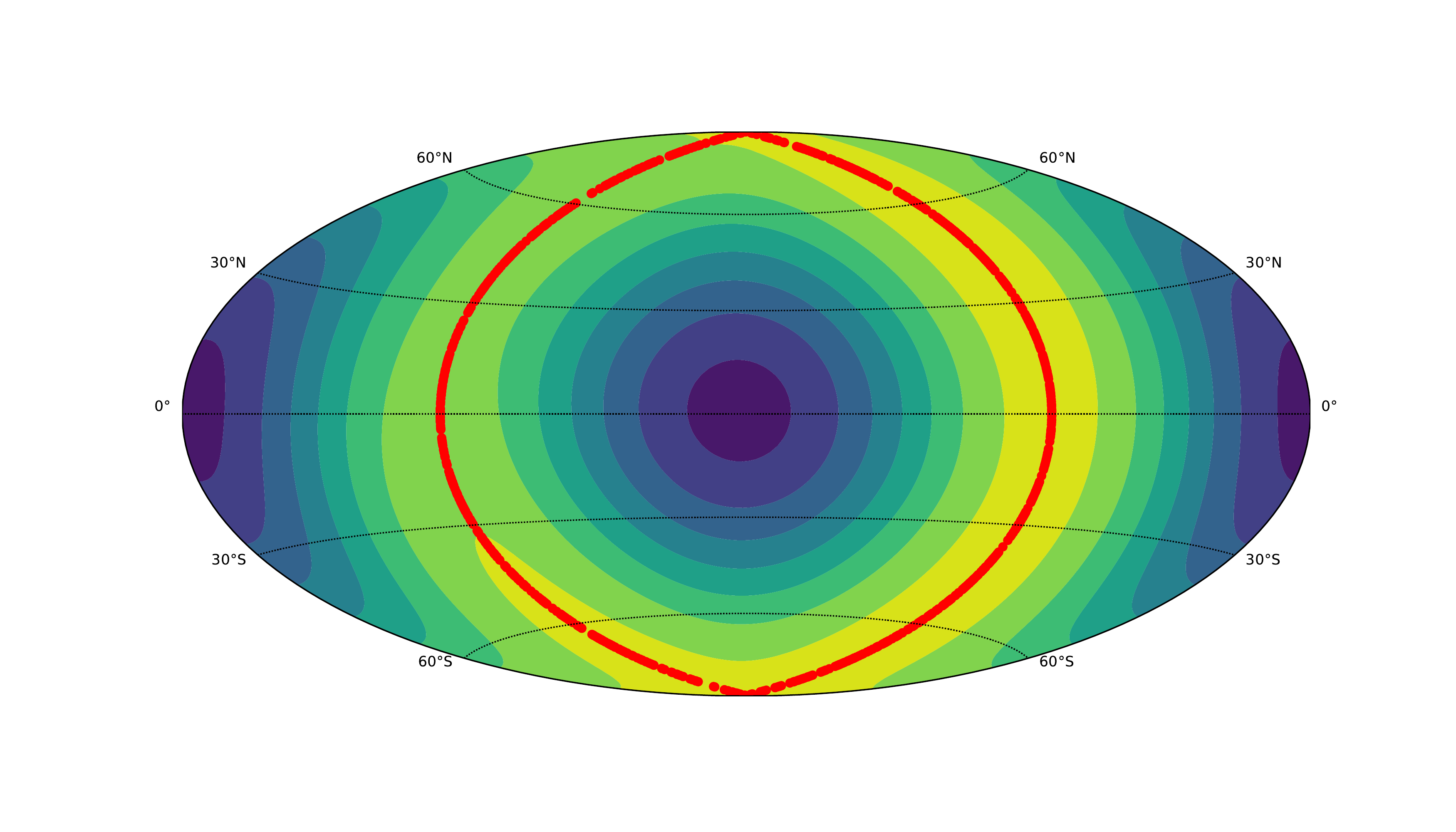}
		\caption{Converged points (hammer projection)}
	\end{subfigure}
	\hfil
	\begin{subfigure}[t]{.32\textwidth}
		\centering
		\includegraphics[width=1\linewidth]{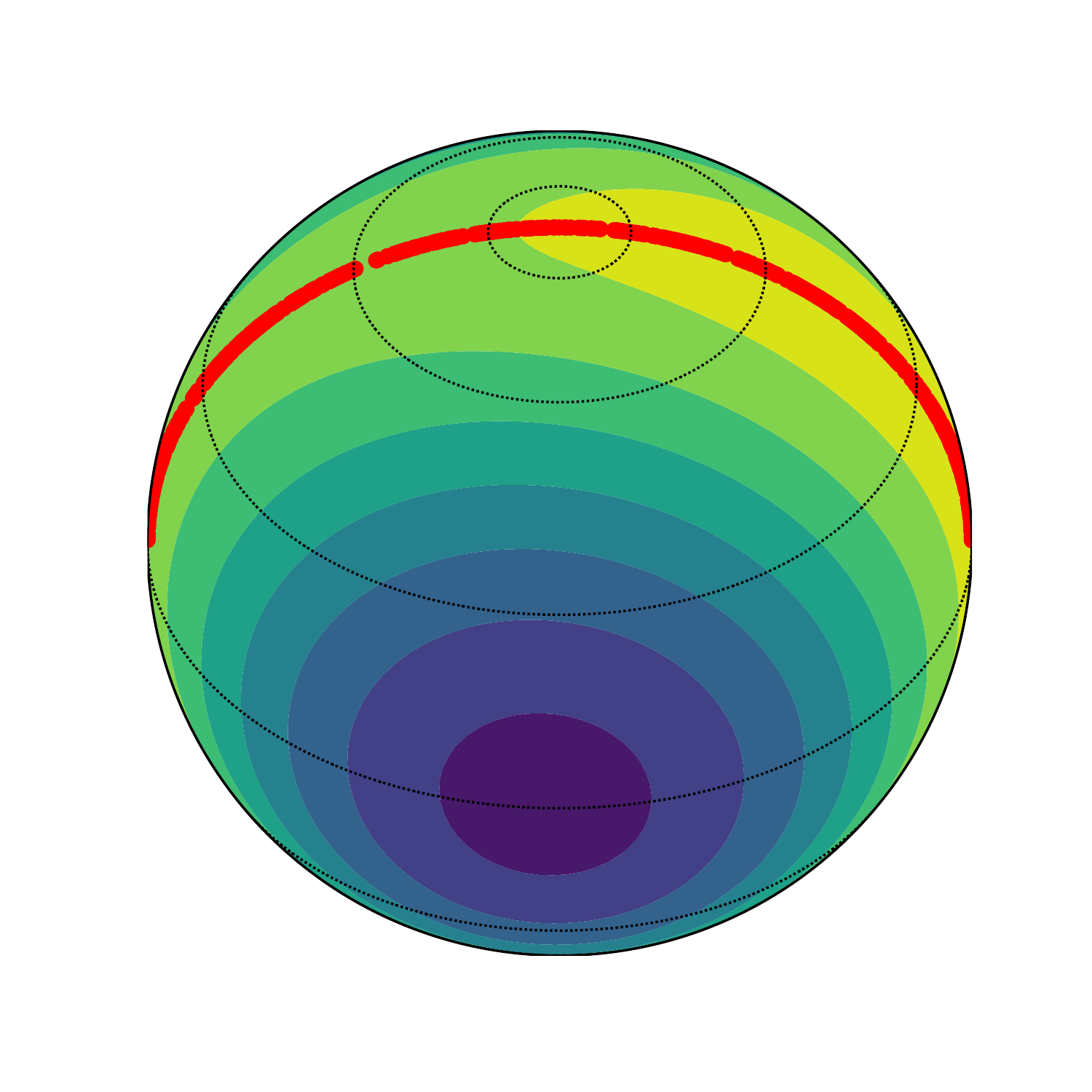}
		\caption{Converged points (orthographic projection)}
	\end{subfigure}
	\caption{Euclidean and directional SCMS algorithms performed on the simulated dataset. 
		{\bf Panels (a)-(c):} Outcomes of the Euclidean SCMS algorithm with the contour plot for the Euclidean KDE.
		{\bf Panels (d)-(f):} Outcomes of our directional SCMS algorithm with the contour plot for the directional KDE.
		{\bf Panels (a)-(b) and (d)-(e)} are shown in the view of Hammer projections (page 160 in \citealt{Snyder1989album}), while {\bf Panels (c) and (f)} are presented under the orthographic projections.
	}
	\label{fig:add_Cart}
\end{figure}

At this point, some readers may have a natural concern: why  we do not directly apply the Euclidean SCMS algorithm to the Cartesian coordinates $\left\{\bm{X}_1,...,\bm{X}_n \right\} \subset \Omega_q$ of the available data points? We discuss the potential downsides of this approach from two different aspects.
	\begin{enumerate}
		\item The Euclidean SCMS algorithm is not intrinsically designed for handling the directional data $\left\{\bm{X}_1,...,\bm{X}_n \right\} \subset \Omega_q$. Directly applying the algorithm to these Cartesian coordinates leads to an estimated ridge not lying on $\Omega_q$. While the $L_2$ normalization is able to standardize the ridge points back to $\Omega_q$, this standardization process will inevitable introduce extra bias.
		
		\item When estimating the underlying density of $\left\{\bm{X}_1,...,\bm{X}_n \right\} \subset \Omega_q$, we know from \eqref{unif_bound} and some KDE literature \citep{Asymp_deri_KDE2011,Scott2015,KDE_t} that the (uniform) rates of convergence of the Euclidean KDE and its derivatives depend on the dimension $(q+1)$ of the ambient space instead of the intrinsic dimension $q$ of directional data. This dimensionality effect also appears in the (linear) convergence of the downstream SCMS algorithm, which, for instance, shrinks the upper bounds of the (linear) convergence radius and step size threshold in Theorem~\ref{SCGA_LC}. Thus, analyzing directional data $\left\{\bm{X}_1,...,\bm{X}_n \right\} \subset \Omega_q$ with the Euclidean KDE and SCMS algorithm will slow down the statistical and algorithmic rates of convergence of the density estimators as well as lower the accuracy of the resulting ridge in recovering the underlying structure inside the dataset.
	\end{enumerate}
	
	\begin{figure}[!t]
		\captionsetup[subfigure]{justification=centering}
		\centering
		\begin{subfigure}[t]{.49\textwidth}
			\centering
			\includegraphics[width=0.9\linewidth]{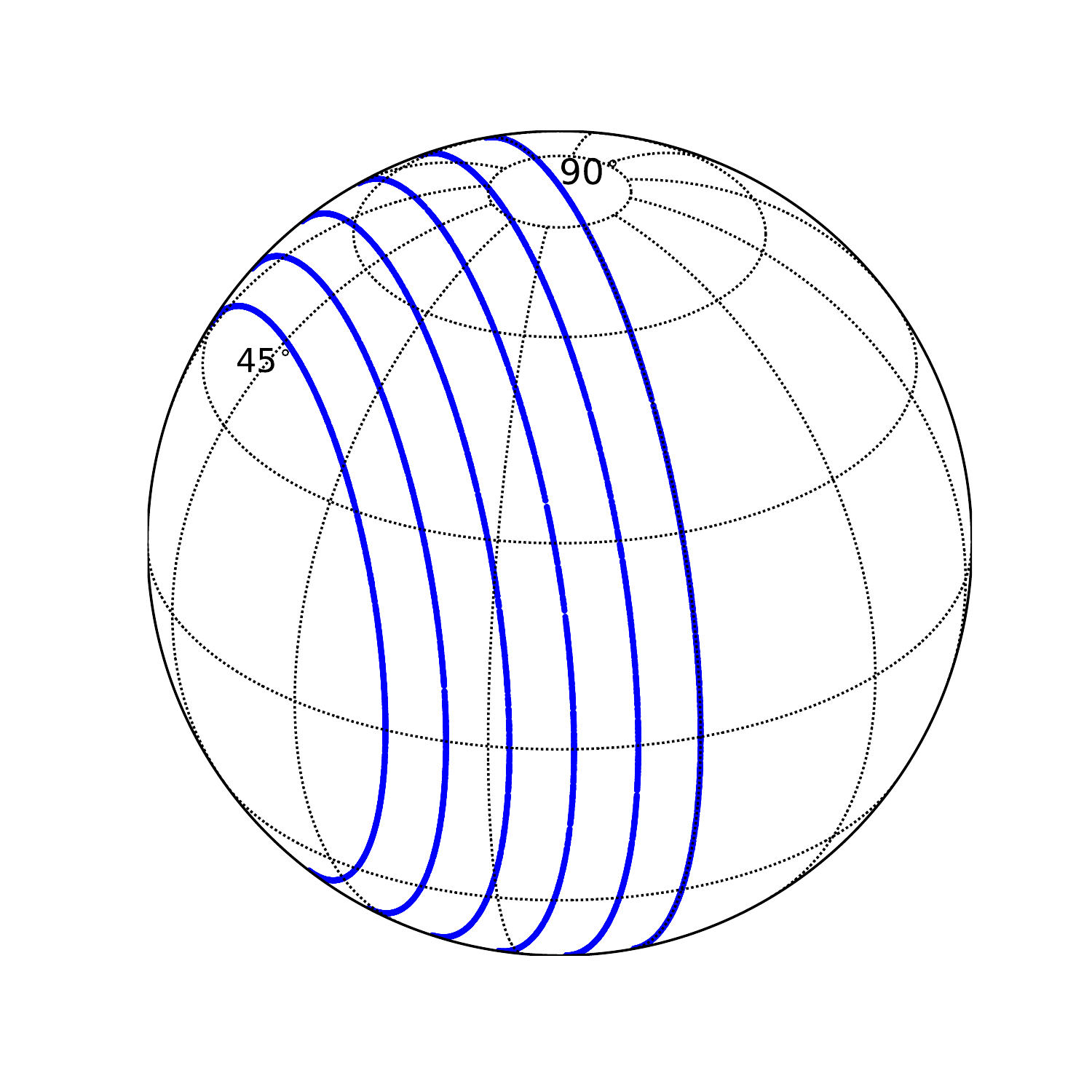}
			\caption{Depiction of underlying true circular structures.}
		\end{subfigure}
		\hfil
		\begin{subfigure}[t]{.49\textwidth}
			\centering
			\includegraphics[width=\linewidth]{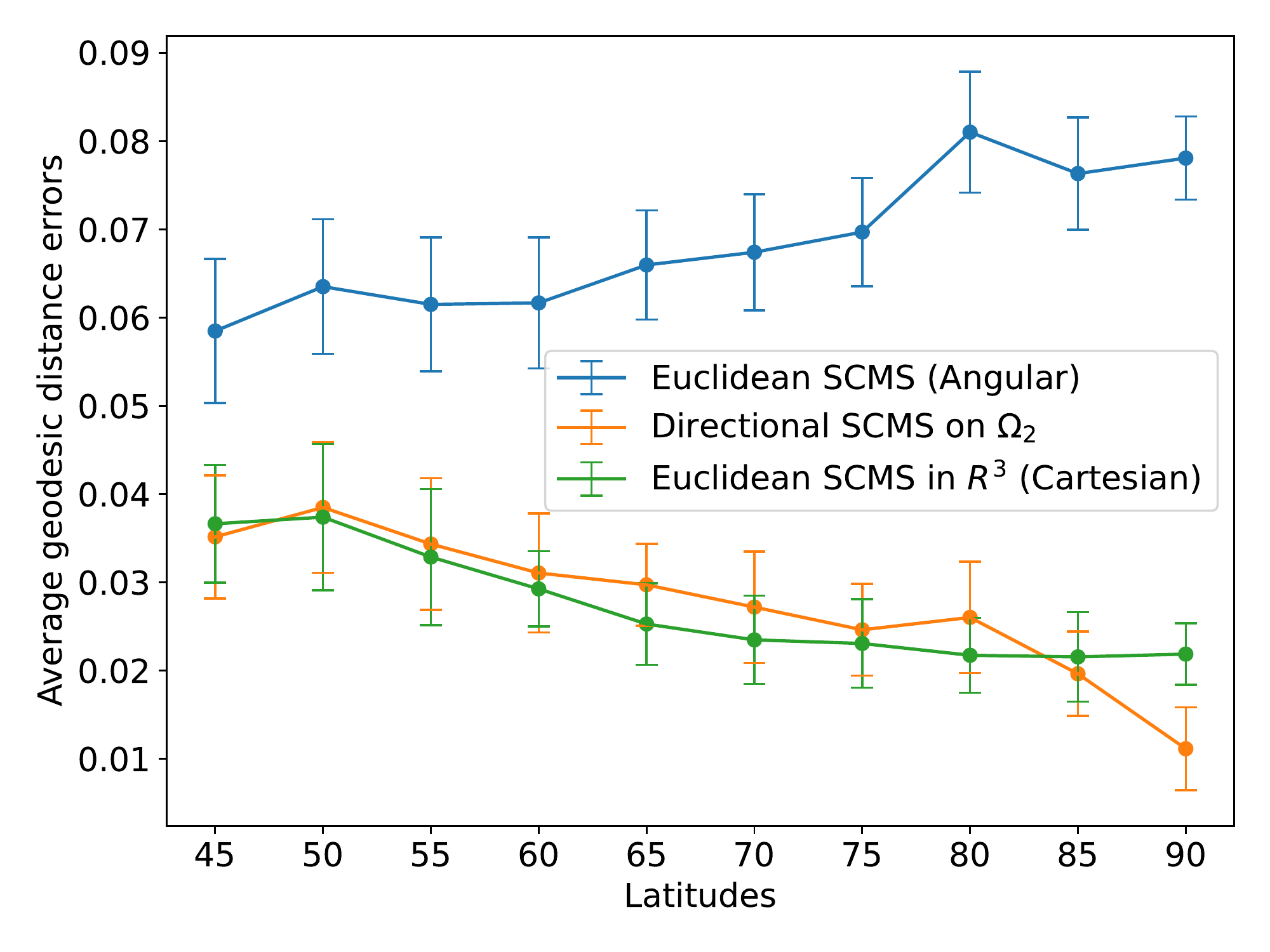}
			\caption{Average geodesic distance errors.}
		\end{subfigure}
		\begin{subfigure}[t]{.49\textwidth}
			\centering
			\includegraphics[width=\linewidth]{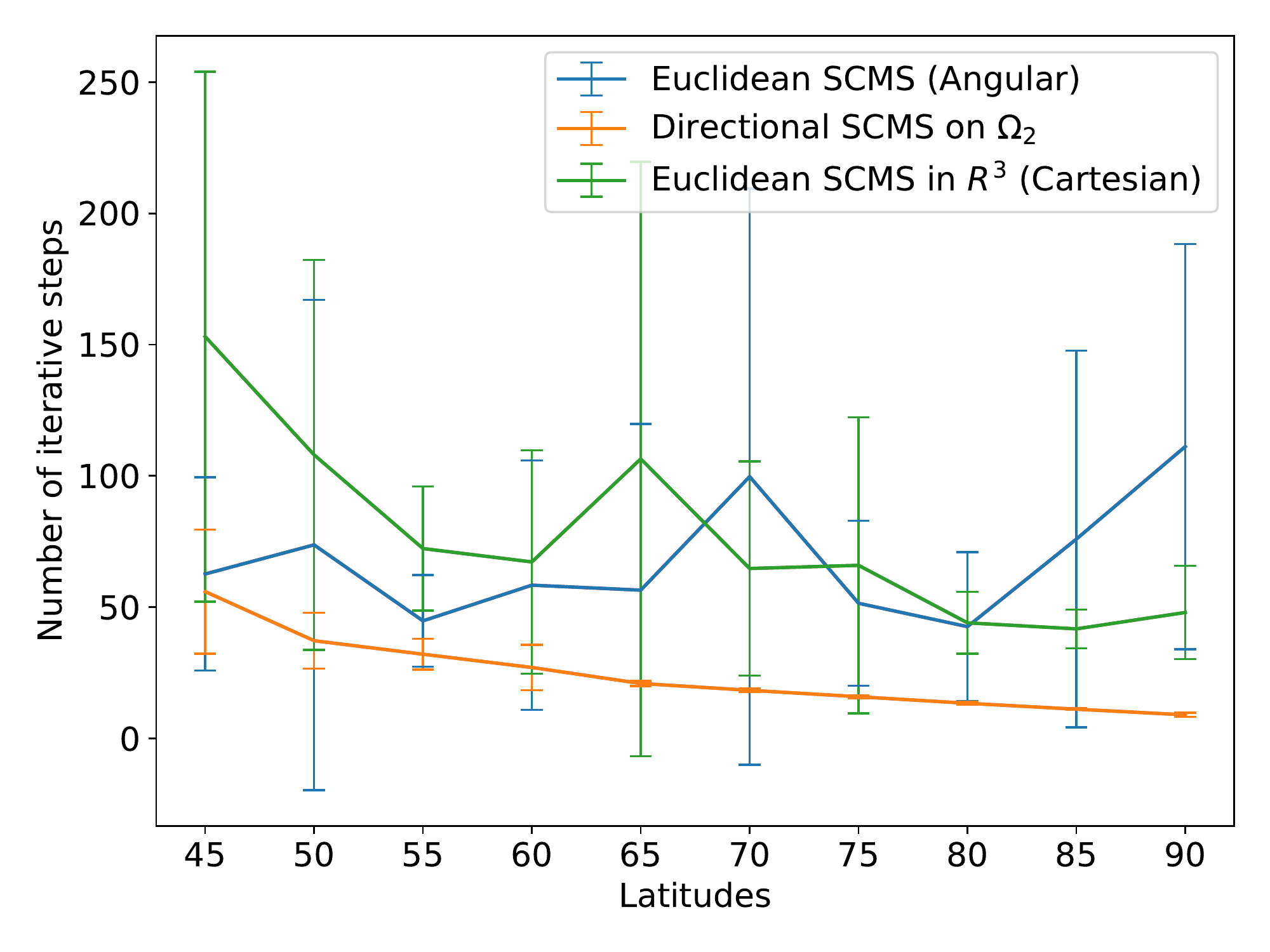}
			\caption{Number of iteration steps.}
		\end{subfigure}
		\hfil
		\begin{subfigure}[t]{.49\textwidth}
			\centering
			\includegraphics[width=\linewidth]{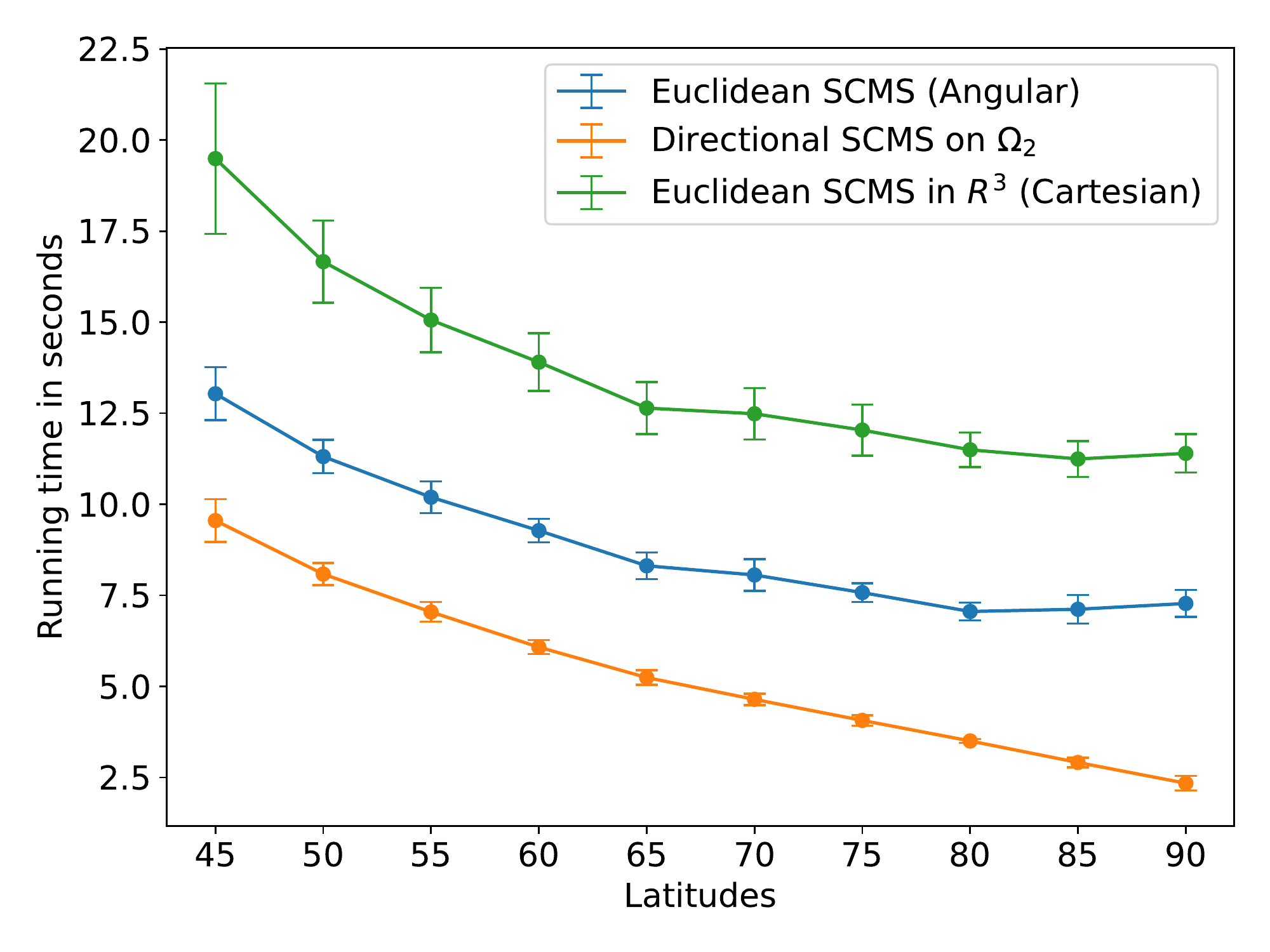}
			\caption{Running time.}
		\end{subfigure}
		\caption{Euclidean and directional SCMS algorithms applied to the simulated datasets whose true structures are circles on $\Omega_2$ attaining their maximum latitudes from $45^{\circ}$ to $90^{\circ}$, respectively. The dots on each line plot in the panels (b-d) are the means of the associated statistics for the repeated experiments, while the error bars indicate their corresponding standard deviations.}
		\label{fig:val_lat}
	\end{figure}
	
	To support our above explanations, we extend our simulation study in Figure~\ref{fig:add_Cart} as follows. We vary the maximum latitude attained by the underlying (intrinsic) circular structure on $\Omega_2$ from $45^{\circ}$ to $90^{\circ}$ while keeping the circle parallel to the original great circle connecting the North and South Poles of $\Omega_2$; see the panel (a) in Figure~\ref{fig:val_lat} for an illustration. For each of these underlying circles, we follow the same sampling scheme as in Figure~\ref{fig:add_Cart}, \emph{i.e.}, sampling 1000 points uniformly on the circle with some i.i.d. additive Gaussian noises $N(0,0.2^2)$ to their Cartesian coordinates and $L_2$ normalization back to $\Omega_2$. The Cartesian coordinates of the simulated points from each circular structure are denoted by $\left\{\bm{X}_1,...,\bm{X}_{1000} \right\} \subset \Omega_2$ while their angular coordinates are represented by $\left\{\bm{Y}_1,...,\bm{Y}_{1000} \right\} \subset [-180^{\circ}, 180^{\circ})\times \left[-90^{\circ}, 90^{\circ}\right]$. Then, we apply our directional SCMS algorithm to $\left\{\bm{X}_1,...,\bm{X}_{1000} \right\}$ from each of these simulated datasets. Moreover, the Euclidean SCMS algorithm is applied to both the angular coordinates $\left\{\bm{Y}_1,...,\bm{Y}_{1000} \right\}$ and Cartesian coordinates $\left\{\bm{X}_1,...,\bm{X}_{1000} \right\}$ from each of these simulated datasets, where we consider $\left\{\bm{X}_1,...,\bm{X}_{1000} \right\}$ as a dataset in the ambient space $\mathbb{R}^3$ in the latter case. Here, the sets of initial points for the Euclidean and directional SCMS algorithms are the simulated datasets themselves. Finally, we compute the average geodesic distance errors on $\Omega_2$ from the resulting ridges to the corresponding true circular structures. To reduce the randomness of our simulation studies, we also repeat the above sampling and experimental procedures 20 times for each true circular structure.
	
	
	We present our comparisons of the Euclidean and directional SCMS algorithms based on three metrics in Figure~\ref{fig:val_lat}: (i) average geodesic distance errors between the estimated ridges and the true circular structures, (ii) the number of iteration steps, and (iii) the running time. Notice that, as the latitudes of the underlying circular structures increase, the distance errors of (Euclidean) ridges based on the Euclidean SCMS algorithm applied on the angular coordinates $\left\{\bm{Y}_1,...,\bm{Y}_{1000} \right\}$ rise. Conversely, the distance errors of directional ridges and the ridges based on the Euclidean SCMS algorithm in $\mathbb{R}^3$ decreases when the true circular structures climb on $\Omega_2$; see the panel (b) of Figure~\ref{fig:val_lat}. While the performances of our directional SCMS algorithm and the Euclidean SCMS algorithm in $\mathbb{R}^3$ are almost indistinguishable in terms of the average geodesic distance errors, our directional SCMS algorithm significantly outperforms the Euclidean SCMS algorithm with regards to time efficiency; see the panels (c-d) of Figure~\ref{fig:val_lat}. Note that the Euclidean SCMS algorithm exhibits high variance in the number of iteration steps under the repeated experiments, because each simulated dataset may contain some outliers that are far away from the true circular structure on $\Omega_2$ and the Euclidean SCMS algorithm requires exceptionally large iterative steps to converge when initialized from these outliers. Our directional SCMS algorithm, however, is stabler in its iterative step due to the fact that it is adaptive to the geometry of $\Omega_2$.

Other potential issues of analyzing directional data with Euclidean methods and ignoring the curvature of $\Omega_2$ can be found in \cite{chrisman2017calculating}. In summary, it is highly inadequate and inefficient to handle directional data with the Euclidean KDE and SCMS algorithm, which calls for the needs to introduce the directional KDE \eqref{Dir_KDE} and propose our well-designed SCMS algorithm for analyzing directional data (Algorithm~\ref{Algo:Dir_SCMS}).

\section{Normal Space of the Euclidean Density Ridge}
\label{App:Normal_space_Eu}

As we will refer to conditions (A1-3) frequently in the next two sections, we restate them here:
\begin{itemize}
	\item {\bf (A1)} (\emph{Differentiability}) We assume that $p$ is bounded and at least four times differentiable with bounded partial derivatives up to the fourth order for every $\bm{x}\in \mathbb{R}^D$.
	\item {\bf (A2)} (\emph{Eigengap}) We assume that there exist constants $\rho>0$ and $\beta_0>0$ such that $\lambda_{d+1}(\bm{y}) \leq -\beta_0$ and $\lambda_d(\bm{y}) -\lambda_{d+1}(\bm{y}) \geq \beta_0$ for any $\bm{y} \in R_d\oplus \rho$.
	\item {\bf (A3)} (\emph{Path Smoothness}) Under the same $\rho, \beta_0>0$ in (A2), we assume that there exists another constant $\beta_1 \in (0,\beta_0)$ such that
	\begin{align*}
	D^{\frac{3}{2}}\norm{U_d^{\perp}(\bm{y}) \nabla p(\bm{y})}_2 \norm{\nabla^3 p(\bm{y})}_{\max} &\leq \frac{\beta_0^2}{2},\\
	d \cdot D^{\frac{3}{2}} \norm{\nabla p(\bm{x})}_2 \norm{\nabla^3 p(\bm{x})}_{\max} & \leq \beta_0(\beta_0-\beta_1)
	\end{align*}
	for all $\bm{y} \in R_d\oplus \rho$ and $\bm{x}\in R_d$. 
\end{itemize}

Given a matrix-valued function $B: \mathbb{R}^D \to \mathbb{R}^{m\times n}$, its gradient $\nabla B(\bm{x})$ will be an $m\times n \times D$ array defined as $\left[B(\bm{x}) \right]_{ijk} = \frac{\partial}{\partial x_k} B(\bm{x})_{ij}$.
The derivative of $B$ in the directional of a vector $\bm{z}\in \mathbb{R}^D$ is defined as:
$$B'(\bm{x};\bm{z}) \equiv \lim_{\epsilon \to 0} \frac{B(\bm{x}+\epsilon \bm{z}) - B(\bm{x})}{\epsilon}=\nabla B(\bm{x})\bm{z}.$$
When the matrix $A(\bm{x})=\nabla\nabla f(\bm{x}) \in \mathbb{R}^{D\times D}$, we will use the notation $\nabla\nabla f'(\bm{x};\bm{z}) = \nabla^3 f(\bm{x}) \bm{z} \equiv \bm{z}^T \nabla^3 f(\bm{x})$ interchangeably to denote its directional derivative along $\bm{z}$.

Recall that an order-$d$ ridge of the density $p$ in $\mathbb{R}^D$ is the collection of points defined as:
\begin{equation*}
R_d = \left\{\bm{x}\in \mathbb{R}^D: G_d(\bm{x})=\bm{0}, \lambda_{d+1}(\bm{x}) <0 \right\} = \left\{\bm{x}\in \mathbb{R}^D: V_d(\bm{x})^T \nabla p(\bm{x})=\bm{0}, \lambda_{d+1}(\bm{x}) <0 \right\}.
\end{equation*} 
Lemma~\ref{normal_reach_prop} below shows that under conditions (A1-3), the Jacobian matrix  $\nabla \left[V_d(\bm{x})^T \nabla p(\bm{x}) \right]$ has rank $D-d$ at every point of $R_d$, and $R_d$ is a $d$-dimensional manifold by the implicit function theorem \citep{Rudin1976}. Consequently, the row space of $\nabla \left[V_d(\bm{x})^T \nabla p(\bm{x}) \right] \in \mathbb{R}^{(D-d)\times D}$ spans the normal space to $R_d$.

If we define $M(\bm{x}) = \nabla \left[V_d(\bm{x})^T \nabla p(\bm{x}) \right]^T = \left[\bm{m}_{d+1}(\bm{x}),...,\bm{m}_D(\bm{x}) \right] \in \mathbb{R}^{D\times (D-d)}$, the derivation in pages 60-63 of \cite{Eberly1996ridges} shows that
\begin{equation}
\label{normal_rows}
\bm{m}_k(\bm{x}) = \left[\lambda_k(\bm{x}) \bm{I}_D +\sum_{i=1}^d \frac{\bm{v}_i(\bm{x})^T \nabla p(\bm{x})}{\lambda_k(\bm{x})-\lambda_i(\bm{x})} \cdot \bm{v}_i(\bm{x})^T \nabla^3 p(\bm{x}) \right] \bm{v}_k(\bm{x})
\end{equation}
for $k=d+1,...,D$, and the column space of $M(\bm{x})$ spans the normal space to $R_d$. Let
\begin{align*}
\Lambda_0(\bm{x}) &= \Diag\left[\lambda_{d+1}(\bm{x}),...,\lambda_D(\bm{x})\right],\\
\Lambda_i(\bm{x}) &= \Diag\left[\frac{1}{\lambda_{d+1}(\bm{x}) - \lambda_i(\bm{x})},..., \frac{1}{\lambda_D(\bm{x}) - \lambda_i(\bm{x})} \right],\\
T_i(\bm{x}) &= \left[\bm{v}_i(\bm{x})^T \nabla p(\bm{x}) \right] \cdot \bm{v}_i(\bm{x})^T \nabla^3 p(\bm{x})
\end{align*}
for $i=1,...,d$. Then, 
\begin{equation}
\label{normal_rows_all}
M(\bm{x}) = V_d(\bm{x}) \Lambda_0(\bm{x}) + \sum_{i=1}^d T_i(\bm{x}) V_d(\bm{x}) \Lambda_i(\bm{x}).
\end{equation}
However, the columns of $M(\bm{x})$ are not orthonormal. Thus, we leverage the orthonormalization in \cite{Asymp_ridge2015} to construct $N(\bm{x})$ whose columns are orthonormal and span the same column space as $M(\bm{x})$ in the following steps. Under the condition that $M(\bm{x}) = \nabla \left[V_d(\bm{x})^T \nabla p(\bm{x}) \right]^T$ has full rank $D-d$ at every point $\bm{x}\in R_d$ (see Lemma~\ref{normal_reach_prop}), $M(\bm{x})^T M(\bm{x})$ is positive definite, and we perform the Cholesky decomposition on it, that is,
\begin{equation}
\label{Cholesky_factor}
M(\bm{x})^T M(\bm{x}) = J(\bm{x}) J(\bm{x})^T,
\end{equation}
where $J(\bm{x}) \in \mathbb{R}^{(D-d)\times (D-d)}$ is a lower triangular matrix whose diagonal elements are positive. We then define 
\begin{equation}
\label{normal_rows_all_ortho}
N(\bm{x}) = M(\bm{x})\left[J(\bm{x})^T \right]^{-1}.
\end{equation} 
Notice that $M(\bm{x}), N(\bm{x}),J(\bm{x})$ intrinsically depend on the dimension $d$ of the ridge $R_d$, but we do not explicate these dependencies in their notations. As discussed in \cite{Asymp_ridge2015}, $M(\bm{x})$ might not be unique because the eigenvalues of $\nabla\nabla p(\bm{x})$ can have their multiplicities greater than 1. Any collection of linearly independent unit eigenvectors of $\nabla\nabla p(\bm{x})$ fits into the above construction for $M(\bm{x})$. However, as will be shown later, this volatility of $M(\bm{x})$ will not affect our results, as we only require the smoothness of $M(\bm{x})^TM(\bm{x})$ to develop a lower bound of $\mathtt{reach}(R_d)$.

\begin{lemma}
	\label{normal_reach_prop}
	Assume conditions (A1-3). Given that $M(\bm{x})$ and $N(\bm{x})$ are defined in \eqref{normal_rows_all} and \eqref{normal_rows_all_ortho}, we have the following properties:
	\begin{enumerate}[label=(\alph*)]
		\item $N(\bm{x})$ and $M(\bm{x})$ have the same column space. In addition,
		$$N(\bm{x})N(\bm{x})^T = M(\bm{x})\left[M(\bm{x})^T M(\bm{x}) \right]^{-1} M(\bm{x})^T.$$
		That is, $N(\bm{x}) N(\bm{x})^T$ is the projection matrix onto the columns of $M(\bm{x})$.
		\item The columns of $N(\bm{x})$ are orthonormal to each other.
		\item For $\bm{x}\in R_d$, the column space of $N(\bm{x})$ is normal to the (tangent) direction of $R_d$ at $\bm{x}$.
		\item For all $\bm{x} \in R_d$, $\rank(N(\bm{x}))=\rank(M(\bm{x})) = D-d$. Moreover, $R_d$ is a $d$-dimensional manifold that contains neither intersections and nor endpoints. Namely, $R_d$ is a finite union of connected and compact manifolds.
		\item For $\bm{x}\in R_d$, all the $(D-d)$ nonzero singular values of $M(\bm{x})$ are greater than $\beta_1>0$ and therefore,
		$$\norm{\left[M(\bm{x})^T M(\bm{x}) \right]^{-1}}_2 \leq \frac{1}{\beta_1^2} \quad \text{ and } \quad \norm{\left[J(\bm{x})^T\right]^{-1}}_2 \leq \frac{1}{\beta_1}.$$
		\item When $\norm{\bm{x}-\bm{y}}_2$ is sufficiently small and $\bm{x},\bm{y} \in R_d\oplus \rho$, 
		$$\norm{N(\bm{x})N(\bm{x})^T - N(\bm{y})N(\bm{y})^T}_{\max} \leq A_0 \left(\norm{p}_{\infty}^{(3)} + \norm{p}_{\infty}^{(4)} \right)^2 \norm{\bm{x}-\bm{y}}_2$$
		for some constant $A_0>0$.
		\item Assume that another density function $q$ also satisfies conditions (A1-3) and $\norm{p-q}_{\infty,3}^*$ is sufficiently small. Then
		$$\norm{N_p(\bm{x}) N_p(\bm{x})^T - N_q(\bm{x}) N_q(\bm{x})^T}_{\max} \leq A_1 \cdot \norm{p-q}_{\infty,3}^*$$
		for some constant $A_1 >0$ and any $\bm{x} \in R_d$, where $N_p(\bm{x})$ is the matrix defined in \eqref{normal_rows_all_ortho} with the underlying density $p$.
		\item The reach of $R_d$ satisfies
		$$\mathtt{reach}(R_d) \geq \min\left\{\frac{\rho}{2}, \frac{\beta_1^2}{A_2 \left(\norm{p}_{\infty}^{(3)} + \norm{p}_{\infty}^{(4)} \right)} \right\}$$
		for some constant $A_2>0$.
	\end{enumerate} 
\end{lemma}

Lemma~\ref{normal_reach_prop} is extended from Lemma 2 in \cite{Asymp_ridge2015} to handle the density ridge $R_d$ with $1\leq d<D$. As our conditions (A1-3) imply the imposed conditions of Lemma 2 in \cite{Asymp_ridge2015}, our proof of Lemma~\ref{normal_reach_prop} essentially follows from their arguments with some minor modifications.

\begin{proof}[Proof of Lemma~\ref{normal_reach_prop}]
	We adopt and generalize parts of the proof of Lemma 2 in \cite{Asymp_ridge2015}.\\
	\noindent (a) This property is a natural corollary of the Cholesky decomposition as:
	$$N(\bm{x}) N(\bm{x})^T = M(\bm{x}) \left[J(\bm{x})^T \right]^{-1} J(\bm{x})^{-1} M(\bm{x})^T = M(\bm{x})\left[M(\bm{x})^T M(\bm{x}) \right]^{-1} M(\bm{x})^T.$$
	
	\noindent (b) Some direct calculations show that
	\begin{align*}
	N(\bm{x})^T N(\bm{x}) &= J(\bm{x})^{-1} M(\bm{x})^T M(\bm{x}) \left[J(\bm{x})^T \right]^{-1}\\
	&= J(\bm{x})^{-1} M(\bm{x})^T M(\bm{x}) \left[J(\bm{x})^T \right]^{-1} J(\bm{x})^{-1} J(\bm{x})\\
	&= J(\bm{x})^{-1} M(\bm{x})^T M(\bm{x}) \left[M(\bm{x})^T M(\bm{x}) \right]^{-1} J(\bm{x})\\
	&= \bm{I}_{D-d}.
	\end{align*}
	
	\noindent (c) It can be proved by the argument of Lemma 1 in \cite{Asymp_ridge2015}. Or, we define an arbitrary parametrized curve $\gamma: (-\epsilon,\epsilon) \to \mathbb{R}^D$ lying within $R_d$ for some $\epsilon >0$. Then $\gamma'(t)$ aligns with the tangent direction at $\gamma(t)\in R_d$. Since $V_d(\gamma(t))^T\nabla p(\gamma(t))=\bm{0}$, taking the derivative with respect to $t$ gives us that
	$$0= \frac{d}{dt} \left[V_d(\gamma(t))^T\nabla p(\gamma(t)) \right] = \nabla\left[V_d(\bm{x})^T\nabla p(\bm{x}) \right] \cdot \gamma'(t)$$
	with $\bm{x}=\gamma(t)$. Hence, by the arbitrariness of $\gamma(t)$, the column of $M(\bm{x}) = \nabla\left[V_d(\bm{x})^T\nabla p(\bm{x}) \right]^T$ is normal to the tangent direction of $R_d$ at $\bm{x}$.\\
	
	\noindent (d) We prove that the $(D-d)$ nonzero singular values of $M(\bm{x}) \in \mathbb{R}^{D\times (D-d)}$ are bounded away from 0. Recall that 
	$$M(\bm{x}) = V_d(\bm{x}) \Lambda_0(\bm{x}) + \sum_{i=1}^d T_i(\bm{x}) V_d(\bm{x}) \Lambda_i(\bm{x})$$
	with 
	\begin{align*}
	\Lambda_0(\bm{x}) &= \Diag\left[\lambda_{d+1}(\bm{x}),...,\lambda_D(\bm{x})\right],\\
	\Lambda_i(\bm{x}) &= \Diag\left[\frac{1}{\lambda_{d+1}(\bm{x}) - \lambda_i(\bm{x})},..., \frac{1}{\lambda_D(\bm{x}) - \lambda_i(\bm{x})} \right],\\
	T_i(\bm{x}) &= \left[\bm{v}_i(\bm{x})^T \nabla p(\bm{x}) \right] \cdot \bm{v}_i(\bm{x})^T \nabla^3 p(\bm{x})
	\end{align*}
	for $i=1,...,d$. Under conditions (A2-3), 
	\begin{align*}
	&\norm{\sum_{i=1}^d T_i(\bm{x}) V_d(\bm{x}) \Lambda_i(\bm{x})}_2 \\
	&\leq \sum_{i=1}^d \norm{T_i(\bm{x})}_2 \cdot \norm{V_d(\bm{x})}_2 \cdot \frac{1}{\beta_0} \quad \text{ by condition (A2)}\\
	&\leq \sum_{i=1}^d \norm{\left[\bm{v}_i(\bm{x})^T \nabla p(\bm{x}) \right] \cdot \bm{v}_i(\bm{x})^T \nabla^3 p(\bm{x})}_2 \cdot \frac{1}{\beta_0} \quad \text{ since } \norm{V_d(\bm{x})}_2=1\\
	&\leq \frac{d \norm{\nabla p(\bm{x})}_2 \cdot D^{\frac{3}{2}} \norm{\nabla^3 p(\bm{x})}_{\max}}{\beta_0}\\
	&\leq \beta_0-\beta_1 \quad \text{ by condition (A3)}.
	\end{align*}
	It shows that all the singular values of $\sum\limits_{i=1}^d T_i(\bm{x}) V_d(\bm{x}) \Lambda_i(\bm{x})$ are less than $\beta_0-\beta_1$. Moreover, under condition (A2) again, all the $(D-d)$ nonzero singular values of $V_d(\bm{x})\Lambda_0(\bm{x})$ are greater than $\beta_0$. By Theorem 3.3.16 in \cite{horn_johnson_1991}, we know that all the $(D-d)$ nonzero singular values of $M(\bm{x})$ are greater than $\beta_0-(\beta_0-\beta_1)=\beta_1 >0$. Therefore, $\rank(M(\bm{x}))=\rank(N(\bm{x}))=D-d$. The rest of the proof follows directly from Claim 4 in \cite{Asymp_ridge2015}.
	
	\noindent (e) By the proof of (d), we already know that all the $(D-d)$ nonzero singular values of $M(\bm{x})$ are greater than $\beta_1 >0$. Thus, $\norm{\left[M(\bm{x})^T M(\bm{x}) \right]^{-1}}_2 \leq \frac{1}{\beta_1^2}$, and
	\begin{align*}
	\norm{\left[J(\bm{x})^T \right]^{-1}}_2 = \sqrt{\frac{1}{\sigma_{\min}\left(J(\bm{x})J(\bm{x})^T \right)}} &= \sqrt{\frac{1}{\sigma_{\min}\left(M(\bm{x})^T M(\bm{x}) \right)}} \\
	&= \sqrt{\norm{\left[M(\bm{x})^T M(\bm{x}) \right]^{-1}}_2} \leq \frac{1}{\beta_1},
	\end{align*}
	where $\sigma_{\min}(A)$ is the smallest singular value of matrix $A$.\\
	
	Finally, the proofs of properties (f), (g), and (h) are essentially the same as the corresponding claims in \cite{Asymp_ridge2015}. We thus omitted them.
\end{proof}

As we have discussed in Remark~\ref{solution_manifold_remark}, property (d) of Lemma~\ref{normal_reach_prop} demonstrates that our imposed assumptions (A1-3) for the ridge $R_d$ is sufficient to imply the critical full-rank condition of its normal space in \cite{YC2020} in order for $R_d$ to be a well-defined solution manifold.

\section{Proofs of Lemma~\ref{limit_step_MS}, Proposition~\ref{SCGA_conv}, and Theorem~\ref{SCGA_LC}}
\label{App:Proofs_Sec3}

\begin{customlem}{3.2}
	Assume conditions (A1) and (E1). The convergence rate of $\hat{g}_n(\bm{x})=\frac{2c_{k,D}}{nh^{D+2}} \sum\limits_{i=1}^n -k'\left(\norm{\frac{\bm{x}-\bm{X}_i}{h}}_2^2 \right)$ is
	\begin{align*}
	h^2 \hat{g}_n(\bm{x}) &= -2c_{k,D} \cdot p(\bm{x}) \int_{\mathbb{R}^D} k'\left(\norm{\bm{u}}_2^2 \right)d\bm{u} +O\left(h^2 \right) + O_P\left(\sqrt{\frac{1}{nh^D}} \right) \\
	&= O(1) + O\left(h^2 \right) + O_P\left(\sqrt{\frac{1}{nh^D}} \right)
	\end{align*}
	for any $\bm{x}\in \mathbb{R}^D$ as $nh^D \to \infty$ and $h\to 0$.
\end{customlem}

Another interpretation of Lemma~\ref{limit_step_MS} is that $\hat{g}_n(\bm{x})$ diverges to infinity at the rate 
$$O\left(h^{-2} \right) + O_P\left(\sqrt{\frac{1}{nh^{D+4}}} \right) = O\left(n^{\frac{2}{D+4}} \right) + O_P\left(1 \right)$$
if we select the bandwidth $h_{\text{opt}} \asymp O\left(n^{-\frac{1}{D+4}} \right)$ to minimize the asymptotic mean integrated square error \citep{KDE_t}, where ``$\asymp$'' stands for the asymptotic equivalence.

\begin{proof}[Proof of Lemma~\ref{limit_step_MS}]
	Note that
	\begin{equation}
	\label{g_n_decom}
	h^2\hat{g}_n(\bm{x}) = \mathbb{E}\left[h^2\hat{g}_n(\bm{x}) \right] + h^2\hat{g}_n(\bm{x}) - \mathbb{E}\left[h^2\hat{g}_n(\bm{x}) \right].
	\end{equation}
	Given the differentiability of $p$ guaranteed by condition (A1), the expectation of $h^2 \hat{g}_n(\bm{x})$ is given by:
	\begin{align*}
	\mathbb{E}\left[h^2\hat{g}_n(\bm{x}) \right] &= \frac{2c_{k,D}}{h^D} \int_{\mathbb{R}^D} -k'\left(\norm{\frac{\bm{x}-\bm{y}}{h}}_2^2 \right) \cdot p(\bm{y}) d\bm{y}\\
	&= 2c_{k,D} \int_{\mathbb{R}^D} -k'\left(\norm{\bm{u}}_2^2 \right) \cdot p(\bm{x}+ h\bm{u}) d\bm{u}  \quad \text{ by } \bm{u}=\frac{\bm{y}-\bm{x}}{h}\\
	&= 2c_{k,D} \int_{\mathbb{R}^D} -k'\left(\norm{\bm{u}}_2^2 \right) \left[p(\bm{x}) + h\nabla p(\bm{x})^T\bm{u} +O(h^2) \right] d\bm{u} \\
	&= 2c_{k,D} \cdot p(\bm{x}) \int_{\mathbb{R}^D} -k'\left(\norm{\bm{u}}_2^2 \right) d\bm{u} +O(h^2).
	\end{align*}
	By condition (E1), the dominating constant $-2c_{k,D} \cdot p(\bm{x}) \int_{\mathbb{R}^D} k'\left(\norm{\bm{u}}_2^2 \right) d\bm{u}$ is finite and therefore, 
	\begin{equation}
	\label{term1}
	\mathbb{E}\left[h^2\hat{g}_n(\bm{x}) \right] = -2c_{k,D} \cdot p(\bm{x}) \int_{\mathbb{R}^D} k'\left(\norm{\bm{u}}_2^2 \right) d\bm{u} + O\left(h^2 \right) = O(1) + O(h^2).
	\end{equation}
	In addition, we calculate the variance of $\hat{g}_n(\bm{x})$ as
	\begin{align*}
	&\text{Var}\left[h^2\hat{g}_n(\bm{x}) \right] \\
	&= \frac{4c_{k,D}^2}{nh^{2D}}\cdot \text{Var}\left[k'\left(\norm{\frac{\bm{x}-\bm{X}_1}{h}}_2^2 \right) \right]\\
	&= \frac{4c_{k,D}^2}{nh^{2D}} \int_{\mathbb{R}^D} k'\left(\norm{\frac{\bm{x}-\bm{y}}{h}}_2^2 \right)^2 p(\bm{y}) \,d\bm{y} - \frac{1}{n} \left\{\mathbb{E}\left[h^2\hat{g}_n(\bm{x}) \right] \right\}^2\\
	&= \frac{4c_{k,D}^2}{nh^D} \int_{\mathbb{R}^D} k'\left(\norm{\bm{u}}_2^2 \right)^2 \cdot p(\bm{x}+h\bm{u}) d\bm{u} -\frac{1}{n}\left[O(1) + O\left(h^2 \right) \right]^2 \quad \text{ by } \bm{u}=\frac{\bm{y}-\bm{x}}{h}\\
	&= \frac{4c_{k,D}^2}{nh^D} \int_{\mathbb{R}^D} k'\left(\norm{\bm{u}}_2^2 \right)^2 \left[p(\bm{x}) + h\nabla p(\bm{x})^T\bm{u} +O(h^2) \right] d\bm{u} + o\left(\frac{1}{nh^D} \right)\\
	&= \frac{4c_{k,D}^2 \cdot p(\bm{x})}{nh^D} \int_{\mathbb{R}^D} k'\left(\norm{\bm{u}}_2^2 \right)^2 d\bm{u} + o\left(\frac{1}{nh^D} \right),
	\end{align*}
	Again, by condition (E1), the dominating constant $4c_{k,D}^2 \cdot p(\bm{x}) \int_{\mathbb{R}^D} k'\left(\norm{\bm{u}}_2^2 \right)^2 d\bm{u}$ is finite. Thus, by the central limit theorem,
	\begin{align}
	\label{term2}
	\begin{split}
	h^2\hat{g}_n(\bm{x}) - \mathbb{E}\left[h^2\hat{g}_n(\bm{x}) \right] &= \sqrt{\text{Var}\left[h^2\hat{g}_n(\bm{x}) \right]} \cdot  \frac{h^2\hat{g}_n(\bm{x}) - \mathbb{E}\left[h^2\hat{g}_n(\bm{x}) \right]}{\sqrt{\text{Var}\left[h^2\hat{g}_n(\bm{x}) \right]}}\\
	&= \sqrt{\text{Var}\left[h^2\hat{g}_n(\bm{x}) \right]} \cdot \bm{Z}_n(\bm{x})\\
	&= O_P\left(\sqrt{\frac{1}{nh^D}} \right),
	\end{split}
	\end{align}
	where $\bm{Z}_n(\bm{x}) \stackrel{d}{\to} \mathcal{N}_D(\bm{0}, \bm{I}_D)$. Combining \eqref{g_n_decom}, \eqref{term1}, and \eqref{term2}, we conclude that 
	\begin{align*}
	h^2 \hat{g}_n(\bm{x}) &= -2c_{k,D} \cdot p(\bm{x}) \int_{\mathbb{R}^D} k'\left(\norm{\bm{u}}_2^2 \right) +O\left(h^2 \right) + O_P\left(\sqrt{\frac{1}{nh^D}} \right) \\
	&= O(1) + O\left(h^2 \right) + O_P\left(\sqrt{\frac{1}{nh^D}} \right)
	\end{align*}
	for any $\bm{x}\in \mathbb{R}^D$ as $nh^D \to \infty$ and $h\to 0$.
\end{proof}

\begin{remark}
	Some previous research papers on the mean shift algorithm \cite{MS2007_pf, MS_EM2007, MS2015_Gaussian,MS_consist2016} have already justified that the algorithm converges to a local mode of the KDE $\hat{p}_n$ when its local modes are isolated and the algorithm starts within some small neighborhoods of these estimated local modes. 
	Lemma~\ref{limit_step_MS} here provides a (probabilistic) perspective on the linear convergence of the mean shift algorithm.
	It is well-known that the set of the true local modes of $p$ can be approximated by the set of estimated modes defined by $\hat{p}_n$ \citep{Mode_clu2016}. Moreover, around the true local modes of the density $p$, one can argue that $\hat{p}_n$ is strongly convex and has a Lipschitz gradient with probability tending to 1 by the uniform consistency of $\nabla \hat{p}_n$ and $\nabla\nabla \hat{p}_n$ as $h\to 0$ and $\frac{nh^{D+4}}{|\log h|} \to \infty$; see the uniform bounds \eqref{unif_bound}. Hence, by some standard results in optimization theory (\emph{e.g.}, Chapter 3 in \cite{SB2015}), a sample-based gradient ascent algorithm with objective function $\hat{p}_n$ converges linearly to (estimated) local modes around their neighborhoods as long as its step size is below some threshold value. Finally, recall that (i) the mean shift algorithm is a special variant of the sample-based gradient ascent method with an adaptive size $\eta_{n,h}^{(t)}$ by \eqref{MS_iteration} and (ii) $\eta_{n,h}^{(t)}$ can be sufficiently small but bounded away from 0 when $nh^D$ is large and $h$ is small by Lemma~\ref{limit_step_MS}; see also Remark~\ref{SCMS_stepsize_remark}. Therefore, the mean shift algorithm will converge linearly with high probability around some small neighborhood of the local modes of $\hat{p}_n$ when the sample size $n$ is sufficiently large and $h$ is chosen to be small. 
\end{remark}

\begin{customprop}{3.3}[Convergence of the SCGA Algorithm]
		For any SCGA sequence $\big\{\bm{x}^{(t)}\big\}_{t=0}^{\infty}$ defined by \eqref{SCGA_update_pop} with $0<\eta < \frac{2}{D\norm{p}_{\infty}^{(2)}}$, the following properties hold.
		\begin{enumerate}[label=(\alph*)]
			\item Under condition (A1), the objective function sequence $\big\{p(\bm{x}^{(t)}) \big\}_{t=0}^{\infty}$ is non-decreasing and converges.
			
			\item Under condition (A1), $\lim_{t\to\infty} \norm{V_d(\bm{x}^{(t)})^T \nabla p(\bm{x}^{(t)})}_2 =\lim_{t\to\infty} \norm{\bm{x}^{(t+1)} -\bm{x}^{(t)}}_2 = 0$.
			
			\item Under conditions (A1-3), $\lim_{t\to\infty} d_E(\bm{x}^{(t)}, R_d) =0$ whenever $\bm{x}^{(0)} \in R_d\oplus r_1$ with the convergence radius $r_1$ satisfying
			$$0<r_1 < \min\left\{\frac{\rho}{2},\; \frac{\beta_1^2}{A_2 \left(\norm{p}_{\infty}^{(3)} + \norm{p}_{\infty}^{(4)} \right)},\; \frac{\beta_1}{A_4(p)} \right\},$$
			where $A_2>0$ is a constant defined in (h) of Lemma~\ref{normal_reach_prop} while $A_4(p) >0$ is a quantity depending on both the dimension $D$ and functional norm $\norm{p}_{\infty,4}^*$ up to the fourth-order (partial) derivatives of $p$. 
		\end{enumerate}
	\end{customprop}
	
	\begin{proof}[Proof of Proposition~\ref{SCGA_conv}]
		(a) We first derive the following fact about the objective function $p$.
		
		$\bullet$ \emph{Fact 1}. Given (A1), $p$ is $D\norm{p}_{\infty}^{(2)}$-smooth, that is, $\nabla p$ is $D\norm{p}_{\infty}^{(2)}$-Lipschitz.
		
		This fact follows from the differentiability of $p$ ensured by condition (A1) and Taylor's theorem that
		\begin{align*}
		\norm{\nabla p(\bm{x}) - \nabla p(\bm{y})}_2 &\leq \norm{\nabla\nabla p(\tilde{\bm{y}})}_2 \cdot \norm{\bm{x}-\bm{y}}_2 \leq D\norm{p}_{\infty}^{(2)} \cdot \norm{\bm{x}-\bm{y}}_2
		\end{align*}
		for any $\bm{x},\bm{y} \in \mathbb{R}^D$, where $\tilde{\bm{y}}$ is within a $\norm{\bm{x}-\bm{y}}_2$-neighborhood of $\bm{y}$. Moreover,
		\begin{align}
		\label{L_smoothness}
		\begin{split}
		\left|p(\bm{y}) - p(\bm{x}) -\nabla p(\bm{x})^T(\bm{y}-\bm{x}) \right| &= \int_0^1 \nabla p(\bm{x}+t(\bm{y}-\bm{x}))^T (\bm{y}-\bm{x}) dt - \nabla p(\bm{x})^T (\bm{y}-\bm{x})\\
		&\leq \int_0^1 \norm{\nabla p(\bm{x}+t(\bm{y}-\bm{x})) - p(\bm{x})}_2 \cdot \norm{\bm{y}-\bm{x}}_2 dt\\
		&\leq \frac{D\norm{p}_{\infty}^{(2)}}{2}\cdot \norm{\bm{y}-\bm{x}}_2^2.
		\end{split}
		\end{align}
		When $0<\eta < \frac{2}{D\norm{p}_{\infty}^{(2)}}$, we have that
		\begin{align}
		\label{SCGA_ascending}
		\begin{split}
		&p(\bm{x}^{(t+1)}) - p(\bm{x}^{(t)}) \\
		&= p\left(\bm{x}^{(t)} + \eta \cdot V_d(\bm{x}^{(t)}) V_d(\bm{x}^{(t)})^T \nabla p(\bm{x}^{(t)}) \right) -p(\bm{x}^{(t)})\\
		&\geq \nabla p(\bm{x}^{(t)})^T \cdot \eta V_d(\bm{x}^{(t)}) V_d(\bm{x}^{(t)})^T \nabla p(\bm{x}^{(t)}) - \frac{D\norm{p}_{\infty}^{(2)}}{2} \eta^2 \norm{V_d(\bm{x}^{(t)})^T \nabla p(\bm{x}^{(t)})}_2^2 \quad \text{ by \eqref{L_smoothness}}\\
		&= \eta \left(1-\frac{D\norm{p}_{\infty}^{(2)} \eta}{2}  \right)\norm{V_d(\bm{x}^{(t)})^T \nabla p(\bm{x}^{(t)})}_2^2 \geq 0,
		\end{split}
		\end{align}
		showing that the objective function $p$ is non-decreasing along $\big\{\bm{x}^{(t)} \big\}_{t=0}^{\infty}$. Given the boundedness of $p$ guaranteed by condition (A1), we know that the sequence $\left\{p(\bm{x}^{(t)}) \right\}_{t=0}^{\infty}$ is bounded. Thus, $\left\{p(\bm{x}^{(t)}) \right\}_{t=0}^{\infty}$ also converges.\\
		
		\noindent (b) From (a), we know that when $0<\eta < \frac{2}{D\norm{p}_{\infty}^{(2)}}$,
		\begin{align*}
		p(\bm{x}^{(t+1)}) - p(\bm{x}^{(t)}) &\geq \eta \left(1-\frac{D\norm{p}_{\infty}^{(2)} \eta}{2}  \right)\norm{V_d(\bm{x}^{(t)})^T \nabla p(\bm{x}^{(t)})}_2^2 \\
		&= \left(\frac{2-D\norm{p}_{\infty}^{(2)}\eta}{2\eta} \right) \norm{\bm{x}^{(t+1)} -\bm{x}^{(t)}}_2^2.
		\end{align*}
		Since the sequence $\left\{p(\bm{x}^{(t)}) \right\}_{t=0}^{\infty}$ converges as $t\to\infty$, it follows that
		$$\lim_{t\to\infty} \norm{V_d(\bm{x}^{(t)})^T \nabla p(\bm{x}^{(t)})}_2 =0 \quad \text{ and } \quad \lim_{t\to\infty} \norm{\bm{x}^{(t+1)} -\bm{x}^{(t)}}_2 = 0.$$
		
		\noindent (c) Given condition (A2) and the fact that $r_1 < \frac{\rho}{2}$, we know that
		$$\bm{x}\in R_d\oplus r_1 \text{ and } \norm{V_d(\bm{x})^T \nabla p(\bm{x})}_2=0 \quad \text{ if and only if} \quad \bm{x}\in R_d.$$ 
		Let $\bm{z}^{(t)} = \pi_{R_d}(\bm{x}^{(t)})$ denote the projection of $\bm{x}^{(t)}$ in the SCGA sequence onto the ridge $R_d$. Since $r_1 \leq \mathtt{reach}(R_d)$ by (h) of Lemma~\ref{normal_reach_prop}, $\bm{z}^{(t)}$ is well-defined when $\bm{x}^{(t)}\in R_d\oplus r_1$. Given that the definition of $M(\bm{z}^{(t)}) = \nabla\left[V_d(\bm{z}^{(t)})^T \nabla p(\bm{z}^{(t)}) \right]^T \in \mathbb{R}^{D\times (D-d)}$ in \eqref{normal_rows_all}, we know that
		\begin{align}
		\label{error_bound}
		\begin{split}
		&\norm{V_d(\bm{x}^{(t)})^T \nabla p(\bm{x}^{(t)})}_2 \\
		&= \norm{V_d(\bm{x}^{(t)})^T \nabla p(\bm{x}^{(t)}) - \underbrace{V_d(\bm{z}^{(t)})^T \nabla p(\bm{z}^{(t)})}_{=0}}_2\\
		&= \norm{\int_0^1 M\left(\bm{z}^{(t)} + \epsilon(\bm{x}^{(t)}-\bm{z}^{(t)})\right)^T (\bm{x}^{(t)} -\bm{z}^{(t)}) \,d\epsilon}_2 \quad \text{ by Taylor's theorem}\\
		&= \norm{M(\bm{z}^{(t)})^T (\bm{x}^{(t)} - \bm{z}^{(t)}) + \int_0^1\left[M\left(\bm{z}^{(t)} + \epsilon(\bm{x}^{(t)}-\bm{z}^{(t)})\right) - M(\bm{z}^{(t)}) \right]^T (\bm{x}^{(t)} -\bm{z}^{(t)}) \,d\epsilon}_2\\
		&\geq \norm{M(\bm{z}^{(t)})^T (\bm{x}^{(t)} - \bm{z}^{(t)})}_2 - \frac{1}{2} \sup_{\epsilon \in [0,1]} \norm{\nabla M\left(\bm{z}^{(t)} + \epsilon(\bm{x}^{(t)}-\bm{z}^{(t)})\right)^T}_2 \norm{\bm{x}^{(t)} -\bm{z}^{(t)}}_2^2\\
		&\stackrel{\text{(i)}}{\geq} \beta_1 \norm{\bm{x}^{(t)} -\bm{z}^{(t)}}_2 - \frac{A_4(p)}{2} \norm{\bm{x}^{(t)} -\bm{z}^{(t)}}_2^2\\
		&= d_E(\bm{x}^{(t)}, R_d) \cdot \left(\beta_1 -\frac{A_4(p)}{2} d_E(\bm{x}^{(t)}, R_d) \right)\\
		&\geq \frac{\beta_1}{2} \cdot d_E(\bm{x}^{(t)}, R_d)
		\end{split}
		\end{align}
		when $\bm{x}^{(t)} \in R_d\oplus r_1$, where we leverage (e) of Lemma~\ref{normal_reach_prop} to obtain the inequality (i). More specifically, $M(\bm{z}^{(t)})^T$ is a full row rank matrix by (d) of Lemma~\ref{normal_reach_prop} and $\bm{x}^{(t)} -\bm{z}^{(t)}$ lies within the row space of $M(\bm{z}^{(t)})^T$ because $\bm{x}^{(t)} -\bm{z}^{(t)}$ is normal to $R_d$ at $\bm{z}^{(t)}$. Since the nonzero singular values of $M(\bm{z}^{(t)})$ are lower bounded by $\beta_1>0$, it follows that $\norm{M(\bm{z}^{(t)})^T (\bm{x}^{(t)} - \bm{z}^{(t)})}_2 \geq \beta_1 \norm{\bm{x}^{(t)} - \bm{z}^{(t)}}_2$. From the above derivation, we also know that $A_4(p) >0$ is indeed the supremum norm of $\nabla M(\bm{x})^T = \nabla\nabla \left[V_d(\bm{x})^T \nabla p(\bm{x})\right]^T$ over the line segment connecting $\bm{x}^{(t)}$ and $\bm{z}^{(t)}$, which depends on the uniform functional norm $\norm{p}_{\infty,4}^*$ of the partial derivatives of $p$ up to the fourth order. The result follows from (b).
\end{proof}

The following Davis-Kahan theorem \citep{Davis_Kahan} is one of the most notable theorems in matrix perturbation theory. We present the theorem in a modified version from \cite{Spectral_Clustering_t,Non_ridge_est2014}. Other useful variants of the Davis-Kahan theorem can be found in \cite{Yu_DK2014}.

\begin{lemma}[Davis-Kahan]
	\label{Davis_K}
	Let $H$ and $\tilde{H}$ be two symmetric matrices in $\mathbb{R}^{D\times D}$, whose spectra (Definition 1.1.4 in \citealt{HJ2012}) are $\sigma(H)$ and $\sigma(\tilde{H})$, and $S_1 \subset \mathbb{R}$ be an interval. Denote by $\sigma_{S_1}(H)$ the set of eigenvalues of $H$ that are contained in $S_1$, and by $V_1$ the matrix whose columns are the corresponding (unit) eigenvectors to $\sigma_{S_1}(H)$ (more formally, $V_1$ is the image of the spectral projection induced by $\sigma_{S_1}(H)$). Denote by $\sigma_{S_1}(\tilde{H})$ and $\tilde{V}_1$ the analogous quantities for $\tilde{H}$. If 
	$$\delta:= \inf\left\{|\lambda -\tilde{\lambda}|: \lambda \in \sigma_{S_1}(H), \tilde{\lambda} \in \sigma(\tilde{H})\setminus \sigma_{S_1}(\tilde{H}) \right\} >0,$$
	then the distance $d(V_1,\tilde{V}_1):=\norm{\sin \bm{\Theta}(V_1,\tilde{V}_1)}$ between two subspaces is bounded by
	$$d(V_1,\tilde{V}_1) \leq \frac{\norm{H - \tilde{H}}}{\delta}$$
	for any orthogonally invariant norm $\norm{\cdot}$, such as the Frobenius norm $\norm{\cdot}_F$ and the $L_2$-operator norm $\norm{\cdot}_2$, where $\bm{\Theta}(V_1,\tilde{V}_1)$ is a diagonal matrix with the ascending principal angles between the column spaces of $V_1$ and $\tilde{V}_1$ on the diagonal.
\end{lemma}

Note that when we take the Frobenius norm in Lemma~\ref{Davis_K}, $\norm{\sin \bm{\Theta}(V_1,\tilde{V}_1)}_F=\frac{\norm{V_1V_1^T - \tilde{V}_1\tilde{V}_1^T}_F}{\sqrt{2}}$ by some simple algebra. Consequently, we will utilize the following inequality from the Davis-Kahan theorem in our subsequent proofs as:
\begin{align}
\label{Davis_Kahan_ineq}
\begin{split}
\norm{V_1V_1^T - \tilde{V}_1\tilde{V}_1^T}_2 \leq \norm{V_1V_1^T - \tilde{V}_1\tilde{V}_1^T}_F = \sqrt{2} \norm{\sin \bm{\Theta}(V_1,\tilde{V}_1)}_F &\leq \frac{\sqrt{2}\norm{H-\tilde{H}}_F}{\delta}\\ 
&\leq \frac{\sqrt{2}D \norm{H-\tilde{H}}_{\max}}{\delta}\\
&\leq \frac{\sqrt{2}D \norm{H-\tilde{H}}_2}{\delta}.
\end{split}
\end{align}

\begin{customthm}{3.6}[Linear Convergence of the SCGA Algorithm]
	Assume conditions (A1-4) throughout the theorem.
	\begin{enumerate}[label=(\alph*)]
		\item \textbf{Q-Linear convergence of $\norm{\bm{x}^{(t)}-\bm{x}^*}_2$}: Consider a convergence radius $r_2>0$ satisfying
		$$0<r_2< \min\left\{\frac{\rho}{2}, \frac{\beta_1^2}{A_2\left(\norm{p}_{\infty}^{(3)} + \norm{p}_{\infty}^{(4)} \right)}, \frac{\beta_1}{A_4(p)}, \frac{3\beta_0}{4D\left(6\norm{p}_{\infty}^{(2)}\beta_2^2 \rho + D^{\frac{1}{2}}\norm{p}_{\infty}^{(3)} \right)} \right\},$$
		where $A_2 >0$ is the constant defined in (h) of Lemma~\ref{normal_reach_prop} and $A_4(p) >0$ is a quantity defined in (c) of Proposition~\ref{SCGA_conv} that depends on both the dimension $D$ and the functional norm $\norm{p}_{\infty,4}^*$ up to the fourth-order derivative of $p$. Whenever $0<\eta \leq \min\left\{\frac{4}{\beta_0}, \frac{1}{D\norm{p}_{\infty}^{(2)}} \right\}$ and the initial point $\bm{x}^{(0)} \in \text{Ball}_D(\bm{x}^*, r_2)$ with $\bm{x}^* \in R_d$, we have that
		$$\norm{\bm{x}^{(t)}-\bm{x}^*}_2 \leq \Upsilon^t \norm{\bm{x}^{(0)}-\bm{x}^*}_2 \quad \text{ with } \quad \Upsilon = \sqrt{1-\frac{\beta_0\eta}{4}}.$$
		\item \textbf{R-Linear convergence of $d_E(\bm{x}^{(t)}, R_d)$}: Under the same radius $r_2>0$ in (a), we have that whenever $0<\eta \leq \min\left\{\frac{4}{\beta_0},\frac{1}{D\norm{p}_{\infty}^{(2)}}\right\}$ and the initial point $\bm{x}^{(0)} \in \text{Ball}_D(\bm{x}^*, r_2)$ with $\bm{x}^* \in R_d$,
		$$d_E(\bm{x}^{(t)}, R_d) \leq \Upsilon^t \norm{\bm{x}^{(0)}-\bm{x}^*}_2 \quad \text{ with } \quad \Upsilon=\sqrt{1-\frac{\beta_0 \eta}{4}}.$$
	\end{enumerate}
	We further assume conditions (E1-2) in the rest of statements. If $h \to 0$ and $\frac{nh^{D+4}}{|\log h|} \to \infty$,
	\begin{enumerate}[label=(c)]
		\item \textbf{Q-Linear convergence of $\norm{\hat{\bm{x}}^{(t)}-\bm{x}^*}_2$}: under the same radius $r_2>0$ and $\Upsilon=\sqrt{1-\frac{\beta_0 \eta}{4}}$ in (a), we have that
		$$\norm{\hat{\bm{x}}^{(t)} -\bm{x}^*}_2 \leq \Upsilon^t \norm{\hat{\bm{x}}^{(0)} -\bm{x}^*}_2 + O(h^2) + O_P\left(\sqrt{\frac{|\log h|}{nh^{D+4}}}\right)$$
		with probability tending to 1 whenever $0<\eta \leq \min\left\{\frac{4}{\beta_0}, \frac{1}{D\norm{p}_{\infty}^{(2)}} \right\}$ and the initial point $\hat{\bm{x}}^{(0)} \in \text{Ball}_D(\bm{x}^*, r_2)$ with $\bm{x}^* \in R_d$.
	\end{enumerate}
	\begin{enumerate}[label=(d)]
		\item \textbf{R-Linear convergence of $d_E(\hat{\bm{x}}^{(t)}, R_d)$}: under the same radius $r_2>0$ and $\Upsilon=\sqrt{1-\frac{\beta_0 \eta}{4}}$ in (a), we have that
		$$d_E(\hat{\bm{x}}^{(t)},R_d) \leq \Upsilon^t \norm{\hat{\bm{x}}^{(0)} -\bm{x}^*}_2 + O(h^2) + O_P\left(\sqrt{\frac{|\log h|}{nh^{D+4}}} \right)$$
		with probability tending to 1 whenever $0< \eta \leq \min\left\{\frac{4}{\beta_0}, \frac{1}{D\norm{p}_{\infty}^{(2)}} \right\}$ and the initial point $\hat{\bm{x}}^{(0)} \in \text{Ball}_D(\bm{x}^*, r_2)$ with $\bm{x}^* \in R_d$.
	\end{enumerate}
\end{customthm}

\begin{proof}[Proof of Theorem~\ref{SCGA_LC}]
	The entire proof is inspired by some standard results in optimization theory. However, the objective function $p$ is no longer strongly concave, and we focus on the SCGA iteration instead of the standard gradient ascent method. We first recall the following two facts.\\
	$\bullet$ \emph{Fact 1}. Given condition (A1), $p$ is $D\norm{p}_{\infty}^{(2)}$-smooth, that is, $\nabla p$ is $D\norm{p}_{\infty}^{(2)}$-Lipschitz.\\
	$\bullet$ \emph{Fact 2}. Given conditions (A1-3), we know that $\norm{V_d(\bm{x}^{(t)})^T\nabla p(\bm{x}^{(t)})}_2>0$ for any $\bm{x}^{(t)} \in \text{Ball}_D(\bm{x}^*,r_2) \setminus R_d$ and
	$$p(\bm{x}^*) - p\left(\bm{x}^{(t)}+\frac{1}{D\norm{p}_{\infty}^{(2)}} \cdot V_d(\bm{x}^{(t)})V_d(\bm{x}^{(t)})^T \nabla p(\bm{x}^{(t)}) \right) \geq 0$$
	for any $\bm{x}^{(t)} \in \text{Ball}_D(\bm{x}^*,r_2)$.
	
	\emph{Fact 1} has been proved in Proposition~\ref{SCGA_conv}, implying that the objective function sequence $\left\{p(\bm{x}^{(t)}) \right\}_{t=0}^{\infty}$ is non-decreasing when $\eta < \frac{2}{D\norm{p}_{\infty}^{(2)}}$.
	\emph{Fact 2} is a natural corollary by Proposition~\ref{SCGA_conv}, because $\bm{x}^{(t)} \in \text{Ball}_D(\bm{x}^*,r_2)$ and $p\left(\bm{x}^{(t)}+\frac{1}{D\norm{p}_{\infty}^{(2)}} \cdot V_d(\bm{x}^{(t)})V_d(\bm{x}^{(t)})^T \nabla p(\bm{x}^{(t)}) \right)$ is the objective function value after one-step SCGA iteration with step size $\frac{1}{D\norm{p}_{\infty}^{(2)}}$. The iteration will move $\bm{x}^{(t)}$ towards the ridge $R_d$. With the help of these two facts, we start the proofs of (a-d).\\

	\noindent (a) We first show that the following claim: for all $t\geq 0$ and the initial point $\bm{x}^{(0)} \in \text{Ball}_D(\bm{x}^*, r_2)$,
		\begin{equation}
		\label{claim_local_SC1}
		p(\bm{x}^*) -p(\bm{x}^{(t)}) \leq \nabla p(\bm{x}^{(t)})^T V_d(\bm{x}^{(t)}) V_d(\bm{x}^{(t)})^T \left(\bm{x}^* -\bm{x}^{(t)} \right) - \frac{\beta_0}{4} \norm{\bm{x}^* -\bm{x}^{(t)}}_2^2 + \epsilon_t,
		\end{equation}
		where $\epsilon_t=\left(D\norm{p}_{\infty}^{(2)}\beta_2^2 \rho + \frac{D^{\frac{3}{2}}\norm{p}_{\infty}^{(3)}}{6} \right) \norm{\bm{x}^*-\bm{x}^{(t)}}_2^3=o\left(\norm{\bm{x}^*-\bm{x}^{(t)}}_2^2 \right)$. By the differentiability of $p$ guaranteed by condition (A1) and Taylor's theorem, we have that
		\begin{align*}
		&p(\bm{x}^*) - p(\bm{x}^{(t)}) \\
		&\leq \nabla p(\bm{x}^{(t)})^T (\bm{x}^* -\bm{x}^{(t)}) + \frac{1}{2}(\bm{x}^* -\bm{x}^{(t)})^T \nabla\nabla p(\bm{x}^{(t)}) (\bm{x}^* -\bm{x}^{(t)}) + \frac{D^{\frac{3}{2}} \norm{p}_{\infty}^{(3)}}{6}\norm{\bm{x}^*-\bm{x}^{(t)}}_2^3\\
		&\stackrel{\text{(i)}}{=} \nabla p(\bm{x}^{(t)})^T V_d(\bm{x}^{(t)}) V_d(\bm{x}^{(t)})^T(\bm{x}^* -\bm{x}^{(t)}) + \nabla p(\bm{x}^{(t)})^T U_d^{\perp}(\bm{x}^{(t)}) (\bm{x}^* -\bm{x}^{(t)}) \\
		&\quad + \frac{1}{2} (\bm{x}^* -\bm{x}^{(t)})^T \left(V_{\diamond}(\bm{x}^{(t)}), V_d(\bm{x}^{(t)}) \right)
		\begin{pmatrix}
		\lambda_1(\bm{x}^{(t)}) & &\\
		& \ddots & \\
		& & \lambda_D(\bm{x}^{(t)})
		\end{pmatrix}
		\begin{pmatrix}
		V_{\diamond}(\bm{x}^{(t)})^T\\
		V_d(\bm{x}^{(t)})^T
		\end{pmatrix}
		(\bm{x}^* -\bm{x}^{(t)})\\
		&\quad + \frac{D^{\frac{3}{2}} \norm{p}_{\infty}^{(3)}}{6}\norm{\bm{x}^*-\bm{x}^{(t)}}_2^3 \\
		&\stackrel{\text{(ii)}}{\leq} \nabla p(\bm{x}^{(t)})^T V_d(\bm{x}^{(t)}) V_d(\bm{x}^{(t)})^T(\bm{x}^* -\bm{x}^{(t)}) + \frac{\beta_0}{4} \norm{\bm{x}^{(t)} -\bm{x}^*}_2^2 \\
		&\quad +\frac{\lambda_1(\bm{x}^{(t)})}{2}\norm{U_d^{\perp}(\bm{x}^{(t)}) (\bm{x}^* -\bm{x}^{(t)})}_2^2 - \frac{\beta_0}{2}\norm{V_d(\bm{x}^{(t)})^T (\bm{x}^* -\bm{x}^{(t)})}_2^2 \\
		&\quad + \frac{D^{\frac{3}{2}} \norm{p}_{\infty}^{(3)}}{6}\norm{\bm{x}^*-\bm{x}^{(t)}}_2^3\\
		&\stackrel{\text{(iii)}}{=} \nabla p(\bm{x}^{(t)})^T V_d(\bm{x}^{(t)}) V_d(\bm{x}^{(t)})^T(\bm{x}^* -\bm{x}^{(t)}) + \frac{\beta_0}{4} \norm{\bm{x}^{(t)} -\bm{x}^*}_2^2 \\
		&\quad + \frac{\left(\lambda_1(\bm{x}^{(t)})+\beta_0 \right)}{2} \norm{U_d^{\perp}(\bm{x}^{(t)}) (\bm{x}^* -\bm{x}^{(t)})}_2^2 - \frac{\beta_0}{2}\norm{\bm{x}^*-\bm{x}^{(t)}}_2^2 + \frac{D^{\frac{3}{2}} \norm{p}_{\infty}^{(3)}}{6}\norm{\bm{x}^*-\bm{x}^{(t)}}_2^3\\
		&\stackrel{\text{(iv)}}{\leq} \nabla p(\bm{x}^{(t)})^T V_d(\bm{x}^{(t)}) V_d(\bm{x}^{(t)})^T(\bm{x}^* -\bm{x}^{(t)}) - \frac{\beta_0}{4} \norm{\bm{x}^{(t)} -\bm{x}^*}_2^2  \\
		&\quad + \frac{\left(\lambda_1(\bm{x}^{(t)})+\beta_0 \right)}{2} \cdot \beta_2^2\norm{\bm{x}^* -\bm{x}^{(t)}}_2^4 +\frac{D^{\frac{3}{2}} \norm{p}_{\infty}^{(3)}}{6}\norm{\bm{x}^*-\bm{x}^{(t)}}_2^3 \\
		&\stackrel{\text{(v)}}{\leq} \nabla p(\bm{x}^{(t)})^T V_d(\bm{x}^{(t)}) V_d(\bm{x}^{(t)})^T(\bm{x}^* -\bm{x}^{(t)}) - \frac{\beta_0}{4}\norm{\bm{x}^*-\bm{x}^{(t)}}_2^2 \\
		&\quad + \left(D\norm{p}_{\infty}^{(2)}\beta_2^2 \rho + \frac{D^{\frac{3}{2}}\norm{p}_{\infty}^{(3)}}{6} \right) \norm{\bm{x}^*-\bm{x}^{(t)}}_2^3,
		\end{align*}
		where we use the equality $\bm{I}_D=  V_d(\bm{x}^{(t)}) V_d(\bm{x}^{(t)})^T + U_d^{\perp}(\bm{x}^{(t)})$ in (i) and (iii), leverage conditions (A2) and (A4) to obtain that $\lambda_D(\bm{x}^{(t)}) \leq \cdots \leq \lambda_{d+1}(\bm{x}^{(t)}) \leq -\beta_0$ and $\norm{\nabla p(\bm{x}^{(t)}) U_d^{\perp}(\bm{x}^{(t)}) (\bm{x}^*-\bm{x}^{(t)})}_2 \leq \frac{\beta_0}{4}\norm{\bm{x}^{(t)} - \bm{x}^*}_2^2$ in (ii), apply the quadratic bound on $\norm{U_d^{\perp}(\bm{x}^{(t)}) (\bm{x}^*-\bm{x}^{(t)})}_2$ from condition (A4) to obtain (iv), and use the fact that $\norm{\bm{x}^*-\bm{x}^{(t)}}_2 \leq \rho$ in (v). We also use the fact that $\max\left\{\beta_0, |\lambda_1(\bm{x}^{(t)})|\right\} \leq D\norm{p}_{\infty}^{(2)}$ and $\norm{\bm{x}^*-\bm{x}^{(t)}}_2 \leq \rho$ in (v). Claim \eqref{claim_local_SC1} thus follows.
	
	Given \emph{Fact 2} and any $\bm{x}^{(t)} \in \text{Ball}_D(\bm{x}^*, r_2)$,
	\begin{align*}
	&p(\bm{x}^{(t)}) - p(\bm{x}^*) \\
	&\leq p(\bm{x}^{(t)}) - p(\bm{x}^*) + p(\bm{x}^*) - p\left(\bm{x}^{(t)}+\frac{1}{D\norm{p}_{\infty}^{(2)}} \cdot V_d(\bm{x}^{(t)})V_d(\bm{x}^{(t)})^T \nabla p(\bm{x}^{(t)}) \right)\\
	&=-\left[p\left(\bm{x}^{(t)}+\frac{1}{D\norm{p}_{\infty}^{(2)}} \cdot V_d(\bm{x}^{(t)})V_d(\bm{x}^{(t)})^T \nabla p(\bm{x}^{(t)}) \right) -p(\bm{x}^{(t)}) \right]\\
	&\leq -\left[\nabla p(\bm{x}^{(t)})^T \frac{1}{D\norm{p}_{\infty}^{(2)}} V_d(\bm{x}^{(t)})V_d(\bm{x}^{(t)})^T \nabla p(\bm{x}^{(t)}) - \frac{1}{2D\norm{p}_{\infty}^{(2)}} \norm{V_d(\bm{x}^{(t)})^T \nabla p(\bm{x}^{(t)})}_2^2 \right]\\
	&= - \frac{1}{2D\norm{p}_{\infty}^{(2)}} \norm{V_d(\bm{x}^{(t)})^T \nabla p(\bm{x}^{(t)})}_2^2,
	\end{align*} 
	where we apply \eqref{L_smoothness} to obtain the inequality. It implies that
	\begin{equation}
	\label{grad_bound_op1}
	\norm{V_d(\bm{x}^{(t)})^T \nabla p(\bm{x}^{(t)})}_2^2 \leq 2D\norm{p}_{\infty}^{(2)} \cdot \left[p(\bm{x}^*) -p(\bm{x}^{(t)}) \right].
	\end{equation}
	Therefore,
	\begin{align*}
	&\norm{\bm{x}^{(t+1)}-\bm{x}^*}_2^2 \\
	&= \norm{\bm{x}^{(t)} + \eta V_d(\bm{x}^{(t)}) V_d(\bm{x}^{(t)})^T \nabla p(\bm{x}^{(t)}) - \bm{x}^*}_2^2\\
	&= \norm{\bm{x}^{(t)}-\bm{x}^*}_2^2 + 2\eta \nabla p(\bm{x}^{(t)})^T V_d(\bm{x}^{(t)}) V_d(\bm{x}^{(t)})^T \left(\bm{x}^{(t)}-\bm{x}^* \right) + \eta^2 \norm{V_d(\bm{x}^{(t)})^T \nabla p(\bm{x}^{(t)})}_2^2\\
	&\stackrel{\text{(i)}}{\leq} \norm{\bm{x}^{(t)}-\bm{x}^*}_2^2 + 2\eta \Bigg[p(\bm{x}^{(t)}) -p(\bm{x}^*) - \frac{\beta_0}{4}\norm{\bm{x}^*-\bm{x}^{(t)}}_2^2 \\
	&\quad + \left(D\norm{p}_{\infty}^{(2)}\beta_2^2 \rho + \frac{D^{\frac{3}{2}}\norm{p}_{\infty}^{(3)}}{6} \right) \norm{\bm{x}^*-\bm{x}^{(t)}}_2^3 \Bigg] + \eta^2 \cdot 2D\norm{p}_{\infty}^{(2)} \cdot \left[p(\bm{x}^*) -p(\bm{x}^{(t)}) \right]\\
	&\stackrel{\text{(ii)}}{\leq} \left(1-\frac{\beta_0\eta}{4} \right) \norm{\bm{x}^{(t)}-\bm{x}^*}_2^2 - 2\eta \left(1- \eta D\norm{p}_{\infty}^{(2)}\right) \underbrace{\left[p(\bm{x}^*) -p(\bm{x}^{(t)}) \right]}_{\geq 0}\\
	&\leq \left(1-\frac{\beta_0\eta}{4} \right) \norm{\bm{x}^{(t)}-\bm{x}^*}_2^2
	\end{align*}
	whenever $0< \eta \leq \min\left\{\frac{4}{\beta_0}, \frac{1}{D\norm{p}_{\infty}^{(2)}} \right\}$, where we apply \eqref{claim_local_SC1} and \eqref{grad_bound_op1} in (i) and use the choice of $r_2 >0$ to argue that 
	\begin{align*}
	&\left(D\norm{p}_{\infty}^{(2)}\beta_2^2 \rho + \frac{D^{\frac{3}{2}}\norm{p}_{\infty}^{(3)}}{6} \right) \norm{\bm{x}^*-\bm{x}^{(t)}}_2^3 \\
	&\leq \left(D\norm{p}_{\infty}^{(2)}\beta_2^2 \rho + \frac{D^{\frac{3}{2}}\norm{p}_{\infty}^{(3)}}{6} \right) r_2 \norm{\bm{x}^*-\bm{x}^{(t)}}_2^2 \\
	&\leq \frac{\beta_0}{8}\norm{\bm{x}^*-\bm{x}^{(t)}}_2^2 
	\end{align*}
	in (ii). By telescoping, we conclude that when $0< \eta \leq \min\left\{\frac{4}{\beta_0}, \frac{1}{D\norm{p}_{\infty}^{(2)}} \right\}$ and $\bm{x}^{(0)} \in \text{Ball}_D(\bm{x}^*, r_2)$,
	$$\norm{\bm{x}^{(t)}-\bm{x}^*}_2 \leq \left(1-\frac{\beta_0\eta}{4} \right)^{\frac{t}{2}} \norm{\bm{x}^{(0)}-\bm{x}^*}_2.$$
	The result follows.\\
	
	\noindent (b) The result follows easily from (a) and the inequality $d_E(\bm{x}^{(t)}, R_d) \leq \norm{\bm{x}^{(t)} - \bm{x}^*}_2$ for all $t\geq 0$.\\
	
	\noindent (c) The proof here is partially inspired by the proof of Theorem 2 in \cite{EM2017}. We write the spectral decompositions of $\nabla\nabla p(\bm{x})$ and $\nabla\nabla \hat{p}_n(\bm{x})$ as 
	$$\nabla\nabla p(\bm{x}) = V(\bm{x}) \Lambda(\bm{x}) V(\bm{x})^T \quad \text{ and } \quad \nabla\nabla \hat{p}_n(\bm{x}) = \hat{V}(\bm{x}) \hat{\Lambda}(\bm{x}) \hat{V}(\bm{x})^T,$$
	where $\Lambda(\bm{x}) = \Diag\left[\lambda_1(\bm{x}),...,\lambda_D(\bm{x}) \right]$ and $\hat{V}(\bm{x}) = \Diag\left[\hat{\lambda}_1(\bm{x}),...,\hat{\lambda}_D(\bm{x}) \right]$.
	By Weyl's theorem (Theorem 4.3.1 in \citealt{HJ2012}) and uniform bounds \eqref{unif_bound}, we know that for any $j=1,...,D$,
	\begin{align*}
	|\lambda_j(\bm{x}) - \hat{\lambda}_j(\bm{x})| &\leq \max_j\left|\lambda_j\left(\nabla\nabla p(\bm{x}) - \nabla\nabla \hat{p}_n(\bm{x})\right)\right| \\
	&= \norm{\nabla\nabla p(\bm{x}) - \nabla\nabla \hat{p}_n(\bm{x})}_2 \\
	&\leq D \norm{p - \hat{p}_n}_{\infty}^{(2)}\\
	&= O(h^2) + O_P\left(\sqrt{\frac{|\log h|}{nh^{D+4}}} \right).
	\end{align*}
	Thus, $\hat{p}_n$ satisfies the first two inequalities in condition (A2) when $h$ is sufficiently small and $\frac{nh^{D+4}}{|\log h|}$ is sufficiently large. According to the Davis-Kahan theorem (Lemma~\ref{Davis_K} here) and uniform bounds \eqref{unif_bound},
	\begin{align*}
	&\norm{G_d(\bm{y}) - \hat{G}_d(\bm{y})}_2 \\
	&= \norm{V_d(\bm{y}) V_d(\bm{y})^T \nabla p(\bm{y}) - \hat{V}_d(\bm{y}) \hat{V}_d(\bm{y})^T \nabla \hat{p}_n(\bm{y})}_2\\
	&\leq \norm{\left[V_d(\bm{y}) V_d(\bm{y})^T - \hat{V}_d(\bm{y}) \hat{V}_d(\bm{y})^T \right] \nabla p(\bm{y})}_2 + \norm{\hat{V}_d(\bm{y}) \hat{V}_d(\bm{y})^T \left[\nabla p(\bm{y}) -\nabla \hat{p}_n(\bm{y}) \right]}_2\\
	&\stackrel{\text{(i)}}{\leq} \frac{\sqrt{2}D\norm{\nabla\nabla p(\bm{y}) - \nabla\nabla \hat{p}_n(\bm{y})}_{\max} \cdot \norm{\nabla p(\bm{y})}_2}{\beta_0} + \norm{\nabla p(\bm{y}) -\nabla \hat{p}_n(\bm{y})}_2\\
	&\leq \frac{\sqrt{2}D\norm{p-\hat{p}_n}_{\infty}^{(2)} \cdot \sqrt{D}\norm{p}_{\infty}^{(1)}}{\beta_0} + \sqrt{D}\norm{p-\hat{p}_n}_{\infty}^{(1)}\\
	&\equiv \epsilon_{n,h} = O(h^2) + O_P\left(\sqrt{\frac{|\log h|}{nh^{D+4}}}\right)
	\end{align*}
	for any $\bm{y}\in R_d\oplus r_2$, where we use \eqref{Davis_Kahan_ineq} and the fact that $\norm{\hat{V}_d(\bm{y}) \hat{V}_d(\bm{y})^T}_2=1$ to obtain (i). Hence, when $h \to 0$ and $\frac{nh^{D+4}}{|\log h|} \to \infty$,
	\begin{equation}
	\label{proj_grad_diff}
	\norm{G_d(\bm{y}) - \hat{G}_d(\bm{y})}_2\leq \epsilon_{n,h} = O(h^2) + O_P\left(\sqrt{\frac{|\log h|}{nh^{D+4}}}\right) \leq \frac{(1-\Upsilon) r_2}{\eta}
	\end{equation}
	with probability tending to 1. \\
	We now claim that $\norm{\hat{\bm{x}}^{(t)}- \bm{x}^*}_2 \leq r_2$ and 
	\begin{equation}
	\label{claim_diff}
	\norm{\hat{\bm{x}}^{(t+1)} -\bm{x}^*}_2 \leq \Upsilon \norm{\hat{\bm{x}}^{(t)} -\bm{x}^*}_2 + \eta \cdot \epsilon_{n,h}
	\end{equation}
	for all $t\geq 0$. We will prove this claim by induction on the iteration number. Note that when $t=1$, we derive from triangle inequality that
	\begin{align*}
	\norm{\hat{\bm{x}}^{(1)} -\bm{x}^*}_2 &= \norm{\hat{\bm{x}}^{(0)} + \eta\hat{V}_d(\hat{\bm{x}}^{(0)}) \hat{V}_d(\hat{\bm{x}}^{(0)})^T \nabla \hat{p}_n(\hat{\bm{x}}^{(0)}) -\bm{x}^*}_2\\
	&\leq \norm{\hat{\bm{x}}^{(0)} + \eta V_d(\hat{\bm{x}}^{(0)}) V_d(\hat{\bm{x}}^{(0)})^T \nabla p(\hat{\bm{x}}^{(0)}) -\bm{x}^*}_2 + \eta\norm{G_d(\hat{\bm{x}}^{(0)}) - \hat{G}_d(\hat{\bm{x}}^{(0)})}\\
	&\leq \Upsilon \norm{\hat{\bm{x}}^{(0)}-\bm{x}^*}_2 + \eta \cdot \epsilon_{n,h},
	\end{align*}
	where we apply the result in (a) to obtain the last inequality. Moreover, by the choice of $\hat{\bm{x}}^{(0)}$ and \eqref{proj_grad_diff}, we are guaranteed that $\norm{\hat{\bm{x}}^{(1)}-\bm{x}^*}_2\leq r_2$. In the induction from $t \mapsto t+1$, we suppose that $\norm{\hat{\bm{x}}^{(t)} -\bm{x}^*}_2 \leq r_2$ and the claim \eqref{claim_diff} holds at iteration $t$. The same argument then implies that the claim \eqref{claim_diff} holds for iteration $t+1$ and that $\norm{\hat{\bm{x}}^{(t+1)} -\bm{x}^*}_2 \leq r_2$. The claim \eqref{claim_diff} is thus proved.\\ 
	Now, given that $\Upsilon = \sqrt{1-\frac{\beta_0\eta}{4}}<1$, we iterate the claim \eqref{claim_diff} to show that
	\begin{align*}
	\norm{\hat{\bm{x}}^{(t)} -\bm{x}^*}_2 &\leq \Upsilon \norm{\hat{\bm{x}}^{(t-1)} -\bm{x}^*}_2 + \eta \cdot \epsilon_{n,h}\\
	&\leq \Upsilon \left[\Upsilon\norm{\hat{\bm{x}}^{(t-2)} -\bm{x}^*}_2 + \eta \cdot \epsilon_{n,h} \right] + \eta \cdot \epsilon_{n,h}\\
	&\leq \Upsilon^t \norm{\hat{\bm{x}}^{(0)} -\bm{x}^*}_2 + \left[\sum_{s=0}^{t-1}\Upsilon^s \right] \eta \cdot \epsilon_{n,h}\\
	&\leq \Upsilon^t \norm{\hat{\bm{x}}^{(0)} -\bm{x}^*}_2 + \frac{\eta \cdot \epsilon_{n,h}}{1-\Upsilon}\\
	&= \Upsilon^t \norm{\hat{\bm{x}}^{(0)} -\bm{x}^*}_2 + O(h^2) + O_P\left(\frac{|\log h|}{nh^{D+4}} \right),
	\end{align*}
	where the fourth inequality follows by summing the geometric series, and the last equality is due to our notation $\epsilon_{n,h} = O(h^2) + O_P\left(\frac{|\log h|}{nh^{D+4}} \right)$. It completes the proof.\\
	
	\noindent (d) The result follows easily from (c) and the inequality $d_E(\hat{\bm{x}}^{(t)}, R_d) \leq \norm{\hat{\bm{x}}^{(t)}-\bm{x}^*}_2$ for all $t\geq 0$.
\end{proof}

\section{Discussion on Condition (A4)}
	\label{App:dis_A4}
	
	In this section, we explore several avenues to derive condition (A4) based on some potentially weaker assumptions. Recall from Section~\ref{Sec:LC_SCGA} that condition (A4) requires the following:
	\begin{itemize}
		\item {\bf (A4)} (\emph{Quadratic Behaviors of Residual Vectors}) We assume that the SCGA sequence $\big\{\bm{x}^{(t)}\big\}_{t=0}^{\infty}$ with step size $0<\eta \leq \min\left\{\frac{4}{\beta_0},\frac{1}{D\norm{p}_{\infty}^{(2)}} \right\}$ and $\bm{x}^* \in R_d$ as its limiting point satisfies that 
		\begin{align*}
		\nabla p(\bm{x}^{(t)})^T U_d^{\perp}(\bm{x}^{(t)}) (\bm{x}^*-\bm{x}^{(t)}) &\leq \frac{\beta_0}{4} \norm{\bm{x}^*-\bm{x}^{(t)}}_2^2,\\
		\norm{U_d^{\perp}(\bm{x}^{(t)}) (\bm{x}^*-\bm{x}^{(t)})}_2 &\leq \beta_2\norm{\bm{x}^*-\bm{x}^{(t)}}_2^2
		\end{align*}
		for some constant $\beta_2>0$, where $\beta_0>0$ is the constant defined in condition (A2).
	\end{itemize}
	
	\subsection{Self-Contractedness Assumption}
	\label{App:self-cont}
	
	One important assumption that connects condition (A4) with the existing conditions (A1-3) in Section~\ref{Sec:Assump} is the so-called self-contracted property \citep{daniilidis2010asymptotic,daniilidis2015rectifiability,gupta2021path}:
	\begin{itemize}
		\item {\bf (A5)} (\emph{Self-Contractedness}) We assume that the SCGA sequence $\big\{\bm{x}^{(t)}\big\}_{t=0}^{\infty}$ satisfies that
		$$\norm{\bm{x}^{(t_3)} - \bm{x}^{(t_2)}}_2 \leq \norm{\bm{x}^{(t_3)} - \bm{x}^{(t_1)}}_2 \quad \text{ for all } \quad 0\leq t_1 \leq t_2 \leq t_3.$$
	\end{itemize}
	
	Condition (A5) requires the SCGA sequence to move towards the ridge $R_d$ under a relatively straight and shrinking path. As we have proved in Proposition~\ref{SCGA_conv} that the SCGA sequence $\big\{\bm{x}^{(t)}\big\}_{t=0}^{\infty}$ converges to $R_d$ when the sequence is initialized near $R_d$ and its step size is small, condition (A5) is indeed a mild assumption as long as the sequence $\big\{\bm{x}^{(t)}\big\}_{t=0}^{\infty}$ does not move erratically around $R_d$. More importantly, we demonstrate by Proposition~\ref{SCC_self_contract} below that condition (A5) can be implied by a subspace constrained version of the concavity assumption on the objective (density) function $p$.
	
	\begin{proposition}
		\label{SCC_self_contract}
		Assume condition (A1) and the following assumption on the objective function $p$ :
		\begin{itemize}
			\item {\bf (A6)} (\emph{Subspace Constrained Concavity}) For any $\bm{x},\bm{y} \in R_d\oplus r_4$ with $r_4 >0$ being a constant radius, it holds that
			\begin{equation*}
			p(\bm{y}) -p(\bm{x}) \leq \nabla p(\bm{x})^T V_d(\bm{x}) V_d(\bm{x})^T (\bm{y}-\bm{x}).
			\end{equation*}
		\end{itemize}
		Then, the SCGA sequence $\big\{\bm{x}^{(t)}\big\}_{t=0}^{\infty}$ defined in \eqref{SCGA_update_pop} with step size $0<\eta \leq \frac{1}{D\norm{p}_{\infty}^{(2)}}$ and initial point $\bm{x}^{(0)} \in R_d\oplus r_4$ is self-contracted.
	\end{proposition}
	
	Notice that the density function \eqref{den_example} satisfies the ``subspace constrained concavity'' condition (A6) around a small neighborhood of its ridge $R_1$. Moreover, it is intuitive to verify that condition (A6) is a weaker assumption compared to our established ``subspace constrained strong concavity'' in Theorem~\ref{SCGA_LC}; see also Remark~\ref{SC_remark}.
	
	\begin{proof}[Proof of Proposition~\ref{SCC_self_contract}]
		The proof is inspired by Lemma 14 in \cite{gupta2021path}. We show the self-contractedness for $s_3=t$ as follows, where $t> 0$ is arbitrary. For all $s< t$ and $0<\eta \leq \frac{1}{D\norm{p}_{\infty}^{(2)}}$ with $\bm{x}^{(0)}\in R_d\oplus r_4$, we calculate that 
		\begin{align*}
		&\norm{\bm{x}^{(s+1)} -\bm{x}^{(t)}}_2^2 \\
		&= \norm{\bm{x}^{(s)} + \eta \cdot V_d(\bm{x}^{(s)}) V_d(\bm{x}^{(s)})^T \nabla p(\bm{x}^{(s)}) -\bm{x}^{(t)}}_2^2\\
		&= \norm{\bm{x}^{(s)} -\bm{x}^{(t)}}_2^2 + 2\eta \nabla p(\bm{x}^{(s)})^T V_d(\bm{x}^{(s)}) V_d(\bm{x}^{(s)})^T (\bm{x}^{(s)} -\bm{x}^{(t)}) \\
		&\quad + \eta^2 \norm{V_d(\bm{x}^{(s)})^T \nabla p(\bm{x}^{(s)})}_2^2\\
		&\stackrel{\text{(i)}}{\leq} \norm{\bm{x}^{(s)} -\bm{x}^{(t)}}_2^2 + 2\eta \left[p(\bm{x}^{(s)}) -p(\bm{x}^{(t)}) \right] + \eta^2 \norm{V_d(\bm{x}^{(s)})^T \nabla p(\bm{x}^{(s)})}_2^2\\
		&\stackrel{\text{(ii)}}{\leq} \norm{\bm{x}^{(s)} -\bm{x}^{(t)}}_2^2 + 2\eta \left[p(\bm{x}^{(s)}) -p(\bm{x}^{(s+1)}) \right] + \eta^2 \norm{V_d(\bm{x}^{(s)})^T \nabla p(\bm{x}^{(s)})}_2^2 \\
		&\stackrel{\text{(iii)}}{\leq} \norm{\bm{x}^{(s)} -\bm{x}^{(t)}}_2^2 -2\eta \cdot \eta \left(1-\frac{\eta D\norm{p}_{\infty}^{(2)}}{2} \right) \norm{V_d(\bm{x}^{(s)})^T \nabla p(\bm{x}^{(s)})}_2^2 \\
		&\quad + \eta^2 \norm{V_d(\bm{x}^{(s)})^T \nabla p(\bm{x}^{(s)})}_2^2\\
		&= \norm{\bm{x}^{(s)} -\bm{x}^{(t)}}_2^2  + \eta^2 \cdot \left(\eta D\norm{p}_{\infty}^{(2)} -1 \right)\norm{V_d(\bm{x}^{(s)})^T \nabla p(\bm{x}^{(s)})}_2^2\\
		&\leq \norm{\bm{x}^{(s)} -\bm{x}^{(t)}}_2^2,
		\end{align*}
		where we apply condition (A6) in inequality (i), use the ascending property of $p$ from (a) of Proposition~\ref{SCGA_conv} to argue that $p(\bm{x}^{(t)}) \geq p(\bm{x}^{(s+1)})$ in inequality (ii), and leverage the inequality \eqref{SCGA_ascending} guaranteed by condition (A1) to obtain (iii). The self-contractedness of the SCGA sequence $\big\{\bm{x}^{(t)}\big\}_{t=0}^{\infty}$ thus follows.
	\end{proof}
	
	Under the self-contractedness condition (A5), we argue by the following lemma that the existing conditions (A1-3) in the literature \citep{Non_ridge_est2014,Asymp_ridge2015} is nearly sufficient to imply the quadratic behavior of the residual vector $U_d^{\perp}(\bm{x}^{(t)}) (\bm{x}^*-\bm{x}^{(t)})$ along the SCGA sequence $\big\{\bm{x}^{(t)}\big\}_{t=0}^{\infty}$. In other words, condition (A4) and the linear convergence of the SCGA algorithm hold without any extra assumption.
	
	\begin{lemma}
		\label{quad_bound_prop}
		Assume condition (A5) throughout the lemma. 
		\begin{enumerate}[label=(\alph*)]
			\item The total length of the SCGA trajectory is of the linear order, \emph{i.e.},
			$$\sum_{s=t}^{\infty} \norm{\bm{x}^{(s+1)} -\bm{x}^{(s)}}_2 \leq 2^{10D^2} \norm{\bm{x}^{(t)} - \bm{x}^*}_2 \quad \text{ for any } t\geq 0.$$
			
			\item We further assume conditions (A1-2). Then,
			\begin{align*}
			&\nabla p(\bm{x}^{(t)})^T U_d^{\perp}(\bm{x}^{(t)}) (\bm{x}^*-\bm{x}^{(t)}) \\
			& \leq \frac{2^{10D^2+1}D^{\frac{5}{2}}\norm{U_d^{\perp}(\bm{x}^{(t)}) \nabla p(\bm{x}^{(t)})}_2 \norm{\nabla^3 p(\bm{x}^{(t)})}_{\max}}{\beta_0} \cdot \norm{\bm{x}^*-\bm{x}^{(t)}}_2^2\\
			&\text{ and }\\
			&\norm{U_d^{\perp}(\bm{x}^{(t)})\left(\bm{x}^* -\bm{x}^{(t)} \right)}_2 \leq \frac{2^{10D^2+1}D^{\frac{5}{2}}\norm{\nabla^3 p(\bm{x}^{(t)})}_{\max}}{\beta_0} \cdot \norm{\bm{x}^*-\bm{x}^{(t)}}_2^2
			\end{align*}
			for any $\bm{x}^{(t)}\in \text{Ball}_D(\bm{x}^*, r_5)$ with some radius $0 < r_5 < \rho$, where we recall that $\rho>0$ is the effective radius in condition (A2) under which the underlying density $p$ has an eigengap $\beta_0>0$ between the $d$-th and $(d+1)$-th eigenvalues of its Hessian matrix $\nabla\nabla p$. 
		\end{enumerate}
	\end{lemma}
	
	\begin{proof}[Proof of Lemma~\ref{quad_bound_prop}]
		(a) This result follows directly from Theorem 15 of \cite{gupta2021path}. Note that although their results are stated for the standard gradient descent path, the associated proof only utilizes the self-contractedness property of the iterative path. Thus, their proofs are applicable to our SCGA setting under condition (A5).
		
		\begin{figure}[t]
			\centering
			\includegraphics[width=0.8\linewidth]{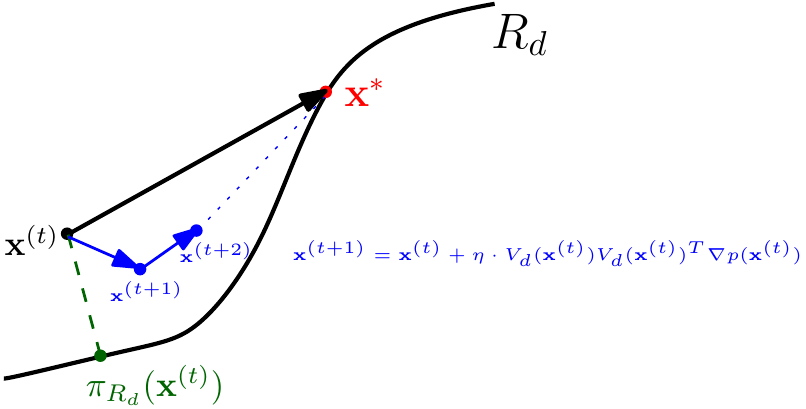}
			\caption{Decomposition of the vector $\bm{x}^*-\bm{x}^{(t)}$ into the summation $\sum\limits_{s=t}^{\infty} \left(\bm{x}^{(s+1)} -\bm{x}^{(s)} \right)$ of subspace constrained gradient ascent iterative vectors.}
			\label{fig:SCGA_decomp}
		\end{figure}
		
		\noindent (b) We first decompose the vector $\bm{x}^*-\bm{x}^{(t)}$ into an infinite sum of SCGA iterations $\bm{x}^{(s+1)} = \bm{x}^{(s)} + \eta \cdot V_d(\bm{x}^{(s)}) V_d(\bm{x}^{(s)})^T \nabla p(\bm{x}^{(s)})$ for $s\geq t$ and obtain that  
		\begin{align}
		\label{x_dist_residual}
		\begin{split}
		& U_d^{\perp}(\bm{x}^{(t)})\left(\bm{x}^* -\bm{x}^{(t)}\right) \\
		&= U_d^{\perp}(\bm{x}^{(t)}) \cdot \sum_{s=t}^{\infty} \left(\bm{x}^{(s+1)} -\bm{x}^{(s)} \right) \\
		& = \sum_{s=t}^{\infty} U_d^{\perp}(\bm{x}^{(t)}) \cdot\eta V_d(\bm{x}^{(s)}) V_d(\bm{x}^{(s)})^T \nabla p(\bm{x}^{(s)}) \\
		&\stackrel{\text{(i)}}{=} \sum_{s=t}^{\infty} U_d^{\perp}(\bm{x}^{(t)}) \cdot\eta \left[V_d(\bm{x}^{(s)}) V_d(\bm{x}^{(s)})^T - V_d(\bm{x}^{(t)}) V_d(\bm{x}^{(t)})^T \right] V_d(\bm{x}^{(s)}) V_d(\bm{x}^{(s)})^T \nabla p(\bm{x}^{(s)})
		\end{split}
		\end{align}
		for any $\bm{x}^{(t)}\in R_d\oplus \rho$, where we leverage the orthogonality between $U_d^{\perp}(\bm{x}^{(t)})$ and $V_d(\bm{x}^{(t)}) V_d(\bm{x}^{(t)})^T$ and the idempotence of $V_d(\bm{x}^{(s)}) V_d(\bm{x}^{(s)})^T$ for all $s\geq t$ in (ii). See also Figure~\ref{fig:SCGA_decomp} for a graphical illustration of the decomposition. 
		By Davis-Kahan theorem (Lemma~\ref{Davis_K} and \eqref{Davis_Kahan_ineq} here) and conditions (A1-2), we deduce that for all $s\geq t$,
		\begin{align*}
		&\norm{V_d(\bm{x}^{(s)}) V_d(\bm{x}^{(s)})^T - V_d(\bm{x}^{(t)}) V_d(\bm{x}^{(t)})^T}_2 \\
		& \leq \frac{\sqrt{2}D\norm{\nabla\nabla p(\bm{x}^{(s)}) - \nabla\nabla p(\bm{x}^{(t)})}_2}{\beta_0}\\
		&\stackrel{\text{(i)}}{\leq} \frac{\sqrt{2}D\norm{\nabla^3 p(\bm{x}^{(t)})}_2\norm{\bm{x}^{(s)}-\bm{x}^{(t)}}_2}{\beta_0} + \frac{\sqrt{2}D^3 \norm{p}_{\infty}^{(4)}}{\beta_0}\cdot \norm{\bm{x}^{(s)}-\bm{x}^{(t)}}_2^2\\
		&\stackrel{\text{(ii)}}{\leq} \frac{2D^{\frac{5}{2}}\norm{\nabla^3 p(\bm{x}^{(t)})}_{\max}\norm{\bm{x}^*-\bm{x}^{(t)}}_2}{\beta_0},
		\end{align*}
		where we use the Taylor's theorem in (i) as well as apply the self-contractedness condition (A5) and possibly shrink the radius $r_5 >0$ so that $\norm{\bm{x}^*-\bm{x}^{(t)}}_2 \leq \frac{(2-\sqrt{2}) \norm{\nabla^3 p(\bm{x}^{(t)})}_{\max}}{\sqrt{2D} \norm{p}_{\infty}^{(4)}}$ in (ii). Hence, by \eqref{x_dist_residual} and the fact that $\norm{U_d^{\perp}(\bm{x}^{(t)})}_2=1$, we obtain that
		\begin{align*}
		&\norm{U_d^{\perp}(\bm{x}^{(t)})\left(\bm{x}^* -\bm{x}^{(t)} \right)}_2 \\
		& \leq \sup_{s\geq t} \norm{V_d(\bm{x}^{(s)}) V_d(\bm{x}^{(s)})^T - V_d(\bm{x}^{(t)}) V_d(\bm{x}^{(t)})^T}_2 \cdot \sum_{s=t}^{\infty} \norm{\eta V_d(\bm{x}^{(s)}) V_d(\bm{x}^{(s)})^T \nabla p(\bm{x}^{(s)})}_2\\
		&\leq \frac{2D^{\frac{5}{2}}\norm{\nabla^3 p(\bm{x}^{(t)})}_{\max}\norm{\bm{x}^*-\bm{x}^{(t)}}_2}{\beta_0} \cdot 2^{10D^2} \norm{\bm{x}^{(t)} - \bm{x}^*}_2\\
		&= \frac{2^{10D^2+1}D^{\frac{5}{2}}\norm{\nabla^3 p(\bm{x}^{(t)})}_{\max}\norm{\bm{x}^*-\bm{x}^{(t)}}_2^2}{\beta_0},
		\end{align*}
		implying the second bound in condition (A4) with $\beta_2= \frac{2^{10D^2+1}D^{\frac{5}{2}}\norm{p}_{\infty}^{(3)}}{\beta_0}$. In addition, 
		\begin{align*}
		&\nabla p(\bm{x}^{(t)})^T U_d^{\perp}(\bm{x}^{(t)}) (\bm{x}^*-\bm{x}^{(t)}) \\
		&\leq \norm{U_d^{\perp}(\bm{x}^{(t)}) \nabla p(\bm{x}^{(t)})}_2 \norm{U_d^{\perp}(\bm{x}^{(t)})\left(\bm{x}^* -\bm{x}^{(t)} \right)}_2\\
		&\leq \frac{2^{10D^2+1}D^{\frac{5}{2}}\norm{U_d^{\perp}(\bm{x}^{(t)}) \nabla p(\bm{x}^{(t)})}_2 \norm{\nabla^3 p(\bm{x}^{(t)})}_{\max}\norm{\bm{x}^*-\bm{x}^{(t)}}_2^2}{\beta_0}.
		\end{align*}
		The results follow.
	\end{proof}
	
	According to (b) of Lemma~\ref{quad_bound_prop}, condition (A4) will hold with $\beta_2= \frac{2^{10D^2+1}D^{\frac{5}{2}}\norm{p}_{\infty}^{(3)}}{\beta_0}$ whenever 
	\begin{equation}
	\label{cond_A3_mod}
	2^{10D^2+1}D^{\frac{5}{2}}\norm{U_d^{\perp}(\bm{x}^{(t)}) \nabla p(\bm{x}^{(t)})}_2 \norm{\nabla^3 p(\bm{x}^{(t)})}_{\max} \leq \frac{\beta_0^2}{4}.
	\end{equation}
	The choice of $\beta_2 >0$ is a valid constant under the differentiability condition (A1). More importantly, \eqref{cond_A3_mod} is essentially the same assumption as the first inequality of condition (A3). Compared to the corresponding condition in (A3), the upper bound in \eqref{cond_A3_mod} for $\norm{U_d^{\perp}(\bm{x}^{(t)}) \nabla p(\bm{x}^{(t)})}_2 \norm{\nabla^3 p(\bm{x}^{(t)})}_{\max}$ around the ridge $R_d$ is only shrunk by a dimension-dependent factor $\frac{1}{D \cdot 2^{10D^2+2}}$. 
	As condition (A3) and \eqref{cond_A3_mod} are local, this adjustment does not induce too much extra strictness on the underlying density $p$.
	
	\subsection{Subspace Constrained Polyak-{\L}ojasiewicz Inequality Assumption}
	
	We have demonstrated in Appendix~\ref{App:self-cont} that the crucial condition (A4) is valid under the self-contractedness assumption on the SCGA sequence $\big\{\bm{x}^{(t)} \big\}_{t=0}^{\infty}$. Consequently, the linear convergence of the SCGA algorithm can be established by slightly modifying the common assumptions (A1-3) in ridge estimation. Nevertheless, the self-contractedness property of the SCGA sequence $\big\{\bm{x}^{(t)} \big\}_{t=0}^{\infty}$ does not always hold in practice, and it may only be implied by the subspace constrained concavity condition (A6) as proved in Proposition~\ref{SCC_self_contract}. 
	
	Given the fact that the underlying density function $p$ or its estimator $\hat{p}_n$ may not satisfy the subspace constrained concavity assumption in many practical applications of SCGA and SCMS algorithms, we present another approach to deduce condition (A4) based on the well-known Polyak-{\L}ojasiewicz inequality \citep{Polyak1963,Lojasiewicz1963}. Given any SCGA sequence $\big\{\bm{x}^{(t)} \big\}_{t=0}^{\infty}$ with limiting point $\bm{x}^* \in R_d$ and step size $0<\eta < \frac{2}{D\norm{p}_{\infty}^{(2)}}$, we consider the following condition:
	\begin{itemize}
		\item {\bf (A7)} (\emph{Subspace Constrained Polyak-{\L}ojasiewicz Inequality}) For all $t\geq 0$, there exists a constant $\beta_3 >0$ such that
		$$\frac{1}{2}\norm{V_d(\bm{x}^{(t)})^T \nabla p(\bm{x}^{(t)})}_2^2 \geq \beta_3 \left[p(\bm{x}^*) - p(\bm{x}^{(t)}) \right].$$
	\end{itemize}
	Similar to the standard Polyak-{\L}ojasiewicz inequality, there exist some objective functions that satisfy the subspace constrained Polyak-{\L}ojasiewicz inequality but fail to be concave in the subspace constrained sense as in condition (A6); see \cite{charles2018stability,fazel2018global} and Equation (36) in \cite{YC2020}. From this aspect, condition (A7) incorporates some extra SCGA sequences satisfying condition (A4) and converging linearly to the ridge $R_d$. However, as the subspace constrained Polyak-{\L}ojasiewicz inequality does not imply condition (A5) or (A6), it should not be regarded as a more general condition. Furthermore, unlike the standard gradient ascent/descent method (Theorem 2 in \citealt{LC_PL_ineq2016}), the error bound condition (\emph{i.e.}, Equation \eqref{error_bound} here) does not imply the subspace constrained Polyak-{\L}ojasiewicz inequality, indicating a challenge in validating condition (A7) in practice. 
	
	Despite these disadvantages, the subspace constrained Polyak-{\L}ojasiewicz inequality condition does give rise to a concise proof for the linear convergence of the objective function value $\big\{p(\bm{x}^{(t)}) \big\}_{t=0}^{\infty}$ along the SCGA sequence $\big\{\bm{x}^{(t)}\big\}_{t=0}^{\infty}$.
	
	\begin{proposition}
		\label{LC_func_SCPL}
		Assume conditions (A1) and (A7). Then, for any SCGA sequence $\big\{\bm{x}^{(t)} \big\}_{t=0}^{\infty}$ with step size $0< \eta <\min\left\{\frac{1}{D\norm{p}_{\infty}^{(2)}}, \frac{1}{\beta_3} \right\}$, we have that
		$$p(\bm{x}^*) - p(\bm{x}^{(t)}) \leq \big[p(\bm{x}^*) - p(\bm{x}^{(0)}) \big] \cdot (1-\eta\beta_3)^t.$$
	\end{proposition}
	
	\begin{proof}[Proof of Proposition~\ref{LC_func_SCPL}]
		The proof is inspired by Theorem 1 in \cite{LC_PL_ineq2016}. From \eqref{SCGA_ascending} and condition (A7), we know that
		\begin{align*}
		p(\bm{x}^{(t+1)}) - p(\bm{x}^{(t)}) &\geq \eta \left(1- \frac{D\norm{p}_{\infty}^{(2)}\eta}{2}\right) \norm{V_d(\bm{x}^{(t)})^T \nabla p(\bm{x}^{(t)})}_2^2 \geq \eta\beta_3 \left[p(\bm{x}^*) - p(\bm{x}^{(t)}) \right]
		\end{align*}
		for all $t\geq 0$ when $0< \eta <\frac{1}{D\norm{p}_{\infty}^{(2)}}$. By some rearrangements, we conclude that
		$$p(\bm{x}^*) - p(\bm{x}^{(t+1)})\leq (1-\eta\beta_3) \left[p(\bm{x}^*) - p(\bm{x}^{(t)}) \right].$$
		The final display follows from telescoping.
	\end{proof}
	
	More importantly, the subspace constrained Polyak-{\L}ojasiewicz inequality controls the total length of the SCGA path to be of the linear order and implicates the quadratic behaviors of residual vectors as required by condition (A4).
	
	\begin{lemma}
		\label{quad_bound_prop_PL}
		Assume conditions (A1) and (A7) throughout the lemma.
		\begin{enumerate}[label=(\alph*)]
			\item The total length of the SCGA trajectory is of the linear order, \emph{i.e.}, 
			$$\sum_{s=t}^{\infty} \norm{\bm{x}^{(s+1)} -\bm{x}^{(s)}}_2 \leq \frac{4D\norm{p}_{\infty}^{(2)}}{\beta_3} \norm{\bm{x}^{(t)} - \bm{x}^*}_2 \quad \text{ for any } t\geq 0.$$
			
			\item We further assume condition (A2). Then,
			\begin{align*}
			&\nabla p(\bm{x}^{(t)})^T U_d^{\perp}(\bm{x}^{(t)}) (\bm{x}^*-\bm{x}^{(t)}) \\
			&\leq \frac{32 D^{\frac{9}{2}} \left(\norm{p}_{\infty}^{(2)}\right)^2\norm{U_d^{\perp}(\bm{x}^{(t)}) \nabla p(\bm{x}^{(t)})}_2 \norm{\nabla^3 p(\bm{x}^{(t)})}_{\max} \norm{\bm{x}^*-\bm{x}^{(t)}}_2^2}{\beta_0\beta_3^2},\\
			&\text{ and }\\
			&\norm{U_d^{\perp}(\bm{x}^{(t)}) (\bm{x}^*-\bm{x}^{(t)})}_2 \leq \frac{32 D^{\frac{9}{2}} \left(\norm{p}_{\infty}^{(2)}\right)^2\norm{\nabla^3 p(\bm{x}^{(t)})}_{\max}\norm{\bm{x}^*-\bm{x}^{(t)}}_2^2}{\beta_0\beta_3^2}
			\end{align*}
			for any $\bm{x}^{(t)}\in \text{Ball}_D(\bm{x}^*, r_6)$ with some radius $0 < r_6 < \rho$, where we recall that $\rho>0$ is the effective radius in condition (A2) under which the underlying density $p$ has an eigengap $\beta_0>0$ between the $d$-th and $(d+1)$-th eigenvalues of its Hessian matrix $\nabla\nabla p$. 
		\end{enumerate}
	\end{lemma}
	
	\begin{proof}[Proof of Lemma~\ref{quad_bound_prop_PL}]
		(a) This part of the proof is inspired by the arguments in Theorem 9 of \cite{gupta2021path}. Based on the proof of (a) in Proposition~\ref{SCGA_conv} under condition (A1), we know from \eqref{SCGA_ascending} that 
		\begin{align*}
		p(\bm{x}^{(t+1)}) - p(\bm{x}^{(t)}) &\geq \eta\left(1-\frac{D\norm{p}_{\infty}^{(2)}\eta}{2} \right) \norm{V_d(\bm{x}^{(t)})^T \nabla p(\bm{x}^{(t)})}_2^2 \\
		&\geq \frac{\eta}{2} \norm{V_d(\bm{x}^{(t)})^T \nabla p(\bm{x}^{(t)})}_2^2
		\end{align*}
		when $0<\eta < \frac{1}{D\norm{p}_{\infty}^{(2)}}$. Using this inequality and condition (A7), we derive that
		\begin{align*}
		\sqrt{p(\bm{x}^*) - p(\bm{x}^{(t+1)})} &= \sqrt{p(\bm{x}^*) - p\left(\bm{x}^{(t)} + \eta \cdot V_d(\bm{x}^{(t)}) V_d(\bm{x}^{(t)})^T \nabla p(\bm{x}^{(t)}) \right)}\\
		&\leq \sqrt{p(\bm{x}^*) -p(\bm{x}^{(t)}) - \frac{\eta}{2} \norm{V_d(\bm{x}^{(t)})^T \nabla p(\bm{x}^{(t)})}_2^2} \\
		&\stackrel{\text{(i)}}{\leq} \sqrt{p(\bm{x}^*) -p(\bm{x}^{(t)})} - \frac{\eta \norm{V_d(\bm{x}^{(t)})^T \nabla p(\bm{x}^{(t)})}_2^2}{4\sqrt{p(\bm{x}^*) -p(\bm{x}^{(t)})}} \\
		&\stackrel{\text{(ii)}}{\leq} \sqrt{p(\bm{x}^*) -p(\bm{x}^{(t)})} - \frac{\eta\sqrt{2\beta_3}}{4} \norm{V_d(\bm{x}^{(t)})^T \nabla p(\bm{x}^{(t)})}_2, 
		\end{align*}
		where we use the inequality $\sqrt{a-b} \leq \sqrt{a} - \frac{b}{2\sqrt{a}}$ to obtain (i) and apply condition (A7) in inequality (ii). Since $\norm{\bm{x}^{(t+1)} - \bm{x}^{(t)}}_2 = \eta \norm{V_d(\bm{x}^{(t)})^T \nabla p(\bm{x}^{(t)})}_2$, some rearrangement of the above inequality suggests that
		$$\sqrt{\frac{\beta_3}{8}} \norm{\bm{x}^{(t+1)} - \bm{x}^{(t)}}_2 \leq \sqrt{p(\bm{x}^*) -p(\bm{x}^{(t)})} - \sqrt{p(\bm{x}^*) -p(\bm{x}^{(t+1)})}.$$
		Therefore, 
		\begin{align*}
		\sum_{s=t}^{\infty} \norm{\bm{x}^{(s+1)} - \bm{x}^{(s)}}_2 &\leq \sqrt{\frac{8}{\beta_3}} \sum_{s=t}^{\infty} \left[\sqrt{p(\bm{x}^*) -p(\bm{x}^{(s)})} - \sqrt{p(\bm{x}^*) -p(\bm{x}^{(s+1)})} \right]\\
		&= \sqrt{\frac{8}{\beta_3}} \cdot \sqrt{p(\bm{x}^*) -p(\bm{x}^{(t)})}\\
		&\stackrel{\text{(i)}}{\leq} \sqrt{\frac{8}{\beta_3}} \cdot \sqrt{\frac{1}{2\beta_3}} \cdot \norm{V_d(\bm{x}^{(t)})^T \nabla p(\bm{x}^{(t)})}_2 \quad \text{ by condition (A4)}\\
		&=\frac{2}{\beta_3} \norm{V_d(\bm{x}^{(t)})^T \nabla p(\bm{x}^{(t)}) - \underbrace{V_d(\bm{x}^*)^T \nabla p(\bm{x}^*)}_{=0}}_2\\
		&\stackrel{\text{(ii)}}{\leq} \frac{2}{\beta_3} \norm{\sup_{\epsilon \in [0,1]} M(\bm{x}^*+\epsilon(\bm{x}^{(t)}-\bm{x}^*))^T \left(\bm{x}^{(t)} -\bm{x}^* \right)}_2\\
		&\stackrel{\text{(iii)}}{\leq} \frac{2}{\beta_3} \left(D\norm{p}_{\infty}^{(2)} +\beta_0-\beta_1 \right) \norm{\bm{x}^*-\bm{x}^{(t)}}_2 \\
		&\leq \frac{4D\norm{p}_{\infty}^{(2)}}{\beta_3} \norm{\bm{x}^*-\bm{x}^{(t)}}_2,
		\end{align*}
		where we leverage condition (A7) again in (i). In addition, to obtain inequalities (ii) and (iii), we recall from the proof of (d) for Lemma~\ref{normal_reach_prop} that $M(\bm{x}) = \nabla\left[V_d(\bm{x})^T \nabla p(\bm{x})\right]^T = V_d(\bm{x}) \Lambda_0(\bm{x}) + \sum_{i=1}^d T_i(\bm{x}) V_d(\bm{x}) \Lambda_i(\bm{x})$, in which the singular values of $V_d(\bm{x}) \Lambda_0(\bm{x})$ is bounded by $D\norm{p}_{\infty}^{(2)}$ and the singular values of $\sum_{i=1}^d T_i(\bm{x}) V_d(\bm{x}) \Lambda_i(\bm{x})$ is bounded by $\beta_0-\beta_1 \leq D\norm{p}_{\infty}^{(2)}$. The result thus follows.\\
		
		\noindent (b) This part of the proof is analogous to our arguments in (b) of Lemma~\ref{quad_bound_prop}, except that the SCGA sequence $\big\{\bm{x}^{(t)}\big\}_{t=0}^{\infty}$ is no longer self-contracted. For the completeness, we still repeat some arguments and highlight the differences here. By Davis-Kahan theorem (Lemma~\ref{Davis_K} and \eqref{Davis_Kahan_ineq} here) and conditions (A1-2), we have that for all $s\geq t$,
		\begin{align*}
		&\norm{V_d(\bm{x}^{(s)}) V_d(\bm{x}^{(s)})^T - V_d(\bm{x}^{(t)}) V_d(\bm{x}^{(t)})^T}_2 \\
		& \leq \frac{\sqrt{2}D \norm{\nabla\nabla p(\bm{x}^{(s)}) - \nabla\nabla p(\bm{x}^{(t)})}_2}{\beta_0}\\
		&\leq \frac{\sqrt{2}D \norm{\nabla^3 p(\bm{x}^{(t)})}_2\norm{\bm{x}^{(s)} - \bm{x}^{(t)}}_2}{\beta_0} + \frac{\sqrt{2} D^3 \norm{p}_{\infty}^{(4)}}{\beta_0} \norm{\bm{x}^{(s)} - \bm{x}^{(t)}}_2^2\\
		&\stackrel{\text{(ii)}}{\leq}  \frac{4\sqrt{2}D^2 \norm{p}_{\infty}^{(2)}\norm{\nabla^3 p(\bm{x}^{(t)})}_2 \norm{\bm{x}^* - \bm{x}^{(t)}}_2}{\beta_0\beta_3} + \frac{16\sqrt{2}D^5 \norm{p}_{\infty}^{(4)} \left(\norm{p}_{\infty}^{(2)} \right)^2 }{\beta_0\beta_3^2} \norm{\bm{x}^{(t)}-\bm{x}^*}_2^2\\
		&\stackrel{\text{(ii)}}{\leq} \frac{8D^{\frac{7}{2}} \norm{p}_{\infty}^{(2)}\norm{\nabla^3 p(\bm{x}^{(t)})}_{\max} \norm{\bm{x}^* - \bm{x}^{(t)}}_2}{\beta_0\beta_3},
		\end{align*}
		where we possibly shrink the radius $r_6 >0$ so that $\norm{\bm{x}^{(t)} -\bm{x}^*}_2 \leq \frac{(\sqrt{2} -1) \beta_3 \norm{\nabla^3 p(\bm{x}^{(t)})}_{\max}}{4D^{\frac{3}{2}} \norm{p}_{\infty}^{(2)} \norm{p}_{\infty}^{(4)}}$ to obtain inequality (ii). Notice also that, since $\norm{\bm{x}^{(s)} -\bm{x}^{(t)}}_2 \leq \norm{\bm{x}^* -\bm{x}^{(t)}}_2$ may not hold without the self-contractedness property, we use a looser bound 
		$$\norm{\bm{x}^{(s)} -\bm{x}^{(t)}}_2 \leq \sum_{s=t}^{\infty} \norm{\bm{x}^{(s+1)} - \bm{x}^{(s)}}_2 \leq \frac{4D\norm{p}_{\infty}^{(2)}}{\beta_3}\norm{\bm{x}^{(t)} -\bm{x}^*}_2$$ 
		from (a) to derive inequality (i). Therefore, by \eqref{x_dist_residual} and the fact that 
		$$\norm{\eta V_d(\bm{x}^{(s)}) V_d(\bm{x}^{(s)})^T \nabla p(\bm{x}^{(s)})}_2 =\norm{\bm{x}^{(s+1)}-\bm{x}^{(s)}}_2,$$ 
		we obtain that
		\begin{align*}
		&\norm{U_d^{\perp}(\bm{x}^{(t)})\left(\bm{x}^* -\bm{x}^{(t)} \right)}_2 \\
		& \leq \sup_{s\geq t} \norm{V_d(\bm{x}^{(s)}) V_d(\bm{x}^{(s)})^T - V_d(\bm{x}^{(t)}) V_d(\bm{x}^{(t)})^T}_2 \cdot \sum_{s=t}^{\infty} \norm{\eta V_d(\bm{x}^{(s)}) V_d(\bm{x}^{(s)})^T \nabla p(\bm{x}^{(s)})}_2\\
		&\leq \frac{8D^{\frac{7}{2}} \norm{p}_{\infty}^{(2)}\norm{\nabla^3 p(\bm{x}^{(t)})}_{\max} \norm{\bm{x}^{(s)} - \bm{x}^{(t)}}_2}{\beta_0\beta_3} \cdot \frac{4D\norm{p}_{\infty}^{(2)}}{\beta_3} \norm{\bm{x}^{(t)} - \bm{x}^*}_2\\
		&= \frac{32 D^{\frac{9}{2}} \left(\norm{p}_{\infty}^{(2)}\right)^2\norm{\nabla^3 p(\bm{x}^{(t)})}_{\max}\norm{\bm{x}^*-\bm{x}^{(t)}}_2^2}{\beta_0\beta_3^2},
		\end{align*}
		which implies the second bound in condition (A4) with $\beta_2=\frac{32 D^{\frac{9}{2}} \left(\norm{p}_{\infty}^{(2)}\right)^2\norm{p}_{\infty}^{(3)}}{\beta_0\beta_3^2}$. Finally, 
		\begin{align*}
		&\nabla p(\bm{x}^{(t)})^T U_d^{\perp}(\bm{x}^{(t)}) (\bm{x}^*-\bm{x}^{(t)}) \\
		&\leq \norm{U_d^{\perp}(\bm{x}^{(t)}) \nabla p(\bm{x}^{(t)})}_2 \norm{U_d^{\perp}(\bm{x}^{(t)})\left(\bm{x}^* -\bm{x}^{(t)} \right)}_2\\
		&\leq \frac{32 D^{\frac{9}{2}} \left(\norm{p}_{\infty}^{(2)}\right)^2\norm{U_d^{\perp}(\bm{x}^{(t)}) \nabla p(\bm{x}^{(t)})}_2 \norm{\nabla^3 p(\bm{x}^{(t)})}_{\max}\norm{\bm{x}^*-\bm{x}^{(t)}}_2^2}{\beta_0\beta_3^2}.
		\end{align*}
		The results follow.
	\end{proof}
	
	The results in (b) of Lemma~\ref{quad_bound_prop_PL} also imply condition (A4) with $\beta_2 =\frac{32 D^{\frac{9}{2}} \left(\norm{p}_{\infty}^{(2)}\right)^2\norm{p}_{\infty}^{(3)}}{\beta_0\beta_3^2}$ whenever
	\begin{equation}
	\label{cond_A3_mod2}
	\frac{32 D^{\frac{9}{2}} \left(\norm{p}_{\infty}^{(2)}\right)^2\norm{U_d^{\perp}(\bm{x}^{(t)}) \nabla p(\bm{x}^{(t)})}_2 \norm{\nabla^3 p(\bm{x}^{(t)})}_{\max}}{\beta_3^2} \leq \frac{\beta_0^2}{4}.
	\end{equation}
	Once again, the choice of $\beta_2$ is feasible under condition (A1) and the upper bound \eqref{cond_A3_mod2} can be viewed as a variant of the first inequality in condition (A3). From this perspective, the subspace constrained Polyak-{\L}ojasiewicz inequality (A7) also leads to an alternative assumptions for condition (A4) and the linear convergence of the SCGA algorithm.
	
	\begin{remark}
		\label{quad_bound_prop_Dir}
		Note that the results in Proposition~\ref{quad_bound_prop_PL} can be generalized to the directional or arbitrary manifold cases under conditions (\underline{A1-3}). First, the subspace constrained Polyak-{\L}ojasiewicz inequality for the SCGA sequence $\big\{\underline{\bm{x}}^{(t)} \big\}_{t=0}^{\infty}$ on $\Omega_q$ or an arbitrary manifold can be modified as:
		$$\frac{1}{2} \norm{\underline{V}_d(\underline{\bm{x}}^{(t)})^T \grad f(\underline{\bm{x}}^{(t)})}_2^2 \geq \underline{\beta}_3 \left[f(\underline{\bm{x}}^*) -f(\underline{\bm{x}}^{(t)}) \right] \quad \text{ for some } \underline{\beta}_3 >0 \text{ and any }t\geq 0,$$
		where $f$ is the objective (density) function. Based on the proof of (a) in Proposition~\ref{SCGA_Dir_conv} and our arguments in (a) of Lemma~\ref{quad_bound_prop_PL}, it follows that the total length of the SCGA trajectory on $\Omega_q$ or an arbitrary manifold is of the linear order, \emph{i.e.},
		$$\sum_{s=t}^{\infty} d_g\left(\underline{\bm{x}}^{(s+1)},\underline{\bm{x}}^{(s)} \right) \leq \frac{4q\norm{\mathcal{H}f}_{\infty}^{(2)}}{\underline{\beta}_3} \cdot d_g\left(\underline{\bm{x}}^{(t)},\underline{\bm{x}}^* \right) \quad \text{ for any } t\geq 0.$$
		Second, to establish the quadratic bounds for $\left\langle \underline{U}_d^{\perp}(\underline{\bm{x}}^{(t)}) \grad f(\underline{\bm{x}}^{(t)}),\, \Exp_{\underline{\bm{x}}^{(t)}}^{-1}(\underline{\bm{x}}^*) \right\rangle$ and $\norm{\underline{U}_d^{\perp}(\underline{\bm{x}}^{(t)}) \Exp_{\underline{\bm{x}}^{(t)}}^{-1}(\underline{\bm{x}}^*)}_2$, one can follow the arguments in the proof of (b) in Lemma~\ref{quad_bound_prop_PL} and leverage the two facts:
		
		1. The tangent vector $\Exp_{\underline{\bm{x}}^{(t)}}^{-1}(\underline{\bm{x}}^*)$ can be decomposed into an infinite sum of SCGA updates \eqref{SCGA_manifold_update} on $\Omega_q$ or an arbitrary manifold as:
		\begin{align*}
		\Exp_{\underline{\bm{x}}^{(t)}}^{-1}(\underline{\bm{x}}^*) &= \sum_{s=t}^{\infty} \Gamma_{\underline{\bm{x}}^{(s)}}^{\underline{\bm{x}}^{(t)}}\left(\Exp_{\underline{\bm{x}}^{(s)}}^{-1}(\underline{\bm{x}}^{(s+1)}) \right).
		\end{align*}
		See Figure~\ref{fig:SCGA_decomp_Dir} for a graphical illustration. This equation is valid because parallel transports preserve inner products and are linear.
		\begin{figure}
			\centering
			\includegraphics[width=0.9\linewidth]{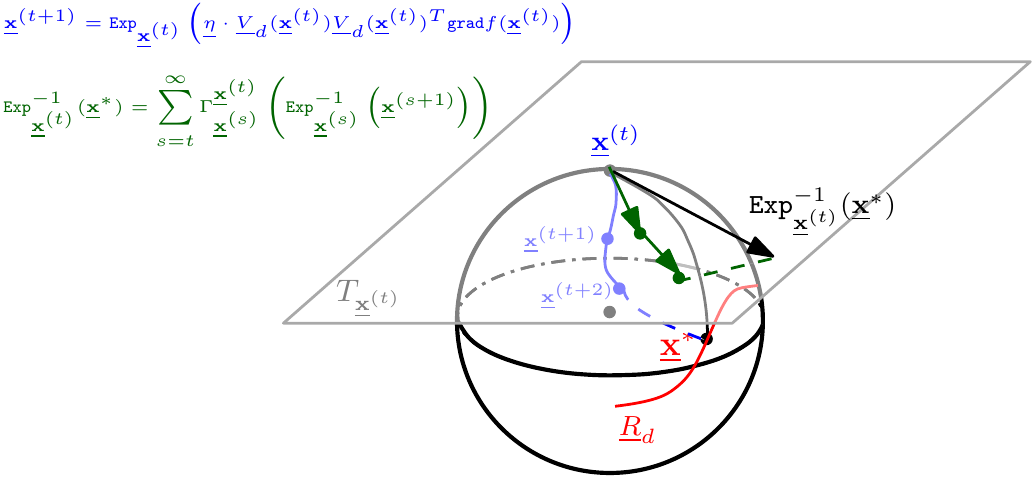}
			\caption{Decomposition of the vector $\Exp_{\underline{\bm{x}}^{(t)}}^{-1}(\underline{\bm{x}}^*)$ within the tangent space $T_{\underline{\bm{x}}^{(t)}}$ into the summation $\sum\limits_{s=t}^{\infty} \Gamma_{\underline{\bm{x}}^{(s)}}^{\underline{\bm{x}}^{(t)}}\left(\Exp_{\underline{\bm{x}}^{(s)}}^{-1}(\underline{\bm{x}}^{(s+1)}) \right)$ of parallel transported SCGA iterative vectors. Here, the blue curves on $\Omega_q$ are iterative paths of the SCGA algorithm, while the green vectors are tangent vectors $\Exp_{\underline{\bm{x}}^{(s)}}^{-1}(\underline{\bm{x}}^{(s+1)}) \in T_{\underline{\bm{x}}^{(s)}}$ after being parallel transported to $T_{\underline{\bm{x}}^{(t)}}$.}
			\label{fig:SCGA_decomp_Dir}
		\end{figure}
		
		2. Under conditions (\underline{A1-2}), we know that 
		\begin{align*}
		&\norm{\underline{U}_d^{\perp}(\underline{\bm{x}}^{(t)}) \cdot  \sum_{s=t}^{\infty} \Gamma_{\underline{\bm{x}}^{(s)}}^{\underline{\bm{x}}^{(t)}}\left(\Exp_{\underline{\bm{x}}^{(s)}}^{-1}(\underline{\bm{x}}^{(s+1)}) \right)}_2\\
		&=\Bigg\|\sum_{s=t}^{\infty} \underline{U}_d^{\perp}(\underline{\bm{x}}^{(t)}) \bigg[\Gamma_{\underline{\bm{x}}^{(s)}}^{\underline{\bm{x}}^{(t)}}\left(\underline{\eta} \underline{V}_d(\underline{\bm{x}}^{(s)}) \underline{V}_d(\underline{\bm{x}}^{(s)})^T \underline{V}_d(\underline{\bm{x}}^{(s)}) \underline{V}_d(\underline{\bm{x}}^{(s)})^T \grad f(\underline{\bm{x}}^{(s)})\right) \\
		&\hspace{20mm} - \underline{\eta}\underline{V}_d(\underline{\bm{x}}^{(t)}) \underline{V}_d(\underline{\bm{x}}^{(t)})^T \underline{V}_d(\underline{\bm{x}}^{(s)}) \underline{V}_d(\underline{\bm{x}}^{(s)})^T  \grad f(\underline{\bm{x}}^{(s)})  \bigg]\Bigg\|_2 \\
		&\leq \sum_{s=t}^{\infty} \tilde{A} \norm{\underline{V}_d(\underline{\bm{x}}^{(s)}) \underline{V}_d(\underline{\bm{x}}^{(s)})^T - \underline{V}_d(\underline{\bm{x}}^{(t)}) \underline{V}_d(\underline{\bm{x}}^{(t)})^T }_2 \cdot \norm{\underline{\bm{x}}^{(s)} - \underline{\bm{x}}^{(t)}}_2\\
		&\leq \tilde{A} \norm{\underline{\bm{x}}^*-\bm{x}^{(t)}}_2 \cdot \sum_{s=t}^{\infty} \norm{\underline{\bm{x}}^{(s)} - \underline{\bm{x}}^{(t)}}_2\\
		&=O\left(\norm{\underline{\bm{x}}^{(s)} - \underline{\bm{x}}^{(t)}}_2^2\right)
		\end{align*}
		for some constant $\underline{A}>0$, where we leverage the fact that the vector field 
		$$X(\gamma(t)) = \underline{V}_d(\gamma(t)) \underline{V}_d(\gamma(t))^T \underline{V}_d(\underline{\bm{x}}^{(s)}) \underline{V}_d(\underline{\bm{x}}^{(s)})^T  \grad f(\underline{\bm{x}}^{(s)})$$
		with $\gamma(0)=\underline{\bm{x}}^{(s)}$ and $\gamma(1)=\underline{\bm{x}}^{(t)}$ has its variation $\norm{X(\gamma(t_1))-X(\gamma(t_2))}_2$ bounded by $$\norm{\underline{V}_d(\underline{\bm{x}}^{(s)}) \underline{V}_d(\underline{\bm{x}}^{(s)})^T - \underline{V}_d(\underline{\bm{x}}^{(t)}) \underline{V}_d(\underline{\bm{x}}^{(t)})^T }_2 = O\left(\norm{\underline{\bm{x}}^{(s)} -\underline{\bm{x}}^{(t)}}_2\right)$$ 
		according to the Davis-Kahan theorem for any $0\leq t_1,t_2\leq 1$. However, we are not sure if the self-contractedness condition can also be adaptive to the directional or general manifold cases, given that the arguments in Theorem 15 of \cite{gupta2021path} are based on the Euclidean geometry. 
	\end{remark}

\section{Other Technical Concepts of Differential Geometry on $\Omega_q$}
\label{App:DG_sphere}

\noindent $\bullet$ {\bf Taylor's Theorem on $\Omega_q$}. Given a smooth function $f$ on $\Omega_q$, its Taylor's expansion is often written as \citep{Intrin_Stat_manifolds2006}:
\begin{equation}
\label{taylor_thm_sphere}
f(\Exp_{\bm x}(\bm{v})) = f(\bm{x}) + \left\langle \grad f(\bm{x}), \bm{v} \right\rangle + \frac{1}{2} \bm{v}^T \mathcal{H} f(\bm{x}) \bm{v} + o\left(\norm{\bm{v}}_2^2\right)
\end{equation} 
for any $\bm{v}\in T_{\bm{x}}$, where $\Exp_{\bm x}: T_{\bm x} \to \Omega_q$ is the \emph{exponential map} at $\bm x\in \Omega_q$. One may replace the exponential map with a more general concept called the \emph{retractions} on an arbitrary manifold; see Section 4.1 and Proposition 5.5.5 in \cite{Op_algo_mat_manifold2008}.

\noindent $\bullet$ {\bf Parallel Transport}. When comparing vectors in two different tangent spaces $T_{\bm{x}}, T_{\bm{y}}$ on $\Omega_q$, we leverage the notion of \emph{parallel transport} $\Gamma_{\bm{x}}^{\bm{y}}: T_{\bm{x}} \to T_{\bm{y}}$ to transport vectors from one tangent space to another along a geodesic. In addition, $\Gamma_{\bm{x}}^{\bm{y}}(\bm{v})$ is a tangent vector in $T_{\bm{y}}$ after being parallel transported from $\bm{v}\in T_{\bm{x}}$ along a geodesic (or great circle) on $\Omega_q$. The parallel transport mapping $\Gamma_{\bm{x}}^{\bm{y}}$ is a linear isometry along any smooth curve on $\Omega_q$, \emph{i.e.}, $\left\langle \Gamma_{\bm x}^{\bm{y}}(\bm{u}), \Gamma_{\bm x}^{\bm{y}}(\bm{v}) \right\rangle = \langle \bm{u}, \bm{v} \rangle$ for any $\bm{u},\bm{v} \in T_{\bm{x}}$; see Proposition 5.5 in \cite{Lee2018Riem_man} or Proposition 1 in Section 4-4 of \cite{doCarmo}.

\noindent $\bullet$ {\bf Sectional Curvature}. \emph{Sectional curvature} is the Gaussian curvature of a two-dimensional submanifold formed as the image of a two-dimensional subspace of a tangent space after exponential mapping; see Section 3-2 in \cite{doCarmo} for detailed discussions about the Gaussian curvature. It is known that a two dimensional submanifold with positive, zero, or negative sectional curvature is locally isometric to a two dimensional sphere, a Euclidean plane, or a hyperbolic plane with the same Gaussian curvature \citep{Geo_Convex_Op2016}.

\noindent $\bullet$ {\bf Geodesically Strong Concavity}. A function $f:\Omega_q \to \mathbb{R}$ is said to be \emph{geodesically concave} if for any $\bm{x},\bm{y} \in \Omega_q$, it holds that
	$$f(\varphi(t)) \geq (1-t) f(\bm{x}) + f(\bm{y})$$
	for any $t\in [0,1]$, where $\varphi: [0,1]\to \Omega_q$ is a geodesic with $\varphi(0)=\bm{x}$ and $\varphi(1) =\bm{y}$. When $f$ is differentiable, an equivalent statement of the geodesic concavity is that (Theorem 11.17 in \citealt{boumal2020introduction}):
	$$f(\bm{y}) - f(\bm{x}) \leq \left\langle \grad f(\bm{x}), \Exp_{\bm x}^{-1}(\bm{y}) \right\rangle.$$
	A function $f:\Omega_q \to \mathbb{R}$ is said to be \emph{geodesically $\mu_g$-strongly concave} if for any $\bm{x},\bm{y} \in \Omega_q$, it holds that
	$$f(\bm{y}) \leq f(\bm{x}) + \left\langle \grad f(\bm{x}), \Exp_{\bm x}^{-1}(\bm{y}) \right\rangle - \frac{\mu_g}{2} \cdot d_g(\bm{x},\bm{y})^2.$$

\section{Normal Space of Directional Density Ridge}
\label{App:Normal_space_Dir}

Recall that we extend the directional density $f$ from its support $\Omega_q$ to $\mathbb{R}^{q+1}\setminus \left\{\bm{0}\right\}$ by defining $f(\bm{x})\equiv f\left(\frac{\bm{x}}{\norm{\bm{x}}_2} \right)$ for all $\bm{x}\in \mathbb{R}^{q+1}\setminus \left\{\bm{0}\right\}$. As we will refer to conditions (\underline{A1-3}) frequently in the next three sections, we restate them here:

\begin{itemize}
	\item {\bf (\underline{A1})} (\emph{Differentiability}) Under the extension \eqref{DirDensity_Ext} of the directional density $f$, we assume that the total gradient $\nabla f(\bm{x})$, total Hessian matrix $\nabla\nabla f(\bm{x})$, and third-order derivative tensor $\nabla^3 f(\bm{x})$ in $\mathbb{R}^{q+1}$ exist, and are continuous on $\mathbb{R}^{q+1} \setminus \{\bm{0} \}$ and square integrable on $\Omega_q$. We also assume that $f$ has bounded fourth order derivatives on $\Omega_q$. 
	
	\item {\bf (\underline{A2})} (\emph{Eigengap}) We assume that there exist constants $\underline{\rho}>0$ and $\underline{\beta}_0 >0$ such that $\underline{\lambda}_{d+1}(\bm{y}) \leq -\underline{\beta}_0$ and $\underline{\lambda}_d(\bm{y}) - \underline{\lambda}_{d+1}(\bm{y}) \geq \underline{\beta}_0$ for any $\bm{y}\in \left(\underline{R}_d \oplus \underline{\rho} \right) \cap \Omega_q$.
	
	\item {\bf (\underline{A3})} (\emph{Path Smoothness}) Under the same $\underline{\rho}, \underline{\beta}_0 >0$ in (\underline{A2}), we assume that there exists another constant $\underline{\beta}_1 \in \left(0,\underline{\beta}_0 \right)$ such that
	\begin{align*}
	\sqrt{2}\cdot q^{\frac{3}{2}} \norm{\underline{U}_d^{\perp}(\bm{y}) \grad f(\bm{y})}_2 \norm{\nabla^3 f(\bm{y})}_{\max} &\leq \frac{\underline{\beta}_0^2}{2},\\
	d \cdot q^{\frac{3}{2}} \norm{\nabla f(\bm{x})}_2 \cdot \norm{\nabla^3 f(\bm{x})}_{\max} & \leq \underline{\beta}_0\left(\underline{\beta}_0-\underline{\beta}_1 \right)
	\end{align*}
	for all $\bm{y} \in \left(\underline{R}_d\oplus \underline{\rho} \right)\cap \Omega_q$ and $\bm{x}\in \underline{R}_d$.
\end{itemize}

Recall that an order-$d$ density ridge of a directional density $f$ on $\Omega_q=\left\{\bm{x}\in \mathbb{R}^{q+1}:\norm{\bm{x}}_2=1 \right\}$ is the set of points defined as:
\begin{equation}
\label{Dir_ridge_app}
\underline{R}_d = \left\{\bm{x}\in \Omega_q: \underline{G}_d(\bm{x})=\bm{0}, \underline{\lambda}_{d+1}(\bm{x}) <0 \right\} = \left\{\bm{x}\in \Omega_q: \underline{V}_d(\bm{x})^T \grad f(\bm{x})=\bm{0}, \underline{\lambda}_{d+1}(\bm{x}) <0 \right\}.
\end{equation}

Lemma~\ref{Dir_norm_reach_prop} below shows that under conditions (\underline{A1-3}), the Jacobian matrices $\nabla \left[\underline{V}_d(\bm{x})^T \grad f(\bm{x}) \right] \in \mathbb{R}^{(q-d)\times (q+1)}$ and $\left(\bm{I}_{q+1} -\bm{x}\bm{x}^T \right)\nabla \left[\underline{V}_d(\bm{x})^T \grad f(\bm{x}) \right]^T \in \mathbb{R}^{(q+1)\times (q-d)}$ (\emph{i.e.}, projecting the columns of $\nabla \left[\underline{V}_d(\bm{x})^T \grad f(\bm{x}) \right]^T$ onto the tangent space $T_{\bm{x}}$) both have rank $q-d$ at every point on $\underline{R}_d$, and $\underline{R}_d$ will be a $d$-dimensional submanifold on $\Omega_q$ by the implicit function theorem \citep{Rudin1976,Lee2012}. 
Analogous to the discussion about the normal space of a Euclidean density ridge in Appendix~\ref{App:Normal_space_Eu}, we define 
\begin{align*}
\underline{M}(\bm{x}) =\nabla \left[\underline{V}_d(\bm{x})^T \grad f(\bm{x}) \right]^T &=  \nabla \left[\underline{V}_d(\bm{x})^T \nabla f(\bm{x}) \right]^T \\
&= \left(\underline{\bm{m}}_{d+1}(\bm{x}),...,\underline{\bm{m}}_q(\bm{x})\right) \in \mathbb{R}^{(q+1)\times (q-d)}.
\end{align*}
Different from the Euclidean density ridge case, it is the column space of 
\begin{equation}
	\label{Dir_normal_ext}
	\underline{M}_E(\bm{x})=\left(\nabla\left[\underline{V}_d(\bm{x})^T \nabla f(\bm{x}) \right]^T, \bm{x} \right)=\left(\underline{\bm{m}}_{d+1}(\bm{x}),...,\underline{\bm{m}}_q(\bm{x}), \bm{x} \right)\in \mathbb{R}^{(q+1)\times (q+1-d)}
\end{equation}
	that spans the normal space of $\underline{R}_d$ within the ambient space $\mathbb{R}^{q+1}$. It can be seen from our Remark~\ref{solution_manifold_remark} that the rows of 
	$$\nabla\Psi(\bm{x}) = \begin{pmatrix}
	\nabla \left[\underline{\bm{v}}_{d+1}(\bm{x})^T \nabla f(\bm{x}) \right]\\
	\vdots\\
	\nabla \left[\underline{\bm{v}}_q(\bm{x})^T \nabla f(\bm{x}) \right]\\
	2\bm{x}
	\end{pmatrix}$$
	spans the normal space of the solution manifold $\underline{R}_d$; see also Lemma 1 in \cite{YC2020}.
Consequently, the column space of $\left(\bm{I}_{q+1} -\bm{x}\bm{x}^T \right)\underline{M}(\bm{x}) = \left(\bm{I}_{q+1} -\bm{x}\bm{x}^T \right)\nabla \left[\underline{V}_d(\bm{x})^T \grad f(\bm{x}) \right]^T$ spans the normal space of $\underline{R}_d$ within the tangent space $T_{\bm{x}}$ at each $\bm{x}\in \underline{R}_d \subset \Omega_q$. 
The technique in pages 60-63 of \cite{Eberly1996ridges} is still valid to argue that
\begin{align}
\label{Dir_normal_rows}
\begin{split}
&\underline{\bm{m}}_k(\bm{x}) \\
&= \left[\underline{\lambda}_k(\bm{x}) \bm{I}_{q+1} +\sum_{i=1}^d \frac{\underline{\bm{v}}_i(\bm{x})^T \nabla f(\bm{x})}{\underline{\lambda}_k(\bm{x})- \underline{\lambda}_i(\bm{x})} \cdot \underline{\bm{v}}_i(\bm{x})^T \nabla^3 f(\bm{x}) + \frac{\bm{x}^T \nabla f(\bm{x})}{\underline{\lambda}_k(\bm{x})}\cdot \bm{x}^T \nabla^3 f(\bm{x}) \right] \underline{\bm{v}}_k(\bm{x})\\
&= \left[\underline{\lambda}_k(\bm{x}) \bm{I}_{q+1} +\sum_{i=1}^d \frac{\underline{\bm{v}}_i(\bm{x})^T \nabla f(\bm{x})}{\underline{\lambda}_k(\bm{x})-\underline{\lambda}_i(\bm{x})} \cdot \underline{\bm{v}}_i(\bm{x})^T \nabla^3 f(\bm{x}) \right] \underline{\bm{v}}_k(\bm{x})
\end{split}
\end{align}
for $k=d+1,...,q$, where we use the fact that $\bm{x}^T\nabla f(\bm{x}) = \bm{x}^T\grad f(\bm{x})=0$ on $\Omega_q$ under the extension of $f$ as in (\underline{A1}). Let
\begin{align*}
\underline{\Lambda}_0(\bm{x}) &= \Diag\left[\underline{\lambda}_{d+1}(\bm{x}),...,\underline{\lambda}_q(\bm{x})\right],\\
\underline{\Lambda}_i(\bm{x}) &= \Diag\left[\frac{1}{\underline{\lambda}_{d+1}(\bm{x}) - \underline{\lambda}_i(\bm{x})},..., \frac{1}{\underline{\lambda}_q(\bm{x}) - \underline{\lambda}_i(\bm{x})} \right],\\
\underline{T}_i(\bm{x}) &= \left[\underline{\bm{v}}_i(\bm{x})^T \nabla f(\bm{x}) \right] \cdot \underline{\bm{v}}_i(\bm{x})^T \nabla^3 f(\bm{x})
\end{align*}
for $i=1,...,d$. Then, 
\begin{equation}
\label{Dir_normal_rows_all}
\underline{M}(\bm{x}) = \underline{V}_d(\bm{x}) \underline{\Lambda}_0(\bm{x}) + \sum_{i=1}^d \underline{T}_i(\bm{x}) \underline{V}_d(\bm{x}) \underline{\Lambda}_i(\bm{x}).
\end{equation}

As in the Euclidean data case, the columns of $\underline{M}_E(\bm{x}) = \left[\underline{M}(\bm{x}), \bm{x}\right] \in \mathbb{R}^{(q+1)\times (q+1-d)}$ are not orthonormal, and we again leverage the orthonormalization technique in \cite{Asymp_ridge2015} to construct $\underline{N}(\bm{x})$ that shares the same column space with $\underline{M}_E(\bm{x})$ but has orthonormal columns. That is, under the condition that $\underline{M}_E(\bm{x})$ has full column rank $q-d$ at every point $\bm{x}\in \underline{R}_d$ (see Lemma~\ref{Dir_norm_reach_prop}),
\begin{equation}
\label{Dir_normal_rows_all_ortho}
\underline{N}(\bm{x}) = \underline{M}_E(\bm{x})\left[\underline{J}(\bm{x})^T \right]^{-1}
\end{equation}
with the Cholesky decomposition $\underline{M}_E(\bm{x})^T \underline{M}_E(\bm{x}) = \underline{J}(\bm{x}) \underline{J}(\bm{x})^T$, where $\underline{J}(\bm{x}) \in \mathbb{R}^{(q+1-d)\times (q+1-d)}$ is a lower triangular matrix whose diagonal elements are positive. Finally, the non-uniqueness of $\underline{M}_E(\bm{x})$ will not affect our subsequent discussions about the properties of directional density ridges.

\begin{lemma}
	\label{Dir_norm_reach_prop}
	Assume conditions (\underline{A1-3}). Given that $\underline{M}(\bm{x})$, $\underline{M}_E(\bm{x})=\left[\underline{M}(\bm{x}), \bm{x}\right]$, and $\underline{N}(\bm{x})$ are defined in \eqref{Dir_normal_rows_all} and \eqref{Dir_normal_rows_all_ortho}, we have the following properties:
	\begin{enumerate}[label=(\alph*)]
		\item $\underline{M}_E(\bm{x})$ and $\underline{N}(\bm{x})$ have the same column space. In addition,
		$$\underline{N}(\bm{x}) \underline{N}(\bm{x})^T = \underline{M}_E(\bm{x})\left[\underline{M}_E(\bm{x})^T \underline{M}_E(\bm{x}) \right]^{-1} \underline{M}_E(\bm{x})^T.$$
		That is, $\underline{N}(\bm{x}) \underline{N}(\bm{x})^T$ is the projection matrix onto the columns of $\underline{M}_E(\bm{x})$.
		\item The columns of $\underline{N}(\bm{x})$ are orthonormal to each other.
		\item For $\bm{x}\in \underline{R}_d$, the column space of $\underline{N}(\bm{x})$ is normal to the (tangent) direction of $\underline{R}_d$ at $\bm{x}$.
		\item For $\bm{x} \in \underline{R}_d$, the smallest eigenvalue $\lambda_{\min}(\underline{M}(\bm{x})^T \underline{M}(\bm{x}))=\lambda_{\min}(\underline{M}(\bm{x})^T (\bm{I}_{q+1}-\bm{x}\bm{x}^T) \underline{M}(\bm{x}))\geq \underline{\beta}_1^2>0$, and 
		$$\rank(\underline{M}(\bm{x})) =\rank\left((\bm{I}_{q+1}-\bm{x}\bm{x}^T) \underline{M}(\bm{x}) \right) = q-d.$$ 
			Moreover, all the nonzero singular values of $\underline{M}_E(\bm{x})$ are greater than $\min\left\{\underline{\beta}_1, 1 \right\} >0$, and
			$$\rank(\underline{N}(\bm{x}))=\rank(\underline{M}_E(\bm{x})) = q+1-d.$$
			Therefore, $\underline{R}_d$ is a $d$-dimensional submanifold that contains neither intersections and nor endpoints on $\Omega_q$. Namely, $\underline{R}_d$ is a finite union of connected and compact submanifolds on $\Omega_q$.
		\item For all $\bm{x} \in \underline{R}_d$,
		$$\norm{\left[\underline{M}(\bm{x})^T \left(\bm{I}_{q+1} -\bm{x}\bm{x}^T \right) \underline{M}(\bm{x})\right]^{-1}}_2\leq \frac{1}{\underline{\beta}_1^2} \quad \text{ and } \quad \norm{\left[\underline{J}(\bm{x})^T \right]^{-1}}_2 \leq \max\left\{\frac{1}{\underline{\beta}_1}, 1 \right\}.$$
		\item When $\norm{\bm{x}-\bm{y}}_2$ is sufficiently small and $\bm{x},\bm{y} \in \left(\underline{R}_d\oplus \underline{\rho} \right) \cap \Omega_q$, 
		$$\norm{\underline{N}(\bm{x}) \underline{N}(\bm{x})^T - \underline{N}(\bm{y}) \underline{N}(\bm{y})^T}_{\max} \leq \underline{A}_0 \left(\norm{f}_{\infty}^{(3)} + \norm{f}_{\infty}^{(4)} \right)^2 \norm{\bm{x}-\bm{y}}_2$$
		for some constant $\underline{A}_0>0$.
		\item Assume that another directional density function $g$ also satisfies conditions (\underline{A1-3}) after the extension $g(\bm{x})\equiv g\left(\frac{\bm{x}}{\norm{\bm{x}}} \right)$ in $\mathbb{R}^{q+1}\setminus \{\bm{0} \}$, and $\norm{f-g}_{\infty,3}^*$ is sufficiently small. Then,
		$$\norm{\underline{N}_f(\bm{x}) \underline{N}_f(\bm{x})^T - \underline{N}_g(\bm{x}) \underline{N}_g(\bm{x})^T}_{\max} \leq \underline{A}_1 \cdot \norm{f-g}_{\infty,3}^*$$
		for some constant $\underline{A}_1 >0$ and any $\bm{x} \in \underline{R}_d$, where $\underline{N}_f(\bm{x})$ is the matrix defined in \eqref{Dir_normal_rows_all_ortho} with directional density $f$.
		\item The reach of $\underline{R}_d$ satisfies
		$$\mathtt{reach}(\underline{R}_d) \geq \min\left\{\underline{\rho}/2, \frac{\min\left\{\underline{\beta}_1, 1\right\}^2}{\underline{A}_2 \left(\norm{f}_{\infty}^{(3)} + \norm{f}_{\infty}^{(4)} \right)} \right\}$$
		for some constant $\underline{A}_2>0$.
	\end{enumerate} 
\end{lemma}

This lemma is a direct extension of Lemma~\ref{normal_reach_prop} to the directional data scenario; thus, its proof is similar to the proof of Lemma~\ref{normal_reach_prop}.

\begin{proof}[Proof of Lemma~\ref{Dir_norm_reach_prop}]
	The proofs of properties (a), (b), and (c) can be inherited from the corresponding ones in Lemma~\ref{normal_reach_prop} with mild modifications and we thus omit them. \\
	
	\noindent(d) We will prove that the $(q-d)$ nonzero singular values of $\underline{M}(\bm{x})$ and $\left(\bm{I}_{q+1} -\bm{x}\bm{x}^T \right) \underline{M}(\bm{x}) \in \mathbb{R}^{(q+1)\times (q-d)}$ are bounded away from 0. Recall that
	$$\underline{M}(\bm{x}) = \underline{V}_d(\bm{x}) \underline{\Lambda}_0(\bm{x}) + \sum_{i=1}^d \underline{T}_i(\bm{x}) \underline{V}_d(\bm{x}) \underline{\Lambda}_i(\bm{x})$$
	with 
	\begin{align*}
	\underline{\Lambda}_0(\bm{x}) &= \Diag\left[\underline{\lambda}_{d+1}(\bm{x}),...,\underline{\lambda}_q(\bm{x})\right]\\
	\underline{\Lambda}_i(\bm{x}) &= \Diag\left[\frac{1}{\underline{\lambda}_{d+1}(\bm{x}) - \underline{\lambda}_i(\bm{x})},..., \frac{1}{\underline{\lambda}_q(\bm{x}) - \underline{\lambda}_i(\bm{x})} \right]\\
	\underline{T}_i(\bm{x}) &= \left[\underline{\bm{v}}_i(\bm{x})^T \nabla f(\bm{x}) \right] \cdot \underline{\bm{v}}_i(\bm{x})^T \nabla^3 f(\bm{x})
	\end{align*}
	for $i=1,...,d$. Under condition (\underline{A2}), 
	\begin{align*}
	&\norm{\left(\bm{I}_{q+1} -\bm{x}\bm{x}^T \right) \cdot  \sum_{i=1}^d \underline{T}_i(\bm{x}) \underline{V}_d(\bm{x}) \underline{\Lambda}_i(\bm{x})}_2\\ &\leq \norm{\sum_{i=1}^d \underline{T}_i(\bm{x}) \underline{V}_d(\bm{x}) \underline{\Lambda}_i(\bm{x})}_2 \quad \text{ since } \norm{\bm{I}_{q+1} -\bm{x}\bm{x}^T}_2=1\\
	&\leq \sum_{i=1}^d \norm{\underline{T}_i(\bm{x})}_2 \cdot \norm{\underline{V}_d(\bm{x})}_2 \cdot \frac{1}{\underline{\beta}_0} \quad \text{ by (\underline{A2})}\\
	&\leq \sum_{i=1}^d \norm{\underline{\bm{v}}_i(\bm{x})^T \nabla f(\bm{x})}_2 \cdot \norm{\underline{\bm{v}}_i(\bm{x})^T \nabla^3 f(\bm{x})}_2 \quad \text{ since } \norm{\underline{V}_d(\bm{x})}_2=1\\
	&\leq \frac{d \norm{\nabla f(\bm{x})}_2 \cdot q^{\frac{3}{2}}\norm{\nabla^3 f(\bm{x})}_{\max}}{\underline{\beta}_0}\\
	&\leq \underline{\beta}_0-\underline{\beta}_1.
	\end{align*}
	It shows that all the singular values of $\left(\bm{I}_{q+1} -\bm{x}\bm{x}^T \right) \sum\limits_{i=1}^d \underline{T}_i(\bm{x}) \underline{V}_d(\bm{x}) \underline{\Lambda}_i(\bm{x})$ or simply $\sum\limits_{i=1}^d \underline{T}_i(\bm{x}) \underline{V}_d(\bm{x}) \underline{\Lambda}_i(\bm{x})$ are less than $\underline{\beta}_0-\underline{\beta}_1$. Moreover, under condition (\underline{A2}) again, all the $(q-d)$ singular values of 
	$$\left(\bm{I}_{q+1}-\bm{x}\bm{x}^T \right) \underline{V}_d(\bm{x}) \underline{\Lambda}_0(\bm{x}) = \underline{V}_d(\bm{x}) \underline{\Lambda}_0(\bm{x})$$ 
	are greater than $\underline{\beta}_1$. \\
	By Theorem 3.3.16 in \cite{horn_johnson_1991}, we know that all the singular values of $\underline{M}(\bm{x})$ and $\left(\bm{I}_{q+1}-\bm{x}\bm{x}^T \right) \underline{M}(\bm{x})$ are greater than 
	$$\sigma_i\left(\underline{V}_d(\bm{x})\underline{\Lambda}_0(\bm{x}) \right) - \sigma_1\left( \sum_{i=1}^d \underline{T}_i(\bm{x}) \underline{V}_d(\bm{x}) \underline{\Lambda}_i(\bm{x}) \right) \geq \underline{\beta}_0-(\underline{\beta}_0-\underline{\beta}_1)=\underline{\beta}_1>0,$$ 
	where $\sigma_i(A), i=1,...,q-d$ are singular values of a matrix $A \in \mathbb{R}^{(q+1)\times (q-d)}$ in their descending order. Therefore, the minimum eigenvalue of $\underline{M}(\bm{x})^T \underline{M}(\bm{x})$ satisfies 
	\begin{equation}
	\label{M_Dir_eigen_bound}
	\lambda_{\min}\left(\underline{M}(\bm{x})^T \underline{M}(\bm{x}) \right)=\lambda_{\min}\left(\underline{M}(\bm{x})^T (\bm{I}_{q+1}-\bm{x}\bm{x}^T) \underline{M}(\bm{x}) \right)\geq \underline{\beta}_1^2>0.
	\end{equation}
	
	Now, given $\underline{M}_E(\bm{x}) = \left[\underline{M}(\bm{x}), \bm{x}\right] \in \mathbb{R}^{(q+1)\times (q+1-d)}$ and $\bm{x}\in \Omega_q$, we know that 
		$$\underline{M}_E(\bm{x})^T \underline{M}_E(\bm{x}) = \begin{pmatrix}
		\underline{M}(\bm{x})^T \underline{M}(\bm{x}) & \underline{M}(\bm{x})^T \bm{x}\\
		\bm{x}^T \underline{M}(\bm{x}) & 1
		\end{pmatrix}.$$
		If we denote the orthonormal eigenvectors of $\underline{M}(\bm{x})^T \underline{M}(\bm{x})$ by $\bm{v}_{\underline{M},1}(\bm{x}),..., \bm{v}_{\underline{M},q-d}(\bm{x}) \in \mathbb{R}^{q-d}$, then 
		\[
		\begin{pmatrix}
		\bm{v}_{\underline{M},1}(\bm{x})\\
		0
		\end{pmatrix},...,
		\begin{pmatrix}
		\bm{v}_{\underline{M},q-d}(\bm{x})\\
		0
		\end{pmatrix},
		\begin{pmatrix}
		\bm{0}\\
		1
		\end{pmatrix} \in \mathbb{R}^{q+1-d}
		\]
		are the orthonormal eigenvectors of $\underline{M}_E(\bm{x})^T \underline{M}_E(\bm{x})$, whose eigenvalues are thus lower bounded by $\min\left\{\underline{\beta}_1,1 \right\}$ due to \eqref{M_Dir_eigen_bound}. Hence, $\rank(\underline{N}(\bm{x}))=\rank(\underline{M}_E(\bm{x})) = q+1-d$.\\
		By the implicit function theorem and the extra constraint $\underline{R}_d \subset \Omega_q$, $\underline{R}_d$ is a $d$-dimensional submanifold on $\Omega_q$. It also implies that $\underline{R}_d$ cannot have intersections, because otherwise the intersected points will violate the rank condition.\\
		Finally, we argue by contradiction that $\underline{R}_d$ has no endpoints. Assume, on the contrary, that $\underline{R}_d$ has an end point $\bm{x}_0$. Our preceding argument has shown that $\underline{M}(\bm{x})$, the derivative of $\underline{V}_d(\bm{x})^T \nabla f(\bm{x})$, is bounded. In addition, $\bm{x}_0\in \underline{R}_d$. However, this contradicts to the implicit function theorem indicating that $\underline{R}_d$ is a $d$-dimensional submanifold on $\Omega_q$, because at the end point $\bm{x}_0\in \underline{R}_d$, there exists no local coordinate chart for $\underline{R}_d$ defined on an open set in $\underline{R}_d$. The results follow.\\
	
	\noindent (e) By the proof of (d), we already know that all the $(q-d)$ nonzero singular values of $\underline{M}(\bm{x})$ and $\left(\bm{I}_D -\bm{x}\bm{x}^T \right)\underline{M}(\bm{x})$ are greater than $\underline{\beta}_1 >0$. Also, all the $(q+1-d)$ nonzero singular values of $\underline{M}_E(\bm{x})$ are greater than $\min\left\{\underline{\beta}_1,1 \right\}$. Thus, the results follow easily from the argument of (e) in Lemma~\ref{normal_reach_prop}.\\ 
	
	Finally, the proofs of properties (f), (g), and (h) are essentially the same as the corresponding claims in \cite{Asymp_ridge2015}. We thus omitted them. For (h), the reader should be aware that we have extended the directional density $f$ from $\Omega_q$ to $\mathbb{R}^{q+1}\setminus \{\bm{0}\}$. In addition, it is the columns of $\underline{M}_E(\bm{x})$ that span the normal space of $\underline{R}_d$ in the ambient space, whose nonzero singular values are lower bounded by $\min\left\{\underline{\beta}_1, 1 \right\}$. The proof of (h) can also be found in Theorem 3 of \cite{YC2020}.
\end{proof}

\section{Stability of Directional Density Ridge}
\label{App:Stability_Dir_Ridge}

\subsection{Subspace Constrained Gradient Flows}
\label{Sec:SCGF}

This subsection is modified from Section 4 in \cite{Non_ridge_est2014} for directional densities and their ridges on $\Omega_q$. A map $\varpi: \mathbb{R} \to \Omega_q$ is a subspace constrained gradient flow with the principal Riemannian gradient $\underline{G}_d$ if $\varpi(0)=\bm{x}\in \Omega_q$ and 
\begin{equation}
\label{SCGF}
\varpi'(t) = \underline{G}_d(\varpi(t)) =\underline{U}_d(\varpi(t)) \cdot \grad f(\varpi(t))= \underline{V}_d(\varpi(t)) \underline{V}_d(\varpi(t))^T \nabla f(\varpi(t)),
\end{equation}
where the last equality follows from \eqref{Riem_grad_new}. Given the definition of the directional density ridge $\underline{R}_d$ in \eqref{Dir_ridge_app}, it consists of the destinations of the subspace constrained gradient flow $\varpi$, \emph{i.e.}, $\bm{y}\in \underline{R}_d$ if $\lim\limits_{t\to \infty} \varpi(t)=\bm{y}$ for some $\varpi$ satisfying \eqref{SCGF}. It will be convenient to parametrize the subspace constrained gradient ascent path with $\varpi$ by arc length. Let $s\equiv s(t)$ be the arc length from $\varpi(t)$ to $\varpi(\infty)$:
$$s(t) = \int_t^{\infty} \norm{\varpi'(u)}_2 du.$$
Denote the inverse of $s(t)$ by $t \equiv t(s)$. Note that 
$$t'(s) = \frac{1}{s'(t)} = -\frac{1}{\norm{\varpi'(t(s))}_2}= -\frac{1}{\norm{\underline{G}_d\left(\varpi(t(s)) \right)}_2}.$$
With $\gamma(s) =\varpi(t(s))$, we have that
\begin{equation}
\label{arc_length_deri}
\gamma'(s) = -\frac{\underline{G}_d\left(\gamma(s) \right)}{\norm{\underline{G}_d\left(\gamma(s) \right)}_2},
\end{equation}
which is a reparametrization of \eqref{SCGF} by arc length. Note that $\gamma$ always lies on $\Omega_q$ because its velocity is within the tangent space $T_{\gamma(s)}$ for every $s\in [0,\infty)$. Lemma 2 in \cite{Non_ridge_est2014} justifies the uniqueness of $\gamma$ passing through any particular point $\bm{x}\in \left((\underline{R}_d\oplus \underline{\rho}) \setminus \underline{R}_d \right) \cap \Omega_q$ under conditions (\underline{A1-3}). The (reversed) subspace constrained gradient flow $\gamma$ can be lifted onto the directional function $f$, as we may define
\begin{equation}
\label{SCGF_Density}
\xi(s) = f(\varpi(\infty)) - f(\varpi(t(s))) = f(\gamma(0)) - f(\gamma(s)).
\end{equation}
Sometimes, we may add the subscript $\bm{x}$ to the curves $\varpi_{\bm x},\gamma_{\bm x},\xi_{\bm x}$ if we want to emphasize that $\varpi,\gamma,\xi$ start from or pass through the specific point $\bm{x}$.

To analyze the behavior of the subspace constrained gradient flow $\xi$ lifted on $f$, we need the derivative of the projection matrix $\underline{U}(\bm{x})\equiv \underline{U}_d(\bm{x})$ along the path $\gamma$. Recall that $\underline{U}(\bm{x}) \equiv \underline{U}_d(\mathcal{H} f(\bm{x})) = \underline{V}_d(\bm{x}) \underline{V}_d(\bm{x})^T$. The collection $\left\{\underline{U}(\bm{x}): \bm{x}\in \Omega_q \right\}$ defines a matrix field: there is a matrix $\underline{U}(\bm{x})$ attached to each point $\bm{x}$. As mentioned earlier, there is a unique path $\gamma$ and unique $s>0$ such that $\bm{x}=\gamma(s)$ for any $\bm{x}\in \left(\underline{R}_d\oplus \underline{\rho} \right) \setminus \underline{R}_d$. Define
\begin{equation}
\label{proj_mat_deriv}
\underline{\dot{U}}_{\gamma(s)} \equiv \lim_{\epsilon \to 0} \frac{\underline{U}(\mathcal{H} f(\gamma(s+\epsilon))) - \underline{U}(\mathcal{H} f(\gamma(s)))}{\epsilon} = \lim_{t\to 0} \frac{\underline{U}(\mathcal{H} f(\gamma(s)) + tE_s) - \underline{U}(\mathcal{H} f(\gamma(s)))}{t},
\end{equation}
where $E_s = \frac{d}{ds} \mathcal{H} f(\gamma(s)) = \bar{\nabla}_{\gamma'(s)} \mathcal{H} f(\gamma(s))$ with $\bar{\nabla}$ being the Riemannian connection on $\Omega_q$. Under conditions (\underline{A1-3}), $\xi$ has a quadratic-like behavior near the directional ridge $\underline{R}_d$, analogous to Lemma 3 in \cite{Non_ridge_est2014}.

\begin{lemma}
	\label{quad_behave_Dir}
	Assume that conditions (\underline{A1-3}) holds. For all $\bm{x} \in \left(\underline{R}_d \oplus \underline{\rho} \right) \cap \Omega_q$, we have the following properties:
	\begin{enumerate}[label=(\alph*)]
		\item $\xi(0)=0$, $\xi'(s) = \norm{\underline{G}_d(\gamma(s))}_2$, and $\xi'(0)=0$. Thus, $\xi(s)$ is non-decreasing in $s$.
		\item The second derivative of $\xi$ satisfies $\xi''(s) \geq \frac{\underline{\beta}_0}{2}$.
		\item $\xi(s) \geq \frac{\underline{\beta}_0}{4} \norm{\gamma(0) -\gamma(s)}_2^2$.
	\end{enumerate}
\end{lemma}

\begin{proof}[Proof of Lemma~\ref{quad_behave_Dir}]
	The proof is adopted from Lemma 3 in \cite{Non_ridge_est2014}.
	
	(a) The first property $\xi(0)=0$ is obvious from the definition \eqref{SCGF_Density}. Then,
	\begin{align*}
	\xi'(s) = -\grad f(\gamma(s))^T \gamma'(s) &= \frac{\nabla f(\gamma(s))^T \underline{G}_d(\gamma(s))}{\norm{\underline{G}_d(\gamma(s))}_2} \quad \text{ by \eqref{arc_length_deri}}\\
	&= \frac{\nabla f(\gamma(s))^T \underline{V}_d(\gamma(s)) \underline{V}_d(\gamma(s))^T \nabla f(\gamma(s))}{\norm{\underline{G}_d(\gamma(s))}_2}\\
	&= \norm{\underline{G}_d(\gamma(s))}_2,
	\end{align*}
	since $\underline{V}_d(\gamma(s))^T \underline{V}_d(\gamma(s)) = \bm{I}_{q-d}$ for all $\gamma(s) \in \Omega_q$. By the definition of $\underline{R}_d$ in \eqref{Dir_ridge_app}, $\underline{G}_d(\gamma(s))=0$ when $\gamma(s)\in \underline{R}_d$. Thus, $\xi'(0)=0$ and $\xi(s)$ is non-decreasing in $s$.\\
	
	\noindent (b) Note that 
	\begin{align*}
	\left(\xi'(s) \right)^2 = \norm{\underline{G}_d(\gamma(s))}_2^2 &=\nabla f(\gamma(s))^T \underline{V}_d(\gamma(s)) \underline{V}_d(\gamma(s))^T \nabla f(\gamma(s)) \\
	&= \left[ \grad f(\gamma(s)) \right]^T \underline{U}(\gamma(s)) \left[ \grad f(\gamma(s)) \right].
	\end{align*}
	Differentiating both sides of the equation, we have that
	\begin{align*}
	2\xi'(s) \xi''(s) &= 2\gamma'(s)^T \mathcal{H} f(\gamma(s)) \underline{U}(\gamma(s)) \left[\grad f(\gamma(s)) \right] \\
	&\quad + \left[ \grad f(\gamma(s)) \right]^T \underline{\dot{U}}(\gamma(s)) \left[ \grad f(\gamma(s)) \right].
	\end{align*}
	Since $\underline{U}(\gamma(s)) \cdot \underline{U}(\gamma(s)) = \underline{U}(\gamma(s))$ (idempotent), we have that $\underline{\dot{U}}(\gamma(s)) = \underline{U}(\gamma(s)) \underline{\dot{U}}(\gamma(s)) + \underline{\dot{U}}(\gamma(s)) \underline{U}(\gamma(s))$, and hence the second term on the right-hand side of the above equation becomes
	\begin{align*}
	&\left[ \grad f(\gamma(s)) \right]^T \underline{\dot{U}}_{\gamma(s)}(\gamma(s)) \left[ \grad f(\gamma(s)) \right]\\ 
	&= \nabla f(\gamma(s))^T \underline{U}(\gamma(s))\underline{\dot{U}}_{\gamma(s)}(\gamma(s)) \nabla f(\gamma(s)) + \nabla f(\gamma(s))^T \underline{\dot{U}}_{\gamma(s)}(\gamma(s)) \underline{U}(\gamma(s)) \nabla f(\gamma(s))\\
	&= 2\nabla f(\gamma(s))^T \underline{\dot{U}}(\gamma(s)) \underline{G}_d(\gamma(s))
	\end{align*}
	Thus,
	$$2\xi'(s)\xi''(s) = 2\gamma'(s)^T \mathcal{H} f(\gamma(s)) \underline{G}_d(\gamma(s)) + 2\nabla f(\gamma(s))^T \underline{\dot{U}}(\gamma(s)) \underline{G}_d(\gamma(s)).$$
	By (a) and \eqref{arc_length_deri}, we conclude that
	\begin{equation}
	\label{SCGF_Den_Expression}
	\xi''(s) = -\frac{\underline{G}_d(\gamma(s))^T \mathcal{H} f(\gamma(s)) \cdot \underline{G}_d(\gamma(s))}{\norm{\underline{G}_d(\gamma(s))}_2^2} + \frac{\nabla f(\gamma(s))^T \underline{\dot{U}}(\gamma(s)) \underline{G}_d(\gamma(s))}{\norm{\underline{G}_d(\gamma(s))}_2}.
	\end{equation}
	Now, we will bound the two terms in \eqref{SCGF_Den_Expression}, respectively. As for the first term $-\frac{\underline{G}_d(\gamma(s))^T \mathcal{H} f(\gamma(s)) \cdot \underline{G}_d(\gamma(s))}{\norm{\underline{G}_d(\gamma(s))}_2^2}$, we notice that $\underline{G}_d(\gamma(s))$ is in the column space of $\underline{V}_d(\gamma(s))$. Hence,
	$$\underline{G}_d(\gamma(s))^T \mathcal{H} f(\gamma(s)) \cdot \underline{G}_d(\gamma(s)) = \underline{G}_d(\gamma(s))^T \left[\underline{V}_d(\gamma(s)) \Lambda_{\underline{R}_d}(\gamma(s)) \underline{V}_d(\gamma(s))^T \right] \underline{G}_d(\gamma(s)),$$
	where $\Lambda_{\underline{R}_d}(\gamma(s)) = \Diag\left[\underline{\lambda}_{d+1}(\gamma(s)),...,\underline{\lambda}_q(\gamma(s)) \right]$. Therefore, from condition (\underline{A2}),
	\begin{align*}
	\frac{\underline{G}_d(\gamma(s))^T \mathcal{H} f(\gamma(s)) \cdot \underline{G}_d(\gamma(s))}{\norm{\underline{G}_d(\gamma(s))}_2^2} &= \frac{\underline{G}_d(\gamma(s))^T \left[\underline{V}_d(\gamma(s)) \Lambda_{\underline{R}_d}(\gamma(s)) \underline{V}_d(\gamma(s))^T \right] \underline{G}_d(\gamma(s))}{\norm{\underline{G}_d(\gamma(s))}_2^2}\\
	&\leq \lambda_{\max}\left[\underline{V}_d(\gamma(s)) \Lambda_{\underline{R}_d}(\gamma(s)) \underline{V}_d(\gamma(s))^T \right] \leq -\underline{\beta}_0
	\end{align*}
	and consequently,
	$$-\frac{\underline{G}_d(\gamma(s))^T \mathcal{H} f(\gamma(s)) \cdot \underline{G}_d(\gamma(s))}{\norm{\underline{G}_d(\gamma(s))}_2^2} \geq \underline{\beta}_0.$$
	As for the second term $\frac{\nabla f(\gamma(s))^T \underline{\dot{U}}_{\gamma(s)} \underline{G}_d(\gamma(s))}{\norm{\underline{G}_d(\gamma(s))}_2}$, we notice that $\underline{U}(\gamma(s)) + \underline{U}^{\perp}(\gamma(s)) = \bm{I}_{q+1}$, where $\underline{U}^{\perp} \equiv \underline{U}_d^{\perp}$, and $\underline{U}(\gamma(s)) \cdot \underline{G}_d(\gamma(s)) =\underline{G}_d(\gamma(s))$. Then,
	\begin{align*}
	&\nabla f(\gamma(s))^T \underline{\dot{U}}(\gamma(s)) \underline{G}_d(\gamma(s))\\ 
	&= \nabla f(\gamma(s))^T \underline{U}(\gamma(s)) \underline{\dot{U}}(\gamma(s)) \underline{G}_d(\gamma(s)) + \nabla f(\gamma(s))^T \underline{U}^{\perp}(\gamma(s)) \underline{\dot{U}}(\gamma(s)) G_d(\gamma(s))\\
	&= \nabla f(\gamma(s))^T \underline{U}(\gamma(s)) \underline{\dot{U}}(\gamma(s)) \underline{U}(\gamma(s)) \cdot \underline{G}_d(\gamma(s)) \\
	&\quad + \nabla f(\gamma(s))^T \underline{U}^{\perp}(\gamma(s)) \underline{\dot{U}}(\gamma(s)) \underline{U}(\gamma(s)) \cdot \underline{G}_d(\gamma(s)).
	\end{align*}
	However, $\left|\nabla f(\gamma(s))^T \underline{U}(\gamma(s)) \underline{\dot{U}}(\gamma(s)) \underline{U}(\gamma(s)) \cdot \underline{G}_d(\gamma(s)) \right|=0$. To see this, note that $\underline{U}(\gamma(s)) \cdot \underline{U}(\gamma(s)) = \underline{U}(\gamma(s))$ and it implies that
	\begin{align*}
	& \quad \underline{U}(\gamma(s)) \underline{\dot{U}}(\gamma(s)) + \underline{\dot{U}}(\gamma(s))\underline{U}(\gamma(s)) = \underline{\dot{U}}(\gamma(s))\\
	&\implies \underline{U}(\gamma(s)) \underline{\dot{U}}(\gamma(s)) \underline{U}(\gamma(s)) + \underline{\dot{U}}(\gamma(s)) \underline{U}(\gamma(s)) = \underline{\dot{U}}(\gamma(s)) \underline{U}(\gamma(s)),
	\end{align*}
	showing that $\underline{U}(\gamma(s)) \underline{\dot{U}}_{\gamma(s)} \underline{U}(\gamma(s))=\bm{0}$. To bound $\nabla f(\gamma(s))^T \underline{U^{\perp}}(\gamma(s)) \underline{\dot{U}}(\gamma(s)) \underline{U}(\gamma(s)) \cdot \underline{G}_d(\gamma(s))$, we proceed as follows. As before, we let $E_s = \frac{d}{ds} \mathcal{H} f(\gamma(s)) = \bar{\nabla} \mathcal{H} f(\gamma(s)) \cdot \gamma'(s)$. Then, by the Davis-Kahan theorem (Lemma~\ref{Davis_K} here),
	\begin{align*}
	&\left|\nabla f(\gamma(s))^T \underline{U}^{\perp}(\gamma(s)) \underline{\dot{U}}(\gamma(s)) \underline{U}(\gamma(s)) \cdot \underline{G}_d(\gamma(s)) \right| \\
	&= \lim_{t\to 0} \frac{\Big|\nabla f(\gamma(s))^T \underline{U}^{\perp}(\gamma(s))\left[\underline{U}(\mathcal{H} f(\gamma(s)) + tE_s) - \underline{U}(\mathcal{H} f(\gamma(s))) \Big] \underline{U}(\gamma(s)) \cdot \underline{G}_d(\gamma(s)) \right|}{t}\\
	&\leq \norm{\underline{U}^{\perp}(\gamma(s)) \nabla f(\gamma(s))}_2 \cdot \lim_{t\to 0} \frac{\norm{\underline{U}(\mathcal{H} f(\gamma(s)) + tE_s) - \underline{U}(\mathcal{H} f(\gamma(s)))}_2}{t} \cdot \norm{\underline{G}_d(\gamma(s))}_2\\
	&\leq \frac{\sqrt{2}\norm{\underline{U}^{\perp}(\gamma(s)) \nabla f(\gamma(s))}_2 \cdot \norm{E_s}_F \cdot \norm{\underline{G}_d(\gamma(s))}_2}{\underline{\beta}_0}.
	\end{align*}
	Note that $\norm{E_s}_F\leq \norm{\bar{\nabla} \mathcal{H} f(\gamma(s))}_{\max} \cdot \norm{\gamma'(s)}_2 \leq q^{\frac{3}{2}} \norm{\nabla^3 f(\gamma(s))}_{\max}$, because $\norm{\gamma'(s)}_2=1$. Thus, from condition (\underline{A3}),
	$$\frac{\left|\nabla f(\gamma(s))^T \underline{\dot{U}}(\gamma(s)) \underline{G}_d(\gamma(s)) \right|}{\norm{\underline{G}_d(\gamma(s))}_2} \leq \frac{\sqrt{2}q^{\frac{3}{2}} \norm{\nabla^3 f(\gamma(s))}_{\max} \norm{\underline{U}^{\perp}(\gamma(s)) \nabla f(\gamma(s))}_2}{\underline{\beta}_0} \leq \frac{\underline{\beta}_0}{2}.$$
	Therefore, $\xi''(s) \geq \underline{\beta}_0-\frac{\underline{\beta}_0}{2} = \frac{\underline{\beta}_0}{2}$.\\
	
	\noindent(c) For some $0\leq \tilde{s} \leq s$,
	$$\xi(s) = \xi(0) + s \xi'(0) + \frac{s^2}{2} \xi''(\tilde{s}) = \frac{s^2}{2} \xi''(\tilde{s}) \geq \frac{\underline{\beta}_0 s^2}{4}$$
	by (a) and (b). As $\gamma$ is parametrized by arc length, we conclude that
	$$\xi(s)-\xi(0) \geq \frac{\underline{\beta}_0}{4} s^2 \geq \frac{\underline{\beta}_0}{4} \norm{\gamma(0)-\gamma(s)}_2^2.$$
	The result follows.
\end{proof}

The statement (c) in Lemma~\ref{quad_behave_Dir} is known as the quadratic growth condition in the optimization literature \citep{anitescu2000degenerate,drusvyatskiy2018error}. Under conditions (\underline{A1-3}), such a quadratic growth of the subspace constrained gradient flow $\xi$ lifted onto the directional density $f$ enables us to quantify the stability of directional ridges under small perturbations on the directional density and develop the linear convergence of the (directional) SCGA algorithms on $\Omega_q$.

\subsection{Proof of Theorem~\ref{Dir_ridge_stability}}

We now show that if two directional densities $f$ and $\tilde{f}$ are close, their corresponding ridges $\underline{R}_d$ and $\underline{\tilde{R}}_d$ are also close. We will use, for instance, $\underline{\tilde{G}}_d$ and $\underline{\tilde{U}}_d$, to refer to the principal (Riemannian) gradient and projection matrix with its columns as the eigenvectors corresponding to the smallest $q-d$ eigenvalues of the (Riemannian) Hessian $\mathcal{H} \tilde{f}$ with the tangent space of $\Omega_q$ defined by $\tilde{f}$. 

\begin{customthm}{4.1}
	Suppose that conditions (\underline{A1-3}) hold for the directional density $f$ and that condition (\underline{A1}) holds for $\tilde{f}$. When $\norm{f-\tilde{f}}_{\infty,3}^*$ is sufficiently small,
	\begin{enumerate}[label=(\alph*)]
		\item conditions (\underline{A2-3}) holds for $\tilde{f}$.
		\item $\Haus(\underline{R}_d, \underline{\tilde{R}}_d) = O\left(\norm{f-\tilde{f}}_{\infty,2}^* \right)$.
		\item $\mathtt{reach}(\underline{\tilde{R}}_d) \geq \min\left\{\underline{\rho}/2, \frac{\min\left\{\underline{\beta}_1, 1\right\}^2}{\underline{A}_2 \left(\norm{f}_{\infty}^{(3)} + \norm{f}_{\infty}^{(4)} \right)} \right\} + O\left(\norm{f-\tilde{f}}_{\infty,3}^* \right)$ for a constant $\underline{A}_2>0$.
	\end{enumerate}
\end{customthm}

\begin{proof}[Proof of Theorem~\ref{Dir_ridge_stability}] 
	Our arguments are modified from the proof of Theorem 4 in \cite{Non_ridge_est2014} as well as Proposition 4 and Theorem 5 in \cite{YC2020}.\\
	
	\noindent(a) We write the spectral decompositions of $\mathcal{H} f$ and $\mathcal{H} \tilde{f}$ as
	$$\mathcal{H} f = V \Lambda V^T \quad \text{ and } \quad \mathcal{H} \tilde{f} = \tilde{V} \tilde{\Lambda} \tilde{V}.$$
	By Weyl's Theorem (Theorem 4.3.1 in \cite{HJ2012}), we know that
	$$|\lambda_j - \tilde{\lambda}_j| \leq \norm{\mathcal{H}f -\mathcal{H} \tilde{f}}_2 \leq q \norm{f-\tilde{f}}_{\infty,2}^*,$$
	where we recall that there are at most $q$ nonzero eigenvalues of the Riemannian Hessian $\mathcal{H}f(\bm{x})$ on $\Omega_q$.
	Thus, $\tilde{f}$ satisfies condition (\underline{A2}). Moreover, since condition (\underline{A3}) depends only on the first and third order derivative of $f$, they hold for $\tilde{f}$ when $\norm{f-\tilde{f}}_{\infty,3}^*$ is small enough.\\
	
	\noindent(b) We present two methods based on two different flows to prove this statement and comment their pros and cons in Remark~\ref{remark_flows}.\\
	\textbf{Method A}: By the Davis-Kahan theorem (Lemma~\ref{Davis_K} and \eqref{Davis_Kahan_ineq}), 
	\begin{align*}
	\norm{\underline{U}_d(\bm{x})- \underline{\tilde{U}}_d(\bm{x})}_2 &= \norm{\underline{V}_d(\bm{x}) \underline{V}_d(\bm{x})^T - \underline{\tilde{V}}_d(\bm{x}) \underline{\tilde{V}}_d(\bm{x})^T}_2 \\
	&\leq \frac{\sqrt{2}\norm{\mathcal{H} f(\bm{x}) - \mathcal{H} \tilde{f}(\bm{x})}_F}{\underline{\beta}_0} \\
	&\leq \frac{\sqrt{2}q\norm{f-\tilde{f}}_{\infty,2}^*}{\underline{\beta}_0}
	\end{align*}
	for any $\bm{x}\in \Omega_q$. Then, given that $\norm{\underline{\tilde{V}}_d(\bm{x}) \underline{\tilde{V}}_d(\bm{x})^T}_2=1$,
	\begin{align*}
	&\norm{\underline{G}_d(\bm{x}) - \underline{\tilde{G}}_d(\bm{x})}_2 \\
	&= \norm{\underline{V}_d(\bm{x}) \underline{V}_d(\bm{x})^T \nabla f(\bm{x}) - \underline{\tilde{V}}_d(\bm{x}) \underline{\tilde{V}}_d(\bm{x})^T \nabla \tilde{f}(\bm{x})}_2\\
	&\leq \norm{\left[\underline{V}_d(\bm{x}) \underline{V}_d(\bm{x})^T - \underline{\tilde{V}}_d(\bm{x}) \underline{\tilde{V}}_d(\bm{x})^T \right] \nabla f(\bm{x})}_2 + \norm{\underline{\tilde{V}}_d(\bm{x}) \underline{\tilde{V}}_d(\bm{x})^T \left[\nabla f(\bm{x}) -\nabla \tilde{f}(\bm{x}) \right]}_2\\
	&\leq \frac{\sqrt{2}q\norm{f-\tilde{f}}_{\infty,2}^*}{\underline{\beta}_0} \cdot \norm{\nabla f(\bm{x})}_2 + \sqrt{q} \norm{f-\tilde{f}}_{\infty,1}^*.
	\end{align*}
	Therefore, by the differentiability of $f$ from (\underline{A1}) and the compactness of $\Omega_q$, we obtain from the above calculations that
	$$\sup_{\bm{x}\in \Omega_q} \norm{\underline{G}_d(\bm{x}) - \underline{\tilde{G}}_d(\bm{x})}_2 \leq C_1 \norm{f-\tilde{f}}_{\infty,2}^*$$
	for some constant $C_1 >0$ that only depends on the dimension $q$.\\ Now, let $\tilde{\bm{x}} \in \underline{\tilde{R}}_d$. Then, $\norm{\underline{\tilde{G}}_d(\tilde{\bm{x}})}_2 =0$, and $\norm{\underline{G}_d(\tilde{\bm{x}})}_2 \leq C_1 \norm{f-\tilde{f}}_{\infty,2}^*$. Let $\gamma$ be the subspace constrained gradient ascent flow through $\tilde{\bm{x}}$ as defined in Section~\ref{Sec:SCGF} so that $\gamma(s)=\tilde{\bm{x}}$ for some $s$. Note that $\gamma(0) \in \underline{R}_d$. From property (a) of Lemma~\ref{quad_behave_Dir}, we have that $\xi'(s)=\norm{\underline{G}_d(\tilde{\bm{x}})}_2$. Moreover, by Taylor's theorem,
	$$C_1 \norm{f -\tilde{f}}_{\infty,2}^* \geq \norm{\underline{G}_d(\tilde{\bm{x}})}_2 = \xi'(s) = \xi'(0) + s\xi''(u)$$
	for some $u$ between $0$ and $s$. Since $\xi'(0)=0$, from property (b) of Lemma~\ref{quad_behave_Dir}, 
	$$C_1 \norm{f -\tilde{f}}_{\infty,2}^* \geq s\xi''(u) \geq \frac{s\underline{\beta}_0}{2},$$
	and consequently, $d_g(\gamma(0),\tilde{\bm{x}}) \leq s\leq \frac{2C_1}{\underline{\beta}_0} \norm{f -\tilde{f}}_{\infty,2}^*$, where $d_g(\bm{x},\bm{y})$ denotes the geodesic distance between $\bm{x}$ and $\bm{y}$ on $\Omega_q$. Therefore,
	$$d_E(\tilde{\bm{x}}, \underline{R}_d) \leq \norm{\gamma(0)-\tilde{\bm{x}}}_2 \leq d_g(\gamma(0),\tilde{\bm{x}}) \leq \frac{2C_1}{\underline{\beta}_0} \norm{f -\tilde{f}}_{\infty,2}^*.$$
	Now let $\bm{x}\in \underline{R}_d$. The same argument shows that $d_E(\bm{x}, \underline{\tilde{R}}_d) \leq \frac{2C_1}{\underline{\beta}_0} \norm{f -\tilde{f}}_{\infty,2}^*$ for some constant $C_2>0$ because conditions (\underline{A1-3}) hold for $\tilde{f}$.\\
	As a result, $\Haus(\underline{R}_d,\underline{\tilde{R}}_d) \leq \frac{2C_1}{\underline{\beta}_0} \norm{f -\tilde{f}}_{\infty,2}^* =O\left(\norm{f -\tilde{f}}_{\infty,2}^* \right)$.
	
	\textbf{Method B}: Since we are only required to bound the maximum Euclidean distance between $\underline{R}_d$ and $\underline{\tilde{R}}_d$, \emph{i.e.}, $\Haus(\underline{R}_d, \underline{\tilde{R}}_d)$, we may view $\underline{R}_d$ and $\underline{\tilde{R}}_d$ as solution manifolds in $\mathbb{R}^{q+1}$ and tentatively ignore the manifold constraint $\underline{R}_d, \underline{\tilde{R}}_d \subset \Omega_q$. Define $h(\bm{x})=\norm{\underline{V}_d(\bm{x})^T \nabla f(\bm{x})}_2= \sqrt{\nabla f(\bm{x})^T \underline{V}_d(\bm{x}) \underline{V}_d(\bm{x})^T \nabla f(\bm{x})}$. Given that $\underline{M}(\bm{x}) = \nabla\left[\underline{V}_d(\bm{x})^T\nabla f(\bm{x}) \right]^T$, the gradient of $h(\bm{x})$,
	\begin{equation}
	\label{grad_of_h}
	\nabla h(\bm{x}) = \frac{\nabla\left[\underline{V}_d(\bm{x})^T\nabla f(\bm{x}) \right]^T \underline{V}_d(\bm{x})^T \nabla f(\bm{x})}{\norm{\underline{V}_d(\bm{x})^T \nabla f(\bm{x})}_2} = \frac{\underline{M}(\bm{x}) \underline{V}_d(\bm{x})^T \nabla f(\bm{x})}{\norm{\underline{V}_d(\bm{x})^T \nabla f(\bm{x})}_2},
	\end{equation}
	is a vector in $\mathbb{R}^{q+1}$. Let $\bm{z}\in \underline{\tilde{R}}_d$. We define a flow $\phi_{\bm{z}}: \mathbb{R} \to \mathbb{R}^{q+1}$ such that
	$$\phi_{\bm{z}}(0)=\bm{z}, \quad \phi_{\bm{z}}'(t) = -\nabla h(\phi_{\bm{z}}(t)).$$
	It can be argued by Theorem 7 in \cite{YC2020} that $\phi_{\bm{z}}(\infty) \in \underline{R}_d$ when $\bm{z}\in \underline{R}_d\oplus \delta_0$ for some small $\delta_0 >0$. In addition, we can always choose $\norm{f-\tilde{f}}_{\infty,3}^*$ to be small enough so that $\underline{\tilde{R}}_d \subset \underline{R}_d \oplus \delta_0$. By Theorem 3.39 in \cite{Irwin2001}, $\phi_{\bm{z}}(t)$ is uniquely defined because the gradient $\nabla h(\bm{z})$ is well-defined for all $\bm{z}\notin \underline{R}_d$. We can also reparametrize $\phi_{\bm{z}}(t)$ by arc length as:
	$$\gamma_{\bm{z}}(0)=\bm{z}, \quad \gamma_{\bm{z}}'(s) = -\frac{\nabla h(\gamma_{\bm{z}}(s))}{\norm{\nabla h(\gamma_{\bm{z}}(s))}_2}.$$
	Let $\mathcal{S}_{\bm{z}}=\inf\left\{s>0: \gamma_{\bm{z}}(s)\in \underline{R}_d \right\}$ be the terminal time/arc-length point and $\gamma_{\bm{z}}(\mathcal{S}_{\bm{z}}) \in \underline{R}_d$ be the destination of $\gamma_{\bm{z}}$ on $\underline{R}_d$. The above argument also demonstrates that the flows $\phi_{\bm{z}}$ or $\gamma_{\bm{z}}$ converge to the manifold $\underline{R}_d$ from the normal direction of $\underline{R}_d$, because we can write 
	\begin{align*}
	&\gamma_{\bm{z}}'(\mathcal{S}_{\bm{z}})=-\frac{\nabla h(\gamma_{\bm{z}}(\mathcal{S}_{\bm{z}}))}{\norm{\nabla h(\gamma_{\bm{z}}(\mathcal{S}_{\bm{z}}))}_2} = \sum\limits_{k=d+1}^q \bm{a}_k \cdot \underline{\bm{m}}_k(\gamma_{\bm{z}}(\mathcal{S}_{\bm{z}})) \\
	&\quad \text{ with } \quad \bm{a}_k = -\frac{\bm{e}_{k-d}^T \underline{V}_d(\gamma_{\bm{z}}(\mathcal{S}_{\bm{z}})) \nabla f(\gamma_{\bm{z}}(\mathcal{S}_{\bm{z}}))}{\norm{\underline{M}(\gamma_{\bm{z}}(\mathcal{S}_{\bm{z}})) \underline{V}_d(\gamma_{\bm{z}}(\mathcal{S}_{\bm{z}})) \nabla f(\gamma_{\bm{z}}(\mathcal{S}_{\bm{z}}))}_2}
	\end{align*}
	and the column space of $\underline{M}(\gamma_{\bm{z}}(\mathcal{S}_{\bm{z}}))$ spans the normal space of $\underline{R}_d$ at $\gamma_{\bm{z}}(\mathcal{S}_{\bm{z}}) \in \underline{R}_d$.\\
	The goal now is to bound $\mathcal{S}_{\bm{z}}$ because its length must be greater or equal to $\norm{\bm{z}-\pi_{\underline{R}_d}(\bm{z})}_2$. We then define $\vartheta_{\bm{z}}(s) = h(\gamma_{\bm{z}}(s)) -h(\gamma_{\bm{z}}(\mathcal{S}_{\bm{z}})) = h(\gamma_{\bm{z}}(s))$. Differentiating $\vartheta_{\bm{z}}(s)$ with respect to $s$ leads to
	\begin{align}
	\label{zeta_bound}
	\begin{split}
	\vartheta_{\bm{z}}'(s) = \frac{d}{ds} h(\gamma_{\bm{z}}(s)) &= \left[\nabla h(\gamma_{\bm{z}}(s)) \right]^T \gamma_{\bm{z}}'(s)\\
	&= -\norm{\nabla h(\gamma_{\bm{z}}(s))}_2\\
	&= - \frac{\norm{\underline{M}(\gamma_{\bm{z}}(s)) \underline{V}_d(\gamma_{\bm{z}}(s))^T \nabla f(\gamma_{\bm{z}}(s))}_2}{\norm{\underline{V}_d(\gamma_{\bm{z}}(s))^T \nabla f(\gamma_{\bm{z}}(s))}_2}\\
	&\leq -\lambda_{\min}\left(\underline{M}(\gamma_{\bm{z}}(s))^T \underline{M}(\gamma_{\bm{z}}(s)) \right) \leq -\tilde{\beta}_1
	\end{split}
	\end{align}
	by (d) in Lemma~\ref{Dir_norm_reach_prop}. (Note that $0<\tilde{\beta}_1\leq \underline{\beta}_1$ because $\lambda_{\min}(\underline{M}(\bm{x})^T \underline{M}(\bm{x})) \geq \underline{\beta}_1$ and by the continuity of $\lambda_{\min}(\underline{M}(\bm{y})^T \underline{M}(\bm{y}))$, we can always choose $\delta_0>0$ such that $\lambda_{\min}(\underline{M}(\bm{y})^T \underline{M}(\bm{y}))\geq \tilde{\beta}_1$ for all $\bm{y}\in \underline{R}_d \oplus \delta_0$.)
	As $\bm{z}\in \underline{\tilde{R}}_d$, by the proof of \textbf{Method A}, we know that
	\begin{align*}
	C_1\norm{f-\tilde{f}}_{\infty,2}^* \geq \norm{\underline{G}_d(\bm{z})} &= \norm{\underline{V}_d(\bm{z})^T \nabla f(\bm{z})}_2\\
	&= h(\gamma_{\bm{z}}(0)) - h(\gamma_{\bm{z}}(\mathcal{S}_{\bm{z}})) \quad \text{ since }  h(\gamma_{\bm{z}}(\mathcal{S}_{\bm{z}}))=0\\
	&= \vartheta_{\bm{z}}(0) -\vartheta_{\bm{z}}(\mathcal{S}_{\bm{z}}) \quad \text{ since } \vartheta_{\bm{z}}(\mathcal{S}_{\bm{z}})=0 \text{ and }  \vartheta_{\bm{z}}(0) = h(\gamma_{\bm{z}}(0))\\
	&= -\mathcal{S}_{\bm{z}} \vartheta_{\bm{z}}'(\mathcal{S}_{\bm{z}}^*) \quad \text{ by the mean value theorem}\\
	&\geq \mathcal{S}_{\bm{z}} \tilde{\beta}_1 \quad \text{ by \eqref{zeta_bound}},
	\end{align*}
	where $\mathcal{S}_{\bm{z}}^*$ is some value between $0$ and $\mathcal{S}_{\bm{z}}$. Hence, $\mathcal{S}_{\bm{z}} \leq \frac{C_1}{\tilde{\beta}_1} \norm{f-\tilde{f}}_{\infty,2}^*= O\left(\norm{f-\tilde{f}}_{\infty,2}^* \right)$, which is independent of $\bm{z}\in \underline{\tilde{R}}_d$. This implies that 
	$$\sup_{\bm{z}\in \underline{\tilde{R}}_d} d_E(\bm{z}, \underline{R}_d) \leq \frac{C_1}{\tilde{\beta}_1} \norm{f-\tilde{f}}_{\infty,2}^*= O\left(\norm{f-\tilde{f}}_{\infty,2}^* \right).$$
	We can exchange the role of $\underline{R}_d$ and $\underline{\tilde{R}}_d$ and apply the same argument to show that
	$$\sup_{\bm{x}\in \underline{R}_d} d_E(\bm{x}, \underline{\tilde{R}}_d) = O\left(\norm{f-\tilde{f}}_{\infty,2}^* \right).$$
	In total, this leads to the conclusion that $\Haus(\underline{R}_d,\underline{\tilde{R}}_d) = O\left(\norm{f-\tilde{f}}_{\infty,2}^* \right)$.\\
	
	\noindent(c) By (h) in Lemma~\ref{Dir_norm_reach_prop}, the reach of $\underline{R}_d$ has a lower bound, $\min\left\{\underline{\rho}/2, \frac{\min\left\{\underline{\beta}_1,1\right\}^2}{\underline{A}_2\left(\norm{f}_{\infty}^{(3)} + \norm{f}_{\infty}^{(4)} \right)} \right\}$. Note that $\underline{\rho}$ and $\underline{\beta}_1$ depend on the first three order derivatives of $f$. Thus, the lower bound for the reach of $\underline{\tilde{R}}_d$ will be identical to the one for $\underline{R}_d$ with an error rate $O\left(\norm{f-\tilde{f}}_{\infty,3}^* \right)$.
\end{proof}

Note that for the stability of directional ridges, one can relax the condition (\underline{A1}) by requiring $f$ to be $\beta$-H{\"o}lder with $\beta \geq 3$.

\begin{remark}
	\label{remark_flows}
	We apply two different methods to establish the stability theorem of directional density ridges. {\bf Method A} utilizes the subspace constrained gradient flow constructed in Section~\ref{Sec:SCGF} and its quadratic behavior (Lemma~\ref{quad_behave_Dir}), while {\bf Method B} defines a normal flow to the ridge $\underline{R}_d$ induced by the column space of $\underline{M}(\bm{x})$. Each of these two flows has its pros and cons. The subspace constrained gradient flow aligns more coherently with our directional SCMS algorithm (Algorithm~\ref{Algo:Dir_SCMS}) to identify the (estimated) directional ridge from data, because it relies only on the first and second order derivatives of the (estimated) density $f$. Nevertheless, the subspace constrained gradient flow does not necessarily converge to $\underline{R}_d$ in the optimal direction, that is, the normal direction to $\underline{R}_d$. This can be seen from the explicit formula \eqref{Dir_normal_rows_all} of $\underline{M}(\bm{x})$, which spans the normal space of $\underline{R}_d$. The normal flow 
	$$\phi_{\bm{z}}(0)=\bm{z}, \quad \phi_{\bm{z}}'(t) = -\frac{\underline{M}(\phi_{\bm{z}}(t)) \underline{V}_d(\phi_{\bm{z}}(t))^T \nabla f(\phi_{\bm{z}}(t))}{\norm{\underline{V}_d(\phi_{\bm{z}}(t))^T \nabla f(\phi_{\bm{z}}(t))}_2}$$ 
	defined in {\bf Method B}, however, converges to $\underline{R}_d$ in its normal direction by construction. In general, the normal flow tends to the ridge $\underline{R}_d$ faster than the subspace constrained gradient flow, but it may be complicated to compute in any practical ridge-finding task due to its involvement with third order derivatives of the (estimated) density $f$. Recently, \cite{qiao2021algorithms} presented explicit formulae for finding density ridges via such a normal flow and its discrete gradient descent approximation. Additionally, they defined a smoothed version of the ridgeness function that also circumvents the computations of third order derivatives of $f$. 
\end{remark}

\section{Proofs of Proposition~\ref{Dir_SCMS_prop}, Proposition~\ref{SCGA_Dir_conv}, and Theorem~\ref{LC_SCGA_Dir}}
\label{App:Proofs_Dir_SCMS}

\begin{customprop}{4.3}
	Assume that the directional kernel $L$ is non-increasing, twice continuously differentiable, and convex with $L(0) <\infty$. Given the directional KDE $\hat{f}_h(\bm{x})=\frac{c_{L,q}(h)}{n}\sum\limits_{i=1}^n L\left(\frac{1-\bm{x}^T\bm{X}_i}{h^2} \right)$ and the directional SCMS sequence $\big\{\underline{\hat{\bm{x}}}^{(t)} \big\}_{t=0}^{\infty} \subset \Omega_q$ defined by \eqref{Dir_SCMS_update_new} or \eqref{Dir_SCMS_fixed_point_new}, the following properties hold:
	\begin{enumerate}[label=(\alph*)]
		\item The estimated density sequence $\left\{\hat{f}_h(\underline{\hat{\bm{x}}}^{(t)}) \right\}_{t=0}^{\infty}$ is non-decreasing and thus converges.
		\item $\lim\limits_{t\to \infty} \norm{\hat{\underline{V}}_d(\underline{\hat{\bm{x}}}^{(t)})^T \nabla \hat{f}_h(\underline{\hat{\bm{x}}}^{(t)})}_2=0$.
		\item If the kernel $L$ is also strictly decreasing on $[0,\infty)$, then $\lim\limits_{t\to \infty} \norm{\underline{\hat{\bm{x}}}^{(t+1)} - \underline{\hat{\bm{x}}}^{(t)}}_2=0$.
	\end{enumerate}
\end{customprop}

\begin{proof}[Proof of Proposition~\ref{Dir_SCMS_prop}]
	(a) The sequence $\left\{\hat{f}_h(\underline{\hat{\bm{x}}}^{(t)}) \right\}_{t=0}^{\infty}$ is bounded if the kernel $L$ is non-increasing with $L(0)<\infty$. Hence, it suffices to show that it is non-decreasing. The convexity and differentiability of kernel $L$ imply that 
	\begin{equation}
	\label{convex_kernel}
	L(x_2) - L(x_1) \geq L'(x_1) \cdot (x_2-x_1)
	\end{equation}
	for all $x_1,x_2 \in [0,\infty)$. Then, with $\nabla \hat{f}_h(\bm{x}) = -\frac{c_{L,q}(h)}{nh^2} \sum\limits_{i=1}^n \bm{X}_i L'\left(\frac{1-\bm{x}^T\bm{X}_i}{h^2} \right)$ and the iterative formula \eqref{Dir_SCMS_fixed_point_new} in the main paper, we derive that
	\begin{align*}
	&\hat{f}_h(\underline{\hat{\bm{x}}}^{(t+1)}) -\hat{f}_h(\underline{\hat{\bm{x}}}^{(t)}) \\
	&= \frac{c_{L,q}(h)}{n} \sum_{i=1}^n \left[L\left(\frac{1-\bm{X}_i^T\underline{\hat{\bm{x}}}^{(t+1)}}{h^2} \right) - L\left(\frac{1-\bm{X}_i^T\underline{\hat{\bm{x}}}^{(t)}}{h^2} \right) \right]\\
	&\geq \frac{c_{L,q}(h)}{nh^2} \sum_{i=1}^n L'\left(\frac{1-\bm{X}_i^T\underline{\hat{\bm{x}}}^{(t)}}{h^2} \right) \bm{X}_i^T \left(\underline{\hat{\bm{x}}}^{(t)} -\underline{\hat{\bm{x}}}^{(t+1)} \right)\\
	&= \nabla \hat{f}_h(\underline{\hat{\bm{x}}}^{(t)})^T \left(\underline{\hat{\bm{x}}}^{(t+1)} -\underline{\hat{\bm{x}}}^{(t)} \right)\\
	&= \nabla \hat{f}_h(\underline{\hat{\bm{x}}}^{(t)})^T \left[\frac{\hat{\underline{V}}_d(\underline{\hat{\bm{x}}}^{(t)}) \hat{\underline{V}}_d(\underline{\hat{\bm{x}}}^{(t)})^T \nabla \hat{f}_h(\underline{\hat{\bm{x}}}^{(t)}) + \norm{\nabla \hat{f}_h(\underline{\hat{\bm{x}}}^{(t)})}_2 \cdot \underline{\hat{\bm{x}}}^{(t)}}{\norm{\hat{\underline{V}}_d(\underline{\hat{\bm{x}}}^{(t)}) \hat{\underline{V}}_d(\underline{\hat{\bm{x}}}^{(t)})^T \nabla \hat{f}_h(\underline{\hat{\bm{x}}}^{(t)}) + \norm{\nabla \hat{f}_h(\underline{\hat{\bm{x}}}^{(t)})}_2 \cdot \underline{\hat{\bm{x}}}^{(t)}}_2} -\underline{\hat{\bm{x}}}^{(t)}\right]\\
	&\stackrel{\text{(i)}}{=} \frac{\norm{\hat{\underline{V}}_d(\underline{\hat{\bm{x}}}^{(t)})^T \nabla \hat{f}_h(\underline{\hat{\bm{x}}}^{(t)})}_2^2}{\sqrt{\norm{\hat{\underline{V}}_d(\underline{\hat{\bm{x}}}^{(t)})^T \nabla \hat{f}_h(\underline{\hat{\bm{x}}}^{(t)})}_2^2 + \norm{\nabla \hat{f}_h(\underline{\hat{\bm{x}}}^{(t)})}_2^2}}\\
	&\quad + \frac{\left[\norm{\nabla \hat{f}_h(\underline{\hat{\bm{x}}}^{(t)})}_2 - \sqrt{\norm{\hat{\underline{V}}_d(\underline{\hat{\bm{x}}}^{(t)})^T \nabla \hat{f}_h(\underline{\hat{\bm{x}}}^{(t)})}_2^2 + \norm{\nabla \hat{f}_h(\underline{\hat{\bm{x}}}^{(t)})}_2^2} \right] \cdot \nabla \hat{f}_h(\underline{\hat{\bm{x}}}^{(t)})^T \underline{\hat{\bm{x}}}^{(t)}}{\sqrt{\norm{\hat{\underline{V}}_d(\underline{\hat{\bm{x}}}^{(t)})^T \nabla \hat{f}_h(\underline{\hat{\bm{x}}}^{(t)})}_2^2 + \norm{\nabla \hat{f}_h(\underline{\hat{\bm{x}}}^{(t)})}_2^2}}\\
	&\stackrel{\text{(ii)}}{=} \frac{\norm{\hat{\underline{V}}_d(\underline{\hat{\bm{x}}}^{(t)})^T \nabla \hat{f}_h(\underline{\hat{\bm{x}}}^{(t)})}_2^2}{\sqrt{\norm{\hat{\underline{V}}_d(\underline{\hat{\bm{x}}}^{(t)})^T \nabla \hat{f}_h(\underline{\hat{\bm{x}}}^{(t)})}_2^2 + \norm{\nabla \hat{f}_h(\underline{\hat{\bm{x}}}^{(t)})}_2^2}}\\
	&\quad \times \frac{\left[\norm{\nabla \hat{f}_h(\underline{\hat{\bm{x}}}^{(t)})}_2 + \sqrt{\norm{\hat{\underline{V}}_d(\underline{\hat{\bm{x}}}^{(t)})^T \nabla \hat{f}_h(\underline{\hat{\bm{x}}}^{(t)})}_2^2 + \norm{\nabla \hat{f}_h(\underline{\hat{\bm{x}}}^{(t)})}_2^2} - \nabla \hat{f}_h(\underline{\hat{\bm{x}}}^{(t)})^T \underline{\hat{\bm{x}}}^{(t)} \right]}{\norm{\nabla \hat{f}_h(\underline{\hat{\bm{x}}}^{(t)})}_2 + \sqrt{\norm{\hat{\underline{V}}_d(\underline{\hat{\bm{x}}}^{(t)})^T \nabla \hat{f}_h(\underline{\hat{\bm{x}}}^{(t)})}_2^2 + \norm{\nabla \hat{f}_h(\underline{\hat{\bm{x}}}^{(t)})}_2^2}}\\
	&\stackrel{\text{(iii)}}{\geq} \frac{\norm{\hat{\underline{V}}_d(\underline{\hat{\bm{x}}}^{(t)})^T \nabla \hat{f}_h(\underline{\hat{\bm{x}}}^{(t)})}_2^2}{\norm{\nabla \hat{f}_h(\underline{\hat{\bm{x}}}^{(t)})}_2 + \sqrt{\norm{\hat{\underline{V}}_d(\underline{\hat{\bm{x}}}^{(t)})^T \nabla \hat{f}_h(\underline{\hat{\bm{x}}}^{(t)})}_2^2 + \norm{\nabla \hat{f}_h(\underline{\hat{\bm{x}}}^{(t)})}_2^2}}\\
	&\stackrel{\text{(iv)}}{\geq} \frac{\norm{\hat{\underline{V}}_d(\underline{\hat{\bm{x}}}^{(t)})^T \nabla \hat{f}_h(\underline{\hat{\bm{x}}}^{(t)})}_2^2}{\left(1+\sqrt{2} \right)\cdot \norm{\nabla \hat{f}_h(\underline{\hat{\bm{x}}}^{(t)})}_2}\\
	&\geq 0,
	\end{align*}
	where we use the orthogonality between $\hat{\underline{V}}_d(\underline{\hat{\bm{x}}}^{(t)}) \hat{\underline{V}}_d(\underline{\hat{\bm{x}}}^{(t)})^T$ and $\underline{\hat{\bm{x}}}^{(t)}$ in (i), multiply $\norm{\nabla \hat{f}_h(\underline{\hat{\bm{x}}}^{(t)})}_2 + \sqrt{\norm{\hat{\underline{V}}_d(\underline{\hat{\bm{x}}}^{(t)})^T \nabla \hat{f}_h(\underline{\hat{\bm{x}}}^{(t)})}_2^2 + \norm{\nabla \hat{f}_h(\underline{\hat{\bm{x}}}^{(t)})}_2^2}$ to both the numerators and denominators of the two summands to obtain (ii), leverage the fact that $\norm{\nabla \hat{f}_h(\underline{\hat{\bm{x}}}^{(t)})}_2 \geq \nabla \hat{f}_h(\underline{\hat{\bm{x}}}^{(t)})^T \underline{\hat{\bm{x}}}^{(t)}$ in (iii), use the inequality $\norm{\hat{\underline{V}}_d(\underline{\hat{\bm{x}}}^{(t)})^T \nabla \hat{f}_h(\underline{\hat{\bm{x}}}^{(t)})}_2 \leq \norm{\nabla \hat{f}_h(\underline{\hat{\bm{x}}}^{(t)})}_2$ in (iv). It thus completes the proof of (a).\\
	
	\noindent (b) Our derivation in (a) already shows that
	$$\norm{\hat{\underline{V}}_d(\underline{\hat{\bm{x}}}^{(t)})^T \nabla \hat{f}_h(\underline{\hat{\bm{x}}}^{(t)})}_2 \leq \sqrt{\left(1+\sqrt{2} \right) \cdot \norm{\nabla \hat{f}_h(\underline{\hat{\bm{x}}}^{(t)})}_2 \cdot \left[\hat{f}_h(\underline{\hat{\bm{x}}}^{(t+1)}) - \hat{f}_h(\underline{\hat{\bm{x}}}^{(t)}) \right]}.$$
	Notice that, on the one hand, the differentiability of kernel $L$ and the compactness of $\Omega_q$ imply that $\norm{\nabla \hat{f}_h(\bm{x})}_2 \leq B_{h,L}$ for all $\bm{x}\in \Omega_q$, where $B_{h,L}>0$ only depends on the bandwidth $h$ and kernel $L$. On the other hand, our argument in (a) already proves the convergence of $\left\{\hat{f}_h(\underline{\hat{\bm{x}}}^{(t)}) \right\}_{t=0}^{\infty}$. Therefore,
	$$\norm{\hat{\underline{V}}_d(\underline{\hat{\bm{x}}}^{(t)})^T \nabla \hat{f}_h(\underline{\hat{\bm{x}}}^{(t)})}_2 \leq \sqrt{\left(1+\sqrt{2} \right) B_{h,L} \cdot \left[\hat{f}_h(\underline{\hat{\bm{x}}}^{(t+1)}) - \hat{f}_h(\underline{\hat{\bm{x}}}^{(t)}) \right]} \to 0$$
	as $t\to \infty$. The result follows.\\
	
	\noindent (c) Given the iterative formula \eqref{Dir_SCMS_fixed_point_new} in the main paper, we deduce that
	\begin{align*}
	&\norm{\underline{\hat{\bm{x}}}^{(t+1)} - \underline{\hat{\bm{x}}}^{(t)}}_2^2 \\
	&= \norm{\frac{\hat{\underline{V}}_d(\underline{\hat{\bm{x}}}^{(t)}) \hat{\underline{V}}_d(\underline{\hat{\bm{x}}}^{(t)})^T \nabla \hat{f}_h(\underline{\hat{\bm{x}}}^{(t)}) + \norm{\nabla \hat{f}_h(\underline{\hat{\bm{x}}}^{(t)})}_2 \cdot \underline{\hat{\bm{x}}}^{(t)}}{\norm{\hat{\underline{V}}_d(\underline{\hat{\bm{x}}}^{(t)}) \hat{\underline{V}}_d(\underline{\hat{\bm{x}}}^{(t)})^T \nabla \hat{f}_h(\underline{\hat{\bm{x}}}^{(t)}) + \norm{\nabla \hat{f}_h(\underline{\hat{\bm{x}}}^{(t)})}_2 \cdot \underline{\hat{\bm{x}}}^{(t)}}_2} -\underline{\hat{\bm{x}}}^{(t)}}_2^2\\
	&\stackrel{\text{(i)}}{=} \frac{\norm{\hat{\underline{V}}_d(\underline{\hat{\bm{x}}}^{(t)}) \hat{\underline{V}}_d(\underline{\hat{\bm{x}}}^{(t)})^T \nabla \hat{f}_h(\underline{\hat{\bm{x}}}^{(t)}) + \left[\norm{\nabla \hat{f}_h(\underline{\hat{\bm{x}}}^{(t)})}_2 - \sqrt{\norm{\hat{\underline{V}}_d(\underline{\hat{\bm{x}}}^{(t)})^T \nabla \hat{f}_h(\underline{\hat{\bm{x}}}^{(t)})}_2^2 + \norm{\nabla \hat{f}_h(\underline{\hat{\bm{x}}}^{(t)})}_2^2}\right] \underline{\hat{\bm{x}}}^{(t)}}_2^2}{\norm{\hat{\underline{V}}_d(\underline{\hat{\bm{x}}}^{(t)})^T \nabla \hat{f}_h(\underline{\hat{\bm{x}}}^{(t)})}_2^2 + \norm{\nabla \hat{f}_h(\underline{\hat{\bm{x}}}^{(t)})}_2^2}\\
	&\stackrel{\text{(ii)}}{=} \frac{2\norm{\hat{\underline{V}}_d(\underline{\hat{\bm{x}}}^{(t)})^T \nabla \hat{f}_h(\underline{\hat{\bm{x}}}^{(t)})}_2^2}{\norm{\hat{\underline{V}}_d(\underline{\hat{\bm{x}}}^{(t)})^T \nabla \hat{f}_h(\underline{\hat{\bm{x}}}^{(t)})}_2^2 + \norm{\nabla \hat{f}_h(\underline{\hat{\bm{x}}}^{(t)})}_2^2}\\
	&\quad + \underbrace{\frac{2\norm{\nabla \hat{f}_h(\underline{\hat{\bm{x}}}^{(t)})}_2^2 - 2\norm{\nabla \hat{f}_h(\underline{\hat{\bm{x}}}^{(t)})}_2 \cdot \sqrt{\norm{\hat{\underline{V}}_d(\underline{\hat{\bm{x}}}^{(t)})^T \nabla \hat{f}_h(\underline{\hat{\bm{x}}}^{(t)})}_2^2 + \norm{\nabla \hat{f}_h(\underline{\hat{\bm{x}}}^{(t)})}_2^2} }{\norm{\hat{\underline{V}}_d(\underline{\hat{\bm{x}}}^{(t)})^T \nabla \hat{f}_h(\underline{\hat{\bm{x}}}^{(t)})}_2^2 + \norm{\nabla \hat{f}_h(\underline{\hat{\bm{x}}}^{(t)})}_2^2}}_{\leq 0}\\
	&\leq \frac{2\norm{\hat{\underline{V}}_d(\underline{\hat{\bm{x}}}^{(t)})^T \nabla \hat{f}_h(\underline{\hat{\bm{x}}}^{(t)})}_2^2}{\norm{\hat{\underline{V}}_d(\underline{\hat{\bm{x}}}^{(t)})^T \nabla \hat{f}_h(\underline{\hat{\bm{x}}}^{(t)})}_2^2 + \norm{\nabla \hat{f}_h(\underline{\hat{\bm{x}}}^{(t)})}_2^2},
	\end{align*}
	where we leverage the orthogonality between $\hat{\underline{V}}_d(\underline{\hat{\bm{x}}}^{(t)}) \hat{\underline{V}}_d(\underline{\hat{\bm{x}}}^{(t)})^T \nabla \hat{f}_h(\underline{\hat{\bm{x}}}^{(t)})$ and $\underline{\hat{\bm{x}}}^{(t)}$ to obtain (i) and (ii). Under the assumption that the kernel $L$ is strictly decreasing and (twice) continuously differentiable, we know that $\norm{\nabla \hat{f}_h(\bm{x})}_2$ is lower bounded away from 0 on $\Omega_q$. Therefore, with the result in (b), the above calculation indicates that
	$$\norm{\underline{\hat{\bm{x}}}^{(t+1)} - \underline{\hat{\bm{x}}}^{(t)}}_2 \leq \frac{\sqrt{2}\norm{\hat{\underline{V}}_d(\underline{\hat{\bm{x}}}^{(t)})^T \nabla \hat{f}_h(\underline{\hat{\bm{x}}}^{(t)})}_2}{\sqrt{\norm{\hat{\underline{V}}_d(\underline{\hat{\bm{x}}}^{(t)})^T \nabla \hat{f}_h(\underline{\hat{\bm{x}}}^{(t)})}_2^2 + \norm{\nabla \hat{f}_h(\underline{\hat{\bm{x}}}^{(t)})}_2^2}} \to 0$$
	as $t\to \infty$. The result follows.
\end{proof}

\begin{remark}
	The conditions imposed on kernel $L$ in Proposition~\ref{Dir_SCMS_prop} is satisfied by some commonly used kernels, such as the von Mises kernel $L(r)=e^{-r}$. However, they can be further relaxed. On the one hand, it is sufficient to assume that the kernel $L$ is twice continuously differentiable except for finitely many points on $[0,\infty)$. On the other hand, as long as the kernel $L$ satisfies $C_{L,q} =-\frac{\int_0^{\infty} L'(r) r^{\frac{q}{2}-1} dr}{\int_0^{\infty} L(r) r^{\frac{q}{2}-1} dr}>0$ and the true directional density $f$ is positive almost everywhere on $\Omega_q$, Lemma~\ref{norm_tot_grad} demonstrates that $\norm{\nabla \hat{f}_h(\bm{x})}_2\to \infty$ with probability tending to 1 when $h\to 0$ and $nh^q\to \infty$. Therefore, our upper bound on $\norm{\underline{\hat{\bm{x}}}^{(t+1)} - \underline{\hat{\bm{x}}}^{(t)}}_2$ in our proof of (c) will be asymptotically valid for all $t\geq 0$, even without the strict decreasing property of kernel $L$. Under such relaxation, our conclusions in Proposition~\ref{Dir_SCMS_prop} are applicable to directional SCMS algorithms with other kernels that have bounded supports on $[0,\infty)$.
\end{remark}

\begin{customprop}{4.4}[Convergence of the SCGA Algorithm on $\Omega_q$]
		For any SCGA sequence $\big\{\underline{\bm{x}}^{(t)} \big\}_{t=0}^{\infty} \subset \Omega_q$ defined by \eqref{SCGA_manifold_update} with $0<\underline{\eta} < \frac{2}{q\norm{\mathcal{H}f}_{\infty}^{(2)}}$, the following properties hold:
		\begin{enumerate}[label=(\alph*)]
			\item Under condition (\underline{A1}), the objective function sequence $\big\{f(\underline{\bm{x}}^{(t)}) \big\}_{t=0}^{\infty}$ is non-decreasing and thus converges.
			
			\item Under condition (\underline{A1}), $\lim_{t\to\infty} \norm{\underline{V}_d(\underline{\bm{x}}^{(t)})^T \grad f(\underline{\bm{x}}^{(t)})}_2 = \lim_{t\to\infty} d_g\left(\underline{\bm{x}}^{(t+1)}, \underline{\bm{x}}^{(t)} \right)=0$.
			
			\item Under conditions (\underline{A1-3}), $\lim_{t\to\infty} d_g\left(\underline{\bm{x}}^{(t)},\underline{R}_d \right)=0$ whenever $\underline{\bm{x}}^{(0)} \in \underline{R}_d\oplus \underline{r}_1$ with the convergence radius $\underline{r}_1$ satisfying
			$$0< \underline{r}_1 < \min\left\{\underline{\rho}/2, \frac{\min\left\{\underline{\beta}_1, 1 \right\}^2}{\underline{A}_2 \left(\norm{f}_{\infty}^{(3)} + \norm{f}_{\infty}^{(4)} \right)}, 2\sin\left(\frac{\underline{\beta}_1}{2\underline{A}_4(f)} \right) \right\},$$
			where $\underline{A}_2$ is a constant defined in (h) of Lemma~\ref{Dir_norm_reach_prop} while $\underline{A}_4(f) >0$ is a quantity depending on both the dimension $q$ and the functional norm $\norm{f}_{\infty,4}^*$ up to the fourth-order (partial) derivatives of $f$.
		\end{enumerate}
	\end{customprop}
	
	\begin{proof}[Proof of Proposition~\ref{SCGA_Dir_conv}]
		The proof is similar to our arguments in Proposition~\ref{SCGA_conv}. For the completeness, we still delineate the detailed steps because the proof requires some nontrivial techniques, such as parallel transports and line integrals, on general Riemannian manifolds.\\
		
		\noindent (a) We first derive the following property of the objective function $f$ supported on $\Omega_q$, which is a counterpart of \emph{Fact 1} in the proof of Proposition~\ref{SCGA_conv}.
		
		$\bullet$ \emph{Property 1}. Given (\underline{A1}), the function $f$ is $q\norm{\mathcal{H}f}_{\infty}^{(2)}$-smooth on $\Omega_q$, that is, $\grad f$ is $q\norm{\mathcal{H}f}_{\infty}^{(2)}$-Lipschitz.\\
		This property follows easily from the differentiability of $f$ guaranteed by condition (\underline{A1}) and Theorem 4.34 in \cite{Lee2018Riem_man} that
		\begin{align}
		\label{L_smooth_man1}
		\begin{split}
		\norm{\grad f(\bm{y}) -\Gamma_{\bm x}^{\bm y}\left(\grad f(\bm{x}) \right)}_2 &\leq \norm{\mathcal{H} f(\tilde{\bm{y}})}_2 \cdot \norm{\Exp_{\bm{y}}^{-1}(\bm{x})}_2\\
		&\leq q\norm{\mathcal{H}f}_{\infty}^{(2)} \cdot \norm{\Exp_{\bm{y}}^{-1}(\bm{x})}_2
		\end{split}
		\end{align}
		for any $\bm{x},\bm{y} \in \Omega_q$, where $\tilde{\bm{y}}$ lies on the geodesic curve $\varphi:[0,1]\to \Omega_q$ with $\varphi(0)=\bm{x}, \varphi(1)=\bm{y}$, and $\varphi'(0)=\Exp_{\bm{x}}^{-1}(\bm{y})$. Then, 
		\begin{align}
		\label{L_smooth_manifold}
		\begin{split}
		&\left|f(\bm{y}) -f(\bm{x}) - \langle \grad f(\bm{x}), \Exp_{\bm{x}}^{-1}(\bm{y}) \rangle \right| \\
		&\stackrel{\text{(i)}}{=} \left|\int_0^1 \langle \grad f(\varphi(t)), \varphi'(t) \rangle dt - \langle \grad f(\bm{x}), \Exp_{\bm{x}}^{-1}(\bm{y}) \rangle \right|\\
		&\stackrel{\text{(ii)}}{=} \left|\int_0^1 \left\langle \Gamma_{\varphi(t)}^{\bm{x}}\left(\grad f(\varphi(t)) \right), \Gamma_{\varphi(t)}^{\bm{x}}\left(\varphi'(t)\right) \right\rangle dt - \langle \grad f(\bm{x}), \Exp_{\bm{x}}^{-1}(\bm{y}) \rangle \right|\\
		&\stackrel{\text{(iii)}}{=} \left|\int_0^1 \left\langle \Gamma_{\varphi(t)}^{\bm{x}}\left(\grad f(\varphi(t)) \right) -\grad f(\bm{x}), \Exp_{\bm x}^{-1}(\bm{y}) \right\rangle \right|\\
		&\leq \int_0^1 \norm{\Gamma_{\varphi(t)}^{\bm{x}}\left(\grad f(\varphi(t)) \right) -\grad f(\bm{x})}_2 \cdot \norm{\Exp_{\bm x}^{-1}(\bm{y})}_2 dt\\
		&\stackrel{\text{(iv)}}{\leq} q \norm{\mathcal{H}f}_{\infty}^{(2)} \norm{\Exp_{\bm x}^{-1}(\bm{y})}_2 \int_0^1 d_g(\varphi(t),\bm{x}) \,dt\\
		&\leq q \norm{\mathcal{H}f}_{\infty}^{(2)} \norm{\Exp_{\bm x}^{-1}(\bm{y})}_2 \int_0^1 \norm{t \cdot \Exp_{\bm x}^{-1}(\bm{y})}_2 \,dt\\
		&= \frac{q\norm{\mathcal{H}f}_{\infty}^{(2)}}{2} \norm{\Exp_{\bm{x}}^{-1}(\bm{y})}_2^2,
		\end{split}
		\end{align}
		where the equality (i) follows from the fundamental theorem for line integrals (Theorem 11.39 in \citealt{Lee2012}), equality (ii) utilizes the isometric property of parallel transports, and inequality (iv) follows from \eqref{L_smooth_man1}. Moreover, since the velocity of the geodesic $\varphi$ is always constant, we deduce that $\Gamma_{\varphi(t)}^{\bm{x}}\left(\varphi'(t) \right) = \varphi'(0)=\Exp_{\bm x}^{-1}(\bm{y})$ and the equality (iii) follows.
		We will make use of the following direction of the inequality \eqref{L_smooth_manifold}:
		\begin{equation}
		f(\bm{y}) -f(\bm{x}) -\left\langle \grad f(\bm{x}), \Exp_{\bm{x}}^{-1}(\bm{y}) \right\rangle \geq -\frac{q\norm{\mathcal{H}f}_{\infty}^{(2)}}{2} \norm{\Exp_{\bm{x}}^{-1}(\bm{y})}_2^2.
		\end{equation}
		Moreover, when $0< \underline{\eta} < \frac{2}{q\norm{\mathcal{H}f}_{\infty}^{(2)}}$,
		\begin{align*}
		&f(\underline{\bm{x}}^{(t+1)}) -f(\underline{\bm{x}}^{(t)})\\
		&= f\left(\Exp_{\underline{\bm{x}}^{(t)}}\left(\underline{\eta}\cdot  \underline{V}_d(\underline{\bm{x}}^{(t)}) \underline{V}_d(\underline{\bm{x}}^{(t)})^T \grad f(\underline{\bm{x}}^{(t)}) \right) \right) - f(\underline{\bm{x}}^{(t)})\\
		&\geq \left\langle \grad f(\underline{\bm{x}}^{(t)}),\, \eta \underline{V}_d(\underline{\bm{x}}^{(t)}) \underline{V}_d(\underline{\bm{x}}^{(t)})^T \grad f(\underline{\bm{x}}^{(t)}) \right\rangle  \\
		&\quad - \frac{q\norm{\mathcal{H}f}_{\infty}^{(2)}}{2} \cdot \eta^2 \norm{\underline{V}_d(\underline{\bm{x}}^{(t)})^T \grad f(\underline{\bm{x}}^{(t)})}_2^2\\
		&= \underline{\eta} \left(1- \frac{q\norm{\mathcal{H}f}_{\infty}^{(2)} \underline{\eta}}{2} \right) \norm{\underline{V}_d(\underline{\bm{x}}^{(t)})^T \grad f(\underline{\bm{x}}^{(t)})}_2^2 \geq 0,
		\end{align*}
		showing that the objective function $f$ is non-decreasing along the SCGA path $\big\{\underline{\bm{x}}^{(t)} \big\}_{t=0}^{\infty}$ on $\Omega_q$. Given the compactness of $\Omega_q$ and the differentiability of $f$, we know that the sequence $\left\{f(\underline{\bm{x}}^{(t)}) \right\}_{t=0}^{\infty}$ is bounded. Thus, it converges.\\
		
		\noindent (b) From (a), we know that when $0< \underline{\eta} < \frac{2}{q\norm{\mathcal{H}f}_{\infty}^{(2)}}$, 
		\begin{align*}
		f(\underline{\bm{x}}^{(t+1)}) -f(\underline{\bm{x}}^{(t)}) &\geq \underline{\eta} \left(1- \frac{q\norm{\mathcal{H}f}_{\infty}^{(2)} \underline{\eta}}{2} \right) \norm{\underline{V}_d(\underline{\bm{x}}^{(t)})^T \grad f(\underline{\bm{x}}^{(t)})}_2^2 \\
		&= \left(\frac{2-q\norm{\mathcal{H}f}_{\infty}^{(2)} \underline{\eta}}{2\underline{\eta}} \right) d_g\left(\underline{\bm{x}}^{(t+1)}, \underline{\bm{x}}^{(t)} \right)^2.
		\end{align*}
		Since the sequence $\left\{f(\underline{\bm{x}}^{(t)}) \right\}_{t=0}^{\infty}$ converges, it follows that 
		$$\lim_{t\to\infty} \norm{\underline{V}_d(\underline{\bm{x}}^{(t)})^T \grad f(\underline{\bm{x}}^{(t)})}_2 = 0 \quad \text{ and } \quad \lim_{t\to\infty} d_g\left(\underline{\bm{x}}^{(t+1)}, \underline{\bm{x}}^{(t)} \right)=0.$$
		Recall from \eqref{Geo_Eu_dist_eq} that $\norm{\underline{\bm{x}}^{(t)} - \underline{\bm{x}}^{(t+1)}}_2=2\sin\left(\frac{d_g\left(\underline{\bm{x}}^{(t+1)}, \underline{\bm{x}}^{(t)} \right)}{2} \right)$, so $\lim_{t\to\infty}\norm{\underline{\bm{x}}^{(t)} - \underline{\bm{x}}^{(t+1)}}_2 =0$ as well.\\
		
		\noindent (c) Given condition (\underline{A2}) and the fact that $\underline{r}_1 < \underline{\rho}/2$, we know that 
		$$\bm{x}\in (\underline{R}_d \oplus \underline{r}_1) \cap \Omega_q \text{ and } \norm{\underline{V}_d(\bm{x})^T \grad f(\bm{x})}_2=0 \quad \text{ if and only if } \quad \bm{x}\in \underline{R}_d.$$
		Let $\underline{\bm{z}}^{(t)}=\pi_{\underline{R}_d}(\underline{\bm{x}}^{(t)})$ be the projection of $\underline{\bm{x}}^{(t)} \in \Omega_q$ in the SCGA sequence onto the directional ridge $\underline{R}_d$. Since $\underline{r}_1 \leq \mathtt{reach}(\underline{R}_d)$ by (h) of Lemma~\ref{Dir_norm_reach_prop}, $\underline{\bm{z}}^{(t)}$ is well-defined when $\underline{\bm{x}}^{(t)}\in \left(\underline{R}_d\oplus \underline{r}_1 \right) \cap \Omega_q$. Recall from \eqref{Dir_normal_rows} that the column space of
		\begin{align*}
		&\left(\bm{I}_{q+1} -\bm{z}^{(t)} \left(\bm{z}^{(t)} \right)^T \right)\underline{M}(\underline{\bm{z}}^{(t)}) \\
		&= \left(\bm{I}_{q+1} -\bm{z}^{(t)} \left(\bm{z}^{(t)} \right) \right)^T \nabla \left[ \underline{V}_d(\underline{\bm{z}}^{(t)})^T \nabla f(\underline{\bm{z}}^{(t)}) \right]^T \in \mathbb{R}^{(q+1)\times (q-d)}
		\end{align*}
		coincides with the normal space of $\underline{R}_d$ within the tangent space $T_{\underline{\bm{z}}^{(t)}}$. We define a geodesic $\varphi:[0,1] \to \Omega_q$ with $\varphi(0)=\underline{\bm{z}}^{(t)}, \varphi(1) = \underline{\bm{x}}^{(t)}, \varphi'(0)=\Exp_{\underline{\bm{z}}^{(t)}}^{-1}(\underline{\bm{x}}^{(t)})$ and calculate that
		\begin{align*}
		&\norm{\underline{V}_d(\underline{\bm{x}}^{(t)})^T \nabla f(\underline{\bm{x}}^{(t)})}_2 \\
		&= \norm{\underline{V}_d(\underline{\bm{x}}^{(t)})^T \nabla f(\underline{\bm{x}}^{(t)}) - \underbrace{\underline{V}_d(\underline{\bm{z}}^{(t)})^T \nabla f(\underline{\bm{z}}^{(t)})}_{=0}}_2\\
		&= \norm{\int_0^1 \left\langle \nabla \left[ \underline{V}_d(\underline{\bm{z}}^{(t)})^T \nabla f(\underline{\bm{z}}^{(t)}) \right]^T, \varphi'(\epsilon) \right\rangle d\epsilon}_2 \quad \text{ by Theorem 11.39 in \cite{Lee2012}}\\
		&\stackrel{\text{(i)}}{=} \norm{\left\langle \underline{M}(\varphi(0)), \varphi'(0) \right\rangle + \int_0^1 \left[\left\langle \underline{M}(\varphi(\epsilon)), \varphi'(\epsilon) \right\rangle - \left\langle \Gamma_{\varphi(0)}^{\varphi(\epsilon)}\left( \underline{M}(\varphi(0)) \right), \Gamma_{\varphi(0)}^{\varphi(\epsilon)}\left(\varphi'(0) \right) \right\rangle \right] d\epsilon}_2 \\
		&\stackrel{\text{(ii)}}{\geq} \norm{\underline{M}(\underline{\bm{z}}^{(t)})^T  \Exp_{\underline{\bm{z}}^{(t)}}^{-1}(\underline{\bm{x}}^{(t)})}_2 - \norm{\int_0^1 \left\langle \underline{M}(\varphi(\epsilon)) - \Gamma_{\varphi(0)}^{\varphi(\epsilon)}\left( \underline{M}(\varphi(0)) \right), \varphi'(\epsilon) \right\rangle d\epsilon}\\
		&\geq \norm{\underline{M}(\underline{\bm{z}}^{(t)})^T\left(\bm{I}_{q+1}-\underline{\bm{z}}^{(t)} \left(\underline{\bm{z}}^{(t)}\right)^T \right) \cdot  \Exp_{\underline{\bm{z}}^{(t)}}^{-1}(\underline{\bm{x}}^{(t)})}_2 \\
		&\quad - \int_0^1 \norm{\underline{M}(\varphi(\epsilon)) - \Gamma_{\varphi(0)}^{\varphi(\epsilon)}\left( \underline{M}(\varphi(0)) \right)}_2\cdot\norm{\varphi'(\epsilon)}_2  d\epsilon\\
		&\geq \norm{\left[\left(\bm{I}_{q+1}-\underline{\bm{z}}^{(t)} \left(\underline{\bm{z}}^{(t)}\right)^T \right)\underline{M}(\underline{\bm{z}}^{(t)}) \right]^T \Exp_{\underline{\bm{z}}^{(t)}}^{-1}(\underline{\bm{x}}^{(t)})}_2 \\
		&\quad - \norm{\Exp_{\underline{\bm{z}}^{(t)}}^{-1}(\underline{\bm{x}}^{(t)})}_2\int_0^1 \sup_{\epsilon \in [0,1]} \norm{\bar{\nabla} \underline{M}\left(\varphi(\epsilon) \right) }_2 \cdot \epsilon\cdot \norm{\Exp_{\underline{\bm{z}}^{(t)}}^{-1}(\underline{\bm{x}}^{(t)})}_2 d\epsilon \\
		& \quad \text{ by Theorem 4.34 in \cite{Lee2018Riem_man}}\\
		&\stackrel{\text{(iii)}}{\geq} \underline{\beta}_1 \cdot d_g\left(\underline{\bm{x}}^{(t)}, \underline{\bm{z}}^{(t)}\right) - \frac{\underline{A}_4(f)}{2} \cdot d_g\left(\underline{\bm{x}}^{(t)}, \underline{\bm{z}}^{(t)}\right)^2 \\
		&= d_g\left(\underline{\bm{x}}^{(t)}, \underline{\bm{z}}^{(t)}\right) \left(\underline{\beta}_1 - \frac{\underline{A}_4(f)}{2}\cdot d_g\left(\underline{\bm{x}}^{(t)}, \underline{\bm{z}}^{(t)}\right) \right)\\
		&\stackrel{\text{(iv)}}{\geq} \frac{\underline{\beta}_1}{2} \cdot d_g\left(\underline{\bm{x}}^{(t)}, \underline{\bm{z}}^{(t)}\right)
		\end{align*}
		where we utilize the isometric properties of parallel transports in (i), note that the velocity of geodesic is constant, \emph{i.e.}, $\varphi'(0)=\varphi'(\epsilon)$ for any $\epsilon\in [0,1]$ to obtain (ii), leverage (d) of Lemma~\ref{Dir_norm_reach_prop} to deduce (iii) and use the fact that $d_g\left(\underline{\bm{x}}^{(t)}, \underline{\bm{z}}^{(t)}\right) \leq \frac{\underline{\beta}_1}{\underline{A}_4(f)}$ when $\underline{\bm{x}}^{(t)} \in \left(\underline{R}_d \oplus \underline{r}_1 \right) \cap \Omega_q$ in the inequality (iv). In particular for the inequality (iii), $\left(\bm{I}_{q+1}-\underline{\bm{z}}^{(t)} \left(\underline{\bm{z}}^{(t)}\right)^T \right)\underline{M}(\underline{\bm{z}}^{(t)})$ is a full column rank matrix and $\Exp_{\underline{\bm{z}}^{(t)}}^{-1}(\underline{\bm{x}}^{(t)})$ lies within the column space of $\left(\bm{I}_{q+1}-\underline{\bm{z}}^{(t)} \left(\underline{\bm{z}}^{(t)}\right)^T \right)\underline{M}(\underline{\bm{z}}^{(t)})$. Since the nonzero singular values of $\left(\bm{I}_{q+1}-\underline{\bm{z}}^{(t)} \left(\underline{\bm{z}}^{(t)}\right)^T \right)\underline{M}(\underline{\bm{z}}^{(t)})$ are lower bounded by $\underline{\beta}_1>0$, it follows that $$\norm{\left[\left(\bm{I}_{q+1}-\underline{\bm{z}}^{(t)} \left(\underline{\bm{z}}^{(t)}\right)^T \right)\underline{M}(\underline{\bm{z}}^{(t)}) \right]^T \Exp_{\underline{\bm{z}}^{(t)}}^{-1}(\underline{\bm{x}}^{(t)})}_2 \geq \underline{\beta}_1 d_g\left(\underline{\bm{x}}^{(t)}, \underline{\bm{z}}^{(t)}\right).$$ 
		In addition, we also know that $\underline{A}_4(f)>0$ comes from the supremum norm of $\bar{\nabla}\underline{M}(\bm{x})$ 
		over the geodesic connecting $\underline{\bm{x}}^{(t)}$ and $\underline{\bm{z}}^{(t)}$ with $\bar{\nabla}$ being the Riemannian connection, which in turn depends on the uniform functional norm $\norm{f}_{\infty,4}^*$ of the partial derivatives of $f$ up to the fourth order. By (b), we deduce that 
		$$\lim_{t\to\infty} d_g\left(\underline{\bm{x}}^{(t)}, \underline{\bm{z}}^{(t)}\right) =d_g\left(\underline{\bm{x}}^{(t)}, \underline{R}_d \right) =0.$$
		The results follow.
\end{proof}

The nonzero curvature structure of the unit (hyper-sphere) $\Omega_q$, on which the objective function (or density) $f$ lies, induces an extra challenge in establishing the linear convergence of population and sample-based SCGA algorithms. Some useful techniques used in analyzing non-asymptotic convergence of first-order methods in $\mathbb{R}^{q+1}$, such as the law of cosines and linearizations of the objective function, would fail on $\Omega_q$ \citep{Geo_Convex_Op2016}. Therefore, we first introduce a practical trigonometric distance bound for the Alexandrov space \citep{Burago1992} with its sectional curvature bounded from below.

\begin{lemma}[Lemma 5 in \cite{Geo_Convex_Op2016}; see also \cite{SGD_Riem2013}]
	\label{trigo_inequality}
	If $a,b,c$ are the sides (\emph{i.e.}, side lengths) of a geodesic triangle in an Alexandrov space with sectional curvature lower bounded by $\kappa$, and $A$ is the angle between sides $b$ and $c$, then
	\begin{equation}
	\label{trigo_inequality_eq}
	a^2 \leq \frac{\sqrt{|\kappa|} c}{\tanh(\sqrt{|\kappa|}c)} b^2 + c^2 -2bc \cos(A).
	\end{equation}
\end{lemma}

The sketching proof of Lemma \ref{trigo_inequality} can be founded in Lemma 5 of \cite{Geo_Convex_Op2016}. Note that the sectional curvature $\kappa=1$ on $\Omega_q$. We inherit the notation in \cite{Geo_Convex_Op2016} and denote $\frac{\sqrt{|\kappa|} c}{\tanh(\sqrt{|\kappa|}c)}$ by $\zeta(\kappa, c)$ for the curvature dependent quantity in the inequality \eqref{trigo_inequality_eq}. One can show by differentiating $\zeta(\kappa, c)$ with respect to $c$ that $\zeta(\kappa, c)$ is strictly increasing and greater than 1 for any $c>0$ and fixed $\kappa\neq 0$. With Lemma~\ref{trigo_inequality} in hand, we are able to state a straightforward corollary indicating an important relation between two consecutive points in the SCGA sequence $\big\{\underline{\bm{x}}^{(t)}\big\}_{t=0}^{\infty}$ on $\Omega_q$ defined by \eqref{SCGA_manifold_update}:
\begin{equation}
\label{SCGA_manifold_update_app}
\underline{\bm{x}}^{(t+1)} = \Exp_{\underline{\bm{x}}^{(t)}}\left(\underline{\eta} \cdot \underline{V}_d(\underline{\bm{x}}^{(t)}) \underline{V}_d(\underline{\bm{x}}^{(t)})^T \grad f(\underline{\bm{x}}^{(t)}) \right).
\end{equation}

\begin{corollary}
	\label{Riemannian_tri_update}
	For any point $\bm{y}, \underline{\bm{x}}^{(t)}$ in a geodesically convex set on $\Omega_q$, the update in \eqref{SCGA_manifold_update_app} satisfies
	\begin{align*}
	&2\underline{\eta}\left\langle \underline{V}_d(\underline{\bm{x}}^{(t)}) \underline{V}_d(\underline{\bm{x}}^{(t)})^T \nabla f(\underline{\bm{x}}^{(t)}), \Exp_{\underline{\bm{x}}^{(t)}}^{-1}(\bm{y}) \right\rangle \\
	&\leq d_g\left(\underline{\bm{x}}^{(t)}, \bm{y}\right)^2 - d_g\left(\underline{\bm{x}}^{(t+1)}, \bm{y} \right)^2 + \zeta\left(1,d_g(\underline{\bm{x}}^{(t)}, \bm{y}) \right) \cdot \underline{\eta}^2 \norm{\underline{V}_d(\underline{\bm{x}}^{(t)})^T \grad f(\underline{\bm{x}}^{(t)})}_2^2,
	\end{align*}
	where $d_g(\bm{x},\bm{y}) = \sqrt{\left\langle \Exp_{\bm{x}}^{-1}(\bm{y}), \Exp_{\bm{x}}^{-1}(\bm{y}) \right\rangle} = \norm{\Exp_{\bm{x}}^{-1}(\bm{y})}_2$ is the geodesic distance between $\bm{x}$ and $\bm{y}$ on $\Omega_q$.
\end{corollary}

\begin{proof}[Proof of Corollary~\ref{Riemannian_tri_update}] 
	Recall that the (population) SCGA iterative formula on $\Omega_q$ is given by $\underline{\bm{x}}^{(t+1)} = \Exp_{\underline{\bm{x}}^{(t)}}\left(\underline{\eta} \cdot \underline{V}_d(\underline{\bm{x}}^{(t)}) \underline{V}_d(\underline{\bm{x}}^{(t)})^T \grad f(\underline{\bm{x}}^{(t)}) \right)$. Note that for the geodesic triangle $\triangle \underline{\bm{x}}^{(t)}\underline{\bm{x}}^{(t+1)}\bm{y}$ with $\bm{y}\in \Omega_q$, we have that
	$$d_g(\underline{\bm{x}}^{(t)}, \underline{\bm{x}}^{(t+1)}) = \underline{\eta}\norm{\underline{V}_d(\underline{\bm{x}}^{(t)}) \underline{V}_d(\underline{\bm{x}}^{(t)})^T \grad f(\underline{\bm{x}}^{(t)})}_2 = \underline{\eta} \norm{\underline{V}_d(\underline{\bm{x}}^{(t)})^T \grad f(\underline{\bm{x}}^{(t)})}_2$$
	and 
	\begin{align*}
	&d_g(\underline{\bm{x}}^{(t)}, \underline{\bm{x}}^{(t+1)}) \cdot d_g(\underline{\bm{x}}^{(t)},\bm{y}) \cdot \cos\left(\angle\underline{\bm{x}}^{(t+1)}\underline{\bm{x}}^{(t)} \bm{y} \right) \\
	&= \underline{\eta}\left\langle \underline{V}_d(\underline{\bm{x}}^{(t)}) \underline{V}_d(\underline{\bm{x}}^{(t)})^T \grad f(\underline{\bm{x}}^{(t)}),\, \Exp_{\underline{\bm{x}}^{(t)}}^{-1}(\bm{y}) \right\rangle.
	\end{align*}
	By letting $a=\overline{\underline{\bm{x}}^{(t+1)}\bm{y}}, b=\overline{\underline{\bm{x}}^{(t+1)}\underline{\bm{x}}^{(t)}}, c=\overline{\underline{\bm{x}}^{(t)}\bm{y}}$, and $A=\angle\underline{\bm{x}}^{(t+1)}\underline{\bm{x}}^{(t)}\bm{y}$ in Lemma~\ref{trigo_inequality}, we obtain that
	\begin{align*}
	d_g(\underline{\bm{x}}^{(t+1)},\bm{y})^2 &\leq \zeta\left(1,d_g(\underline{\bm{x}}^{(t)},\bm{y})\right) \cdot \underline{\eta}^2 \norm{\underline{V}_d(\underline{\bm{x}}^{(t)})^T \grad f(\underline{\bm{x}}^{(t)})}_2^2 \\
	&\quad + d_g(\underline{\bm{x}}^{(t)},\bm{y})^2 -2\underline{\eta} \left\langle V_d(\underline{\bm{x}}^{(t)}) V_d(\underline{\bm{x}}^{(t)})^T \grad f(\underline{\bm{x}}^{(t)}),\, \Exp_{\underline{\bm{x}}^{(t)}}^{-1}(\bm{y}) \right\rangle.
	\end{align*}
	Some rearrangements will yield the final display.
\end{proof}

Note that $\left(\underline{R}_d \oplus \underline{\rho} \right) \cap \Omega_q$ in our conditions (\underline{A2-3}) is a geodesically convex set, where the minimal geodesic between two points in the set $\left(\underline{R}_d \oplus \underline{\rho} \right) \cap \Omega_q$ always lies within the set. Hence, Corollary~\ref{Riemannian_tri_update} is applicable to our interested SCGA algorithm initialized within $\left(\underline{R}_d \oplus \underline{\rho} \right) \cap \Omega_q$.

\begin{customthm}{4.6}[Linear Convergence of the SCGA Algorithm on $\Omega_q$]
	Assume conditions (\underline{A1-4}) throughout the theorem.
	\begin{enumerate}[label=(\alph*)]
		\item \textbf{Q-Linear convergence of $d_g(\underline{\bm{x}}^{(t)}, \underline{\bm{x}}^*)$}: Consider a convergence radius $\underline{r}_2>0$ satisfying 
		\begin{align*}
		0<\underline{r}_2 &\leq \min\Bigg\{\underline{\rho}/2, \frac{\underline{\beta}_1^2}{\underline{A}_2\left(\norm{f}_{\infty}^{(3)} + \norm{f}_{\infty}^{(4)} \right)}, \frac{\underline{\beta}_1}{\underline{A}_4(f)},\\
		& 2\sin\Bigg[\frac{3\underline{\beta}_0}{8q \left(12\norm{\mathcal{H}f}_{\infty}^{(2)}\underline{\beta}_2^2 \arcsin\left(\underline{\rho}/2 \right) + \sqrt{q}\norm{f}_{\infty}^{(3)} \right)}\Bigg] \Bigg\},
		\end{align*}
		where $\underline{A}_2>0$ is the constant defined in (h) of Lemma~\ref{Dir_norm_reach_prop} and $\underline{A}_4(f) >0$ is a quantity defined in (c) of Proposition~\ref{SCGA_Dir_conv} that depends on both the dimension $q$ and the functional norm $\norm{f}_{\infty,4}^*$ up to the fourth-order (partial) derivatives of $f$. Whenever $0 < \underline{\eta} \leq \min\left\{\frac{4}{\underline{\beta}_0}, \frac{1}{q\norm{\mathcal{H}f}_{\infty}^{(2)} \cdot \zeta(1,\underline{\rho})} \right\}$ and the initial point $\underline{\bm{x}}^{(0)} \in \text{Ball}_{q+1}(\underline{\bm{x}}^*, \underline{r}_2) \cap \Omega_q$ with $\underline{\bm{x}}^* \in \underline{R}_d$, we have that
		$$d_g(\underline{\bm{x}}^{(t)}, \underline{\bm{x}}^*) \leq \underline{\Upsilon}^t \cdot d_g(\underline{\bm{x}}^{(0)}, \underline{\bm{x}}^*) \quad \text{ with } \quad \underline{\Upsilon}= \sqrt{1-\frac{\underline{\beta}_0\underline{\eta}}{4}}.$$
		\item \textbf{R-Linear convergence of $d_g(\underline{\bm{x}}^{(t)}, \underline{R}_d)$}: Under the same radius $\underline{r}_2>0$ in (a), we have that whenever $0 < \underline{\eta} \leq \min\left\{\frac{4}{\underline{\beta}_0}, \frac{1}{q\norm{\mathcal{H}f}_{\infty}^{(2)} \cdot \zeta(1,\underline{\rho})} \right\}$ and the initial point $\underline{\bm{x}}^{(0)} \in \text{Ball}_{q+1}(\underline{\bm{x}}^*, \underline{r}_2) \cap \Omega_q$ with $\underline{\bm{x}}^* \in \underline{R}_d$,
		$$d_g(\underline{\bm{x}}^{(t)}, \underline{R}_d) \leq \underline{\Upsilon}^t \cdot d_g(\underline{\bm{x}}^{(0)}, \underline{\bm{x}}^*) \quad \text{ with } \quad \underline{\Upsilon}=\sqrt{1-\frac{\underline{\beta}_0\underline{\eta}}{4}}.$$
	\end{enumerate}
	We further assume (D1-2) in the rest of statements. Suppose that $h \to 0$ and $\frac{nh^{q+4}}{|\log h|} \to \infty$.
	\begin{enumerate}[label=(c)]
		\item \textbf{Q-Linear convergence of $d_g(\hat{\underline{\bm{x}}}^{(t)}, \underline{\bm{x}}^*)$}: Under the same radius $\underline{r}_2>0$ and $\underline{\Upsilon}=\sqrt{1-\frac{\underline{\beta}_0\underline{\eta}}{4}}$ in (a), we have that
		$$d_g(\hat{\underline{\bm{x}}}^{(t)}, \underline{\bm{x}}^*) \leq \underline{\Upsilon}^t \cdot d_g(\hat{\underline{\bm{x}}}^{(0)}, \underline{\bm{x}}^*) + O(h^2) + O_P\left(\sqrt{\frac{|\log h|}{nh^{q+4}}} \right)$$
		with probability tending to 1 whenever $0 < \underline{\eta} \leq \min\left\{\frac{4}{\underline{\beta}_0}, \frac{1}{q\norm{\mathcal{H}f}_{\infty}^{(2)} \cdot \zeta(1,\underline{\rho})} \right\}$ and the initial point $\hat{\underline{\bm{x}}}^{(0)} \in \text{Ball}_{q+1}(\underline{\bm{x}}^*, \underline{r}_2) \cap \Omega_q$ with $\underline{\bm{x}}^* \in \underline{R}_d$.
	\end{enumerate}
	\begin{enumerate}[label=(d)]
		\item \textbf{R-Linear convergence of $d_g(\hat{\underline{\bm{x}}}^{(t)}, \underline{R}_d)$}: Under the same radius $\underline{r}_2>0$ and $\underline{\Upsilon}=\sqrt{1-\frac{\underline{\beta}_0 \underline{\eta}}{4}}$ in (a), we have that
		$$d_g(\hat{\underline{\bm{x}}}^{(t)}, \underline{R}_d) \leq \underline{\Upsilon}^t \cdot d_g(\hat{\underline{\bm{x}}}^{(0)}, \underline{\bm{x}}^*) + O(h^2) + O_P\left(\sqrt{\frac{|\log h|}{nh^{q+4}}} \right)$$
		with probability tending to 1 whenever $0 < \underline{\eta} \leq \min\left\{\frac{4}{\underline{\beta}_0}, \frac{1}{q\norm{\mathcal{H}f}_{\infty}^{(2)} \cdot \zeta(1,\underline{\rho})} \right\}$ and the initial point $\hat{\underline{\bm{x}}}^{(0)} \in \text{Ball}_{q+1}(\underline{\bm{x}}^*, \underline{r}_2) \cap \Omega_q$ with $\underline{\bm{x}}^* \in \underline{R}_d$.
	\end{enumerate}
\end{customthm}

\begin{proof}[Proof of Theorem~\ref{LC_SCGA_Dir}]
	The proof is similar to our argument in Theorem~\ref{SCGA_LC}, except that the objective function $f$ is supported on a nonlinear manifold $\Omega_q$ here. The key arguments are credited to Corollary~\ref{Riemannian_tri_update}. We first recall the following two properties.\\
	$\bullet$ \emph{Property 1}. Given (\underline{A1}), the function $f$ is $q\norm{\mathcal{H}f}_{\infty}^{(2)}$-smooth on $\Omega_q$, that is, $\grad f$ is $q\norm{\mathcal{H}f}_{\infty}^{(2)}$-Lipschitz.\\
	$\bullet$ \emph{Property 2}. Given conditions (\underline{A1-3}), we know that $\norm{\underline{V}_d(\underline{\bm{x}}^{(t)})^T \grad f(\underline{\bm{x}}^{(t)})}_2 > 0 \text{ for any } \underline{\bm{x}}^{(t)} \in \big(\text{Ball}_{q+1}(\underline{\bm{x}}^*, \underline{r}_2) \cap \Omega_q \big) \setminus \underline{R}_d$ and 
	$$f(\underline{\bm{x}}^*) - f\left(\Exp_{\underline{\bm{x}}^{(t)}}\left(\frac{1}{q\norm{\mathcal{H}f}_{\infty}^{(2)}} \cdot \underline{V}_d(\underline{\bm{x}}^{(t)}) \underline{V}_d(\underline{\bm{x}}^{(t)})^T \grad f(\underline{\bm{x}}^{(t)}) \right) \right) \geq 0$$
	for any $\underline{\bm{x}}^{(t)} \in \text{Ball}_{q+1}(\underline{\bm{x}}^*, \underline{r}_2) \cap \Omega_q$ with $\underline{\bm{x}}^* \in \underline{R}_d$.
	
	\emph{Property 1} has been established in the proof of Proposition~\ref{SCGA_Dir_conv}, indicating that the objective function sequence $\left\{f(\underline{\bm{x}}^{(t)}) \right\}_{t=0}^{\infty}$ is non-decreasing when $0<\underline{\eta} < \frac{2}{q\norm{\mathcal{H}f}_{\infty}^{(2)}}$. \emph{Property 2} is a natural corollary by Proposition~\ref{SCGA_Dir_conv}, because $\underline{\bm{x}}^{(t)} \in \text{Ball}(\underline{\bm{x}}^*, \underline{r}_2) \cap \Omega_q$ and
	$$f\left(\Exp_{\underline{\bm{x}}^{(t)}}\left(\frac{1}{q\norm{\mathcal{H}f}_{\infty}^{(2)}} \cdot \underline{V}_d(\underline{\bm{x}}^{(t)}) \underline{V}_d(\underline{\bm{x}}^{(t)})^T \grad f(\underline{\bm{x}}^{(t)}) \right) \right)$$
	is the objective function value after one-step SCGA iteration on $\Omega_q$ with step size $\frac{1}{q\norm{\mathcal{H}f}_{\infty}^{(2)}}$. The iteration will move $\underline{\bm{x}}^{(t)}$ closer to the directional ridge $\underline{R}_d$. With the help of these two properties, we start the proofs of (a-d).\\
	
	\noindent (a) We first prove the following claim using Lemma~\ref{quad_bound_prop_Dir}: for all $t\geq 0$ and $\underline{\bm{x}}^{(0)} \in \text{Ball}_{q+1}(\underline{\bm{x}}^*, \underline{r}_2) \cap \Omega_q$,
	\begin{align}
		\label{SC_claim_Dir}
		\begin{split}
		f(\underline{\bm{x}}^*) -f(\underline{\bm{x}}^{(t)}) &\leq \left\langle \underline{V}_d(\underline{\bm{x}}^{(t)}) \underline{V}_d(\underline{\bm{x}}^{(t)})^T \grad f(\underline{\bm{x}}^{(t)}), \Exp_{\underline{\bm{x}}^{(t)}}^{-1}(\underline{\bm{x}}^*) \right\rangle - \frac{\underline{\beta}_0}{4}\cdot  d_g(\underline{\bm{x}}^*, \underline{\bm{x}}^{(t)})^2 + \underline{\epsilon}_t,
		\end{split}
		\end{align}
		where $\underline{\epsilon}_t = \left[2q\norm{\mathcal{H}f}_{\infty}^{(2)}\underline{\beta}_2^2 \arcsin\left(\underline{\rho}/2\right) + \frac{q^{\frac{3}{2}} \norm{f}_{\infty}^{(3)}}{6} \right] =o\left(d_g(\underline{\bm{x}}^*, \underline{\bm{x}}^{(t)})^2 \right)$.
		By the differentiability of $f$ ensured by condition (\underline{A1}) and Taylor's theorem on $\Omega_q$, we deduce that
		\begin{align*}
		&f(\underline{\bm{x}}^*) -f(\underline{\bm{x}}^{(t)}) \\
		&\leq \left\langle \grad f(\underline{\bm{x}}^{(t)}), \Exp_{\underline{\bm{x}}^{(t)}}^{-1}(\underline{\bm{x}}^*) \right\rangle + \frac{1}{2} \Exp_{\underline{\bm{x}}^{(t)}}^{-1}(\underline{\bm{x}}^*)^T \left[\mathcal{H} f(\underline{\bm{x}}^{(t)}) \right] \Exp_{\underline{\bm{x}}^{(t)}}^{-1}(\underline{\bm{x}}^*) \\
		&\quad + \frac{q^{\frac{3}{2}}\norm{f}_{\infty}^{(3)}}{6} \cdot \norm{\Exp_{\underline{\bm{x}}^{(t)}}^{-1}(\underline{\bm{x}}^*)}_2^3\\
		&\stackrel{\text{(i)}}{=} \left\langle \underline{V}_d(\underline{\bm{x}}^{(t)}) \underline{V}_d(\underline{\bm{x}}^{(t)})^T \grad f(\underline{\bm{x}}^{(t)}), \Exp_{\underline{\bm{x}}^{(t)}}^{-1}(\underline{\bm{x}}^*) \right\rangle + \left\langle \underline{U}_d^{\perp}(\underline{\bm{x}}^{(t)}) \grad f(\underline{\bm{x}}^{(t)}), \Exp_{\underline{\bm{x}}^{(t)}}^{-1}(\underline{\bm{x}}^*) \right\rangle\\
		&\quad + \frac{1}{2} \Exp_{\underline{\bm{x}}^{(t)}}^{-1}(\underline{\bm{x}}^*)^T \left(\underline{V}_{\diamond}(\underline{\bm{x}}^{(t)}), \underline{V}_d(\underline{\bm{x}}^{(t)}) \right) \begin{pmatrix}
		0 & & & \\
		& \underline{\lambda}_1(\underline{\bm{x}}^{(t)}) & & \\
		& & \ddots & \\
		& & & \underline{\lambda}_q(\underline{\bm{x}}^{(t)})
		\end{pmatrix}
		\begin{pmatrix}
		\underline{V}_{\diamond}(\underline{\bm{x}}^{(t)})\\
		\underline{V}_d(\underline{\bm{x}}^{(t)})
		\end{pmatrix}
		\Exp_{\underline{\bm{x}}^{(t)}}^{-1}(\underline{\bm{x}}^*)\\
		&\quad + \frac{q^{\frac{3}{2}}\norm{f}_{\infty}^{(3)}}{6} \cdot d_g(\underline{\bm{x}}^*, \underline{\bm{x}}^{(t)})^3 \\
		&\stackrel{\text{(ii)}}{\leq} \left\langle \underline{V}_d(\underline{\bm{x}}^{(t)}) \underline{V}_d(\underline{\bm{x}}^{(t)})^T \grad f(\underline{\bm{x}}^{(t)}), \Exp_{\underline{\bm{x}}^{(t)}}^{-1}(\underline{\bm{x}}^*) \right\rangle +  \frac{\underline{\beta}_0}{4}\cdot d_g\left(\underline{\bm{x}}^{(t)},\underline{\bm{x}}^* \right)^2 \\
		&\quad + \frac{\max\left\{0, \underline{\lambda}_1(\underline{\bm{x}}^{(t)})\right\}}{2} \norm{\underline{U}_d^{\perp}(\underline{\bm{x}}^{(t)}) \Exp_{\underline{\bm{x}}^{(t)}}^{-1}(\underline{\bm{x}}^*)}_2^2 - \frac{\underline{\beta}_0}{2} \norm{\underline{V}_d(\underline{\bm{x}}^{(t)})^T \Exp_{\underline{\bm{x}}^{(t)}}^{-1}(\underline{\bm{x}}^*)}_2^2\\
		&\quad + \frac{q^{\frac{3}{2}}\norm{f}_{\infty}^{(3)}}{6} \cdot d_g(\underline{\bm{x}}^*, \underline{\bm{x}}^{(t)})^3 \\
		&\stackrel{\text{(iii)}}{\leq} \left\langle \underline{V}_d(\underline{\bm{x}}^{(t)}) \underline{V}_d(\underline{\bm{x}}^{(t)})^T \grad f(\underline{\bm{x}}^{(t)}), \Exp_{\underline{\bm{x}}^{(t)}}^{-1}(\underline{\bm{x}}^*) \right\rangle + \frac{\underline{\beta}_0}{4}\cdot d_g\left(\underline{\bm{x}}^{(t)},\underline{\bm{x}}^* \right)^2\\
		&\quad + \frac{\left(\underline{\beta}_0+\max\left\{0, \underline{\lambda}_1(\underline{\bm{x}}^{(t)})\right\} \right)}{2} \norm{\underline{U}_d^{\perp}(\underline{\bm{x}}^{(t)}) \Exp_{\underline{\bm{x}}^{(t)}}^{-1}(\underline{\bm{x}}^*)}_2^2 - \frac{\underline{\beta}_0}{2} \norm{\Exp_{\underline{\bm{x}}^{(t)}}^{-1}(\underline{\bm{x}}^*)}_2^2 \\
		&\quad + \frac{q^{\frac{3}{2}}\norm{f}_{\infty}^{(3)}}{6} \cdot d_g(\underline{\bm{x}}^*, \underline{\bm{x}}^{(t)})^3 \\
		&\stackrel{\text{(iv)}}{\leq} \left\langle \underline{V}_d(\underline{\bm{x}}^{(t)}) \underline{V}_d(\underline{\bm{x}}^{(t)})^T \grad f(\underline{\bm{x}}^{(t)}), \Exp_{\underline{\bm{x}}^{(t)}}^{-1}(\underline{\bm{x}}^*) \right\rangle- \frac{\underline{\beta}_0}{4}\cdot d_g\left(\underline{\bm{x}}^{(t)},\underline{\bm{x}}^* \right)^2\\
		&\quad +\frac{\left(\underline{\beta}_0+\max\left\{0, \underline{\lambda}_1(\underline{\bm{x}}^{(t)})\right\} \right)}{2} \cdot \underline{\beta}_2^2 \cdot d_g\left(\underline{\bm{x}}^{(t)},\underline{\bm{x}}^* \right)^4 + \frac{q^{\frac{3}{2}}\norm{f}_{\infty}^{(3)}}{6} \cdot d_g(\underline{\bm{x}}^*, \underline{\bm{x}}^{(t)})^3 \\
		&\stackrel{\text{(v)}}{\leq} \left\langle \underline{V}_d(\underline{\bm{x}}^{(t)}) \underline{V}_d(\underline{\bm{x}}^{(t)})^T \grad f(\underline{\bm{x}}^{(t)}), \Exp_{\underline{\bm{x}}^{(t)}}^{-1}(\underline{\bm{x}}^*) \right\rangle - \frac{\underline{\beta}_0}{4} \cdot d_g(\underline{\bm{x}}^*, \underline{\bm{x}}^{(t)})^2 \\
		&\quad + \left[2q\norm{\mathcal{H}f}_{\infty}^{(2)}\underline{\beta}_2^2 \arcsin\left(\underline{\rho}/2\right) + \frac{q^{\frac{3}{2}} \norm{f}_{\infty}^{(3)}}{6} \right]d_g(\underline{\bm{x}}^*, \underline{\bm{x}}^{(t)})^3,
		\end{align*}
		where we leverage the equality $\bm{I}_{q+1} = \underline{V}_d(\underline{\bm{x}}^{(t)}) \underline{V}_d(\underline{\bm{x}}^{(t)})^T + \underline{U}_d^{\perp}(\underline{\bm{x}}^{(t)})$ in (i) and (iii), use conditions (\underline{A2}) and (\underline{A4}) that $\underline{\lambda}_q(\underline{\bm{x}}^{(t)}) \leq \cdots \leq \underline{\lambda}_{d+1}(\underline{\bm{x}}^{(t)})<-\underline{\beta}_0$ and 
		$$\left\langle \underline{U}_d^{\perp}(\underline{\bm{x}}^{(t)}) \grad f(\underline{\bm{x}}^{(t)}), \Exp_{\underline{\bm{x}}^{(t)}}^{-1}(\underline{\bm{x}}^*) \right\rangle \leq \frac{\underline{\beta}_0}{4}\cdot d_g\left(\underline{\bm{x}}^{(t)},\underline{\bm{x}}^* \right)^2$$ 
		in (ii), apply the quadratic bound for $\norm{\underline{U}_d^{\perp}(\underline{\bm{x}}^{(t)}) \Exp_{\underline{\bm{x}}^{(t)}}^{-1}(\underline{\bm{x}}^*)}_2$ in condition (\underline{A4}) to obtain (iv), and leverage the facts that $\max\left\{\underline{\beta}_0, 0, \underline{\lambda}_1(\underline{\bm{x}}^{(t)})\right\} \leq q \norm{\mathcal{H}f}_{\infty}^{(2)}$ and $d_g(\underline{\bm{x}}^*, \underline{\bm{x}}^{(t)}) \leq 2\arcsin\left(\underline{\rho}/2 \right)$ when $\norm{\underline{\bm{x}}^{(t)} - \underline{\bm{x}}^*}_2 \leq \underline{\rho}$ in (v); recall \eqref{Geo_Eu_dist_eq}. Our claim \eqref{SC_claim_Dir} is thus proved.
	
	In addition, given \emph{Property 2} and any $\underline{\bm{x}}^{(t)} \in \underline{R}_d \oplus \underline{r}_2$, we derive that
	\begin{align*}
	&f(\underline{\bm{x}}^{(t)}) - f(\underline{\bm{x}}^*) \\
	&\leq f(\underline{\bm{x}}^{(t)}) - f(\underline{\bm{x}}^*) + f(\underline{\bm{x}}^*) - f\left(\Exp_{\underline{\bm{x}}^{(t)}}\left(\frac{1}{q\norm{\mathcal{H}f}_{\infty}^{(2)}} \cdot \underline{V}_d(\underline{\bm{x}}^{(t)}) \underline{V}_d(\underline{\bm{x}}^{(t)})^T \grad f(\underline{\bm{x}}^{(t)}) \right) \right)\\
	&= -\left[f\left(\Exp_{\underline{\bm{x}}^{(t)}}\left(\frac{1}{q\norm{\mathcal{H}f}_{\infty}^{(2)}} \cdot \underline{V}_d(\underline{\bm{x}}^{(t)}) \underline{V}_d(\underline{\bm{x}}^{(t)})^T \grad f(\underline{\bm{x}}^{(t)}) \right) \right) - f(\underline{\bm{x}}^{(t)}) \right]\\
	&\leq -\Bigg[\left\langle \grad f(\underline{\bm{x}}^{(t)}),\, \frac{1}{q\norm{\mathcal{H}f}_{\infty}^{(2)}} \underline{V}_d(\underline{\bm{x}}^{(t)}) \underline{V}_d(\underline{\bm{x}}^{(t)})^T \grad f(\underline{\bm{x}}^{(t)}) \right\rangle \\
	&\quad \quad - \frac{q\norm{\mathcal{H}f}_{\infty}^{(2)}}{2} \cdot \norm{\frac{1}{q\norm{\mathcal{H}f}_{\infty}^{(2)}}\underline{V}_d(\underline{\bm{x}}^{(t)}) \underline{V}_d(\underline{\bm{x}}^{(t)})^T \grad f(\underline{\bm{x}}^{(t)})}_2^2 \Bigg]\\
	&= - \frac{1}{2q\norm{\mathcal{H}f}_{\infty}^{(2)}} \norm{\underline{V}_d(\underline{\bm{x}}^{(t)})^T \grad f(\underline{\bm{x}}^{(t)})}_2^2,
	\end{align*}
	where we apply \eqref{L_smooth_manifold} to obtain the inequality. This indicates that
	\begin{equation}
	\label{SC_grad_Dir_bound}
	\norm{\underline{V}_d(\underline{\bm{x}}^{(t)})^T \grad f(\underline{\bm{x}}^{(t)})}_2^2 \leq 2q\norm{\mathcal{H}f}_{\infty}^{(2)} \left[f(\underline{\bm{x}}^*) -f(\underline{\bm{x}}^{(t)}) \right]
	\end{equation}
	for any $\underline{\bm{x}}^{(t)} \in \underline{R}_d \oplus \underline{r}_3$. Therefore, by Corollary~\ref{Riemannian_tri_update}, we obtain that
	\begin{align*}
	&d_g(\underline{\bm{x}}^{(t+1)}, \underline{\bm{x}}^*) \\
	& \stackrel{\text{(i)}}{\leq} d_g(\underline{\bm{x}}^{(t)}, \underline{\bm{x}}^*) -2\underline{\eta} \left\langle \underline{V}_d(\underline{\bm{x}}^{(t)}) \underline{V}_d(\underline{\bm{x}}^{(t)})^T \grad f(\underline{\bm{x}}^{(t)}), \Exp_{\underline{\bm{x}}^{(t)}}^{-1}(\underline{\bm{x}}^*) \right\rangle \\
	&\quad + \zeta\left(1,\underline{\rho} \right) \cdot \underline{\eta}^2 \norm{\underline{V}_d(\underline{\bm{x}}^{(t)})^T \grad f(\underline{\bm{x}}^{(t)})}_2^2\\
	&\stackrel{\text{(ii)}}{\leq} d_g(\underline{\bm{x}}^{(t)}, \underline{\bm{x}}^*)^2 + 2\underline{\eta}\Bigg[f(\underline{\bm{x}}^{(t)}) -f(\underline{\bm{x}}^*) - \frac{\underline{\beta}_0}{4} \cdot d_g(\underline{\bm{x}}^*, \underline{\bm{x}}^{(t)})^2 \\
	&\hspace{35mm} + \left(2q\norm{\mathcal{H}f}_{\infty}^{(2)}\underline{\beta}_2^2 \arcsin\left(\underline{\rho}/2\right) + \frac{q^{\frac{3}{2}} \norm{f}_{\infty}^{(3)}}{6} \right) d_g(\underline{\bm{x}}^*, \underline{\bm{x}}^{(t)})^3 \Bigg]\\
	&\quad + \zeta\left(1,\underline{\rho} \right) \cdot \underline{\eta}^2 \cdot 2q \norm{\mathcal{H}f}_{\infty}^{(2)} \left[f(\underline{\bm{x}}^*) - f(\underline{\bm{x}}^{(t)}) \right] \\
	&\stackrel{\text{(iii)}}{\leq} \left(1-\frac{\underline{\beta}_0\underline{\eta}}{4} \right) \cdot d_g(\underline{\bm{x}}^*, \underline{\bm{x}}^{(t)})^2 - 2\underline{\eta} \left[1- \underline{\eta} \cdot\zeta\left(1,\underline{\rho}\right) \cdot q\norm{\mathcal{H}f}_{\infty}^{(2)} \right] \cdot \underbrace{\left[f(\underline{\bm{x}}^*) - f(\underline{\bm{x}}^{(t)}) \right]}_{\geq 0}\\
	&\leq \left(1-\frac{\underline{\beta}_0\underline{\eta}}{4} \right) \cdot d_g(\underline{\bm{x}}^*, \underline{\bm{x}}^{(t)})^2
	\end{align*}
	whenever $0< \underline{\eta} \leq \min\left\{\frac{4}{\underline{\beta}_0}, \frac{1}{q\norm{\mathcal{H}f}_{\infty}^{(2)} \cdot\zeta\left(1,\underline{\rho}\right)} \right\}$, where we utilize Corollary~\ref{Riemannian_tri_update} and the monotonicity of $\zeta(1,c)$ with respect to $c$ in (i), apply \eqref{SC_claim_Dir} and \eqref{SC_grad_Dir_bound} to obtain (ii), and use the choice of $\underline{r}_2$ to argue that 
	\begin{align*}
	&\left(2q\norm{\mathcal{H}f}_{\infty}^{(2)}\underline{\beta}_2^2 \arcsin\left(\underline{\rho}/2\right) + \frac{q^{\frac{3}{2}} \norm{f}_{\infty}^{(3)}}{6} \right) d_g(\underline{\bm{x}}^*, \underline{\bm{x}}^{(t)})^3 \\
	&\leq \left(2q\norm{\mathcal{H}f}_{\infty}^{(2)}\underline{\beta}_2^2 \arcsin\left(\underline{\rho}/2\right) + \frac{q^{\frac{3}{2}} \norm{f}_{\infty}^{(3)}}{6} \right) d_g(\underline{\bm{x}}^*, \underline{\bm{x}}^{(t)})^2 \cdot 2 \arcsin\left(\underline{r}_2/2 \right)\\
	&\leq \frac{\underline{\beta}_0}{8} \cdot d_g(\underline{\bm{x}}^*, \underline{\bm{x}}^{(t)})^2
	\end{align*}
	in (iii). By telescoping, we conclude that when $0< \underline{\eta} \leq \min\left\{\frac{4}{\underline{\beta}_0}, \frac{1}{q\norm{\mathcal{H}f}_{\infty}^{(2)} \cdot\zeta\left(1,\underline{\rho}\right)} \right\}$ and $\underline{\bm{x}}^{(0)} \in \underline{R}_d \oplus \underline{r}_2$, 
	$$d_g(\underline{\bm{x}}^*, \underline{\bm{x}}^{(t)}) \leq \left(1-\frac{\underline{\beta}_0\underline{\eta}}{4} \right)^{\frac{t}{2}} d_g(\underline{\bm{x}}^*, \underline{\bm{x}}^{(0)}).$$
	The result follows.\\
	
	\noindent (b) The result follows obviously from (a) and the fact that $d_g(\underline{\bm{x}}^{(t)}, \underline{R}_d) \leq d_g(\underline{\bm{x}}^{(t)}, \underline{\bm{x}}^*)$ for all $t\geq 0$.\\
	
	\noindent (c) The proof is logically similar to the proof of (c) in Theorem~\ref{SCGA_LC}. We write the spectral decompositions of $\mathcal{H} f(\bm{x})$ and $\mathcal{H} \hat{f}_h(\bm{x})$ as:
	$$\mathcal{H} f(\bm{x}) = \underline{V}(\bm{x}) \underline{\Lambda}(\bm{x}) \underline{V}(\bm{x})^T \quad \text{ and } \quad \mathcal{H} \hat{f}_h(\bm{x}) = \hat{\underline{V}}(\bm{x}) \hat{\underline{\Lambda}}(\bm{x}) \hat{\underline{V}}(\bm{x})^T.$$
	By Weyl's theorem (Theorem 4.3.1 in \citealt{HJ2012}) and uniform bounds \eqref{Dir_KDE_unif_bound},
	\begin{align*}
	|\underline{\lambda}_j(\bm{x}) - \hat{\underline{\lambda}}_j(\bm{x})| &\leq \norm{\mathcal{H} f(\bm{x}) - \mathcal{H} \hat{f}_h(\bm{x})}_2\\
	&\leq q \norm{f(\bm{x}) - \hat{f}_h(\bm{x})}_{\infty}^{(2)}\\
	&= O(h^2) + O_P\left(\sqrt{\frac{|\log h|}{nh^{q+4}}} \right).
	\end{align*}
	Thus, $\hat{f}_h$ will satisfy conditions (\underline{A2}) with high probability when $h$ is sufficiently small and $\frac{nh^{q+4}}{|\log h|}$ is sufficiently large. According to Davis-Kahan theorem (Lemma~\ref{Davis_K} here), uniform bounds \eqref{Dir_KDE_unif_bound}, and the continuity of exponential maps, we have that
	\begin{align*}
	&d_g\left(\Exp_{\bm{y}}\left(\underline{\eta} \cdot \hat{\underline{V}}_d(\bm{y}) \hat{\underline{V}}_d(\bm{y})^T \grad \hat{f}_h(\bm{y}) \right), \, \Exp_{\bm{y}}\left(\underline{\eta} \cdot \underline{V}_d(\bm{y}) \underline{V}_d(\bm{y})^T \grad f(\bm{y}) \right) \right)\\
	&\leq \underline{\eta} C_3 \norm{\hat{\underline{V}}_d(\bm{y}) \hat{\underline{V}}_d(\bm{y})^T \grad \hat{f}_h(\bm{y}) - \underline{V}_d(\bm{y}) \underline{V}_d(\bm{y})^T \grad f(\bm{y})}_2\\
	&\leq \underline{\eta} C_3 \norm{\hat{\underline{V}}_d(\bm{y}) \hat{\underline{V}}_d(\bm{y})^T \left[\grad \hat{f}_h(\bm{y})- \grad f(\bm{y}) \right]}_2 \\
	&\quad + \underline{\eta} C_3 \norm{\left[\hat{\underline{V}}_d(\bm{y}) \hat{\underline{V}}_d(\bm{y})^T  - \underline{V}_d(\bm{y}) \underline{V}_d(\bm{y})^T \right] \grad f(\bm{y})}_2\\
	&\stackrel{\text{(i)}}{\leq} \underline{\eta} C_3 \norm{\grad \hat{f}_h(\bm{y})- \grad f(\bm{y})}_2 + \underline{\eta} C_3 \cdot \frac{\norm{\mathcal{H}f(\bm{y}) - \mathcal{H}\hat{f}_h(\bm{y})}_2 \cdot \norm{\grad f(\bm{y})}_2}{\underline{\beta}_0}\\
	&\leq \underline{\eta} C_3 \sqrt{q} \norm{\hat{f}_h -f}_{\infty}^{(1)} + \underline{\eta} C_3 \cdot \frac{q\norm{\hat{f}_h-f}_{\infty}^{(2)} \sqrt{q+1}\norm{f}_{\infty}^{(1)}}{\underline{\beta}_0}\\
	&\equiv \underline{\epsilon}_{n,h}= O(h^2) + O_P\left(\sqrt{\frac{|\log h|}{nh^{q+4}}}\right)
	\end{align*}
	for any $\bm{y}\in \text{Ball}_{q+1}(\underline{\bm{x}}^*, \underline{r}_2) \cap \Omega_q$, where we utilize the Davis-Kahan theorem and $\norm{\hat{\underline{V}}_d(\bm{y}) \hat{\underline{V}}_d(\bm{y})^T}_2=1$ in (i). Hence, when $h \to 0$ and $\frac{nh^{q+4}}{|\log h|} \to \infty$,
	\begin{align}
	\label{Exp_proj_grad_diff_bd}
	\begin{split}
	&d_g\left(\Exp_{\bm{y}}\left(\underline{\eta} \cdot \hat{\underline{V}}_d(\bm{y}) \hat{\underline{V}}_d(\bm{y})^T \grad \hat{f}_h(\bm{y}) \right), \, \Exp_{\bm{y}}\left(\underline{\eta} \cdot \underline{V}_d(\bm{y}) \underline{V}_d(\bm{y})^T \grad f(\bm{y}) \right) \right)\\
	&\leq \underline{\epsilon}_{n,h}= O(h^2) + O_P\left(\sqrt{\frac{|\log h|}{nh^{q+4}}}\right) \leq (1-\Upsilon) \cdot 2\arcsin\left(\underline{r}_2/2 \right)
	\end{split}
	\end{align}
	with probability tending to 1.\\
	We now claim that $d_g(\hat{\underline{\bm{x}}}^{(t)}, \underline{\bm{x}}^*) \leq 2\arcsin\left(\underline{r}_2/2 \right)$ and 
	\begin{equation}
	\label{claim_Dir_sam_SCGA}
	d_g(\hat{\underline{\bm{x}}}^{(t+1)}, \underline{\bm{x}}^*) \leq \Upsilon \cdot d_g(\hat{\underline{\bm{x}}}^{(t)}, \underline{\bm{x}}^*) + \underline{\epsilon}_{n,h}
	\end{equation}
	for all $t\geq 0$. We again prove this claim by induction on the iteration number. Note that when $t=1$, we derive that
	\begin{align*}
	&d_g(\hat{\underline{\bm{x}}}^{(1)}, \underline{\bm{x}}^*) \\
	&= d_g\left(\Exp_{\hat{\underline{\bm{x}}}^{(0)}}\left(\underline{\eta} \cdot \hat{\underline{V}}_d(\hat{\underline{\bm{x}}}^{(0)}) \hat{\underline{V}}_d(\hat{\underline{\bm{x}}}^{(0)})^T \grad \hat{f}_h(\hat{\underline{\bm{x}}}^{(0)}) \right),\, \underline{\bm{x}}^* \right)\\
	&\stackrel{\text{(i)}}{\leq} d_g\left(\Exp_{\hat{\underline{\bm{x}}}^{(0)}}\left(\underline{\eta}\cdot \underline{V}_d(\hat{\underline{\bm{x}}}^{(0)}) \underline{V}_d(\hat{\underline{\bm{x}}}^{(0)})^T \grad f(\hat{\underline{\bm{x}}}^{(0)}) \right),\, \underline{\bm{x}}^* \right) \\
	&\quad + d_g\left(\Exp_{\hat{\underline{\bm{x}}}^{(0)}}\left(\underline{\eta} \cdot \hat{\underline{V}}_d(\hat{\underline{\bm{x}}}^{(0)}) \hat{\underline{V}}_d(\hat{\underline{\bm{x}}}^{(0)})^T \grad \hat{f}_h(\hat{\underline{\bm{x}}}^{(0)}) \right),\, \Exp_{\hat{\underline{\bm{x}}}^{(0)}}\left(\underline{\eta} \cdot \underline{V}_d(\hat{\underline{\bm{x}}}^{(0)}) \underline{V}_d(\hat{\underline{\bm{x}}}^{(0)})^T \grad f(\hat{\underline{\bm{x}}}^{(0)}) \right) \right)\\
	&\stackrel{\text{(ii)}}{\leq} \Upsilon \cdot d_g(\hat{\underline{\bm{x}}}^{(0)}, \underline{\bm{x}}^*) + \underline{\epsilon}_{n,h},
	\end{align*}
	where we apply the triangle inequality in (i) and leverage the result in (a) and \eqref{Exp_proj_grad_diff_bd} to obtain (ii). The triangle inequality is valid in this context because the geodesic measures the minimal distance between two points on $\Omega_q$. Moreover, by the choice of $\hat{\underline{\bm{x}}}^{(0)}$ and \eqref{Exp_proj_grad_diff_bd}, we are ensured that $d_g(\hat{\underline{\bm{x}}}^{(1)}, \underline{\bm{x}}^*) \leq 2\arcsin\left(\underline{r}_2/2 \right)$. In the induction from $t\mapsto t+1$, we suppose that $d_g(\hat{\underline{\bm{x}}}^{(t)}, \underline{\bm{x}}^*) \leq 2\arcsin\left(\underline{r}_2/2 \right)$ and the claim \eqref{claim_Dir_sam_SCGA} holds at iteration $t$. The same argument then implies that the claim \eqref{claim_Dir_sam_SCGA} holds for iteration $t+1$ and that $d_g(\hat{\underline{\bm{x}}}^{(t+1)}, \underline{\bm{x}}^*) \leq 2\arcsin\left(\underline{r}_2/2 \right)$. The claim \eqref{claim_Dir_sam_SCGA} is thus verified. \\
	Now, given that $\Upsilon= \sqrt{1-\frac{\underline{\beta}_0\underline{\eta}}{4}} <1$, we iterate the claim \eqref{claim_Dir_sam_SCGA} to show that
	\begin{align*}
	d_g(\hat{\underline{\bm{x}}}^{(t)}, \underline{\bm{x}}^*) & \leq \Upsilon \cdot d_g(\hat{\underline{\bm{x}}}^{(t-1)}, \underline{\bm{x}}^*) + \underline{\epsilon}_{n,h}\\
	&\leq \Upsilon\left[\Upsilon \cdot d_g(\hat{\underline{\bm{x}}}^{(t-2)}, \underline{\bm{x}}^*) + \underline{\epsilon}_{n,h}\right] + \underline{\epsilon}_{n,h}\\
	&\leq \Upsilon^t \cdot d_g(\hat{\underline{\bm{x}}}^{(0)}, \underline{\bm{x}}^*) + \left[\sum_{s=0}^{t-1} \Upsilon^s \right] \underline{\epsilon}_{n,h}\\
	&\leq \Upsilon^t \cdot d_g(\hat{\underline{\bm{x}}}^{(0)}, \underline{\bm{x}}^*) + \frac{\underline{\epsilon}_{n,h}}{1-\Upsilon}\\
	&= \Upsilon^t \cdot d_g(\hat{\underline{\bm{x}}}^{(0)}, \underline{\bm{x}}^*) + O(h^2) + O_P\left(\sqrt{\frac{|\log h|}{nh^{q+4}}} \right),
	\end{align*}
	where the fourth inequality follows by summing the geometric series, and the last equality is due to our notation $\underline{\epsilon}_{n,h} = O(h^2) + O_P\left(\sqrt{\frac{|\log h|}{nh^{q+4}}} \right)$. It completes the proof.\\
	
	\noindent (d) The result follows directly from (c) and the inequality $d_g(\hat{\underline{\bm{x}}}^{(t)}, \underline{R}_d) \leq d_g(\hat{\underline{\bm{x}}}^{(t)}, \underline{\bm{x}}^*)$ for all $t\geq 0$.
\end{proof}

\end{appendix}

\end{document}